\documentclass{article}

\usepackage{microtype}
\usepackage{graphicx}
\usepackage{subcaption}
\usepackage{booktabs} %
\usepackage{url}
\usepackage{multirow}
\usepackage{amsfonts}
\usepackage{amsthm}
\usepackage{amsmath, amssymb, mathtools}
\usepackage{xcolor} %
\usepackage{hyperref}
\usepackage{rotating}
\usepackage[nohints]{minitoc}
\usepackage{etoc}
\usepackage{algorithmic}
\usepackage{algorithm}
\usepackage{wrapfig}
\usepackage[final]{neurips_2024}
\usepackage[capitalize,noabbrev]{cleveref}
\usepackage{adjustbox}
\usepackage[utf8]{inputenc} %
\usepackage[T1]{fontenc}    %
\usepackage{nicefrac}       %

\theoremstyle{plain}
\newtheorem{theorem}{Theorem}[section]
\newtheorem{proposition}[theorem]{Proposition}
\newtheorem{lemma}[theorem]{Lemma}

\theoremstyle{definition}
\newtheorem{definition}[theorem]{Definition}
\newtheorem{assumption}[theorem]{Assumption}
\theoremstyle{remark}

\usepackage[textsize=tiny]{todonotes}

\title{
Dissecting 
the Failure of Invariant Learning on Graphs
}

\author{%
  Qixun Wang\textsuperscript{1} \qquad Yifei Wang \textsuperscript{2} \qquad Yisen Wang \textsuperscript{1,3}\thanks{Corresponding author: Yisen Wang (yisen.wang@pku.edu.cn).} \qquad Xianghua Ying \textsuperscript{1} \\
\textsuperscript{1} State Key Laboratory of General Artificial Intelligence,\\School of Intelligence Science and Technology, Peking University\\
\textsuperscript{2} CSAIL, MIT\\
\textsuperscript{3} Institute for Artificial Intelligence, Peking University\\
}

\begin{document}
\maketitle
\doparttoc
\faketableofcontents

\begin{abstract}
    Enhancing node-level Out-Of-Distribution (OOD) generalization on graphs remains a crucial area of research. In this paper, we develop a Structural Causal Model (SCM) to theoretically dissect the performance of two prominent invariant learning methods---Invariant Risk Minimization (IRM) and Variance-Risk Extrapolation (VREx)---in node-level OOD settings. Our analysis reveals a critical limitation: due to the lack of class-conditional invariance constraints, these methods may struggle to accurately identify the structure of the predictive invariant ego-graph and consequently rely on spurious features. To address this, we propose Cross-environment Intra-class Alignment (CIA), which explicitly eliminates spurious features by aligning cross-environment representations conditioned on the same class, bypassing the need for explicit knowledge of the causal pattern structure. To adapt CIA to node-level OOD scenarios where environment labels are hard to obtain, we further propose CIA-LRA (Localized Reweighting Alignment) that leverages the distribution of neighboring labels to selectively align node representations, effectively distinguishing and preserving invariant features while removing spurious ones, all without relying on environment labels. We theoretically prove CIA-LRA's effectiveness by deriving an OOD generalization error bound based on PAC-Bayesian analysis. Experiments on graph OOD benchmarks validate the superiority of CIA and CIA-LRA, marking a significant advancement in node-level OOD generalization. The codes are available at \url{https://github.com/NOVAglow646/NeurIPS24-Invariant-Learning-on-Graphs}.
\end{abstract}

\section{Introduction}
\label{sec_intro}
Generalizing to unseen testing distributions that differ from the training distributions, known as Out-Of-Distribution (OOD) generalization, is one of the key challenges in machine learning. 
Invariant learning, which aims to capture predictive features that remain consistent under distributional shifts, is a crucial strategy for addressing OOD generalization. Numerous invariant learning methods have been proposed to tackle OOD problems in computer vision (CV) tasks \citep{arjovsky2019invariant, krueger2021out, bui2021exploiting, rame2022fishr,wang2019symmetric,mahajan2021domain,zhang2021can, wang2022out,yi2022breaking, wang2022improving,xin2022connection}. While in recent years, enhancing OOD generalization on graph data is an emerging research direction attracting increasing attention. In this work, we focus on the challenge of node-level OOD generalization on graphs.

Straightforwardly adapting the above methods to node-level graph OOD scenarios presents several challenges: 1) the prediction of a node's label depends on its neighbored samples in an ego-subgraph, causing a gap from conventional CV OOD scenarios where samples are independently predicted; and 2) environment labels in node-level tasks are often unavailable \citep{wu2021handling, li2023invariant, liu2023flood}, rendering invariant learning methods based on environment partitioning infeasible. To illustrate the failure of directly adapting traditional invariant learning to graphs, we evaluate two representative OOD approaches, Invariant Risk Minimization (IRM \citep{arjovsky2019invariant}) and Variance-Risk Extrapolation (VREx \citep{krueger2021out}), in OOD node classification scenarios. We choose IRM and VREx for two reasons: 1) Numerous node-level graph OOD strategies \citep{zhang2021stable, wu2021handling,   liu2023flood,  tian2024graphs} utilize VREx as invariance regularization (details in Appendix \ref{works_using_vrex_sec}). Therefore, analyzing VREx can cover a significant portion of graph-OOD methods; and 2) IRM and VREx are two prominent OOD methods that we can theoretically prove to be effective on non-graph data (Proposition \ref{prop_VREx_IRM_suc}). By testing their performance on graph data, we can better reveal the differences between graph and non-graph data.
We choose real-world graph datasets: Arxiv, Cora, and WebKB; synthetic datasets: CBAS and a toy graph OOD dataset with spurious correlations between node features and labels for evaluation. From Table \ref{moti_acc}, we observe that IRM and VREx offer marginal or no improvement over ERM on both real-world and synthetic node-level graph OOD datasets. This naturally raises some questions here:
\begin{quote}
    \emph{On graphs, why do traditional invariant learning methods fail? How to make them work again?}
\end{quote}

\begin{wrapfigure}{r}{0.35\textwidth}
\vspace{-10pt}
    \centering
    \resizebox{0.99\linewidth}{!}{
        \begin{tabular}{cccc}
            \toprule
             Algorithms & Large-Cov. & Large-Con. & Toy  \\
            \midrule
             ERM & 57.74 & 59.57  & 33.6 \\
             \midrule
             IRM & 57.59 & 59.46  & 34.9 \\
             \midrule
             VREx & 58.46 & 59.83  & 33.9 \\
             \midrule
             CIA (ours) & 59.68 & 60.89  & 37.0 \\
             \midrule
             CIA-LRA (ours) & \textbf{61.94} & \textbf{63.03}  & \textbf{39.1} \\
            \bottomrule
        \end{tabular}
        }
    \captionof{table}{\textit{Real-Cov./Con.} are average OOD accuracy on the covariate/concept shift of Arxiv, Cora, CBAS, and WebKB. \textit{Toy} denotes results on our toy synthetic graph OOD dataset.}
    \label{moti_acc}

\vspace{-10pt}
\end{wrapfigure}

To theoretically analyze their failure modes, we build a Structural Causal Model (SCM) to model the data generation process under two types of distributional shifts: concept shift and covariate shift, and gain a high-level understanding of the challenges in node-level OOD generalization: To correctly predict a node's label, the structure of a predictive invariant neighboring ego-graph (which we call it a \textit{causal pattern}) and their invariant node features must be learned. However, identifying the correct structure of the causal pattern presents additional optimization requirements (compared to CV scenarios) for Graph Neural Networks (GNNs) since they must adjust the aggregation parameters (such as the attention weights in GAT \citep{veličković2018graph}) to achieve this. IRM and VREx lack class-conditional invariance constraints, which causes insufficient supervision for regularizing the training of these aggregation parameters, leading to non-unique solutions of GNN parameters and potentially resulting in the learning of spurious features. (detailed analysis is in Section \ref{theo_sec}).  
To overcome this, we propose Cross-environment Intra-class Alignment (\textbf{CIA}) that further considers class-conditional invariance to identify causal patterns\footnote{Although numerous node-level OOD strategies from different perspectives have been proposed such as GNN architecture design \citep{wu2024graph} or feature augmentation \citep{li2023graph}, we focus on developing an invariant learning objective that could be integrated into and improve other graph-OOD frameworks in a plug-and-play manner (validated in Section \ref{plug_and_play}), serving as a general solution.}. We theoretically demonstrate that by aligning node representations of the same class and different environments, CIA can eliminate spurious features and learn the correct causal pattern, as same-class different-environment samples share similar causal patterns while exhibiting different spurious features. Table \ref{moti_acc} shows CIA's empirical gains. To leverage the advantage of CIA and adapt it to scenarios without environment labels, we further propose \textbf{CIA-LRA} (Localized Reweighting Alignment), utilizing localized label distribution to find node pairs with significant differences in spurious features and small differences in causal ones for alignment, to eliminate the spurious features while alleviating the collapse of the causal ones. Our contributions are summarized as follows:
\begin{enumerate}
    \item By constructing an SCM, we provide a theoretical analysis revealing that VREx and IRM could rely on spurious features when using a GAT-like GNN (Section \ref{classic_fail}) in node-level OOD scenarios, revealing a key challenge of invariant learning on graphs. 
    \item We propose CIA and theoretically prove its effectiveness in learning invariant representations on graphs (Section \ref{CIA_theo}). To adapt CIA to node-level OOD scenarios where environment labels are unavailable, we further propose CIA-LRA that requires no environment labels or complex environmental partitioning processes to achieve invariant learning (Section \ref{CIA-LRA_method_sec}), with theoretical guarantees on its generalization performance (Section \ref{bound_sec}).
    \item We evaluate CIA and CIA-LRA on the Graph OOD benchmark (GOOD) \citep{gui2022good} on GAT and GCN \citep{kipf2016semi}. The results demonstrate CIA's superiority over non-graph invariant learning methods, and CIA-LRA achieves state-of-the-art performance (Section \ref{maiN^exp_sec}).
\end{enumerate}

We leave comparisons of our method and existing node-level OOD works in Appendix \ref{comparison_sec}.

\section{Dissecting Invariant Learning on Graphs
}
\label{theo_sec}

For OOD node classification, we are given a single training graph $\mathcal{G}=(A,X,Y)$ containing $N$ nodes $\mathcal{V}= \{v_i\}_{i=1}^N$ from multiple training environments $e \in \mathcal{E}_{\text{tr}}$. $A\in {\{0,1\}}^{N\times N}$ is the adjacency matrix, $A_{i,j}=1$ iff there is an edge between $v_i$ and $v_j$. $X\in \mathbb{R}^{N \times D}$ are node features. The $i$-th row $X_i\in \mathbb{R}^{D}$ represents the original node feature of $v_i$. $Y \in \{0,1,...,C-1\}^{N}$ are the labels, $C$ is the number of the classes.  Denote the subgraph containing nodes of environment $e$ as $\mathcal{G}^e=({A^e},X^e,Y^e)$, which follows the distribution $p_e$. Let ${A^e}\in {\{0,1\}}^{N^e\times N^e}$ and $D^e$ be the adjacency matrix and the diagonal degree matrix for nodes from environment $e$ respectively, where $D^e_{ii}=\sum_{j=1}{A^e}_{ij}$, $N^e$ is the number of samples in $e$. Denote the normalized adjacency matrix as $\bar{A^e}=(D^e+I_{N^e})^{-\frac{1}{2}}{A^e}(D^e+I_{N^e})^{-\frac{1}{2}}$, $I_{N^e}$ is the identity matrix. Let $\tilde{A^e}=\bar{A}+(D^e+I_{N^e})^{-\frac{1}{2}}I_{N^e}(D^e+I_{N^e})^{-\frac{1}{2}}$. Suppose the unseen test environments are $e'\in \mathcal{E}_{\text{te}}$. The test distribution $p_{e'}\neq p_{e} \forall e'\in \mathcal{E}_{\text{te}}, ~ \forall e\in \mathcal{E}_{\text{tr}}$. OOD generalization aims to minimize the prediction error over test distributions.

 To understand the obstacles in invariant learning on graphs, we start by examining whether IRMv1 (practical implementation of the original challenging IRM objective, proposed by \cite{arjovsky2019invariant}) and VREx can be successfully transferred to node-level graph OOD tasks. Their objectives are as follows:
\begin{equation}
\begin{aligned}
    \text{(IRMv1)}~&\min_{w, \phi} \mathbb{E}_e[ \mathcal{L}\left( w\circ \phi(X^e) ,Y^e \right) + \beta \|\nabla_{w|w=1.0} \mathcal{L}\left( w\circ \phi(X^e) ,Y^e \right)\|_2^2 ],\\
    \text{(VREx)}~&\min_{w, \phi} \mathbb{E}_e[ \mathcal{L}\left( w\circ \phi(X^e) ,Y^e \right)  ] +\beta \text{Var}_e[\mathcal{L}\left( w\circ \phi(X^e) ,Y^e \right)],
\end{aligned}
\end{equation}
where $w$ and $\phi$ denote the classifier and feature extractor, respectively. $\mathcal{L}$ is the cross-entropy loss. $\beta$ is some hyperparameter.

\subsection{A Causal Data Model on Graphs}
\begin{wrapfigure}{r}{0.35\textwidth}
    \centering
    \begin{subfigure}{0.4\linewidth}
        \centering
        \includegraphics[width=\linewidth]{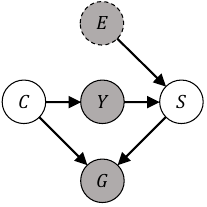}
        \caption{concept shift}
        \label{SCM_con}
    \end{subfigure}%
    \hspace{5mm} %
    \begin{subfigure}{0.4\linewidth}
        \centering
        \includegraphics[width=\linewidth]{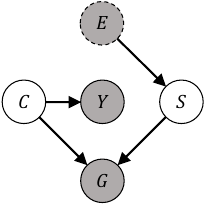}
        \caption{covariate shift}
        \label{SCM_cov}
    \end{subfigure}%
    \caption{Causal graphs of the SCMs considered in our work.}
    \label{SCM_figure}
    \vspace{-10pt}
\end{wrapfigure}
\label{theo_model}
  \textbf{Data generation process.} We construct an SCM to characterize two kinds of distribution shifts: concept shift (Figure \ref{SCM_con}) and covariate shift (Figure \ref{SCM_cov}). $C$ and $S$ denote unobservable causal/spurious latent variables that affect the generation of the graph $G$, and dashed $E$ are environmental variables usually unobservable. We consider a simple case that each node $v$ in environment $e$ has a 2-dim feature $[x_v^1, x_v^2]^\top$, $x_v^1, ~x_v^2\in \mathbb{R}$. Denote the concatenated node features of all nodes in $e$ as  $X_1\in \mathbb{R}^{N^e\times 1}$ and $X_2^e\in \mathbb{R}^{N^e\times 1}$ corresponding to $x^1_v$ and ${x^2_v}$, respectively. For the SCM in Figure \ref{SCM_con}\footnote{Due to space limitation, we only present the concept shift case in the main text and leave the covariate shift case in Appendix \ref{theo_sec_cov}. Our following results in Section \ref{theo_sec} hold for both shifts.}, the data generation process of environment $e$ is
\begin{equation}
\label{data_gene}
\begin{aligned}
    &Y^e=(\tilde{A^e})^kX_1+n_1,~X_2^e=(\tilde{A^e})^mY^e+n_2+{\epsilon^e},
\end{aligned}
\end{equation} where $k\in \mathbb{N}^+$ represents the "depth" (number of hops) of the causal pattern, and $m\in \mathbb{N}^+$ is the depth of the ego-graph determining the spurious node features. $n_1\in \mathbb{R}^{N^e\times 1}$ and $n_2\in\mathbb{R}^{N^e\times 1}$ represent random Gaussian noise. $\epsilon^e$ stands for an environmental variable, causing the spurious correlations between $X_2^e$ and $Y$. A detailed description of the model is in Appendix \ref{theo_model_detail_sec}.

\textbf{How the proposed model considers both node feature shifts and structural shifts?} $X_1$ represent invariant node features causing $Y^e$. $X_2^e$ denotes spurious node features that vary with environments. As for structural shifts, we consider an environmental $\tilde{A^e}$ in Equation (\ref{data_gene}), which means the structure can vary with environments. For example, there could be a spurious correlation between certain structures and the label; or, the graph connectivity or size may shift \citep{buffelli2022sizeshiftreg, xia2024learning}. We model the invariant structural feature as the structure of a node's $k$-layer neighboring ego-graph. See Appendix \ref{sec_structural_shifts} for more discussions of the structural shifts.

We also have the following assumption about the stability of the causal patterns across environments:

\begin{assumption}
\label{assump_stable_causal} \textbf{(Stability of the causal patterns)} 
The $k$-layer causal pattern in Equation (\ref{data_gene}) is invariant across environments for every class $c$.
\end{assumption}

\textbf{A simple multi-layer GNN.} Consider a $L$-layer GNN $f$ parameterized by $\Theta=\left\{\theta_1, \theta_2, {\theta^1_1}^{(l)}, {\theta^2_1}^{(l)}, {\theta^1_2}^{(l)}, {\theta^2_2}^{(l)} \right\}$, $l=1,2,...,L-1$:
\begin{equation}
\label{GNN}
    \begin{aligned}
         & f_{\Theta}(A,X)=H^{(L)}_1\theta_1+H^{(L)}_2\theta_2, ~\text{where}                                                      \\
         &\begin{pmatrix}
            H_1^{(l)} & H_2^{(l)}
        \end{pmatrix}
        =
        \begin{pmatrix}
            \bar{A} & \bar{I}
        \end{pmatrix}
        \begin{pmatrix}
            {\theta^1_1}^{(l-1)} & {\theta^1_2}^{(l-1)} \\
           {\theta^2_1}^{(l-1)} & {\theta^2_2}^{(l-1)}
        \end{pmatrix}
        \begin{pmatrix}
            H_1^{(l-1)} & 0\\
            0 & H_2^{(l-1)}
        \end{pmatrix},
        ~l=2,...,L,~ H^{(1)}_1=X_1, ~H^{(1)}_2=X_2,
    \end{aligned}
\end{equation}
where $\theta_i$ and ${\theta^i_j}^{(l)}$ are scalars for $i,j\in\{1,2\}$, $\forall l$. $H_i^{(l)}\in \mathbb{R}^{N\times D}$ are GNN representations.

\textbf{Remark.} In this GNN, we keep only the top-layer weight matrix $[\theta_1~\theta_2]^\top$, and let the weight matrix of lower layers $1,..., L-1$ be an identity matrix. This architecture resembles an SGC \citep{pmlr-v97-wu19e}. $\theta_1$ and $\theta_2$ are for invariant/spurious features, respectively.  ${\theta^1_1}^{(l)}$, ${\theta^1_2}^{(l)}$ are weights for aggregating features from neighboring nodes and ${\theta^2_1}^{(l)}$, ${\theta^2_2}^{(l)}$ are weights for features of a centered node, this setup can be seen as a GAT. When all lower-layer parameters equal 1, the GNN degenerates to a GCN (see Appendix \ref{sec_structural_shifts} for justification of the choice of the GNN). 

We consider a regression problem that we aim to minimize the MSE loss over all environments $\mathbb{E}_e[ R(e)]=\mathbb{E}_e \left[\mathbb{E}_{n_1,n_2}\left[\left\|f_\Theta({A^e}, X^e)-Y^e\right\|_2^2\right]\right]$. The optimal invariant parameter set $\Theta^*$ is
\begin{equation}
\label{optimal_solu}
    \left\{
    \begin{array}{l}
        \theta_1=1                                                             \\
        \theta_2=0   \quad \text {or}   \quad  \exists l\in \{1,...,L-1\} ~\text{s.t.}~ {\theta_2^1}^{(l)}={\theta_2^2}^{(l)}=0\\
        {\theta^1_1}^{(l)}=1, {\theta^2_1}^{(l)}=1, \quad l= L-1, ..., L-k+1 \\
        {\theta^1_1}^{(l)}=0, {\theta^2_1}^{(l)}=1, \quad l=L-k, L-k-1, ..., 1 \\
    \end{array}
    \right..
\end{equation} In Equation (\ref{optimal_solu}), the GNN parameters for spurious features (line 2) is zero, which means it removes spurious node features. Also, it learns the correct depth $k$ of the causal pattern $\tilde{A}^kX_1$ (line 3-4).

\subsection{Intriguing Failure of VREx and IRM on Graphs}
\label{classic_fail}
Now we are ready to present the failure cases in this node-level OOD task: optimizing IRMv1 and VREx induces a model that relies on spurious features $X_2^e$ to predict, leading to poor OOD generalization. To illustrate that this failure arises from the graph data, we first prove that IRMv1 and VREx can learn invariant features under the non-graph version of SCM of Equation (\ref{data_gene}).

\begin{proposition}
\label{prop_VREx_IRM_suc}
    \textbf{(IRMv1 and VREx can learn invariant features for non-graph tasks, proof is in Appendix \ref{proof_VREx_IRM_suc}.)} For the non-graph version of the SCM in Equation (\ref{data_gene}), 
    \begin{equation}
    \label{non_graph_SCM}
        Y^e=X_1+n_1,~X_2^e=Y^e+n_2+{\epsilon^e},    
    \end{equation}
    VREx and IRMv1 can learn invariant features when using a linear network: $f(X)=\theta_1 X_1 + \theta_2 X_2$.
\end{proposition}

Now we will give the main theorem revealing the failure of VREx and IRMv1 on graphs. 
\begin{theorem}
    \label{prop1}
    \textbf{(IRMv1 and VREx will use spurious features on graphs, informal)} Under the SCM of Equation (\ref{data_gene}), the IRMv1 and VREx objectives have non-unique solutions for parameters of the GNN (\ref{GNN}), and there exist solutions that use spurious features, i.e. $\theta_2\neq 0$.
\end{theorem}

\textbf{Intuitive illustration of the failure.} From Theorem \ref{prop1}, we find that the main reason for the failure lies in the message-passing mechanism in representation learning. Let's provide some key steps in the proof of the IRMv1 case as an illustration. For the non-graph OOD task Equation (\ref{non_graph_SCM}), we can verify that when IRMv1 objective is solved, i.e. $\nabla_w R(e)=0$ for all $e$, the invariant solution $\theta_2=0$ leads to $(\theta_1)^2X_1^\top X_1-\theta_1X_1^\top X_1=0$,
which can be satisfied when $\theta=1$. However, in the graph case, $\theta_2=0$ leads to 
\begin{equation}
\label{illu_t1_solu2}
    (\theta_1)^2 ((\tilde{A^e})^{s}X_1)^\top (\tilde{A^e})^sX_1 - \theta_1 ((\tilde{A^e})^{k}X_1)^\top (\tilde{A^e})^sX_1=0, ~\forall e,
\end{equation}
where $(\tilde{A^e})^{s}X_1$ is the learned representation of the GNN, $0<s\leq L$. $k$ is the depth of the causal pattern. Now we explain why the invariant solution may not hold on graphs. When the depth of the learned aggregation pattern $s\neq k$, Equation (\ref{illu_t1_solu2}) cannot hold for a fixed $\theta_1$ (since $\theta_1$ will depend on $e$ then). This means that identifying the underlying structure of the causal pattern imposes additional difficulty for invariant learning.  Moreover, even if the GNN can learn representations of different depths (e.g. GAT)\footnote{If the GNN is a GCN that has a fixed aggregation depth $L$, i.e. $s=L$, it will be even impossible to learn the true causal pattern if we choose an $L\neq k$ in advance.}, the proof in Appendix \ref{proof_irm_con_sec} shows that IRM failed to provide sufficient supervision to optimize the aggregation parameters $\theta^i_j$, $i,j\in \{1,2\}$ such that $s=k$. A similar analysis holds for VREx. In general, successful invariant learning on graphs requires capturing both invariant node features and the structure of the causal pattern, while methods like IRM and VREx that solely enforce a cross-environment invariance at the loss level\footnote{IRM minimizes the loss gradient w.r.t. the classifier in each environment, and VREx minimizes the loss variance across environments.} may not be able to achieve these goals. The formal versions and proof are in Appendix \ref{proof_IRM} (IRM) and \ref{proof_VREx} (VREx).%

\section{The Proposed Methods}
\subsection{Cross-environment Intra-class Alignment}
\label{CIA_theo}
Inspired by the examples of VREx and IRMv1, we aim to introduce additional invariance regularization to guide the model in identifying the underlying invariant node features and structures. We propose CIA (Cross-environment Intra-class Alignment), which aligns the representations from the same class across different environments. Intuitively, since such node pairs share similar invariant features and causal pattern structures while differ in spurious features, aligning their representations will help achieve our targets. Denote the representation of node $i$ as $\phi_\Theta(i)$ and the classifier parameterized by $\theta_h$ as $h_{\theta_h}$ CIA's objective is:
\begin{equation}
    \begin{aligned}
        &\min_{\theta_h, \Theta} ~ \mathbb{E}_e\left[\mathcal{L}(h_{\theta_h}\circ \phi_\Theta({A^e},X^e), Y^e)\right] \quad \text{s.t.} ~ &\min_\Theta ~\mathcal{L}_{\text{CIA}}= \mathbb{E}_{\substack{e, e'\\e\neq e'}}  \mathbb{E}_{c} \mathbb{E}_{\substack{i,j\\(i,j)\in \Omega^{e,e'}_c}}\left[\mathcal{D}(\phi_\Theta(i), \phi_\Theta(j))\right]
    \end{aligned}
\end{equation}
where $\Omega^{e,e'}_c=\{(i,j)|i\neq j \land Y^e_i=Y^{e'}_j=c \land E_i=e,~E_j=e' \}$ is the set of nodes with same label $c$ and from two different environments. $\mathcal{L}$ is the cross-entropy loss.  $\mathcal{D}$ is some distance metric and we adopt L-2 distance. Now we prove that CIA can learn invariant representations regardless of the unknown causal patterns:
\begin{theorem}
    \label{prop3}
    Under the SCM of Equation (\ref{data_gene}) and Assumption \ref{assump_stable_causal}, optimizing the CIA objective will lead to the optimal invariant solution $\Theta^*$ in Equation (\ref{optimal_solu}) for parameters of the GNN (\ref{GNN}).
\end{theorem}

The proof is in Appendix \ref{proof_CIA}. By enforcing class-conditional invariance, which is not considered in VREx and IRMv1, CIA overcomes the above obstacles and eliminates spurious features. As long as a GNN has the capacity to adaptively learn the true depth of the causal pattern (such as the one considered in Equation (\ref{GNN})) or a GAT), CIA can identify the invariant causal pattern. 

\textbf{Remark.} One might note that the objective of CIA is analogous to MatchDG \citep{mahajan2021domain}. However, we are the first to adapt such an idea to node-level OOD tasks and theoretically reveal its advantage. In Appendix \ref{related_inv_learning}, we compare our extension and the original MatchDG in detail.

\subsection{Localized Reweighting Alignment: an Adaptation to Graphs without Environment Labels}
\label{CIA-LRA_method_sec}

\begin{figure}[htbp]
  \centering
  \adjustbox{center}{
  \includegraphics[scale=0.4]{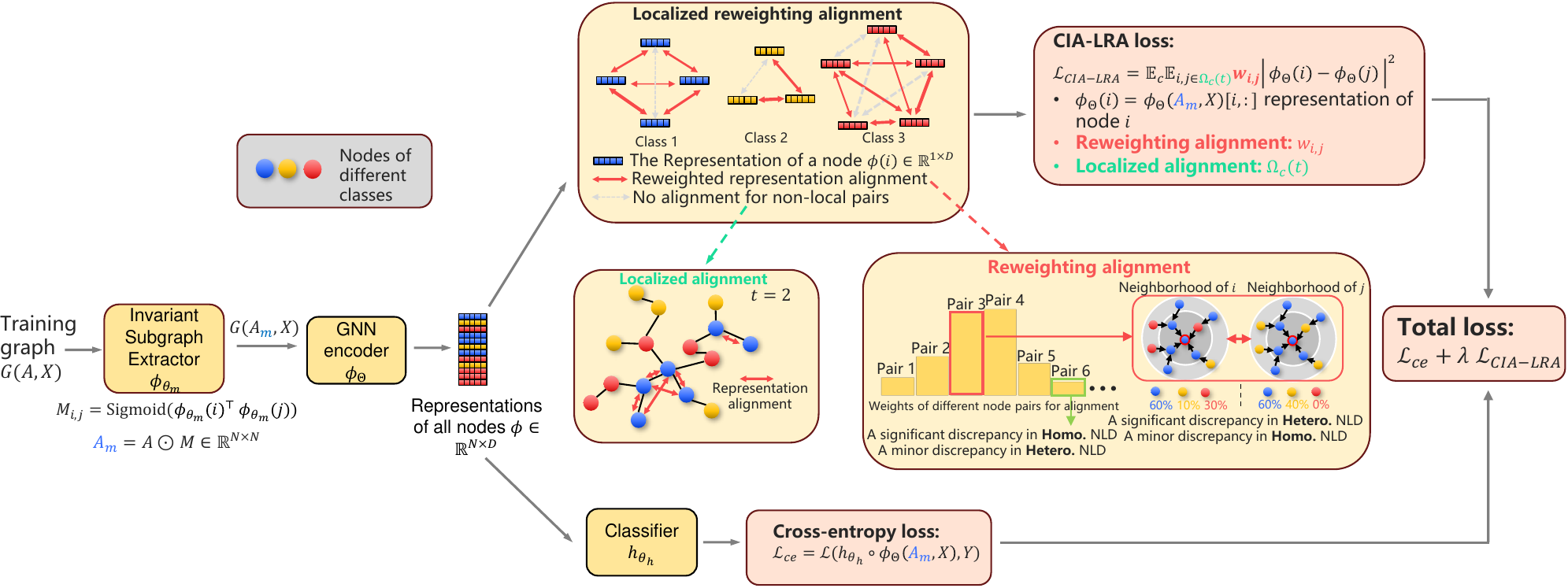}
  }
  \caption{The overall framework of our proposed CIA-LRA. The invariant subgraph extractor $\phi_{\theta_m}$ identifies the invariant subgraph for each node. Then the GNN encoder $\phi_{\Theta}$ aggregates information from the estimated invariant subgraphs to output node representations. CIA-LRA mainly contains two strategies: localized alignment and reweighting alignment. Localized alignment: we restrict the alignment to a local range to avoid overalignment that may cause the collapse of invariant features (shown in Appendix \ref{app:excessive-align-collapse}). Reweighting alignment: to better eliminate spurious features and preserve invariant features without using environment labels, we assign large weights to node pairs with significant discrepancies in heterophilic Neighborhood Label Distribution (NLD) and minor discrepancies in homophilic NLD. See Section \ref{CIA-LRA_method_sec} for a detailed analysis of CIA-LRA. 
  }
  \label{fig:methodIllu}
\end{figure}
\vspace{-0.2cm} %

So far, we have theoretically and empirically validated CIA's advantage on graphs, but it still requires environmental labels that are challenging to obtain in most node classification tasks \citep{wu2021handling, liu2023flood,li2023invariant}. In this section, we propose CIA-LRA (Localized Reweighting Alignment) that realizes CIA's objective without using environment labels by identifying node pairs with significant/minor differences in spurious/invariant features and then aligning their representations. As illustrated in Figure \ref{fig:methodIllu}, CIA-LRA mainly incorporates three components:

\textbf{Localized alignment.} To avoid learning a collapsed representation of invariant features, it is crucial to align node pairs that share similar invariant features. To achieve this, we align nodes close to each other (about 2 to 6 hops). This is based on two observations. First, we observe that spurious features tend to exhibit larger changes within local graph areas than invariant ones, and nodes from the same class that are too distant may differ more in their invariant features than the closer ones (evidence in Appendix \ref{intu1}). This is because invariant features are generally more stable than spurious ones, according to \cite{chen2022learning, scholkopf2021toward}. Second, we empirically find that alignment over an extensive range or too many nodes yields only marginal performance improvements, or even leads to performance degradation (see Appendix \ref{app:excessive-align-collapse}), while increasing computational costs. This may also be attributed to the feature collapse caused by excessive alignment of too many node pairs. Formally, the local pairs are defined as $\Omega_c(t)=\{(i,j)|i\neq j \land Y_i=Y_j=c \land d(i,j)\leq t\}$, where $d(i,j)$ represents the the shortest path length from node $v_i$ to $v_j$, $ t\in \mathbb{N}^+$ is a hyperparameter. Also, we propose to assign smaller weights to pairs more distant away from each other. 

\textbf{Reweighting Alignment.} Since environment labels are unavailable, we need a metric to reflect the distribution of the spurious and invariant representations so that node pairs with significant/small differences in spurious/invariant features can be identified. Since we assume the causal patterns of the same class are similar (Assumption \ref{assump_stable_causal}), the label distribution of homophilic (i.e., same-class) neighbors directly affects the invariant features aggregated to the centered node (empirical evidence in Appendix \ref{reflect_inv}). Therefore, pairs with smaller differences in the ratio of homophilic neighbors should be assigned larger weights for alignment. The ratio discrepancy can be calculated as follows: $r^{\text{same}}(c)_{i,j}=\left |r^c_i-r^c_j \right |$, where $r^c_i$ is the ratio of the neighbors of $v_i$ of class $c$ within $L$ hops ($L$ is the number of layers of the GNN). As for spurious features, we utilize \textit{Heterophilic} (i.e., different-class) \textit{Neighborhood Label Distribution} (HeteNLD) as a measurement, as it affects the two kinds of main distributional shifts on graphs: 1) environmental node feature shifts, and 2) \textit{Neighborhood Label Distribution} (NLD) shift (both empirically verified in Appendix \ref{reflect_sp}). HeteNLD determines the first kind of shift when correlations exist between labels and spurious node features, e.g., concept shift. The second kind, NLD shift, which is affected by HeteNLD, can be regarded as a structural shift as the discrepancy in neighborhood distribution will induce a gap in the aggregated representations (Theorem \ref{bound} shows this shift increases OOD error). Although aligning the representations significantly differing in homophilic neighbor ratio mitigates these two kinds of shifts, it also leads to the collapsed invariant representations and suboptimal performance (Table \ref{ablation_table} shows this effect). Therefore we assign larger weights to the pair with a larger discrepancy in HeteNLD when alignment. The discrepancy in HeteNLD is calculated as follows: $r^{\text{diff}}(c)_{i,j}=\sum_{c'\neq c}\left | r^{c'}_i-r^{c'}_j\right |$.  

\textbf{Invariant Subgraph Extractor.}
In Section \ref{theo_model}, the structural invariant features are defined as the $k$-hop neighboring ego-graph for ease of analysis. However, in practice, the invariant structure may merely be a subgraph of the neighborhood nodes. To better capture the invariant ego subgraph, we train an invariant subgraph extractor inspired by \cite{li2023invariant}. Concretely, we learn an auxiliary GNN encoder $\phi_{\theta_m}$ (parameterized by $\theta_m$) to predict an soft edge mask $M\in \mathbb{R}^{N\times N}$, and then apply it during training and test:
\begin{equation}
\label{edge_mask}
    M_{i,j}=\text{Sigmoid}(\phi_{\theta_m}(i)^\top\phi_{\theta_m}(j)), ~A_{m}=A\odot M~, \text{where}~\odot~\text{is sample-wise multiplication.}
\end{equation}

Now we are ready to present the formal objective of CIA-LRA: 
\begin{equation}
\label{LoReCIA}
\begin{aligned}
    &\min_{\theta_h, \Theta, \theta_m} ~ \mathcal{L}(h_{\theta_h}\circ \phi_\Theta({A_m},X), Y) 
    ~ \text{s.t.} ~ \min_{\Theta,\theta_m} ~ \mathbb{E}_{c}\mathbb{E}_{\substack{i,j\\(i,j)\in \Omega_c(t)}}\left[w_{i,j}\mathcal{D}(\phi_\Theta(i), \phi_\Theta(j))\right],\\
    &\text{where}~w_{i,j}=\text{Norm}\left(\frac{r^{\text{diff}}(c)_{i,j}}{d(i,j) r^{\text{same}}(c)_{i,j}}\right), ~\text{$\text{Norm}(\cdot)$ denotes a Min-Max normalization.}
\end{aligned}
\end{equation}

 In Equation (\ref{LoReCIA}), $\mathcal{L}$ is the cross-entropy loss, $\Omega_c(t)$, $r^{\text{diff}}(c)_{i,j}$, $r^{\text{same}}(c)_{i,j}$ and $d(i,j)$ are defined in the above analysis. In practice, we use CIA-LRA as a regularization term added to the cross-entropy loss with a weight $\lambda$ as a hyperparameter. The detailed implementation of CIA-LRA is in Appendix \ref{algo_sec}.

\section{Theoretical Justification: an OOD Generalization Error Bound}
\label{bound_sec}
Now will derive an OOD generalization error bound to show that optimizing CIA-LRA can minimize OOD error. To achieve this, we adopt a PAC-Bayesian framework following \cite{ma2021subgroup} and establish a Contextual Stochastic Block Model (CSBM, 
\citep{deshpande2018contextual}) for OOD multi-classification. The proposed CSBM-OOD is as follows (more discussions are in Appendix \ref{discussion_CSBM}):
\begin{definition}
\label{csbm}
    \textbf{(CSBM-OOD).} For node $i$ of class $c$ from environment $e$, its node feature $x_i\in \mathbb{R}^{D}$ consists of two parts, $x_i=[x_{\text{inv}}^\top; x_{\text{sp}}^\top]^\top$, 
    where $x_{\text{inv}}\in \mathbb{R}^{\frac{D}{2}}$ sampled from the Gaussian distribution $\mathcal{N}(\mu_c,\sigma^2I)$ is the invariant feature and 
    $x_{\text{sp}}\in \mathbb{R}^{\frac{D}{2}}$ sampled from the $\mathcal{N}(\mu_c^e,\sigma^2I)$ is the spurious feature.  \footnote{For OOD scenarios that spurious features are not correlated with $Y$, we just need to let $\mu_c^e$ are the same for all $c$ in the environment $e$, 
    so this this definition is without loss of generality.} Suppose $\{\mu_c\}$ and $\{\mu_c^e\}$ for all $c$ and $e$ form sets of orthonormal basis. We use $p^{\text{hm}}_i$ to denote the 
    homophilic ratio of node $i$'s one-hop neighbors and use $p^{\text{ht}}_i(c')$ to denote the heterophilic ratio of node $i$'s one-hop neighbors 
    of class $c'$ ($c\neq c'$). We assume $\text{Pr}(y_i=c)$ are the same for all classes $c$.
\end{definition}

\textbf{The GNN model used for deriving the error bound (following \cite{ma2021subgroup}):} The GNN model has a 1-layer mean aggregation $g$ that outputs the aggregated feature $g_i\in \mathbb{R}^D$ for 
node $i$. The GNN classifier $h$ on top of $g$ is a ReLU-activated $L$-layer MLP with $W_1, . . . , W_L$ as parameters for each layer. $h$ is from a function family $\mathcal{H}$. The prediction for node $i$
is $h_i\in \mathbb{R}^C $ with $h_i[c]$ representing the predicted logit for class $c$. Denote the largest width of all the hidden layers as $b$.

\textbf{Notations.} Denote nodes of environment $e$ as $V_e$. We consider the error of generalizing from a mixed training environment $e^{\text{tr}}$ to any test environment $e^{\text{te}}\in \mathcal{E}_{\text{te}}$, where $V_{e^{\text{tr}}}:=\cup_{e\in \mathcal{E}_{\text{tr}}} V_e$ represents all training nodes.
To guarantee the generalization, we need to characterize the distance between  $V_{e^{\text{te}}}$ and $V_{e^{\text{tr}}}$:
define $\epsilon_{e^{\text{te}},e^{\text{tr}}}=\max_{j \in V_{e^{\text{te}}}}\min_{i\in V_{e^{\text{tr}}}}\|g_i-g_j\|_2$ as the aggregated feature distance between the training and test subgroup. Define the number of nodes in environment $e$ as $N_e$. We consider the margin loss of environment $e$ that is used by \cite{ma2021subgroup, mao2023demystifying}: 
$\widehat{\mathcal{L}}_e^\gamma(h):=\frac{1}{N^e} \sum_{v_i \in V_e} 
    \mathbf{1}\left[h_i\left[y_i\right] \leq \gamma+\max _{c \neq y_i} h_i[c]\right]$. 

Now we introduce some assumptions adapted from \cite{ma2021subgroup} that are used in our proof.

\begin{assumption}
\label{ass:ma-ass-2}
    (Equal-sized and disjoint near sets, adapted from Assumption 2 of \cite{ma2021subgroup}) For each node $v_i\in V_{e^{\text{tr}}}$, define $V_{e^{\text{te}}}^{(i)}:=\left\{j \in V_{e^{\text{te}}} \mid\left\|g_i-g_j\right\|_2 \leq \epsilon_{e^{\text{tr}},e^{\text{te}}}\right\}$. For any test environment $e^{\text{te}}$, assume $V_{e^{\text{te}}}^{(i)}$ of each $v_i\in V_{e^{\text{tr}}}$  are disjoint and have the same size $N_{e^{\text{te}}}\in \mathbb{N}^+$.
\end{assumption}

\begin{assumption}
    \label{ass:ma-ass-3} (concentrated expected loss difference, adapted from Assumption 3 of \cite{ma2021subgroup}) Let $P$ be a distribution on $\mathcal{H}$, defined by sampling the vectorized MLP parameters from $\mathcal{N}\left(0, \sigma^2 I\right)$ for some $\sigma^2 \leq \frac{\left(\gamma / 8\epsilon_{e^{\text{te}},e^{\text{tr}}}\right)^{2 / L}}{2 b\left(\lambda N_{e^{\text{tr}}}^{-\alpha}+\ln 2 b L\right)}$. For any $L$ layer GNN classifier $h \in \mathcal{H}$ with model parameters $W_1^h, \ldots, W_L^h$, define $T_h:=\max _{l=1, \ldots, L}\left\|W_l\right\|_2$. Assume that there exists some $0<\alpha<\frac{1}{4}$ satisfying
$$
\operatorname{Pr}_{h \sim P}\left(\mathcal{L}_{e^{\text{te}}}^{\gamma / 4}(h)-\mathcal{L}_{e^{\text{tr}}}^{\gamma / 2}(h)>N_{e^{\text{tr}}}^{-\alpha}+H C \epsilon_{e^{\text{te}},e^{\text{tr}}} \left\lvert\, T_h^L \epsilon_{e^{\text{te}},e^{\text{tr}}}>\frac{\gamma}{8}\right.\right) \leq e^{-N_{e^{\text{tr}}}^{2 \alpha}}
$$

\end{assumption}

Now we present the node-level OOD generalization bound (proof in Appendix \ref{bound_proof}):

\begin{theorem}
\label{bound}
\textbf{ (Subgroup OOD Generalization Bound for GNNs, informal)}. Let $\tilde{h}$ be any classifier in a function family $\mathcal{H}$ with parameters 
$\left\{\widetilde{W}_l\right\}_{l=1}^L$. 
Under Assumption \ref{ass:ma-ass-2} and \ref{ass:ma-ass-3}, for any $e^{\text{te}}\in \mathcal{E}_{\text{te}}, \gamma \geq 0$, 
and large enough $N_{e^{\text{tr}}}$, there exist $0<\alpha<\frac{1}{4}$ with probability at least $1-\delta$, we have\begin{equation}
    \begin{aligned}
        &\mathcal{L}_{e^{\text{te}}}^0(\tilde{h}) \leq \widehat{\mathcal{L}}_{e^{\text{tr}}}^\gamma(\tilde{h})
        +O( \underbrace{\frac{1}{\sigma^2}(\sum^{C}_{c=1} \sum_{c'\neq c}(\sqrt{|[(\mu_c-\mu_{c'})^\top;(\mu_c^{e^\text{te}}-\mu_{c'}^{e^\text{te}})^\top]|})\epsilon_{e^{\text{te}},e^{\text{tr}}}}_{\textbf{(a)}}\\
        + & \underbrace{2\sum_{c=1}^C(C-1)B_{e^{\text{te}}}|\mu_c^{e^\text{te}}-\mu_c^{e^\text{tr}}|)}_{\textbf{(b)}} 
        + \underbrace{ \frac{1}{2 \sigma^2} \frac{1}{N_{e^{\text{tr}}}}\sum_{i\in V_{e^{\text{tr}}}}\frac{1}{N_{e^{\text{te}}}}\sum_{j \in V_{e^{\text{te}}}^{(i)}}\sum^{C}_{c=1} \sum_{c'\neq c}
            |p^{\text{ht}}_j(c'|c)-p^{\text{ht}}_i(c'|c)|}_{\textbf{(c)}})+const\\
    \end{aligned}
\end{equation}
where $p^{\text{ht}}_i(c'|c)$ is the ratio of heterophilic neighbors of class $c'$ when $y_i=c$, $B_{e^{\text{te}}}=\max_{i\in V_{e^{\text{tr}}}\cup V_{e^{\text{te}}}}\|g_i\|_2$ is the maximum feature norm, $V_{e^{\text{te}}}^{(i)}:=\left\{j \in V_{e^{\text{te}}} \mid\left\|g_i-g_j\right\|_2 \leq \epsilon_{e^{\text{tr}},e^{\text{te}}}\right\}$. $const$ is a constant depending on $\alpha$, $\delta$, and $\gamma$. 
\end{theorem}
The observations from Theorem \ref{bound} is summarized as follows: Term $\textbf{(a)}$ reflects the separability of the original features of different classes $|[(\mu_c-\mu_{c'})^\top;(\mu_c^{e^\text{te}}-\mu_{c'}^{e^\text{te}})^\top]|$ and the distance of the aggregated features between the training and test set $\epsilon_{e^{\text{te}},e^{\text{tr}}}$. The former factor is the nature of the dataset itself. Term $\textbf{(b)}$ is the distributional discrepancy between the training and test subgroups, caused by the distribution shifts in spurious features. When there exist correlations between labels and spurious features, CIA-LRA can minimize this term by minimizing the representation distance of node pairs with large discrepancy in the label distribution of heterophilic neighbors\footnote{Although the term $|\mu_c^{e^\text{te}}-\mu_c^{e^\text{tr}}|$ is the cross-environment distance of the original data, it can be minimized implicitly by aligning the representation induced by the two subgroups. This is because minimizing the distance between learnable representations of two node groups is equivalent to using a fixed mean aggregation on two groups with closer original features (in Theorem \ref{bound} we fix the feature extractor to be a mean aggregation layer).}. Term $\textbf{(c)}$ measures the shift in HeteNLD between the training and test subgroups, representing the OOD error caused by the shift in the aggregated features of the same class.  CIA-LRA minimizes this term by enforcing stronger alignment on pairs with greater HeteNLD differences.

\renewcommand{\arraystretch}{1.05} %
\begin{table*}[htbp]
    \centering
    \caption{OOD test accuracy (\%). Our methods are marked in \textbf{bold}. The best and second-best results are shown in \textbf{bold} and \underline{underlined}, respectively. The values in parentheses are standard deviations. 'OOM' denotes Out Of Memory. Results of EERM on Cora GCN marked by '*' are from \cite{gui2022good} since we got OOM during our running. }
    \label{main_exp_table}
\begin{subtable}{\textwidth}
    \centering
    \caption{Results on GAT}
    \resizebox{\linewidth}{!}{
        \begin{tabular}{cccccccc|ccccccc}
            \toprule
            \multicolumn{15}{c}{Non-graph-specific methods}\\
            \midrule
            \multicolumn{8}{c|}{Covariate shift} & \multicolumn{7}{c}{Concept shift} \\
            \cmidrule{1-15}
            Dataset & \multicolumn{2}{c}{Arxiv} & \multicolumn{2}{c}{Cora} & CBAS & WebKB &   \multirow{2}{*}{average}& \multicolumn{2}{c}{Arxiv} & \multicolumn{2}{c}{Cora} & CBAS & WebKB &    \multirow{2}{*}{average}\\
            \cmidrule{1-7} \cmidrule{9-14}
            Domain & degree & time & degree & word & color & university &
            & degree & time & degree & word & color & university &\\
            \toprule
           ERM  &59.12(0.09)& 71.53(0.16)& 55.30(0.34) & 63.75(0.39)  & 65.24(2.69) & 31.48(6.12) & 57.74
           &65.70(0.18)& 65.79(0.05)& 61.36(0.38) & 64.38(0.13)  & 74.52(1.87) & 25.69(1.30) & 59.57\\
            \midrule
            IRM &\underline{59.16(0.14)}&\underline{71.60(0.13)}& 55.07(0.30)   & 63.75(0.26)& 65.24(1.78) & 30.69(8.63) & 57.59
            &65.79(0.35)&66.24(0.20) &61.42(0.34) &64.31(0.39) &73.33(0.89) &25.69(1.98)) & 59.46\\
            \midrule
            VREx &59.10(0.12)&71.41(0.20) & 55.34(0.16)  & 64.13(0.05) & 66.67(2.93)   & 34.13(8.27) & \underline{58.46}
            &\underline{65.94(0.03)}&66.16(0.22)& 61.77(0.30) &64.02(0.22) & 73.57(3.50) & 27.52(6.74) & 59.83\\
            \midrule
            GroupDRO  & 59.04(0.15) & 71.30(0.04) & 55.03(0.45) & 63.82(0.06) & \underline{67.62(2.43)}  & 31.75(2.82) & 58.09 &65.93(0.22) & 66.14(0.15) & 61.31(0.20) & 64.07(0.25) &73.81(1.78)  & 26.91(2.16) & 59.70\\
            \midrule
            Deep Coral & 59.12(0.10) & 71.43(0.12) & 55.35(0.45)  & 63.96(0.07) & 64.76(2.43) & 33.86(6.93) & 58.08
            &65.76(0.21) & \underline{66.25(0.19)} & 61.62(0.26) & 64.26(0.28) & \textbf{75.95(3.21)} & \underline{30.27(3.26)} & \underline{60.69}\\
            \midrule
            IGA & 59.01(0.17) & 71.49(0.32) & 55.56(0.12) & 65.07(0.25) & 65.71(5.08) & 29.89(6.56) & 57.79	& 65.87(0.19) & 65.93(0.16) & 62.62(0.13) & 64.56(0.23) & 70.95(3.51) & 28.14(2.16) & 59.68\\
            \midrule
             MatchDG & OOM & OOM & \underline{55.74(0.25)} & \underline{65.30(0.11)} & 67.15(5.08) & \underline{35.45(8.10)} & - & OOM & OOM & \textbf{62.91(0.27)} & \underline{64.83(0.08)} & \underline{75.24(2.36)} & 29.36(5.40)  & - \\
            \midrule
            \textbf{CIA}  &\textbf{59.26(0.04)}& \textbf{71.65(0.25)}& \textbf{56.34(0.35)}  & \textbf{65.55(0.19)} & \textbf{69.05(3.75)} & \textbf{36.24(0.75)} & \textbf{59.68} & \textbf{65.96(0.07)}&\textbf{66.39(0.25)} & \underline{62.72(0.22)}  &\textbf{64.92(0.28)} & 74.76(2.63) & \textbf{30.58(3.12)} & \textbf{60.89} \\
            \toprule
            \multicolumn{15}{c}{Graph-specific methods}\\
            \midrule
            EERM &OOM&OOM&  46.63(1.75)  & 62.57(0.50) & 60.47(4.10)  & \underline{33.33(14.60)} & - &OOM& OOM& 48.05(2.03) & 53.02(1.23)  & 60.95(3.56)  & 25.38(4.26) & - \\
            \midrule
            SRGNN & 58.76(0.20)  & 71.37(0.37) &55.87(0.32) & 64.50(0.35) &  \underline{68.09(0.67)} & 28.84(1.35) & \underline{57.91} &\underline{65.87(0.35)} & \underline{66.02(0.14)} & 61.21(0.29) & 64.10(0.28) & 72.38(1.22)  &23.55(1.56) & 58.86 \\
            \midrule
            Mixup &59.32(0.11)&\underline{71.78(0.08)}& \underline{56.77(0.36)} & \underline{65.70(0.28)} & 63.33(8.60) & 20.37(11.38) & 56.21  &63.11(0.10)&65.33(0.30)& \underline{63.97(0.18)} & \underline{65.42(0.32)} & 73.33(1.47)  & \textbf{38.53(0.75)} &\underline{61.62} \\
            \midrule
            GTrans & OOM & OOM &  51.49(0.23) & 62.48(0.25) & 61.90(3.56) & 21.69(7.17) & -  &OOM & OOM & 60.93(0.37) & 62.68(0.28) & 73.57(2.10) & 25.08(1.88) & - \\
            \midrule
            CIT & OOM & OOM & 53.13(2.05) & 63.76(0.20) & 61.43(3.09) & 24.60(7.47) & - & OOM & OOM & 60.89(0.36) & 63.60(0.48) & 70.24(1.68) & 24.16(5.26) & -\\
            \midrule
            CaNet & OOM & OOM & 55.35(0.14) & 62.76(0.25) & \underline{68.09(1.78)} & 23.81(15.16) & - & OOM & OOM & 60.97(0.07) & 63.73(0.44) & \underline{75.95(3.41)} & 24.77(3.97) & -\\
            \midrule
            \textbf{CIA-LRA} &\textbf{59.44(0.10)}& \textbf{71.79(0.13)}& \textbf{57.95(0.13)} & \textbf{68.59(0.26)}  & \textbf{75.24(1.78)} & \textbf{38.62(3.57)} & \textbf{61.94}  &\textbf{66.41(0.22)}&\textbf{66.47(0.16)}&  \textbf{67.08(0.26)} &\textbf{68.05(0.14)} & \textbf{78.34(3.51)} & \underline{31.80(1.88)} & \textbf{63.03}\\
            \bottomrule
        \end{tabular}
        }
    \label{main_exp_non_graph_spec}
\end{subtable}

\begin{subtable}{\textwidth}
    \centering
    \caption{Results on GCN}
    \resizebox{\linewidth}{!}{
        \begin{tabular}{cccccccc|ccccccc}
            \toprule
            \multicolumn{15}{c}{Non-graph-specific methods}\\
            \midrule
            \multicolumn{8}{c|}{Covariate shift} & \multicolumn{7}{c}{Concept shift} \\
            \cmidrule{1-15}
            Dataset & \multicolumn{2}{c}{Arxiv} & \multicolumn{2}{c}{Cora} & CBAS & WebKB &   \multirow{2}{*}{average}& \multicolumn{2}{c}{Arxiv} & \multicolumn{2}{c}{Cora} & CBAS & WebKB &    \multirow{2}{*}{average}\\
            \cmidrule{1-7} \cmidrule{9-14}
            Domain & degree & time & degree & word & color & university &
            & degree & time & degree & word & color & university &\\
            \toprule
           ERM  &58.92(0.14)& 70.98(0.20)& 55.78(0.52) & 64.76(0.30)  & 78.57(2.02) & 16.14(1.35) & 57.52
           &62.92(0.21)& 67.36(0.07)& 60.24(0.40) & 64.32(0.15)  & 82.14(1.17) & 27.52(0.75) & 60.75\\
            \midrule
            IRM &58.93(0.17)&70.86(0.12)& 55.77(0.66)  & 64.81(0.33)& 78.57(1.17) & 13.75(4.91) & 57.12
            &62.79(0.11)&67.42(0.08) &\underline{61.23(0.08)}  & 64.42(0.18) &81.67(0.89) & 27.52(0.75) & 60.84\\
            \midrule
            VREx &58.75(0.16)&69.80(0.21) & 55.97(0.53)  & 64.43(0.38) & \underline{79.05(1.78)} & \underline{17.72(11.27)} & \underline{57.62}
            &63.06(0.43)&67.42(0.07)& 60.69(0.42) &64.32(0.22) & 82.86(1.17) & 27.52(1.50) & 60.98\\
            \midrule
            GroupDRO  &58.87(0.00)& 70.93(0.09)& 55.64(0.50) & 64.62(0.30) & 79.52(0.67) & 14.29(2.59) & 57.31 & 62.98(0.53)& 67.41(0.27)& 60.59(0.36) & 64.34(0.25) & 82.38(0.67) & \underline{28.44(0.00)} & \underline{61.02}\\
            \midrule
            Deep Coral &\textbf{59.04(0.16)}& \underline{71.04(0.07)}&  56.03(0.37)  & 64.75(0.26) &  78.09(0.67) & 11.90(1.72) & 56.81
            &\underline{63.09(0.28)}&\underline{67.43(0.24)}& 60.41(0.27) & 64.34(0.17) & 82.86(0.58) &26.61(0.75)  & 60.79\\
            \midrule
            IGA & 58.87(0.17) & 70.99(0.33)  & 55.94(0.58) & \underline{64.89(0.38)} & \underline{79.05(1.78)} & 15.87(2.82) & 57.60	&  62.04(0.02)  & 66.07(0.19) & 61.06(0.36) & 64.32(0.15) & 82.38(0.89) & \underline{28.44(1.50)} & 60.72\\
            \midrule
            MatchDG & OOM & OOM & \underline{56.57(0.46)} & 64.72(0.45) & 77.14(1.17)  & 16.14(5.88) & - & OOM & OOM & 60.49(0.14) & \textbf{64.71(0.33)} & \underline{84.05(0.89)} & 27.83(2.41) & -\\
            \midrule
            \textbf{CIA}  &\underline{59.03(0.39)}& \textbf{71.10(0.15)}& \textbf{56.80(0.54)}  & \textbf{65.07(0.52)} & \textbf{80.00(2.02)} & \textbf{18.25(2.33)} & \textbf{58.38}  &\textbf{63.87(0.26)}&\textbf{67.62(0.04)} & 
 \textbf{61.59(0.18)}  &\underline{64.61(0.11)} & \textbf{85.71(0.72)} & \textbf{28.75(0.87)} & \textbf{61.83} \\
            \toprule
            \multicolumn{15}{c}{Graph-specific methods}\\
            \midrule
            EERM &OOM&OOM&  56.88(0.32)*  & 61.98(0.10)* &  40.48(9.78) &  16.21(5.67) & - &OOM& OOM& 58.38(0.04)* & 63.09(0.36)* & 61.43(1.17)&  28.04(11.67) & - \\
            \midrule
            SRGNN &\underline{58.47(0.00)} &70.83(0.10) & 57.13(0.25) & 64.50(0.35) & 73.81(4.71) & 16.40(1.63) & 56.86 &\underline{62.80(0.25)}&\underline{67.17(0.23)}& 61.21(0.29) & 64.53(0.27) &  80.95(0.67) & 27.52(0.75)  & \underline{60.70} \\
            \midrule
            Mixup &57.80(0.19)&\underline{71.62(0.11)}& \underline{57.89(0.27)} & \underline{65.07(0.22)} & 70.00(5.34) &  16.67(1.12) & 56.51  &62.33(0.34)&65.28(0.43)& \underline{63.65(0.37)} & 64.45(0.12) & 65.48(0.67)  & \underline{30.28(1.50)} & 58.58 \\
            \midrule
            GTrans & OOM & OOM &  52.70(0.52) & 63.37(0.27) & 72.38(2.43) & 10.58(0.99) & -  &OOM & OOM & 59.74(0.14) & 63.56(0.18) & 78.81(1.47) & 26.91(1.56) & - \\
            \midrule
            CIT & OOM & OOM & 56.14(0.45) & 64.79(0.29) & 75.24(2.43) & \underline{19.31(4.32)} & - & OOM & OOM & 60.12(0.30) & 64.26(0.42) & 83.10(0.89) & 28.14(1.14) & - \\
            \midrule
            CaNet & OOM & OOM & 57.35(0.04) & 64.66(0.36) & \underline{80.95(0.67)} & 15.61(5.51) & -  & OOM & OOM & 60.34(0.20) & \underline{64.65(0.39)} & \underline{85.24(3.32)} & 26.30(0.43) & - \\
            \midrule
            \textbf{CIA-LRA} &\textbf{59.85(0.14)}& \textbf{71.81(0.20)}& \textbf{58.40(0.59)} & \textbf{65.95(0.04)}  & \textbf{82.86(1.17)} & \textbf{19.84(2.83)} & \textbf{59.79}  &\textbf{64.34(0.65)}&\textbf{67.52(0.10)}&  \textbf{63.71(0.32)} &\textbf{65.07(0.21)} & \textbf{94.53(0.33)} & \textbf{36.70(0.75)} & \textbf{65.31}\\
            \bottomrule
        \end{tabular}
        }
    \label{GCN_exp_non_graph_spec}
\end{subtable}

\vspace{-10pt}
\end{table*}

\section{Experiments}

\subsection{Experiment Setup}
We run experiments using 3-layer GAT and GCN on GOOD \citep{gui2022good}, a graph OOD benchmark. We reported the results on both covariate shift and concept shift. The detailed experimental setup and hyperparameter settings are in Appendix \ref{exp_detail}. %
We compare our methods with the following algorithms: \textbf{ERM} \citep{vapnik1999overview}; traditional invariant learning methods: \textbf{IRM}, \textbf{VREx},  \textbf{GroupDRO} \citep{sagawa2019distributionally}, \textbf{Deep Coral} \citep{sun2016deep}, \textbf{IGA} \citep{koyama2020invariance}; graph OOD methods: \textbf{EERM}, \textbf{SRGNN}, \textbf{CIT} \citep{xia2024learning}, \textbf{CaNet} \citep{wu2024graph}; graph data augmentation: \textbf{Mixup} \citep{wang2021mixup}, \textbf{GTrans} \citep{jin2022empowering}.

\subsection{Main Results of OOD Generalization }
\label{maiN^exp_sec}

Table \ref{main_exp_table} reports the main OOD generalization results. The observations are summarized as follows: 1) CIA-LRA improves the best baseline methods by 2.44\% and 3.23\% on GAT and GCN, respectively, achieving state-of-the-art performance. 2) CIA outperforms IRM and VREx on all splits, which validates our theoretical findings in Section \ref{theo_sec}. Notably, it performs best among the non-graph-specific methods. 3) CIA-LRA improves CIA in most cases. This suggests that our reweighting strategy can enhance generalization on graphs even without environment labels. 4) MatchDG outperforms IRM and VREx on 12 out of 16 splits but underperforms CIA on average (averaged over 16 splits except on Arxiv, CIA: 57.56, MatchDG: 56.73). %

\subsection{CIA can be Integrated into and Improve other Graph-OOD Methods}
\begin{wrapfigure}{r}{0.4\textwidth}
\vspace{-12pt}
    \begin{minipage}{\linewidth}
    \centering
    \resizebox{0.99\linewidth}{!}{
        \begin{tabular}{ccccccc}
            \toprule
             Algorithms & Cora degree & Cora word & CBAS & WebKB \\
            \midrule
             EERM & 47.34 & 57.80  & 60.71 & 29.36 \\
             \midrule
             EERM-CIA & \underline{57.27} & \underline{62.37} & \underline{65.01}  & \underline{29.50} \\
             \midrule
             CIA & \textbf{59.53} & \textbf{65.24}  & \textbf{71.91} & \textbf{33.41} \\
            \bottomrule
        \end{tabular}
        }
    \captionof{table}{By replacing VREx in EERM with CIA (marked as \textit{EERM-CIA}), the performance is significantly improved.}
    \label{EERM_CIA}
\end{minipage}
\vspace{-30pt}
\end{wrapfigure}

\label{plug_and_play}
 We replace VREx with CIA in the loss function of EERM to show that CIA can improve generalization in a plug-and-play manner. Table \ref{EERM_CIA} shows that this improves original EERM by a large margin or has comparable performances, indicating the performance of node-level OOD algorithms can be limited by VREx.

\subsection{Empirical Understanding of the Role of CIA-LRA}
\label{empiri_understand_sec}
\textbf{A synthetic dataset.} We construct a synthetic dataset (mentioned in Section \ref{sec_intro}) to validate the role of each module in CIA-LRA in eliminating spurious features and preventing the collapse of invariant representations. We generate a random graph and create a 4-class OOD classification task. Each node has a 4-dim feature, with the first/last two dimensions representing invariant/spurious features (details in Appendix \ref{syn_dataset_sec}), so we can disentangle the learned invariant and spurious representations. Figure \ref{toy_exp_fig} depicts the OOD accuracy, the variance of the invariant representation, and the norm of the spurious representation across training epochs for CIA and CIA-LRA. The observations are summarized below.

\begin{figure}[htbp]
\begin{minipage}{0.7\linewidth}
    \centering
    \begin{minipage}{0.05\linewidth}
    \centering
    \rotatebox{90}{Concept} %
    \end{minipage}%
    \begin{minipage}{0.9\linewidth}
    \includegraphics[width=\linewidth]{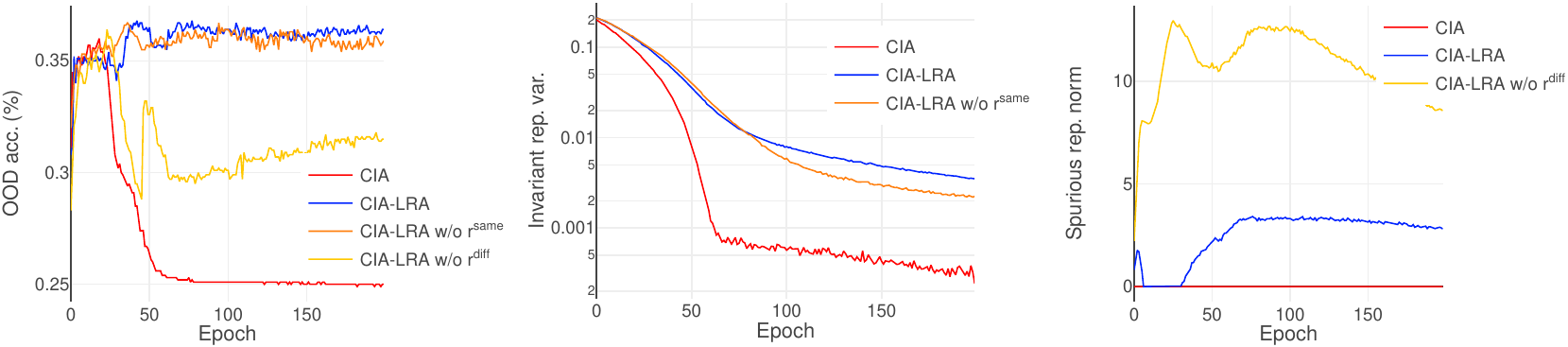} %
    \end{minipage}%

    \centering
    \begin{minipage}{0.05\linewidth}
        \centering
        \rotatebox{90}{Covariate} %
    \end{minipage}%
    \begin{minipage}{0.9\linewidth}
    \includegraphics[width=\linewidth]{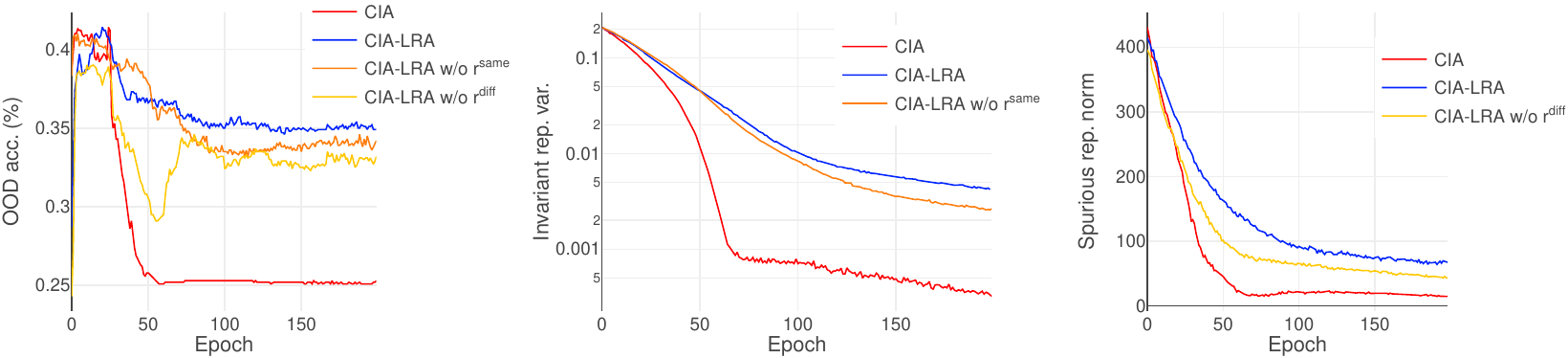} %
    \end{minipage}%
    \caption{Left: OOD test accuracy. Mid: the variance of the invariant representation. Right: the norm of the spurious representation. CIA and CIA-LRA use $\lambda=0.5$ in this figure.}
    \label{toy_exp_fig}
\end{minipage}
    \begin{minipage}{0.25\linewidth}
    \centering
    \resizebox{0.99\linewidth}{!}{
        \begin{tabular}{cc}
            \toprule
             Algorithms & Acc. (\%)  \\
             \midrule
             IRM  & 61.14  \\
             \midrule
             VREx & 61.32  \\
            \midrule
             no $r^{\text{diff}}(c)_{i,j}$  & 63.22  \\
             \midrule
             no $\frac{1}{r^{\text{same}}(c)_{i,j}}$ & 63.91 \\
             \midrule
             no $\frac{1}{d(i,j)}$ & 63.64  \\
             \midrule
             no $M$ & 63.39 \\
             \midrule
             use $r^{\text{same}}(c)_{i,j}$ in numerator & 62.70\\
             \midrule
             full CIA-LRA & \textbf{65.42} \\
            \bottomrule
        \end{tabular}
        }
    \captionof{table}{Ablation study of CIA-LRA. Results are averaged on the four splits of Cora.}
    \label{ablation_table}
\end{minipage}
    \vspace{-10pt}
    
\end{figure}

 1) Aligning the large discrepancy in HeteNLD helps to eliminate spurious features on concept shift and improves generalization. As evident from the right column, incorporating $r^\text{diff}$ diminishes the norm of spurious features under concept shift. For covariate shift, while $r^\text{diff}$ will not remove environmental spurious features due to their independence from labels, it still helps generalization since it reduces the error caused by shifts in HeteNLD as predicted by Theorem \ref{bound}.  
 2) CIA-LRA alleviates collapse of causal representation that CIA may suffer when adopting a substantial $\lambda$.  When using a large $\lambda$ ($=0.5$), the performance of CIA deteriorates to the level of random guessing (25\%) after approximately 50 epochs. In contrast, CIA-LRA sustains its accuracy at a high level because it avoids excessive alignment by aligning only local pairs and reweighting (further evidence in Appendix \ref{sec_fc}). The mid column shows that the invariant features learned by CIA progressively collapse, even if CIA removes most spurious features (right column). 
 3) Maintaining the discrepancy in homophilic neighboring label distribution $r^{\text{same}}$ helps keep the variance of the invariant representation, slightly improving performance.

\textbf{Ablation study.} We also conduct ablation studies on CIA-LRA. Table \ref{ablation_table} shows that removing any module causes a significant performance drop, demonstrating the effectiveness of each module.

\subsection{Effects of the Hyperparameters of CIA-LRA}
\begin{wrapfigure}{r}{0.5\textwidth}
\vspace{-20pt}
    \centering
    \begin{subfigure}{0.23\textwidth}
        \includegraphics[width=\textwidth]{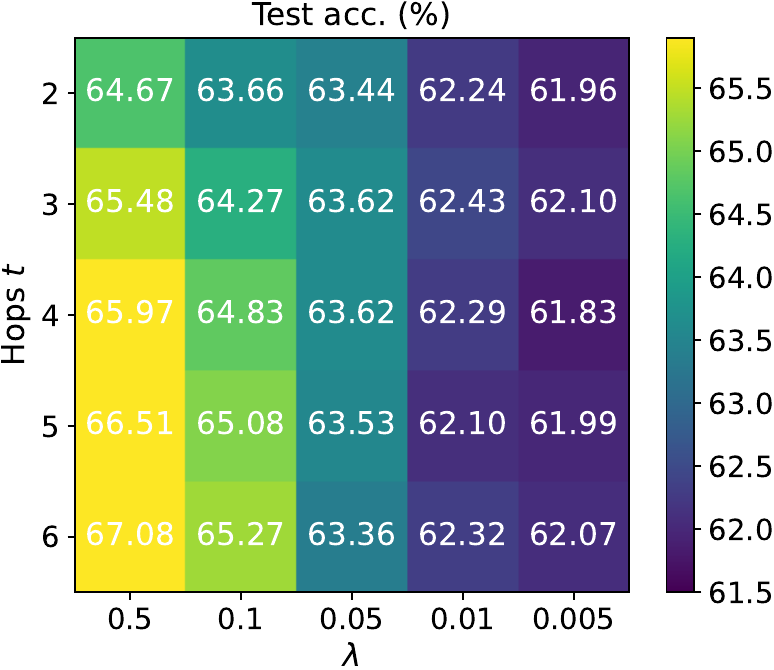}
        \caption{Cora degree concept shift. ERM: 61.36, VREx: 61.42, IRM: 61.77.}
    \end{subfigure}
    \begin{subfigure}{0.23\textwidth}
        \includegraphics[width=\textwidth]{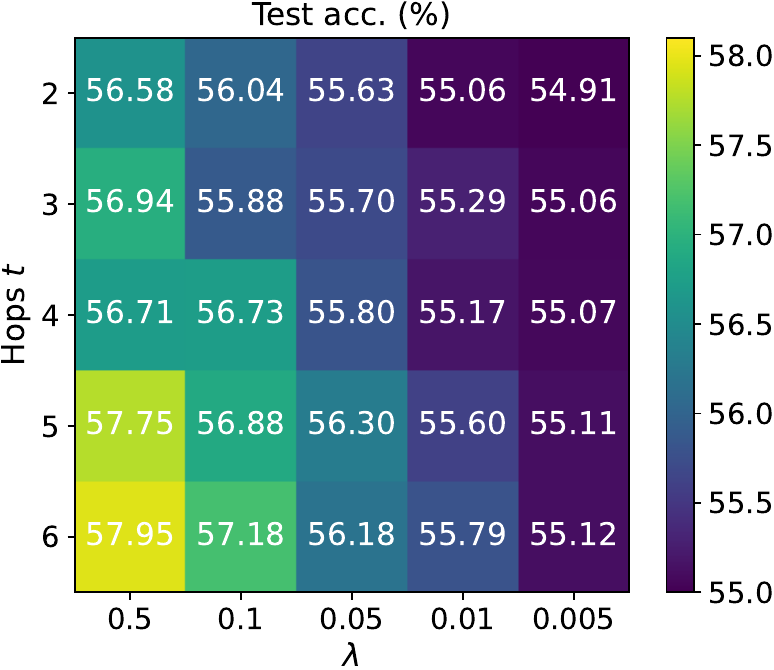}
        \caption{Cora degree covariate shift. ERM: 55.30, VREx: 55.07, IRM: 55.34.}
    \end{subfigure}
    \caption{Effect of $\lambda$ and the number of hops $t$ on OOD test accuracy (\%).}
    \label{para_heat_map}
    \vspace{-38pt}
\end{wrapfigure}
This section analyzes the effect of $\lambda$ and $t$ of CIA-LRA. Figure \ref{para_heat_map} shows that the test accuracy increases with $\lambda$ when $\lambda\leq0.5$. Too small $t$ leads to a sub-optimal performance due to insufficient regularization from aligning only a few pairs. Also, most parameter combinations outperform the baseline methods, indicating that CIA-LRA leads to consistently superior performance. Additional studies of the effects of $\lambda$ and $t$ are in Appendix \ref{sec_fc} and \ref{hops_sec}, respectively.

\section{Conclusion}
In this work, by theoretically dissecting the failure of IRM and VREx in node-level graph OOD tasks, we attribute it to the difficulty in identifying the graph-specific causal pattern structures. To address this, we propose CIA with additional class-conditional invariance constraints and its environment-label-free variant CIA-LRA tailored for graph OOD scenarios. Further theoretical and experimental results validate their efficacy. Notably, CIA can be incorporated in other graph OOD frameworks, serving as a better invariant learning objective than the widely-used VREx on graphs. %

\section*{Acknowledgement}
Yisen Wang was supported by National Key R\&D Program of China (2022ZD0160300), National Natural Science Foundation of China (92370129, 62376010), and Beijing Nova Program (20230484344, 20240484642). Xianghua Ying was supported by the National Natural Science Foundation of China (62371009).

\bibliographystyle{unsrtnat}
\bibliography{refs}

\newpage
\appendix

\onecolumn

\renewcommand\thepart{}
    \renewcommand \partname{}
    \part{Appendix}
    \setcounter{secnumdepth}{4}
    \setcounter{tocdepth}{4}
    \parttoc

\newpage

\section{Additional Related Works}
\label{related_work}
\subsection{Invariant Learning for OOD Generalization}
\label{related_inv_learning}
Invariant learning seeks to find stable features across multiple training environments to achieve OOD generalization \citep{arjovsky2019invariant, krueger2021out, bui2021exploiting, rame2022fishr, shi2021gradient, wang2020improving,mahajan2021domain, wang2022out,yi2022breaking, wang2022improving}. IRM \citep{arjovsky2019invariant} and VREx are two of the most well-known methods. The goal of IRM  is to learn a representation that elicits a classifier to achieve optimality in all training environments. VREx \citep{krueger2021out} reduces the variance of risks across training environments to improve robustness against distribution shifts. \cite{mahajan2021domain} argues that the invariant features can also change across environments. Hence, they proposed MatchDG to align the representations of the so-called same "object" of different environments. To the best of our knowledge, we are the first work to theoretically analysis the limitations of IRM and VREx in OOD node classification problems. 

It's worth emphasizing that although the CIA objective is similar to MatchDG \citep{mahajan2021domain}, our extensions of MatchDG on graphs are non-trivial:

\textbf{1) The extensions from MatchDG to CIA are mainly theoretical.}
\begin{itemize}
    \item We extend the idea of MatchDG to the node-level OOD task by providing a theoretical characterization of CIA's working mechanism on graphs (Theorem \ref{prop3}), revealing its superiority in node-level OOD scenarios for the first time.
    \item We establish the connection between node-level OOD error and two kinds of distributional shifts: (1) environmental node feature shifts and (2) the heterophilic neighborhood label distribution shifts (see Theorem \ref{bound}), giving further explanation of CIA and CIA-LRA's performance gain in node-level OOD generalization.
\end{itemize}
\textbf{2) The methodological extensions of CIA-LRA:}
\begin{itemize}
\item It identifies the node pairs with significant differences in spurious features without using environment labels, providing a new perspective that the widely adopted environment partition paradigm \cite{wu2021handling,yu2023mind,li2023graph,li2023invariant,li2022learning,liu2023flood} may not be necessary for node-level OOD generalization. One can remove spurious features by leveraging neighboring label distributions (analysis in Section \ref{CIA-LRA_method_sec}), shedding light on the role of neighborhood label distribution as compensation for the absence of environment labels.
\item It's the first node-level OOD method explicitly considering the OOD error caused by shifts in heterophilic neighborhood label distribution, (pointed out by term (c) in Theorem \ref{bound}). Such shifts can be regarded as a kind of structural shift that has been first observed in \cite{mao2023demystifying}.
\item It's the first node-level OOD method using homophilic neighborhood label distribution to reflect the find-grained distribution of the invariant features, avoiding the collapse of invariant features.
\end{itemize}

\subsection{Graph-OOD Works Using VREx and Similar Variants}
\label{works_using_vrex_sec}
We summarize the graph-OOD works that include VREx or IGA as part of the training objectives below. 
Node-level works, which can be covered by our analysis:

\begin{itemize}
    \item \textbf{EERM} \citep{wu2021handling}, Equation (5): \begin{equation}
\min _\theta \operatorname{Var}\left(\left\{L\left(g_{w_k^*}(G), Y ; \theta\right): 1 \leq k \leq K\right\}\right)+\frac{\beta}{K} \sum_{k=1}^K L\left(g_{w_k^*}(G), Y ; \theta\right),
\end{equation} where $g_{w_k^*}(G)$ is the generated $k$-th environment.

    \item \textbf{INL} \citep{li2023invariant}, Equation (8): \begin{equation}
    \mathbb{E}_{\text {e } \operatorname{supp}\left(\mathcal{E}_{\text {infer }}\right)} \mathcal{R}^e\left(f\left(\mathrm{G}_{\mathrm{v}}\right), \mathrm{y} ; \theta\right)+\lambda \text { trace }\left(\operatorname{Var}_{\mathcal{E}_{\text {infer }}}\left(\nabla_\theta \mathcal{R}^e\right)\right),
    \end{equation} where $\mathcal{E}_{\text {infer }}$ is the inferred environment label.

    \item \textbf{FLOOD} \citep{liu2023flood}, Equation (12): \begin{equation}
    \min _{\theta, \omega, \psi} \mathcal{L}_{\text {train }}=\mathcal{L}(\theta, \omega)+\alpha \mathcal{R}_{\mathrm{V}-\mathrm{REx}}(\omega, \psi)
    \end{equation}

    \item \cite{zhang2021stable} Equation (10): \begin{equation}
    \mathcal{L}_{\text {global }}=\sum_{e, e^{\prime} \in \mathcal{E}}\left(\mathcal{L}_{\text {pred }}^{0, e}-\mathcal{L}_{\text {pred }}^{0, e^{\prime}}\right)^2,
    \end{equation} as they claimed: "We note that minimizing the pair-wise distance of losses as defined in Equation 10 is equivalent to minimizing the variance of losses."
    
    \item \textbf{GLIDER} \citep{tian2024graphs}, Equation (9): \begin{equation}
    \min \mathbb{V}_e\left[\ell\left(f_c\left(g\left(G_v^e\right)\right), y_v^e\right)\right]+\alpha \mathbb{E}_e\left[\ell\left(f_c\left(g\left(G_v^e\right)\right), y_v^e\right)\right]
\end{equation}
\end{itemize}

Graph-level works, which are not covered in our analysis:

\begin{itemize}

    \item \textbf{LiSA} \citep{yu2023mind}, Equation (15): \begin{equation}
    \begin{aligned}
    & \min _f \mathcal{L}_{c l s}\left(f,\left\{g_i^*\right\}_{i=1}^n\right)+\operatorname{Var}_e\left(\mathcal{L}_{c l s}\left(f, g_i^*\right)\right), i=1 \sim n ,\\
    \end{aligned}
    \end{equation}where $g^*_i$ are the augmented training subgraphs of representing different environments.

    \item  \textbf{G-splice} \citep{li2023graph}, Equation (6): \begin{equation}
\psi^*:=\underset{\psi}{\operatorname{argmin}} \mathbb{E}_{(G, y) \sim \cup_{\varepsilon \in\left\{\mathcal{E} \cup \mathcal{E}_A\right\}} P_{\varepsilon}}\left[\ell\left(f_\psi(G), y\right)\right]+\gamma \operatorname{Var}_{\varepsilon \in\left\{\mathcal{E} \cup \mathcal{E}_A\right\}}\left[\mathbb{E}_{(G, y) \sim P_{\varepsilon}} \ell\left(f_\psi(G), y\right)\right],
\end{equation} where $\mathcal{E}_A$ are the augmented environments.

    \item \textbf{GIL} \citep{li2022learning}, Equation (8): \begin{equation}
    \mathbb{E}_{e \in \operatorname{supp}\left(\mathcal{E}_{i n f e r}\right)} \mathcal{R}^e(f(\mathrm{G}), \mathrm{Y} ; \theta)+\lambda \operatorname{trace}\left(\operatorname{Var}_{\mathcal{E}_{\text {infer }}}\left(\nabla_\theta \mathcal{R}^e\right)\right),
    \end{equation}  where $\mathcal{E}_{\text {infer }}$ is the inferred environment label.

\end{itemize}

\subsection{Comparison with Existing Node-level OOD Generalization Works}
\label{comparison_sec}
 Among the graph OOD methods, one line of research focuses on the node-level OOD generalization \citep{ wu2021handling, liu2023flood,  zhu2021shift, li2023invariant, xia2024learning}. 
 We summarize the drawbacks of previous node-level OOD methods as follows. 1) it is hard for environment-inference-based methods to generate reliable environments. To generate environments, FLOOD \citep{liu2023flood} uses random data augmentation that lacks sufficient prior. EERM \citep{wu2021handling} generates environments by maximizing loss variance, which may not necessarily enlarge differences in spurious features across environments that have been proven to be crucial for invariant learning \citep{chen2023rethinking}. Also, the adversarial learning of its environment generation process may lead to unstable performance (Table \ref{main_exp_table}) and high training costs. INL \citep{li2023invariant} relies on an estimated invariant ego-graph of each node, whose quality could significantly affect performance. Moreover, all these methods need to manually specify the number of environments, which could be inaccurate. 2) previous node-level invariant learning objectives also have some limitations. For instance, \cite{zhang2021stable, wu2021handling, liu2023flood, li2023invariant, tian2024graphs} use VREx \cite{krueger2021out} as their invariance regularization. However, we theoretically prove its potential failure on node-level OOD tasks in Section \ref{classic_fail}. SRGNN \citep{zhu2021shift} only aligns the marginal distribution $p(X)$ between the biased training distribution and the given unbiased distribution, which has been proved to have failure cases \citep{johansson2019support}. 3) some work are based on intuitive guidelines and lacks theoretical guarantees on the OOD generalization performance \citep{liu2023flood, zhu2021shift, yang2023individual}. Our proposed CIA-LRA achieves invariant learning without complex environment inference that could be unstable through a representation alignment manner. Additionally, we provide theoretical guarantees for our invariant learning objective and empirically validate its working mechanism. 

\textbf{Comparison between our Theorem \ref{prop1} and Theorem 1 of \cite{wu2021handling}}
It's worth mentioning that Theorem 1 of \cite{wu2021handling} proves $\min_\Theta \mathbb{V}_e[R(e)]$ will $\min I(y,e|z)$, where $q(z|x)$ is the induced distribution by encoder $\phi$. This seems to conflict with Theorem \ref{prop1}. This is because the upper bound derived in \cite{wu2021handling} $I(y,e|z)\leq D_{\text{KL}}(q(y|z)\|\mathbb{E}_e[q(y|z)])\leq \mathbb{V}_e[R(e)]$ is not tight. Thus, minimizing the variance of loss across training environments does not necessarily minimize mutual information between the label and the environment given the learned representation.

\subsection{Graph-level OOD Generalization}
There has been a substantial amount of work focusing on the OOD generalization problem on graphs. The vast majority have centered on graph classification tasks\citep{chen2022learning,  wu2022discovering, li2022learning, li2023graph, yang2022learning, yu2023mind, chen2023rethinking, buffelli2022sizeshiftreg, jia2023graph, li2022ood, chen2024does, chen2024unifying, gui2024joint}. Most works aim at identifying the invariant subgraph of a whole graph through specific regularization so that the model can use it when inference. Compared to node-level OOD generalization, the graph-level one is more akin to traditional OOD generalization, as the individual samples (graphs) are independently distributed. We focus on the more challenging node-level OOD generalization in this work.

\section{Additional Theoretical Results of the Covariate Shift Case}
\label{theo_sec_cov}
\subsection{Theoretical Model Setup of the Covariate Shift Case}
In this section, we will extend our theoretical model in the main text to the covariate shift setting. The causal graph of the covariate shift is shown in Figure \ref{SCM_cov}. For the covariate shift setting, 
spurious features are independent of $Y$, while $X$ changes with environment $e$. Thus we can model the data generation process for environment $e$ as
\begin{equation}
\label{data_gene_cov}
    Y^e=(\tilde{A^e})^kX_1+n_1,\quad X_2^e=n_2+{\epsilon^e},
\end{equation}
where the definition of $n_1$ and $n_2$ are the same as Section \ref{theo_sec}, ${\epsilon^e}$ represents environmental spurious features. $\epsilon^e_i$ (each dimension of $\epsilon^e$) is a random variable that are independent for $i=1,...,N^e$. We assume the intra-environment expectation of the environment spurious variable is $\mathbb{E}_{\epsilon_i\sim p_e}[\epsilon_i]=\mu^e\in \mathbb{R}$ since spurious features are consistent in a certain environment. We further assume the cross-environment expectation $\mathbb{E}_e[\epsilon^e]=\mathbf{0}$ and cross-environment variance $\mathbb{V}_e[\epsilon_i^e]=\sigma^2$, $i=1,...,N^e$ for simplicity. This is consistent with the covariate shift case that 
$p(X)$ can arbitrarily change across different domains, and the support set of $X$ may vary. Also, we require $L\geq k$ to ensure the GNN has enough capacity to learn the causal representations.

\subsection{Theoretical Results of the Covariate Shift Case}
Now we will present similar conclusions as the concept shift case. Even if VREx and IRMv1 can successfully capture invariant features in the non-graph task, they induce a model that uses spurious features. Still, CIA can learn invariant representations under covariate shift. 

\begin{proposition}
\label{prop_VREx_IRM_suc_cov}
    \textbf{(VREx and IRMv1 learn invariant features for non-graph tasks under covariate shift, proof is in Appendix \ref{proof_IRM_VREx_suc_cov})} For the non-graph version of the SCM in Equation (\ref{data_gene_cov}), 
    \begin{equation}
        Y^e=X_1+n_1,~X_2^e=n_2+{\epsilon^e},    
    \end{equation}
    Optimizing VREx $\min_\Theta \mathcal{L}_{\text{VREx}}=\mathbb{V}_e[R(e)]$ and IRMv1 $\min_\Theta \mathcal{L}_{\text{IRMv1}}=\mathbb{E}_e[\|\nabla_{w|w=1.0} R(e)\|^2]$ will learn invariant features when using a 1-layer linear network: $f(X)=\theta_1 X_1 + \theta_2 X_2$.
\end{proposition}

\begin{proposition}
\label{VREx_fail_cov}
    \textbf{(VREx will use spurious features on graphs under covariate shift)}  Under the SCM of Equation (\ref{data_gene_cov}), the objective $\min_\Theta \mathbb{V}_e[R(e)]$ has non-unique solutions for parameters of the GNN (\ref{GNN}) when part of the model parameters $\{{\theta^1_1}^{(l)},{\theta^2_1}^{(l)},{\theta^1_2}^{(l)},{\theta^2_2}^{(l)}\}$ take the values 
    \begin{equation}
    \label{spec_fail_solution}
        \Theta_0=\left\{
        \begin{array}{l}
            {\theta^1_1}^{(l)}=1, {\theta^2_1}^{(l)}=1, \quad l=L-1, ..., L-s+1  \\
            {\theta^1_1}^{(l)}=0, {\theta^2_1}^{(l)}=1, \quad l=L-s, L-s-1, ..., 1 \\
            {\theta^1_2}^{(l)}=0, {\theta^2_2}^{(l)}=1, \quad l=L-1, ..., 1       \\
        \end{array}
        \right.,
    \end{equation} $0<s<L$ is some positive integer. Specifically, the VREx solutions of $\theta_1$ and $\theta_2$ are the sets of solutions of the cubic equation, some of which are spurious solutions that $\theta_2\neq 0$ (although $\theta_2=0$ is indeed one of the solutions, VREx is not guaranteed to reach this solution):
    \begin{equation}
        \left\{
        \begin{array}{l}
            c_1\sigma^2(2\theta_1 \theta_2+ (\theta_2)^2-2 c_2 \sigma^2 \theta_2)+c_3\theta_2-\mathbb{E}_e[N^e]c_1 \sigma^2 \theta_1 \theta_2 + \mathbb{E}_e[N^e]c_2 \sigma^2 \theta_2 =0 \\
            \left[ c_3\theta_2-\mathbb{E}_e[N^e]c_1 \sigma^2 \theta_1 \theta_2 + \mathbb{E}_e[N^e]c_2 \sigma^2 \theta_2 \right]c_4-c_5(\theta_2)^2=0
        \end{array}
        \right. .
    \end{equation}
    where  $ c_1= \mathbb{E}[ (\tilde{{A^e}}^{s}X_1)^\top (\tilde{{A^e}}^{s}X_1)]$, $c_2= \mathbb{E}[(\tilde{{A^e}}^{s}X_1)^\top (\tilde{{A^e}}^{k}X_1)] $, 
    $ c_3=\mathbb{E}_e[{\epsilon^e}^\top {\epsilon^e} {\epsilon^e}^\top (\tilde{{A^e}}^{s} X_1)] \sigma^2$, $c_4=\mathbb{E}_e\left[  ({(\tilde{A^e})^kX_1})^\top \mathbf{1_{N^e}} \right]\sigma^2 $,
    $ c_5=\mathbb{E}_e \left[N^e\left(\text{tr} ( (\tilde{{A^e}}^{k})^\top \tilde{{A^e}}^{k}) + N^e (1+\sigma^2)  \right)   \right] $.
\end{proposition}

\begin{proposition}
\label{IRM_fail_cov}
    \textbf{(IRMv1 will use spurious features on graphs under covariate shift)}   Under the SCM of Equation (\ref{data_gene_cov}),  there exists $s\in \mathbb{N}^+$ that satisfies $0<s<L$ and $s\neq k$ such that optimizing the IRMv1 objective $\min_\Theta \mathbb{E}_e[\|\nabla_{w|w=1.0} R(e)\|^2]$ will not lead to the invariant solution $\theta_2=0$ for parameters of the GNN (\ref{GNN}) when $\{{\theta^1_1}^{(l)},{\theta^2_1}^{(l)},{\theta^1_2}^{(l)},{\theta^2_2}^{(l)}\}$ take the special solution:
    \begin{equation}
        \Theta_0=\left\{
        \begin{array}{l}
            {\theta^1_1}^{(l)}=1, {\theta^2_1}^{(l)}=1, \quad l=L-1, ..., L-{s}+1                                                                                                                                                                                                                                                                                                      \\
            {\theta^1_1}^{(l)}=0, {\theta^2_1}^{(l)}=1, \quad l=L-{s}, L-{s}-1, ..., 1                                                                                                                                                                                                                                                                                                     \\
            {\theta^1_2}^{(l)}=0, {\theta^2_2}^{(l)}=1, \quad l=L-1, ..., 1                                                                                                                                                                                                                                                                                                           \\
        \end{array}
        \right..
    \end{equation}
\end{proposition}

\begin{proposition}
    Optimizing the CIA objective will lead to the optimal solution $\Theta^*$:
    \begin{equation}
        \left\{
        \begin{array}{l}
            \theta_1= 1                                                                                      \\
            \theta_2=0 \\
            {\theta^1_1}^{(l)}=1, {\theta^2_1}^{(l)}=1, \quad l=L-1, ..., L-k+1                           \\
            {\theta^1_1}^{(l)}=0, {\theta^2_1}^{(l)}=1, \quad l=L-k, L-k-1, ..., 1                           \\
        \end{array}
        \right..
    \end{equation}
\end{proposition}

\section{Detailed Experimental Setup}
\label{exp_detail}
\subsection{Basic Settings}
All experimental results were averaged over three runs with different random seeds. Following \cite{gui2022good}, we use an OOD validation set for model selection. 

\textbf{GAT experiments.} For experiments on GAT, we adopt the learning rate of $0.01$ for Arxiv, $0.001$ for Cora, $0.003$ for CBAS, and $0.1$ for WebKB. The reason why we didn't use the default learning rate in \cite{gui2022good} is that since the original GOOD benchmark didn't implement GAT, so we chose to tune a learning rate for adapting GAT to reach a decent performance. The settings of batch size, training epochs, weight decay, and dropout follow \cite{gui2022good}. 

\textbf{GCN experiments.} For experiments on GCN, we follow the default settings of batch size, training epochs, learning rate, weight decay, and dropout provided by \cite{gui2022good}. 

It's worth mentioning that we choose different normalization strategies for the invariant edge mask of CIA-LRA to achieve better performance for GCN. In Equation (\ref{edge_mask}), we use Sigmoid as normalization. However, we find it is better to use a Min-Max normalization for GCN on some of the datasets. Specifically, for experiments on GCN, we use Sigmoid normalization for CBAS, Arxiv (time concept); Sigmoid for training and Min-Max for testing for WebKB (covariate); Min-Max normalization for the other dataset splits. For GAT, we use Sigmoid during training and testing for all datasets.

\subsection{Hyperparameter Settings of the Main OOD Generalization Results}
Most hyperparameter settings are adopted from \cite{gui2022good}, except that for EERM we reduce the number of generated environments from 10 to 7 and reduce the number of adversarial steps from 5 to 1 for memory and time complexity concerns. For each method, we conduct a grid search for about 3$\sim$7 values of each hyperparameter. The hyperparameter search space is presented in Table \ref{hyper2}. 
\begin{table}[htbp] 
  \caption{Hyperparameter setting of the experiments.}
  \label{hyper2}
  \centering
  \resizebox{0.7\linewidth}{!}{
  \begin{tabular}{ll}
    \toprule
      Algorithm       & Search Space    \\
    \toprule
     IRM  & $\lambda$= 0.1, 1, 10, 100    \\
     \midrule
     VREx     & $\lambda$=1, 10, 100, 1000     \\
     \midrule
     GroupDRO    & $\lambda$=0.001, 0.01, 0.1         \\
     \midrule
     Deep Coral        &$\lambda$=0.01, 0.1, 1 \\
     \midrule
     IGA        &$\lambda$=0.1, 1, 10, 100 \\
     \midrule
     Mixup        & $\alpha$=0.4, 1.0, 2.0 \\
     \midrule
     \multirow{5}{*}{EERM}    & $\beta$=0.5, 1, 3 \\
     & number of generated environments $k$=7 \\
     & adversarial training steps $t$=1\\
     & numbers of nodes for each node should be modified the link with $s$=5\\
     & subgraph generator learning rate $r$=0.0001, 0.001, 0.005, 0.01\\
     \midrule
     SRGNN & 0.000001, 0.00001, 0.0001\\
     \midrule
     GTrans & feature learning rate $r_1$=0.000005, 0.00001, 0.0001\\
     & structure learning rate $r_2$=0.01, 0.1\\
     & $\delta$ optimization steps $t$=5, 10 \\
     \midrule
    \multirow{1}{*}{CIA} & $\lambda$=0.0001, 0.0005, 0.001, 0.005, 0.01, 0.05, 0.1 (GCN)\\
                         & $\lambda$=0.0001, 0.0005, 0.001, 0.01, 0.05 (GAT) \\
    \midrule
     \multirow{2}{*}{CIA-LRA} &$\lambda$= 0.001, 0.005, 0.01, 0.05, 0.1 (GCN)\\
     &edge mask GNN learning rate $r$=0.001 (GCN + Arxiv, Cora, CBAS); $r=0.1$ (GCN + WebKB)\\
                            & $\lambda$= 0.005, 0.05, 0.1, 0.5 (GAT) \\
                            & edge mask GNN learning rate $r$=0.0001 (GAT) \\
         & hops $t$=2, 3, 4, 5, 6 \\
    \bottomrule
  \end{tabular}
  }
\end{table}

\subsection{Details of the Toy Dataset}
\label{syn_dataset_sec}
Now we introduce the setting of the toy synthetic task of Figure \ref{toy_exp_fig}. The synthetic dataset consists of four classes. Each node has a 4-dim node feature. The first/last two dimensions correspond to the invariant/spurious feature for each of the four classes, as shown in Table \ref{syn_inv_feature_table} and \ref{syn_sp_feature_table}. We artificially create both concept shift and covariate shift in this dataset. 

\begin{table}[htbp]
\centering
    \caption{Information of the toy dataset of Figure \ref{toy_exp_fig}.}
    \label{toy_dataset_setting}
\begin{subtable}{0.49\textwidth}
    \centering
    \caption{The invariant part (the first two dimensions) of the node features in the synthetic dataset.}
    \resizebox{\linewidth}{!}{
        \begin{tabular}{cccc}
            \toprule
             class 0 & class 1 & class 2 & class 3 \\
            \midrule
              $\mathcal{N}((1,1),I)$  & $\mathcal{N}((-1,1),I)$& $\mathcal{N}((-1,-1),I)$ & $\mathcal{N}((1,-1),I)$\\
            \bottomrule
        \end{tabular}
    }
    \label{syn_inv_feature_table}
\end{subtable}
\begin{subtable}{0.49\textwidth}
    \centering
    \caption{The spurious part (the last two dimensions) of the nodes features in the synthetic dataset.}
    \resizebox{\linewidth}{!}{
        \begin{tabular}{cccccc}
            \toprule
             \multirow{2}{*}{environment} & \multicolumn{4}{c}{concept shift}   &    \multirow{2}{*}{covariate shift}\\
            \cmidrule{2-5}
              & class 0 & class 1 & class 2 & class 3 &   \\
            \toprule
            training 1& $\mathcal{N}((3,3),I)$ & $\mathcal{N}((-3,3),I)$ & $\mathcal{N}((-3,-3),I)$  & $\mathcal{N}((3,-3),I)$ & $\mathcal{N}((2,2),I)$ \\
            \midrule
             training 2& $\mathcal{N}((-3,3),I)$ & $\mathcal{N}((-3,-3),I)$ & $\mathcal{N}((3,-3),I)$  & $\mathcal{N}((3,3),I)$ & $\mathcal{N}((4,4),I)$ \\
             \midrule
             test& $\mathcal{N}((-3,-3),I)$ & $\mathcal{N}((3,-3),I)$ & $\mathcal{N}((3,3),I)$  & $\mathcal{N}((-3,3),I)$ & $\mathcal{N}((6,6),I)$ \\
            \bottomrule
        \end{tabular}
        }
    \label{syn_sp_feature_table}
\end{subtable}
\end{table}

To explicitly disentangle the learned invariant and the spurious components for quantitative analysis, we employ a 1-layer GCN. We take the output of the first/last two dimensions of the weight matrix as the invariant/spurious representation.

\section{Additional Experimental Results}
\label{additional_exp}

\subsection{Excessive Alignment Leads to the Collapse of the Invariant Features.}
\label{app:excessive-align-collapse}
One may wonder why CIA underperforms CIA-LRA even if CIA uses the ground truth environment labels. In this section, we will show that CIA may suffer excessive alignment, which will lead to the collapse of the learned invariant features and consequently hurt generalization. We use the intra-class variance of the representation corresponding to invariant features (averaged over all classes) to measure the degree of collapse of invariant features. Base on this measurement, the excessive alignment can be caused by:

\textbf{1) Using a that is too large. } Evidence: on the toy dataset of Figure \ref{toy_exp_fig}, a larger $\lambda$ leads to smaller intra-class variance of invariant representations. We also compute the intra-class variance of invariant representations at epoch 50 on the toy dataset, 

\textbf{2) Aligning the representations of too many nodes.} Evidence: we show that aligning fewer node pairs can alleviate the collapse of invariant representation. By modifying CIA to align local pairs (same-class, different-environment nodes within 1 hop), termed "CIA-local", the results in Table \ref{tab:inv-collapse-local-align} show that when by aligning local pairs instead of all pairs, CIA-local avoids the collapse that CIA suffers and achieves better performance.

\begin{table}[htbp]
\centering
    \caption{Experimental evidence of the factors that can cause the collapse of the learned invariant representations.}
    \label{toy_dataset_setting}
\begin{subtable}{0.49\textwidth}
    \centering
    \caption{\textbf{ Using a that is too large can cause the collapse.} The intra-class variance of invariant representations on the toy dataset of Figure \ref{toy_exp_fig} at epoch 50.}
    \resizebox{\linewidth}{!}{
        \begin{tabular}{cccc}
            \toprule
             CIA & $\lambda=0.05$ & $\lambda=0.1$ & $\lambda=0.5$\\
            \midrule
              variance of the invariant representation & 0.061 & 0.039 & 0.011\\
            \bottomrule
        \end{tabular}
    }
    \label{tab:inv-collapse-large-lam}
\end{subtable}
\begin{subtable}{0.49\textwidth}
    \centering
    \caption{\textbf{ Aligning the representations of too many nodes can cause the collapse.} Accuracy and the variance of the invariant representations on the toy dataset of Figure \ref{toy_exp_fig} at epoch 200.}
    \resizebox{\linewidth}{!}{
        \begin{tabular}{ccc}
            \toprule
             & CIA & CIA-local \\
            \midrule
              Concept shift, accuracy		& 0.253& 0.354 \\
                \midrule
              Concept shift, 	 variance of the invariant representation	& 0.0003& 0.2327 \\
              \midrule
              Concept shift, accuracy		& 0.250 & 0.312 \\
              \midrule
              Concept shift,  variance of the invariant representation		& 0.0002& 0.1699 \\
            \bottomrule
        \end{tabular}
    }
    \label{tab:inv-collapse-local-align}
\end{subtable}
\end{table}

\subsection{CIA-LRA Alleviates Representation Collapse Caused by Excessive Alignment}
\label{sec_fc}
In Figure \ref{lambda_figure}, we illustrate the impact of $\lambda$ on OOD accuracy. Both CIA and CIA-LRA experience a performance decline at $\lambda=0.5$, indicating that excessive alignment can hinder generalization. Furthermore, CIA shows an earlier and more pronounced performance drop than CIA-LRA. This suggests that the CIA-LRA method mitigates representation collapse by aligning fewer pairs and selectively focusing on pairs with smaller differences in invariant features.

\begin{figure}[htbp]
    \centering
    \begin{subfigure}{0.24\textwidth}
        \centering
        \includegraphics[width=\linewidth]{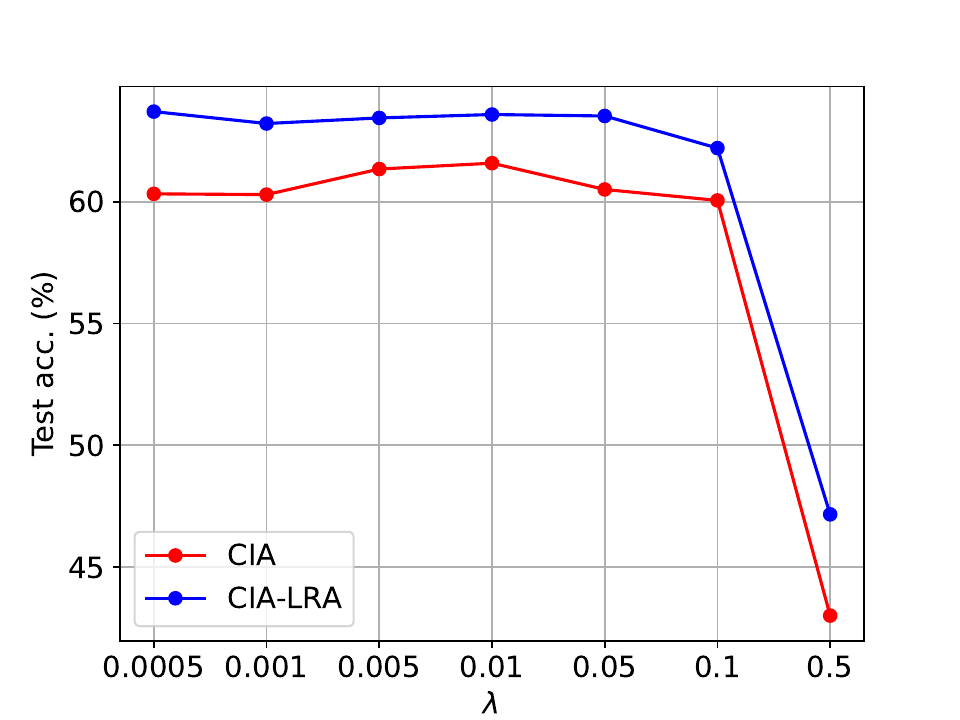}
        \caption{Cora degree concept}
    \end{subfigure}%
    \begin{subfigure}{0.24\textwidth}
        \centering
        \includegraphics[width=\linewidth]{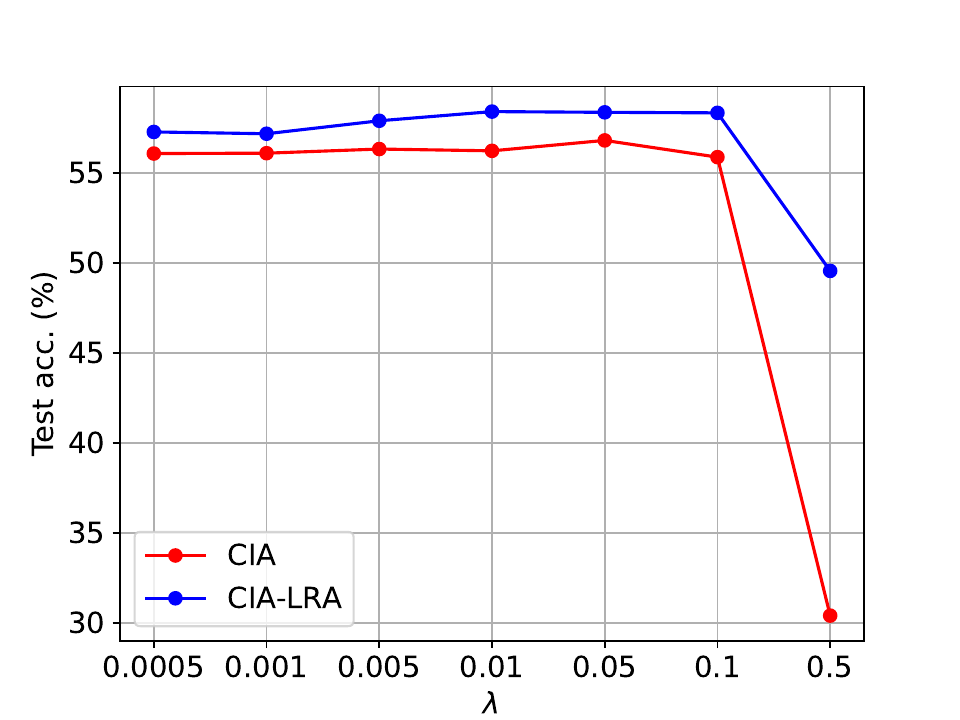}
        \caption{Cora degree covariate}
    \end{subfigure}%
    \begin{subfigure}{0.24\textwidth}
        \centering
        \includegraphics[width=\linewidth]{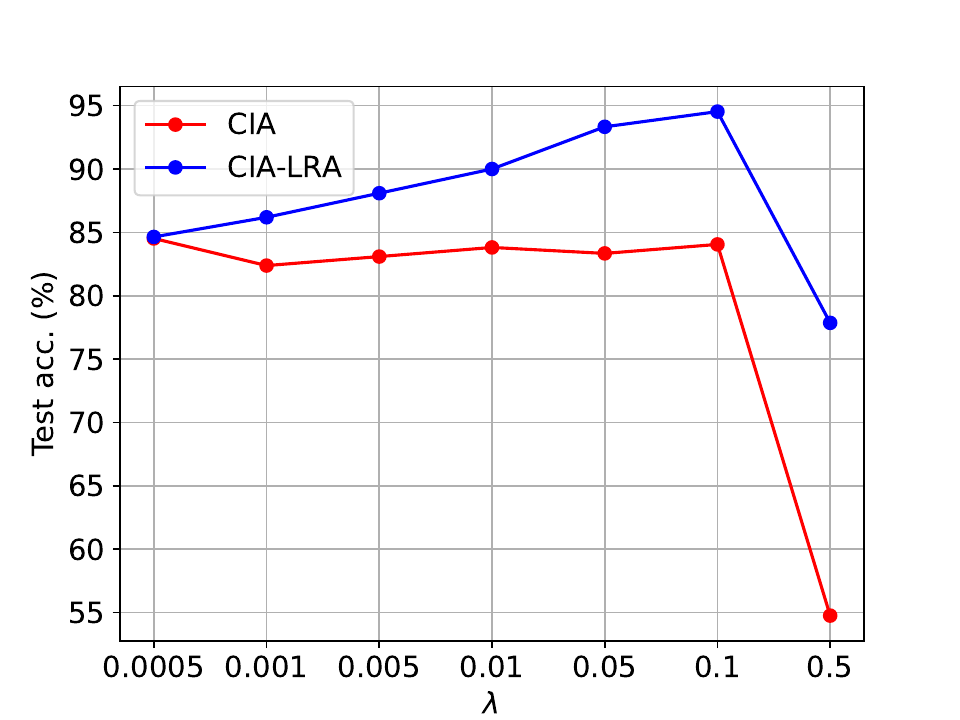}
        \caption{CBAS color concept}
    \end{subfigure}%
    \begin{subfigure}{0.24\textwidth}
        \centering
        \includegraphics[width=\linewidth]{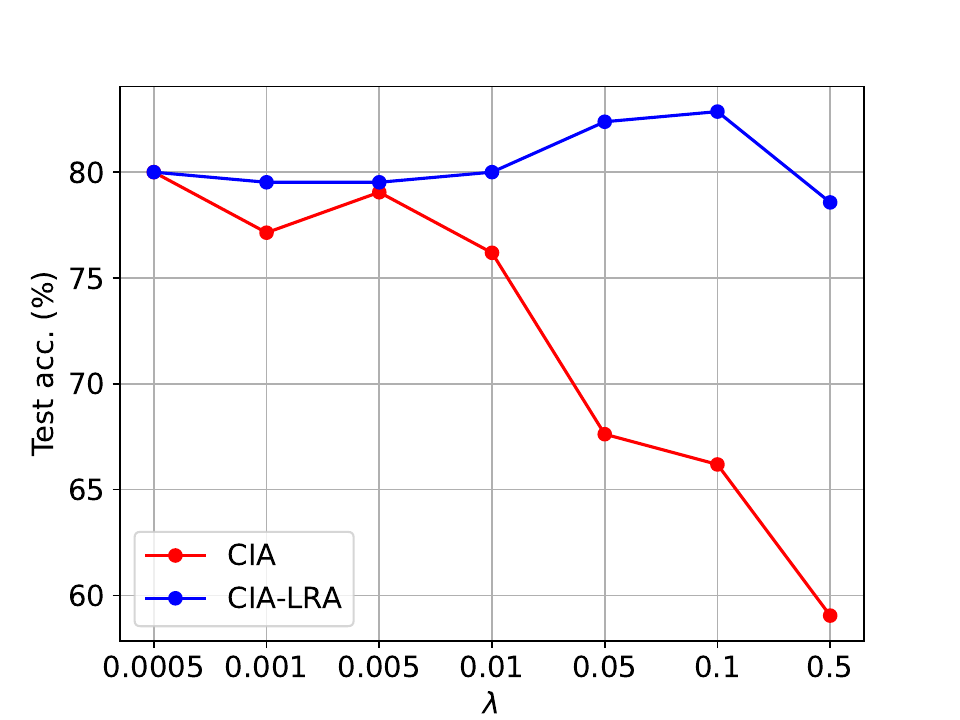}
        \caption{CBAS color covariate}
    \end{subfigure}
    \caption{The effect of $\lambda$ on OOD accuracy. CIA exhibits an earlier and more severe performance drop than CIA-LRA, demonstrating that CIA-LRA can alleviate the feature collapse caused by excessive alignment.}
    \label{lambda_figure}
\end{figure}

The role of CIA-LRA in alleviating the collapse of the invariant features can also be reflected in Figure \ref{collapse_figure}, in which the representation learned by CIA collapsed to a compact region. However, CIA-LRA does not exhibit such collapse, maintaining the diversity of the causal representation.

\begin{figure}[htbp]
    \centering
    \begin{subfigure}{0.24\textwidth}
        \centering
        \includegraphics[width=\linewidth]{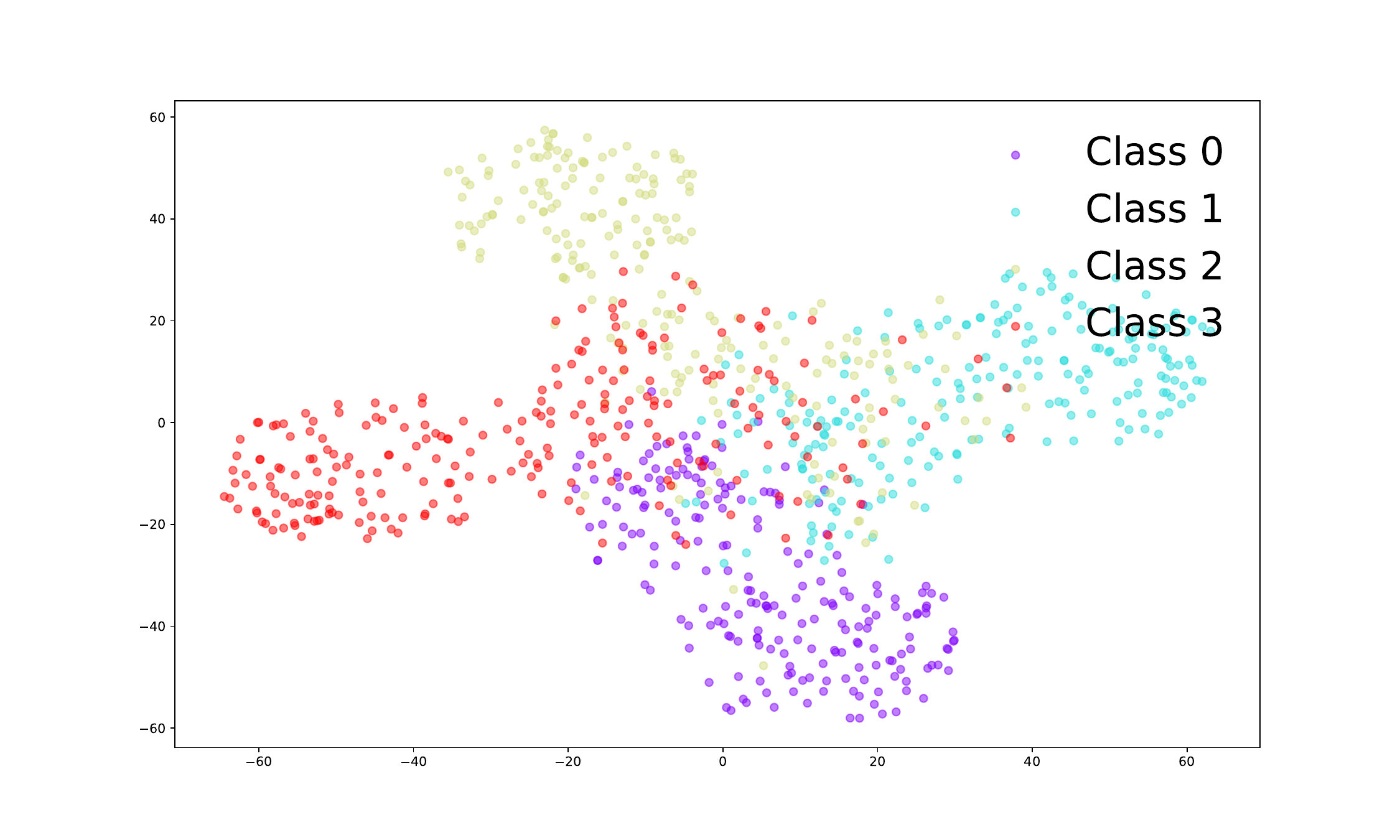}
        \caption{ERM}
    \end{subfigure}%
    \begin{subfigure}{0.24\textwidth}
        \centering
        \includegraphics[width=\linewidth]{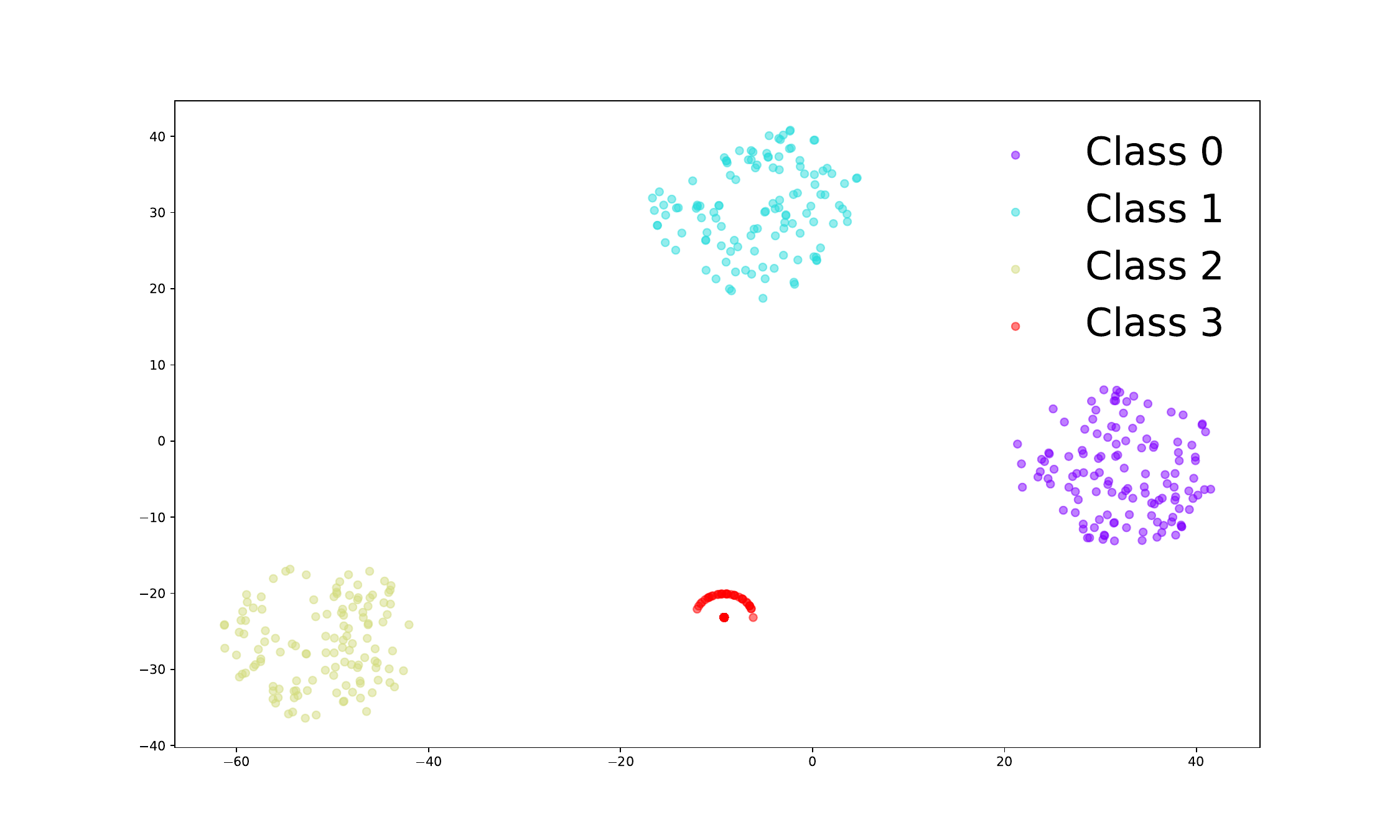}
        \caption{CIA }
    \end{subfigure}%
    \begin{subfigure}{0.24\textwidth}
        \centering
        \includegraphics[width=\linewidth]{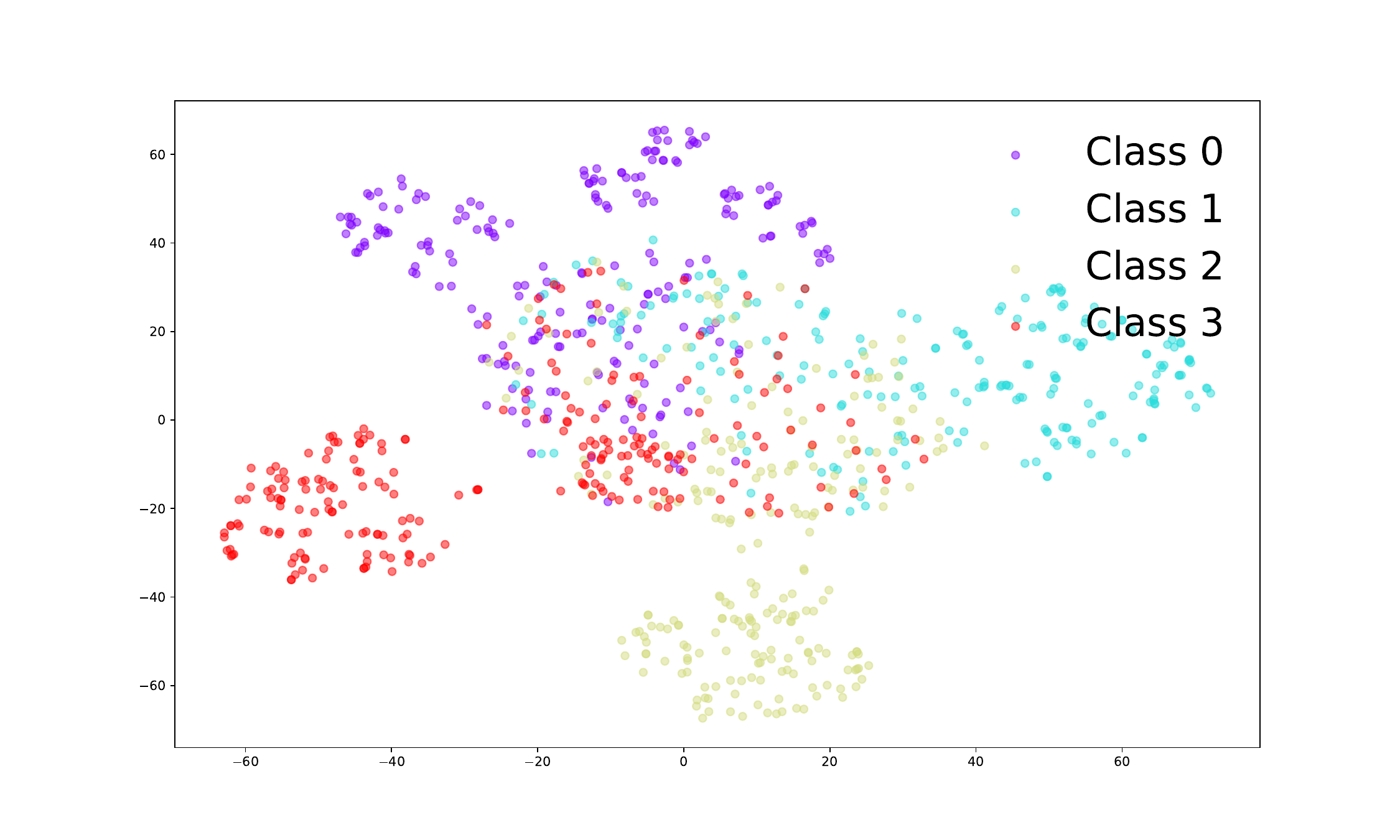}
        \caption{CIA-LRA }
    \end{subfigure}%
    \begin{subfigure}{0.24\textwidth}
        \centering
        \includegraphics[width=\linewidth]{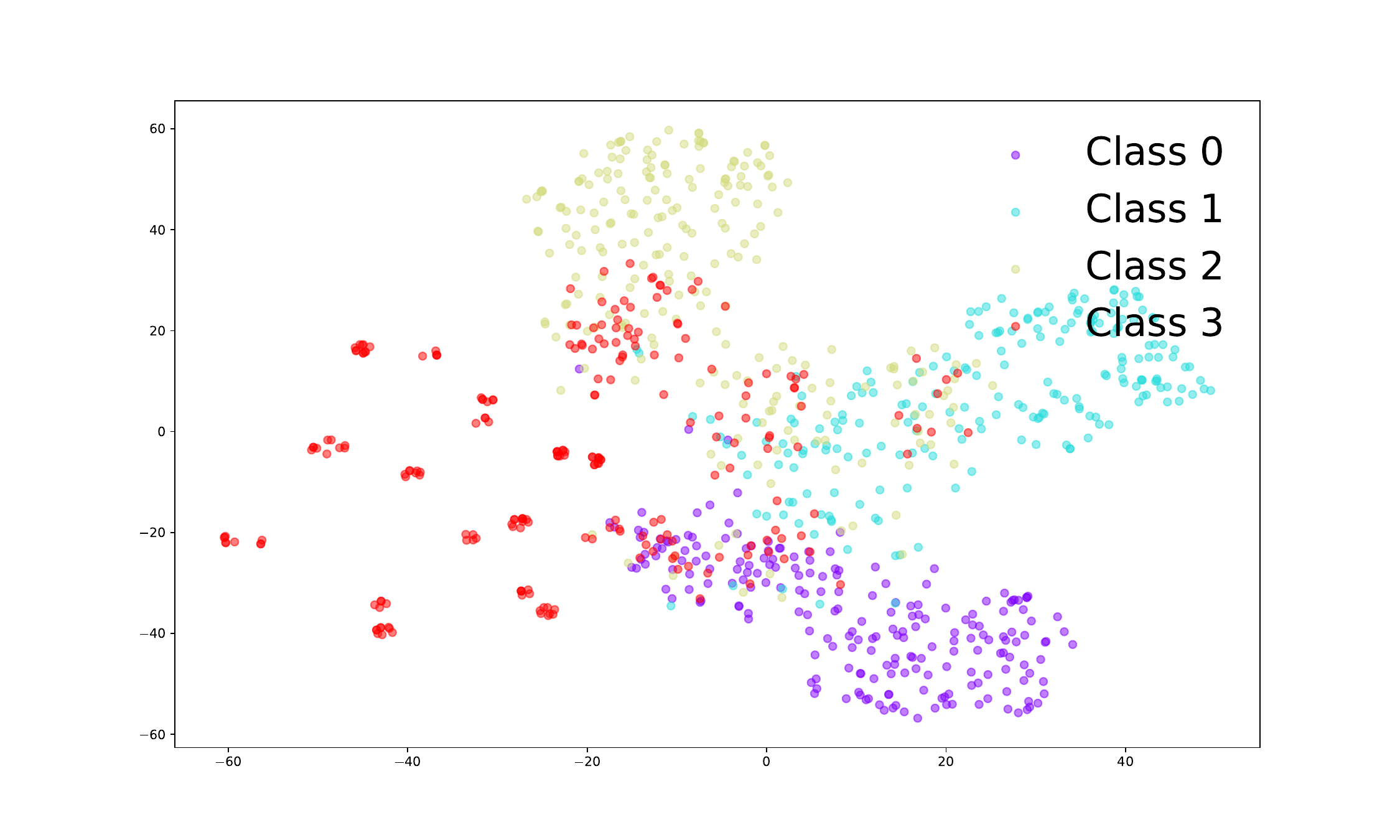}
        \caption{CIA-LRA w/o $r^{\text{same}}$ }
    \end{subfigure}
    \caption{Visualization of the learned representations at epoch 100 on the toy dataset (concept shift). Classes are distinguished by color. $\lambda=0.5$ for CIA and CIA-LRA.}
    \label{collapse_figure}
\end{figure}

\subsection{Effect of the Number of Hops for Localized Alignment}
\label{hops_sec}
In Figure \ref{hops_figure}, we plot the OOD accuracy curve of CIA-LRA against the number of hops $t$ for localized alignment (with $\lambda=0.05$). CIA-LRA achieves optimal performance within a local range of 6 to 10 hops. Performance is notably lower at smaller hops ($t=2$), due to limited regularization from aligning only a few pairs of representations. As $t$ increases, performance gains diminish and can even degrade, particularly on the CBAS color covariate. This underscores the importance of localized alignment: optimal OOD performance is attained by aligning nodes within about 10 hops. Extending the alignment range further does not enhance performance significantly and may lead to performance drops and higher computational costs. These findings support the hypothesis in Appendix \ref{intu1} that invariant features distant on the graph differ substantially, and their alignment could induce invariant feature collapse, leading to a suboptimal generalization performance.

\begin{figure}[htbp]
    \centering
    \begin{subfigure}{0.24\textwidth}
        \centering
        \includegraphics[width=\linewidth]{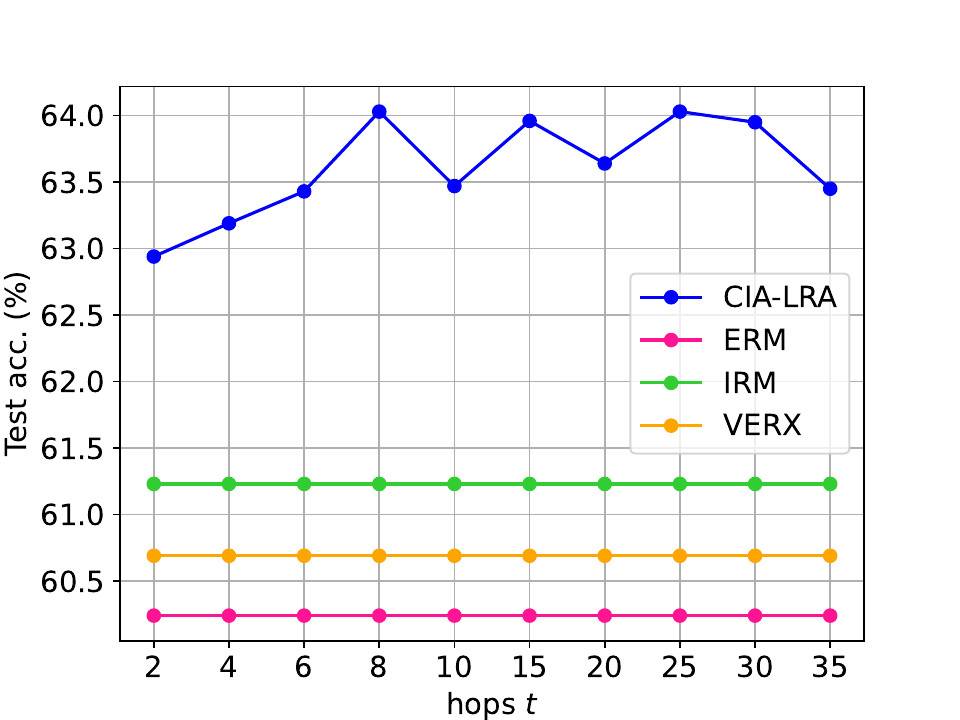}
        \caption{Cora degree concept}
    \end{subfigure}%
    \begin{subfigure}{0.24\textwidth}
        \centering
        \includegraphics[width=\linewidth]{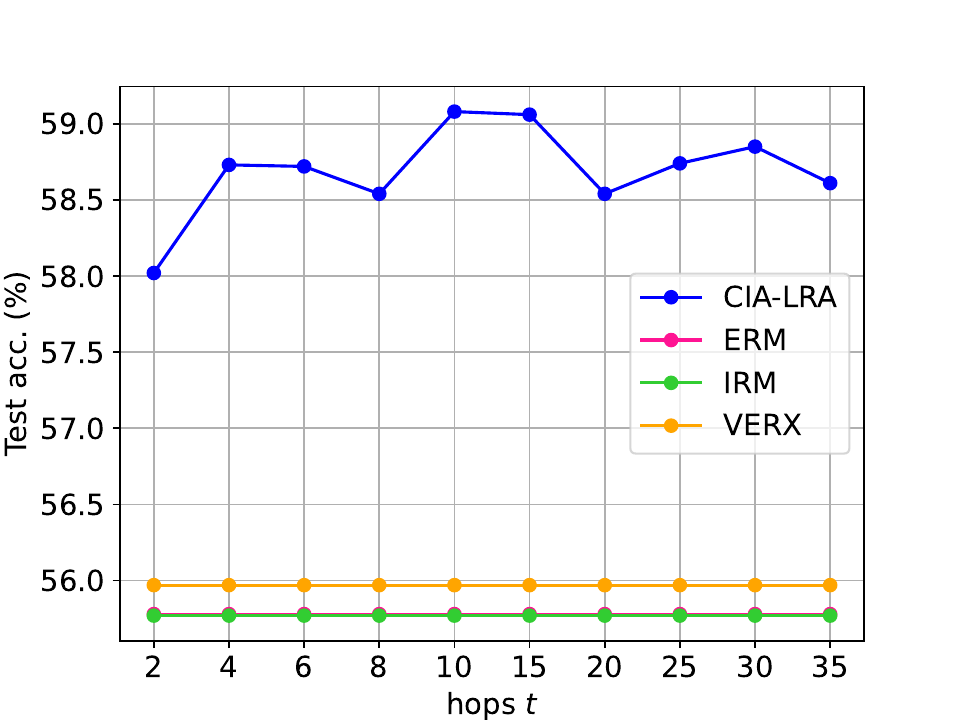}
        \caption{Cora degree covariate}
    \end{subfigure}%
    \begin{subfigure}{0.24\textwidth}
        \centering
        \includegraphics[width=\linewidth]{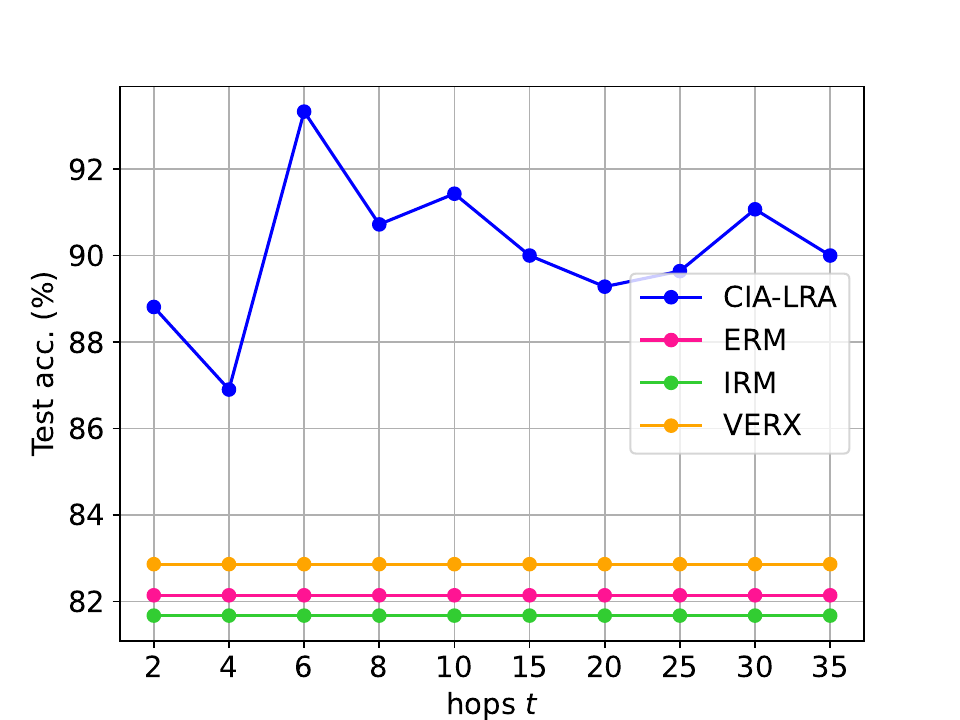}
        \caption{CBAS color concept}
    \end{subfigure}%
    \begin{subfigure}{0.24\textwidth}
        \centering
        \includegraphics[width=\linewidth]{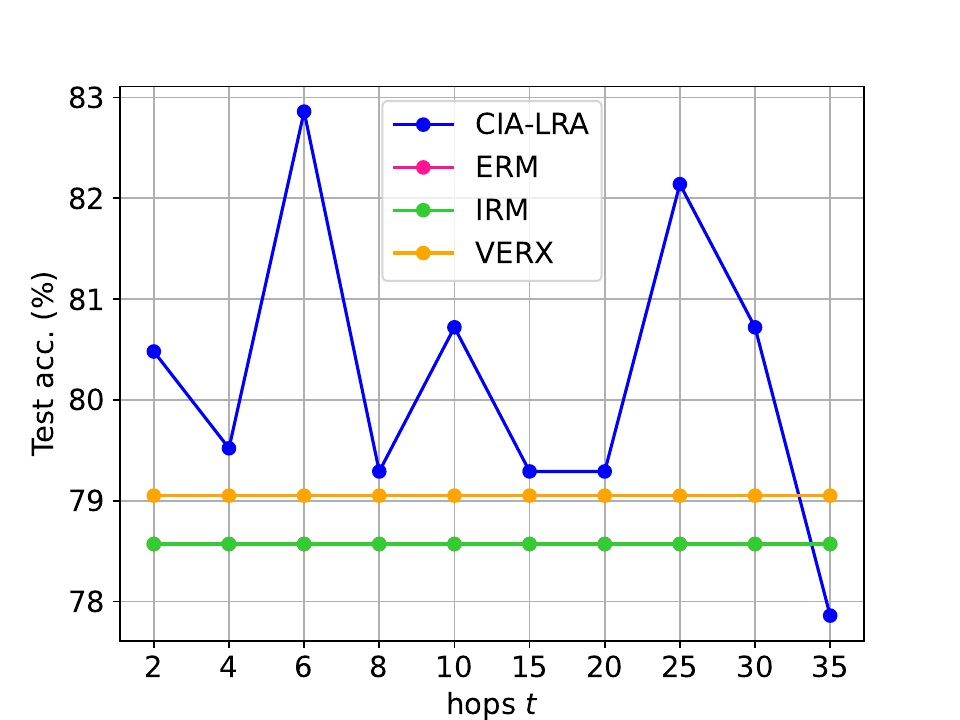}
        \caption{CBAS color covariate}
    \end{subfigure}
    \caption{The effect of the number of hops $t$ for localized alignment on OOD accuracy. Too small $t$ will lead to suboptimal performance. Too large $t$ brings limited performance gain or even deteriorates the performance.}
    \label{hops_figure}
\end{figure}

\subsection{Discussion and Validation of the Assumption on the Rate of Change of Causal and Spurious Features w.r.t Spatial Position}
\label{intu1}
To verify the intuition presented in Section \ref{CIA-LRA_method_sec} that spurious features exhibit larger changes within 
a local range (about 5 to 10 hops) on a graph compared to invariant features, we conduct experiments on real-world datasets Arxiv and Cora. To extract invariant features, we use a pre-trained VREx model and take the output of the last layer as invariant features\footnote{though we reveal in our theory that VREx may rely on spurious features, we still use VREx here to approximately extract invariant features as many previous graph OOD works have done since VREx already demonstrated some advantages in their works}. To obtain spurious features, we train an ERM model to predict the environment label and take the output of the last layer as spurious features. For each class, we randomly sample 10 nodes and generate corresponding 10 paths using Breadth-First Search (BFS). We extract invariant and spurious features of the nodes on each path and plot the L-2 distances between the node representations on the paths and the starting node. The results of Cora are in Figure \ref{rate_cora_1} and \ref{rate_cora_2}, and the results of Arxiv are in Figure \ref{rate_arxiv_1} and \ref{rate_arxiv_2}. We chose some of the classes to avoid excessive paper length; the results for the other classes are similar. 

\textbf{We observe that: despite the curve's slight fluctuations, the invariant feature difference shows a clear positive correlation with the distance from the starting point.} Specifically, within about 5$\sim$10 hops, the changes of spurious features grow more rapidly than those in invariant ones. This insight led us to align the representations of adjacent nodes to better eliminate spurious features and avoid the collapse of the invariant features. This also explains why we add a weighting term $d(i,j)$ in our loss function to assign smaller weight node pairs farther apart. Additional experimental evidence supporting the importance of localized alignment is in Appendix \ref{hops_sec}, which shows that alignment over a large range may lead to suboptimal performance and increasing computational costs.

This assumption aligns with those adopted in a series of previous works on causality and invariant learning \citep{chen2022learning, burshtein1992minimum, scholkopf2022causality, scholkopf2021toward}. These works assume that invariant features are better clustered than spurious features. In the node-level graph OOD scenario, we observe this phenomenon primarily within local parts of a graph. In some cases, when two nodes are too far apart, their invariant features can vary more than the spurious features, as seen in Figure \ref{rate_arxiv_2} (a) path 1,2,4,6,9 and 10. Therefore, matching the representations in a local region helps alleviate the invariant feature collapse problem.

\begin{figure}[htbp]
    \centering
    \begin{subfigure}{0.91\textwidth}
        \includegraphics[width=\textwidth]{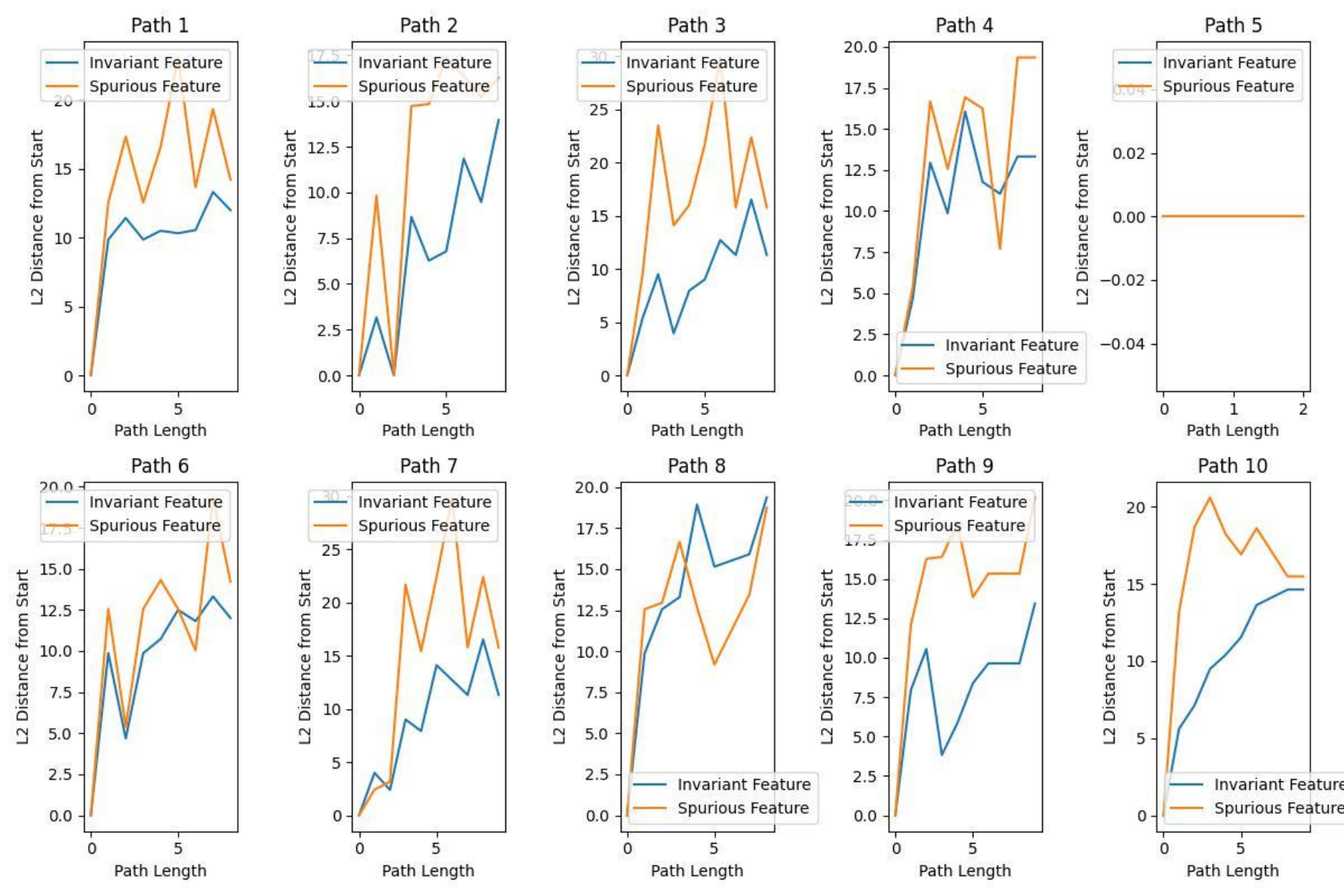}
        \caption{class 16 of Cora}
        \label{fig:image7}
    \end{subfigure}
    
    \begin{subfigure}{0.91\textwidth}
        \includegraphics[width=\textwidth]{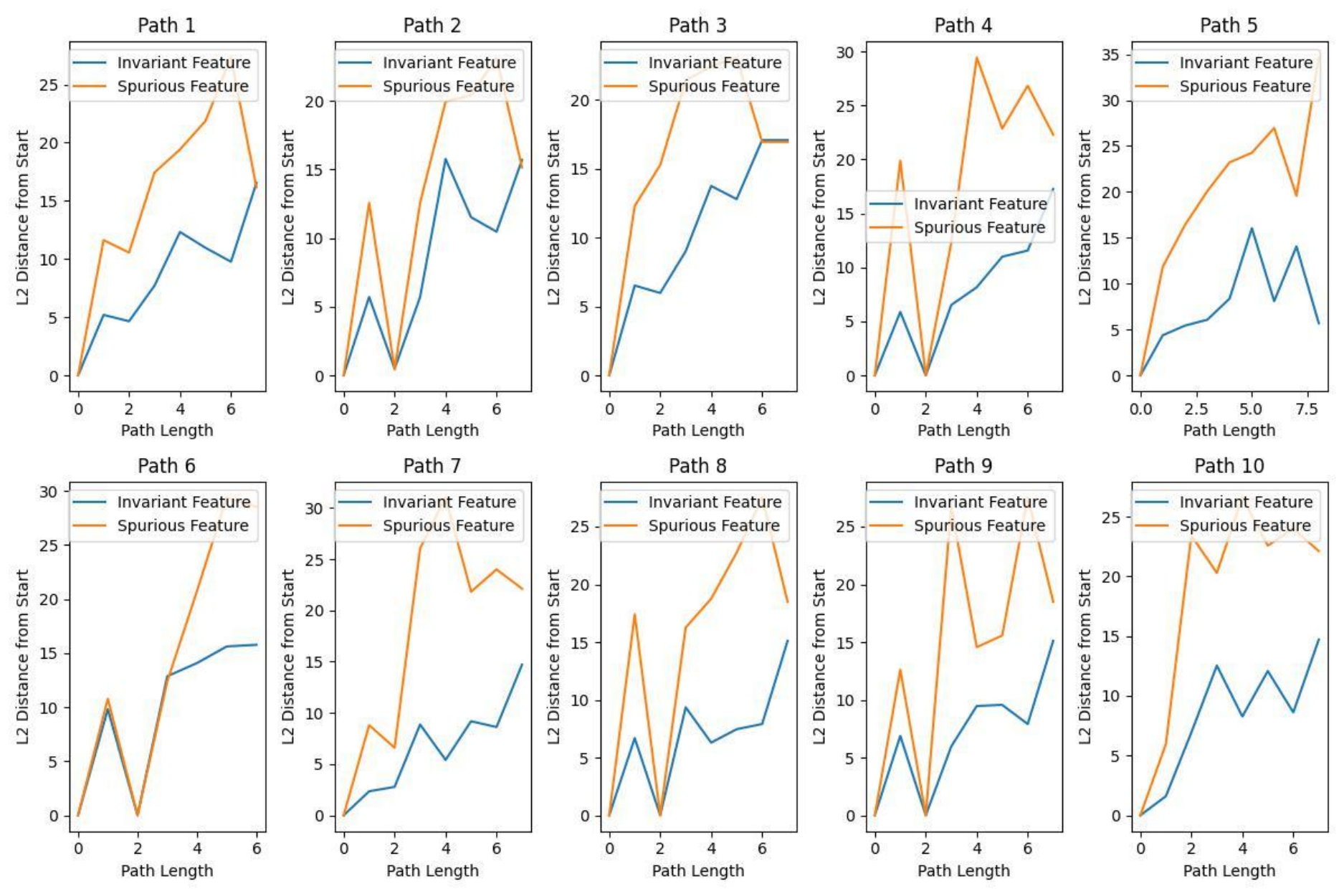}
        \caption{class 17 of Cora}
        \label{fig:image8}
    \end{subfigure}
    
    \caption{Visualization of the rate of change of invariant features and spurious features on Cora (part 1).}
    \label{rate_cora_1}
\end{figure}  

\begin{figure}[htbp]
    \centering
    \begin{subfigure}{0.91\textwidth}
        \includegraphics[width=\textwidth]{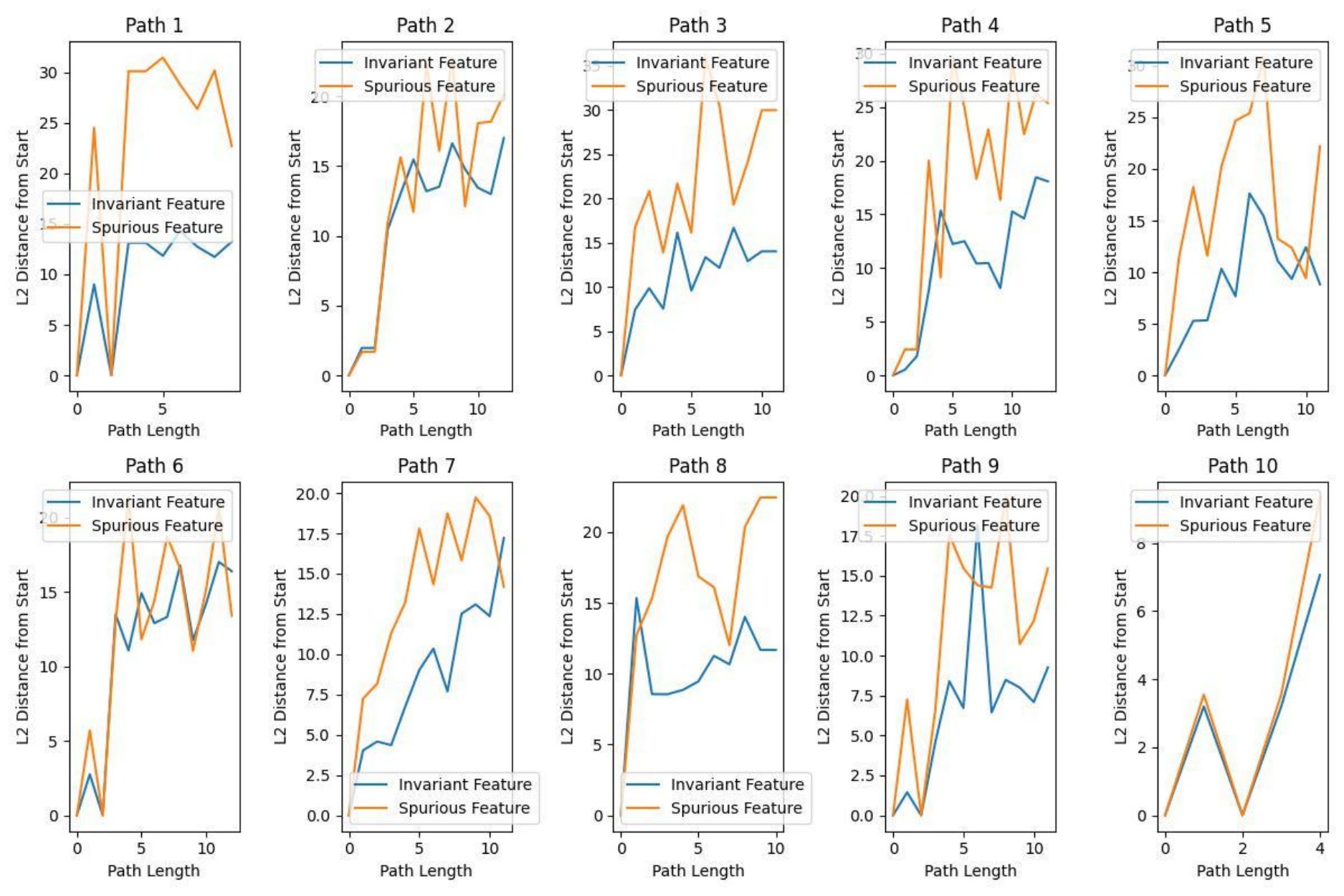}
        \caption{class 39 of Cora}
        \label{fig:image9}
    \end{subfigure}

    \begin{subfigure}{0.91\textwidth}
        \includegraphics[width=\textwidth]{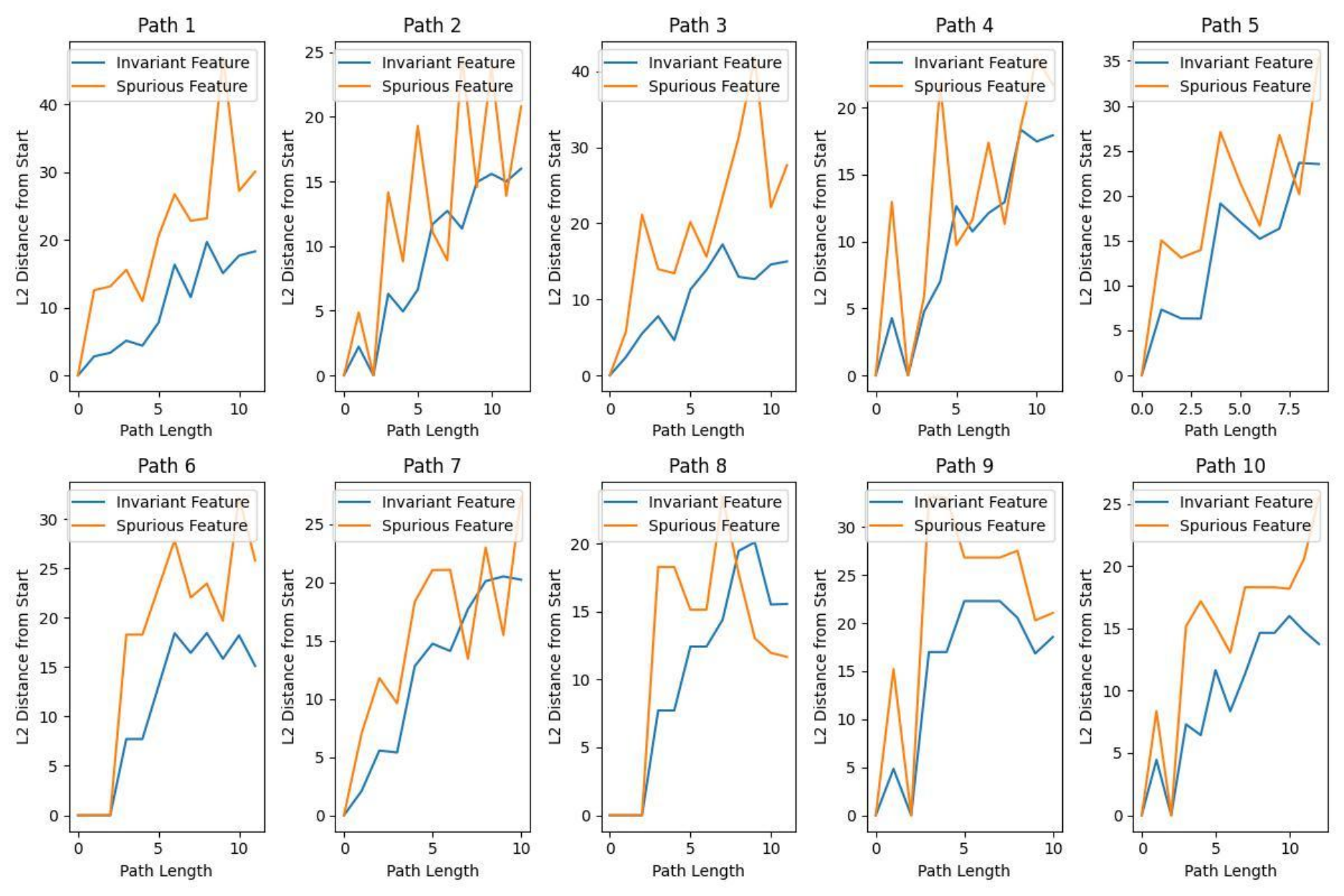}
        \caption{class 41 of Cora}
        \label{fig:image10}
    \end{subfigure}
    \caption{Visualization of the rate of change of invariant features and spurious features on Cora (part 2).}
    \label{rate_cora_2}
\end{figure}
 
\begin{figure}[htbp]
    \centering
    \begin{subfigure}{0.91\textwidth}
        \includegraphics[width=\textwidth]{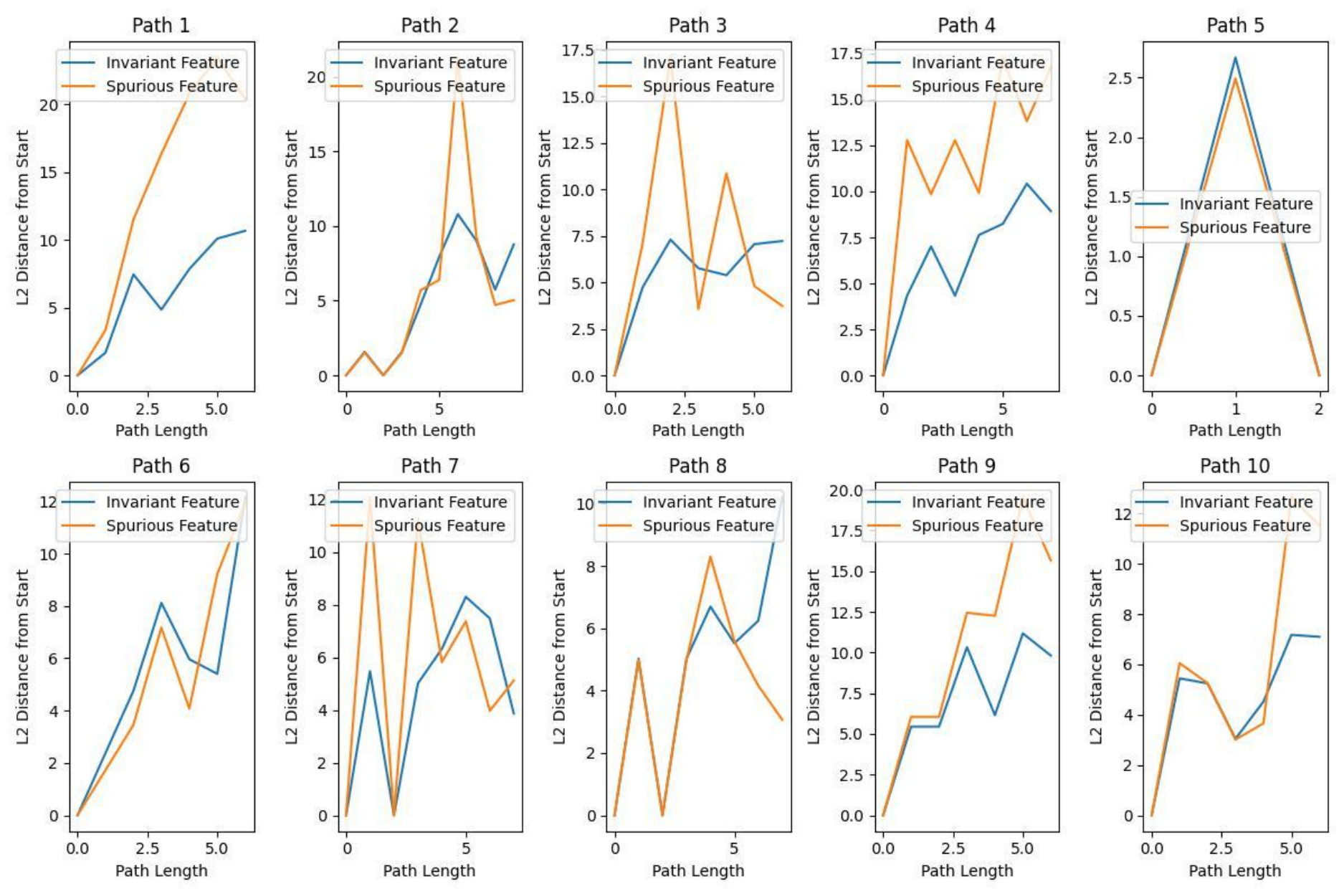}
        \caption{class 25 of Arxiv}
        \label{fig:image11}
    \end{subfigure}
    
    \begin{subfigure}{0.91\textwidth}
        \includegraphics[width=\textwidth]{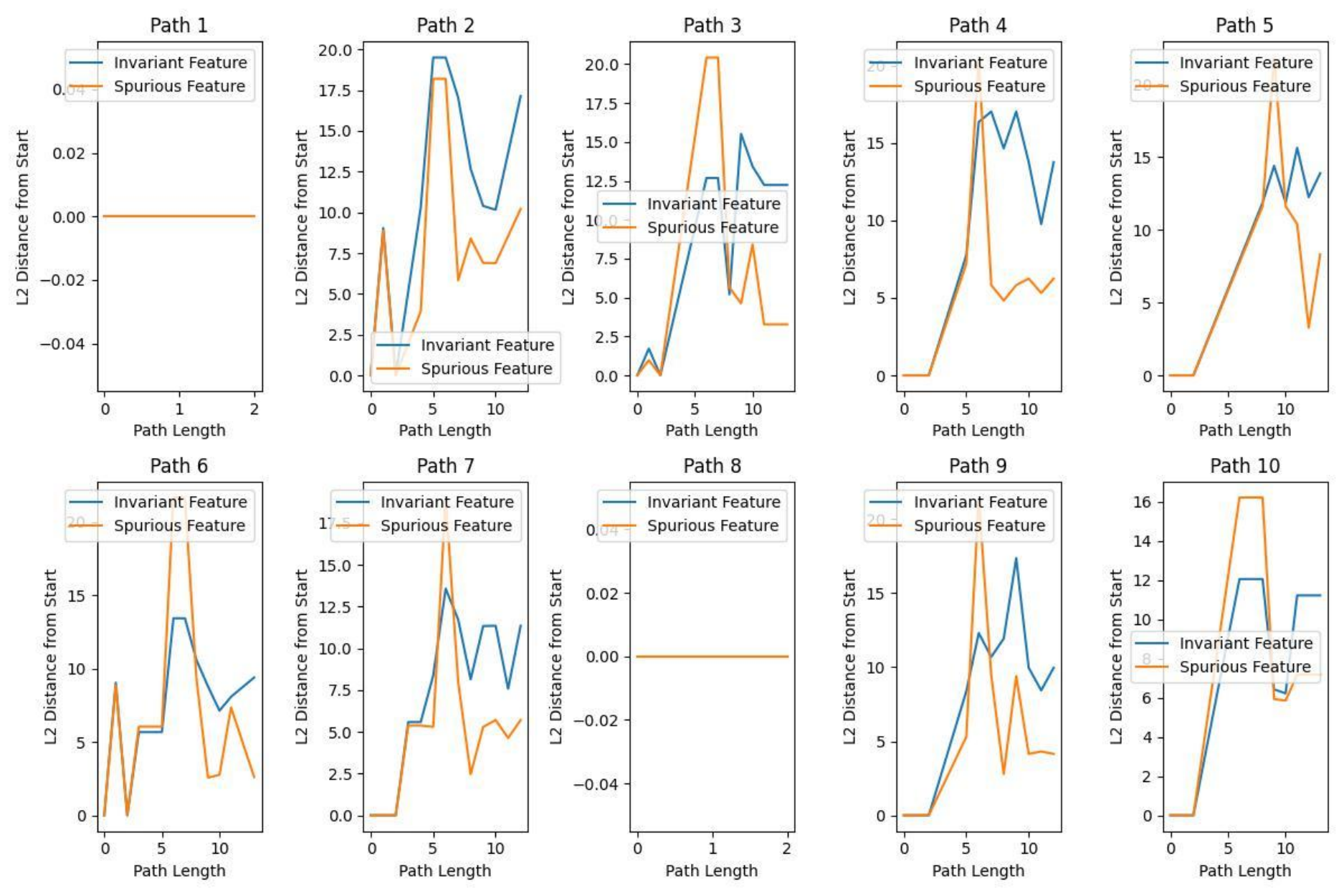}
        \caption{class 29 of Arxiv}
        \label{fig:image12}
    \end{subfigure}
    
    \caption{Visualization of the rate of change of invariant features and spurious features on Arxiv (part 1).}
    \label{rate_arxiv_1}
\end{figure}  

\begin{figure}[htbp]
    \centering
    
    \begin{subfigure}{0.91\textwidth}
        \includegraphics[width=\textwidth]{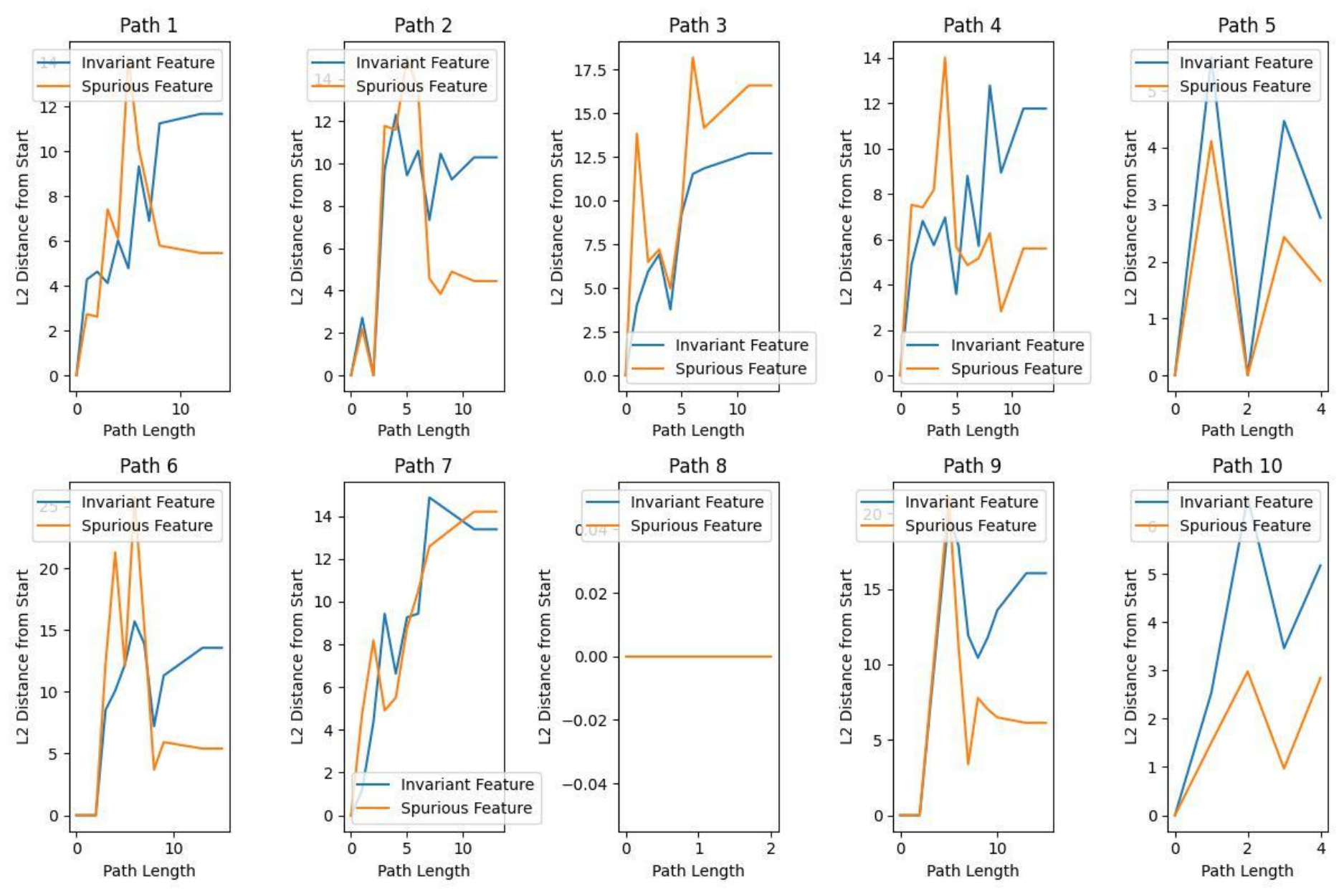}
        \caption{class 13 of Arxiv}
        \label{fig:image13}
    \end{subfigure}

    \begin{subfigure}{0.91\textwidth}
        \includegraphics[width=\textwidth]{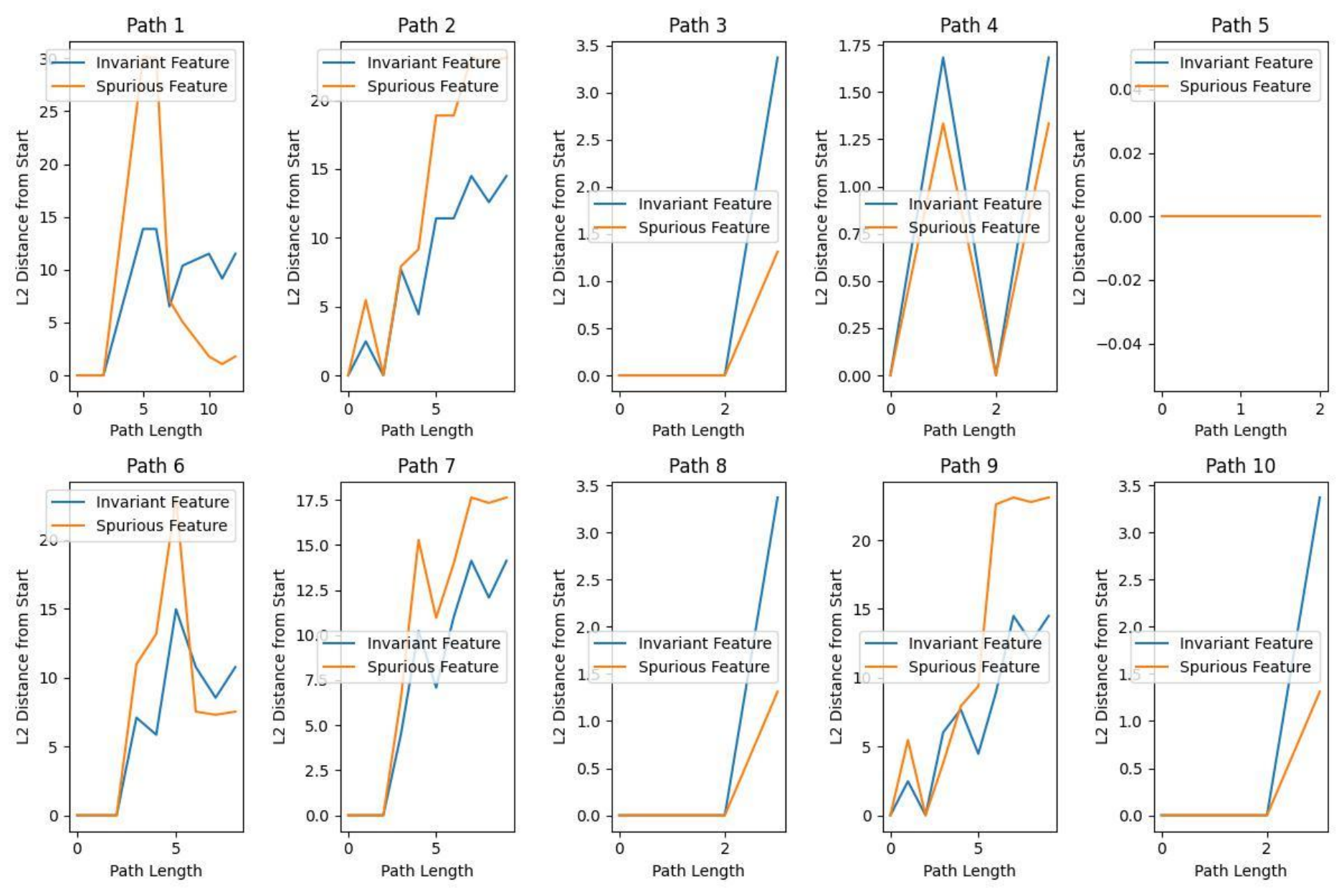}
        \caption{class 17 of Arxiv}
        \label{fig:image14}
    \end{subfigure}

    \caption{Visualization of the rate of change of invariant features and spurious features on Arxiv (part 2).}
    \label{rate_arxiv_2}
\end{figure}

\subsection{Discussion and Validation of the Assumption on the Feature Distance and Neighborhood Label Distribution Discrepancy}

\subsubsection{Heterophilic Neighborhood Labels Distribution Reflect Spurious Feature Distribution}
\label{reflect_sp}
In this section, we will empirically validate the key intuition of CIA-LRA: the label distribution of the neighbors from different classes (which we call \textit{Heterophilic Neighborhood Label Distribution}, HeteNLD) reflects the spurious representation of the centered node.   %
In node-level OOD scenarios, the distributional shifts of spurious features originate from two main sources: (1) the shifts in spurious node features associated with environments, and (2) the shifts in \textit{Neighborhood Label Distribution} (NLD), which affects the aggregated representation of the centered node.   The first type of spurious feature is analogous to those defined in Computer Vision (CV) OOD domains, while the second type is specific to graph structures. The NLD shift is a more general instance of the graph heterophily problem \citep{ma2021homophily, huang2023revisiting, mao2023demystifying}, where changes in the ratio of homophilic neighbors from training to test graphs can degrade performance. This occurs because the changes in the homophilic ratio lead to the distributional shift in the aggregated representation of the same-class nodes. Most previous methods \citep{ma2021homophily, huang2023revisiting, mao2023demystifying} only focus on the binary-classification setting, where changes in the homophilic neighbor ratio are equivalent to changes in the heterophilic neighbor ratio. However, we consider the more general multi-classification tasks. Therefore, we propose to use HeteNLD as a measurement, considering every class different from the central class and using their distribution to reflect shifts in the aggregated representation. Although the ratio of homophilic neighbors also affects environmental spurious features and NLD, it affects the invariant representation as well. Assigning larger weights to the pair with significant differences in the ratio of homophilic neighbors will simultaneously eliminate environmental spurious features and learn a collapsed invariant representation. As evidenced in Table \ref{ablation_table}, moving the $r^{\text{same}}(c)_{i,j}$ to the numerator of Equation (\ref{LoReCIA}) will lead to a significant performance decrease. Hence we use $\frac{1}{r^{\text{same}}(c)_{i,j}}$ instead of $r^{\text{same}}(c)_{i,j}$ in $w_{ij}$.

In the following part, we will empirically validate our intuition that HeteNLD can reflect the two spurious representation distributions on \textit{concept shift}, where $p(Y|X)$ varies across environments, and \textit{covariate shift}, where $p(X)$ changes with environments, respectively. We will show that HeteNLD affects the spurious features of the centered node in different manners under concept shift and covariate shift.

 \textbf{Covariate shift.} For covariate shifts on graphs, since spurious features are not necessarily correlated with labels, the environmental spurious features cannot be reflected by HeteNLD. However, we can still measure how HeteNLD affects the aggregated neighborhood representation. To obtain neighborhood representation, we train a 1-layer GCN that aggregates neighboring features and discards the features of the centered node. We hope to observe whether the gap of HeteNLD accurately reflects the distance of neighborhood representation. 
 To ensure that the discrepancy in the aggregated neighboring feature is caused solely by heterophilic neighbors, we only use point pairs with the same number of homophilic neighbors. 
 Specifically, we compute the L-2 distance between the neighborhood representations of two nodes with the same number of class-same neighbors, and plot its trend w.r.t. the distance of HeteNLD (according to the definition of $r_{i,j}^\text{diff}$ in Equation (\ref{LoReCIA}), except that we didn't normalize by the node degree here). We run experiments on Cora to verify this. We evaluate on both \textit{word} shifts (node feature shifts) and \textit{degree} (graph structure shifts) for a comprehensive understanding. We show the results of the first 30 classes of Cora. \textbf{The results in Figure \ref{reflect_sp_cora_word_cov} and \ref{reflect_sp_cora_deg_cov} show a clear positive correlation between the neighborhood representation distance and HeteNLD discrepancy under covariate shifts, indicating HeteNLD discrepancy can reflect the distance of the aggregated representation.} 

\begin{figure}
    \centering
    \begin{subfigure}{.19\textwidth}
        \centering
        \includegraphics[width=\linewidth]{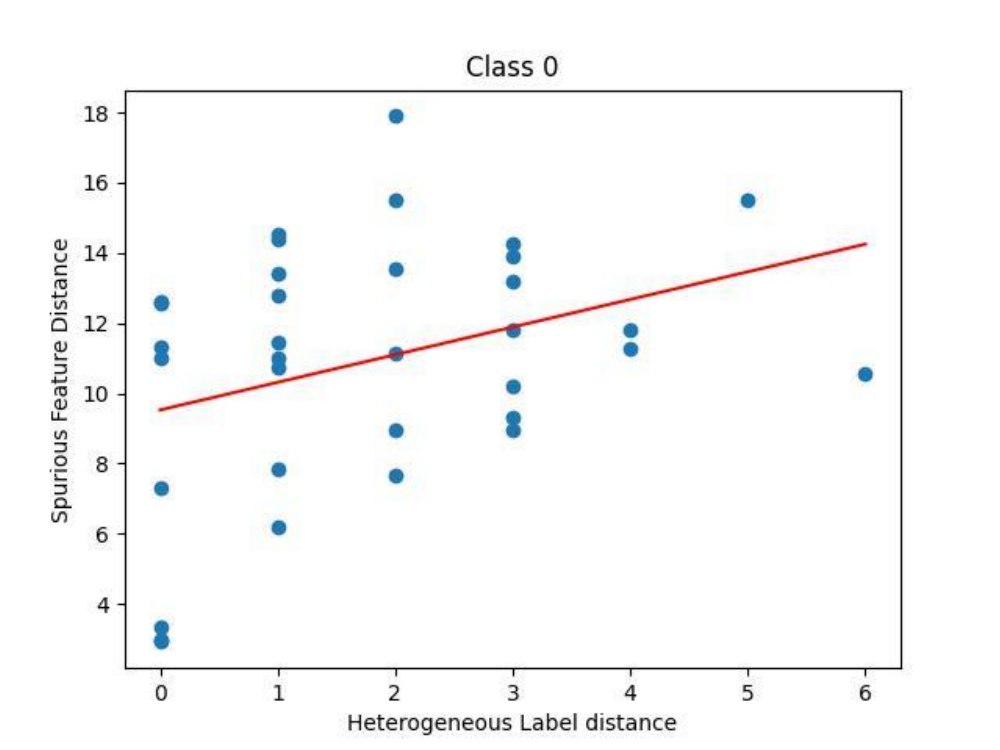}
    \end{subfigure}%
    \begin{subfigure}{.19\textwidth}
        \centering
        \includegraphics[width=\linewidth]{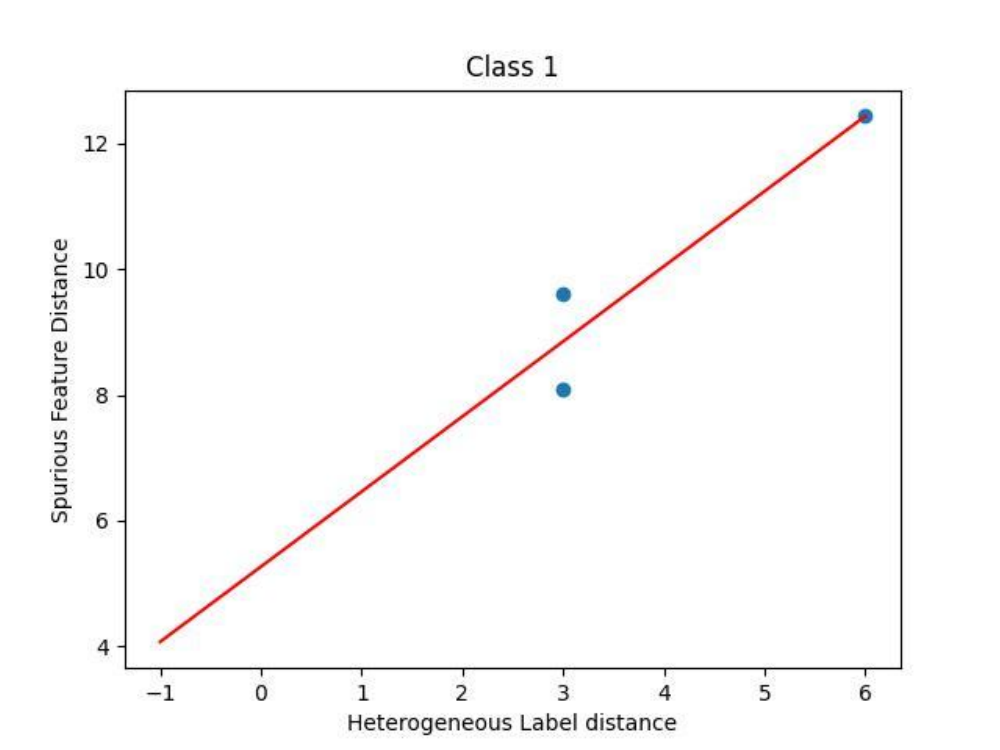}
    \end{subfigure}
    \begin{subfigure}{.19\textwidth}
        \centering
        \includegraphics[width=\linewidth]{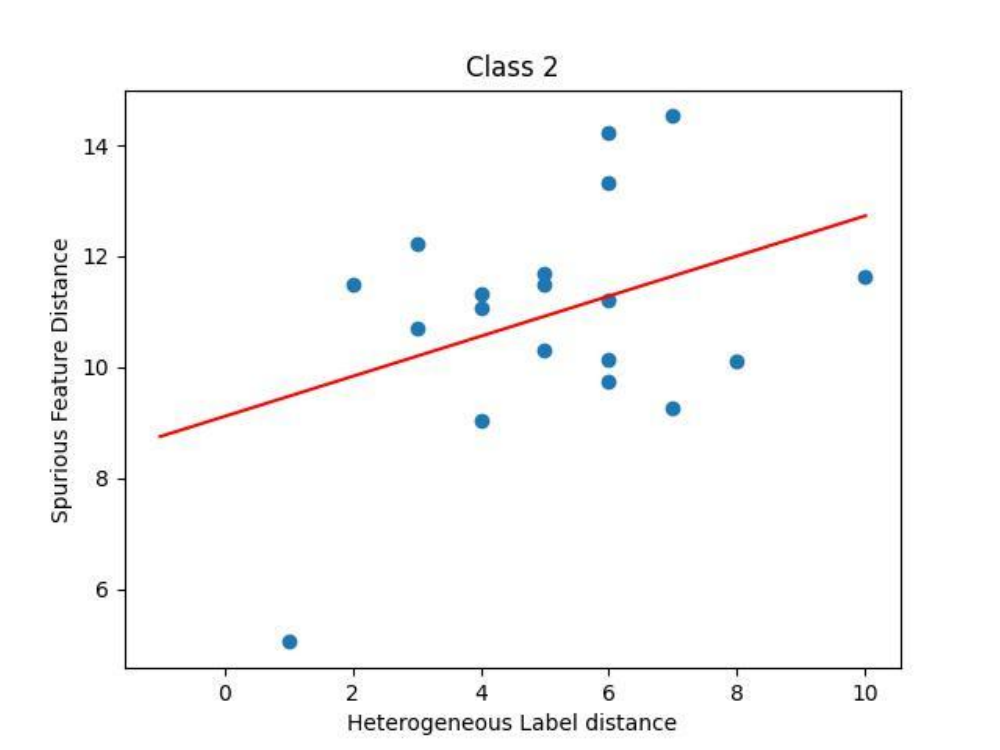}
    \end{subfigure}
    \begin{subfigure}{.19\textwidth}
        \centering
        \includegraphics[width=\linewidth]{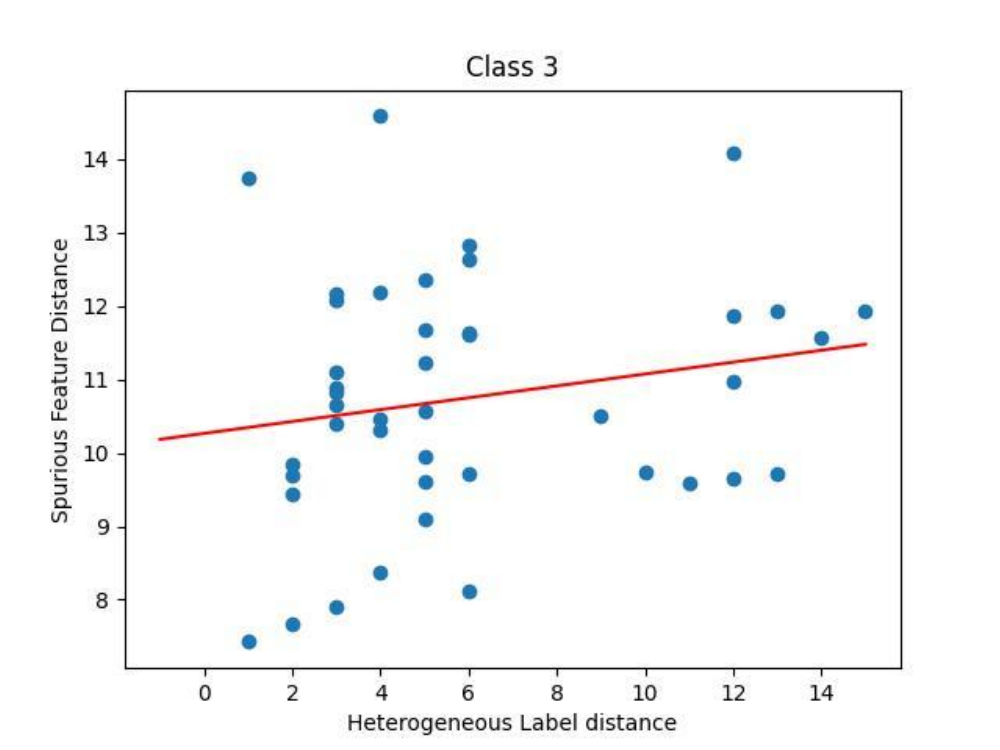}
    \end{subfigure}
    \begin{subfigure}{.19\textwidth}
        \centering
        \includegraphics[width=\linewidth]{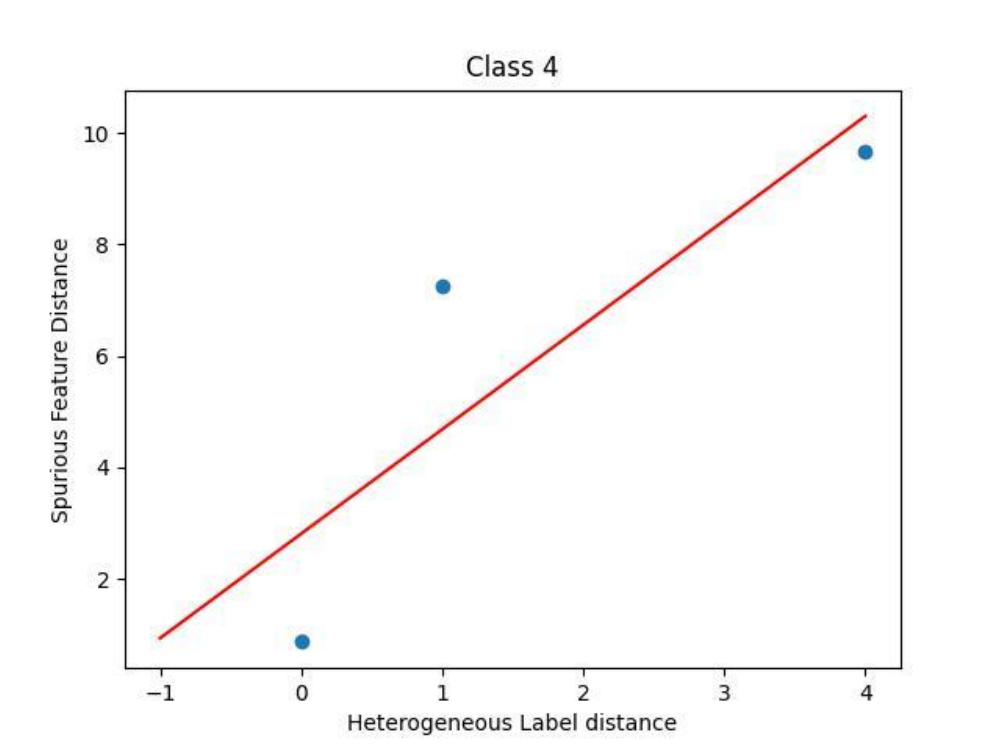}
    \end{subfigure}
    
    \begin{subfigure}{.19\textwidth}
        \centering
        \includegraphics[width=\linewidth]{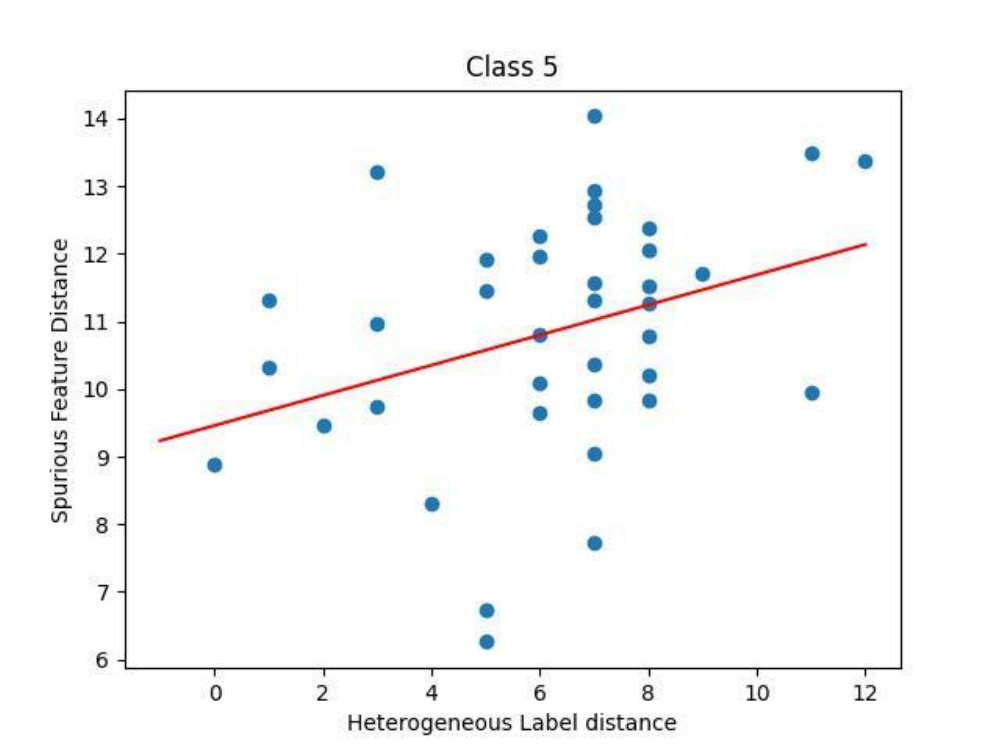}
    \end{subfigure}
    \begin{subfigure}{.19\textwidth}
        \centering
        \includegraphics[width=\linewidth]{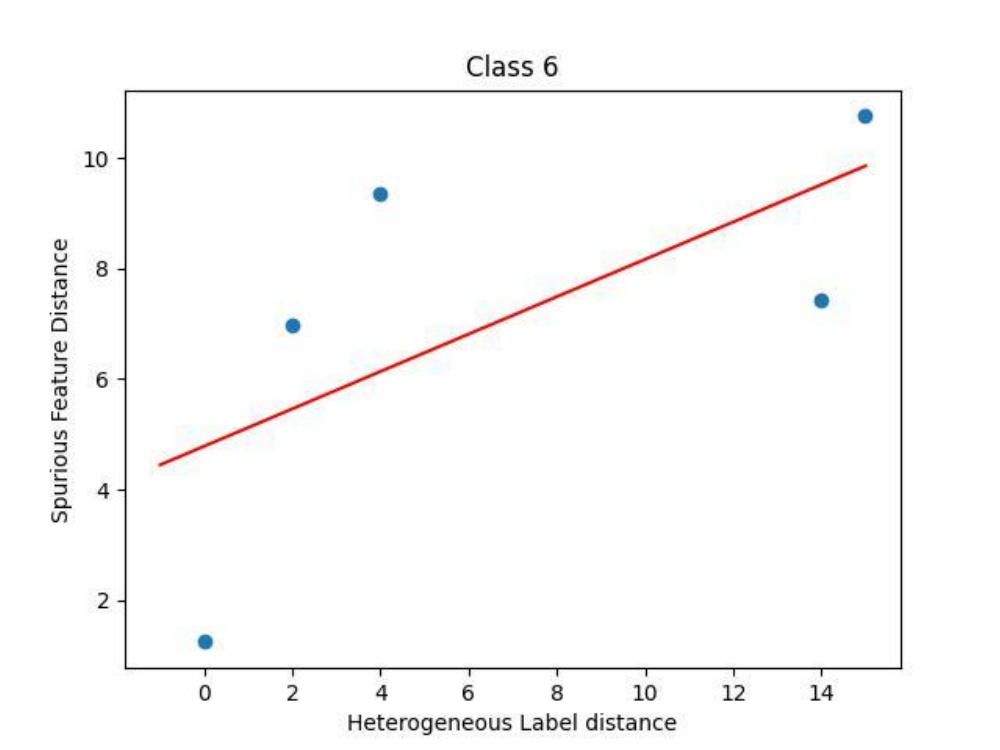}
    \end{subfigure}%
    \begin{subfigure}{.19\textwidth}
        \centering
        \includegraphics[width=\linewidth]{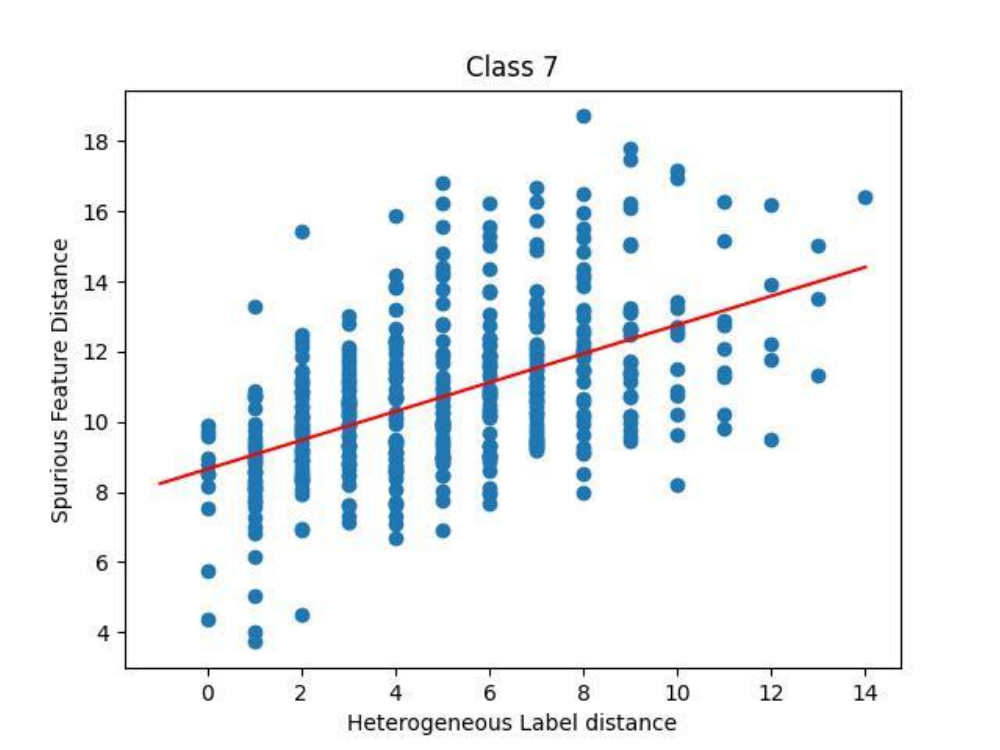}
    \end{subfigure}
    \begin{subfigure}{.19\textwidth}
        \centering
        \includegraphics[width=\linewidth]{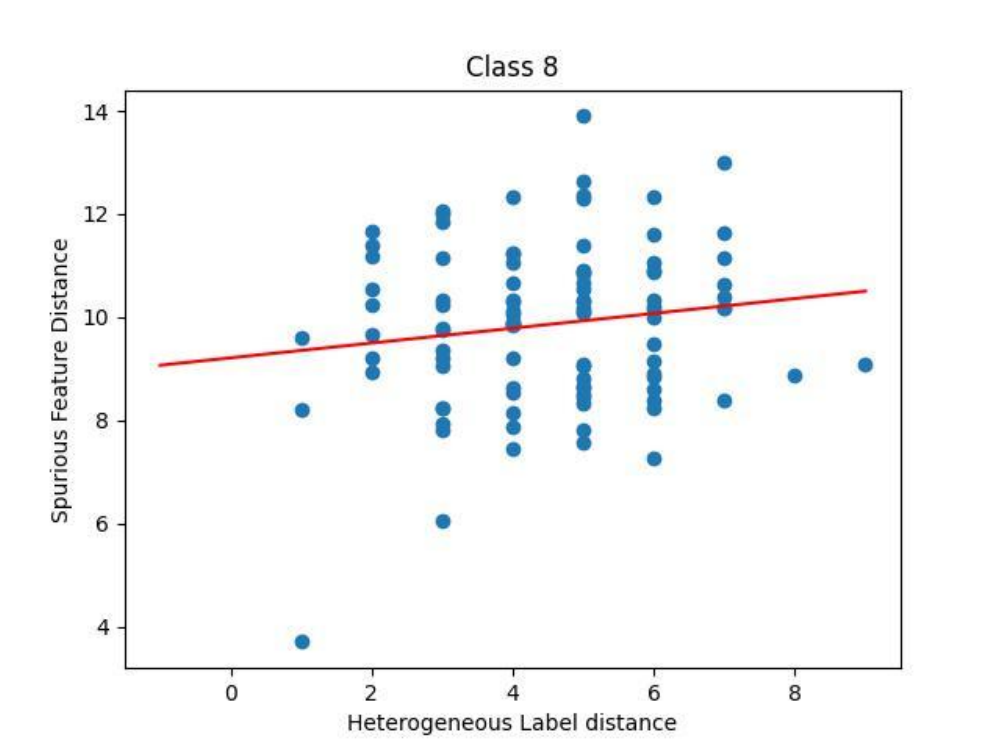}
    \end{subfigure}
    \begin{subfigure}{.19\textwidth}
        \centering
        \includegraphics[width=\linewidth]{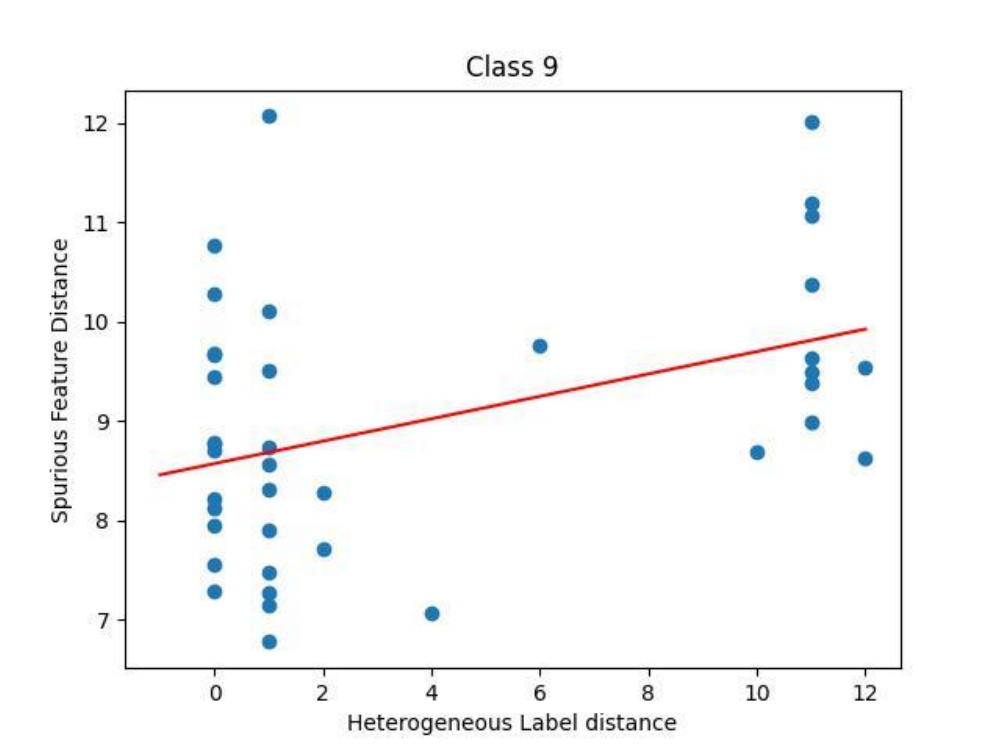}
    \end{subfigure}

    \begin{subfigure}{.19\textwidth}
        \centering
        \includegraphics[width=\linewidth]{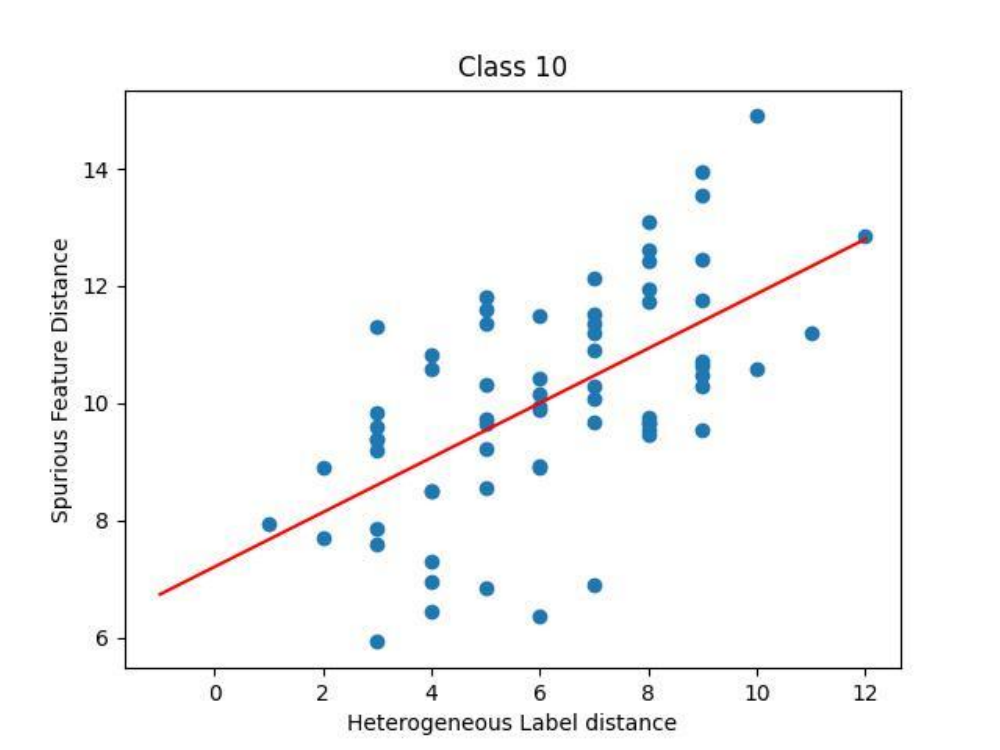}
    \end{subfigure}
    \begin{subfigure}{.19\textwidth}
        \centering
        \includegraphics[width=\linewidth]{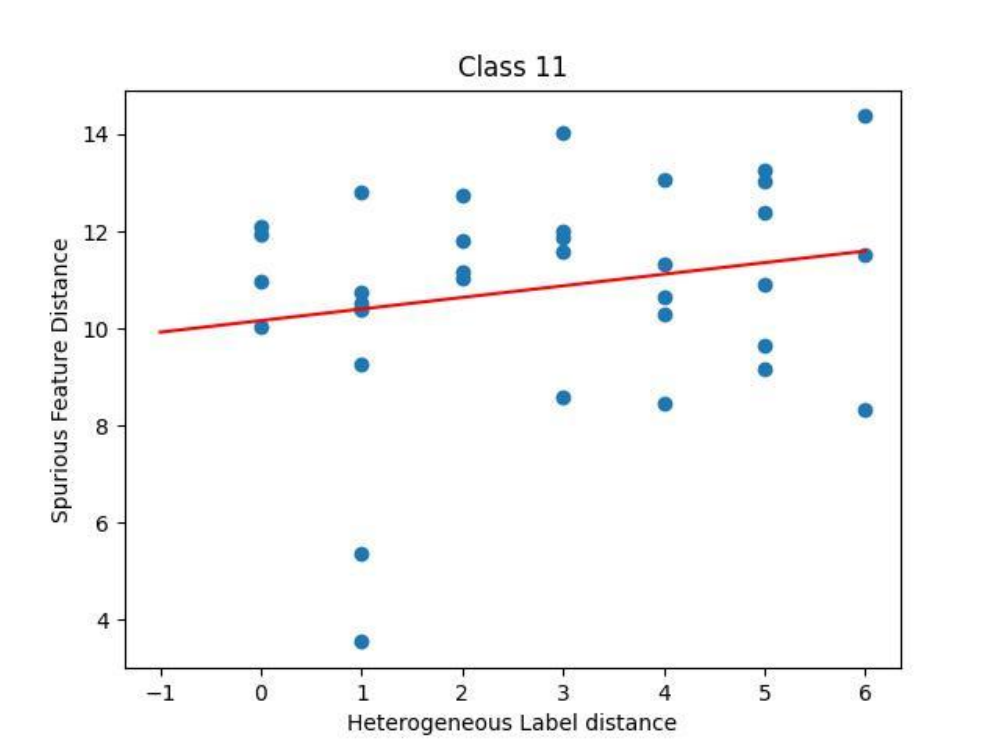}
    \end{subfigure}%
    \begin{subfigure}{.19\textwidth}
        \centering
        \includegraphics[width=\linewidth]{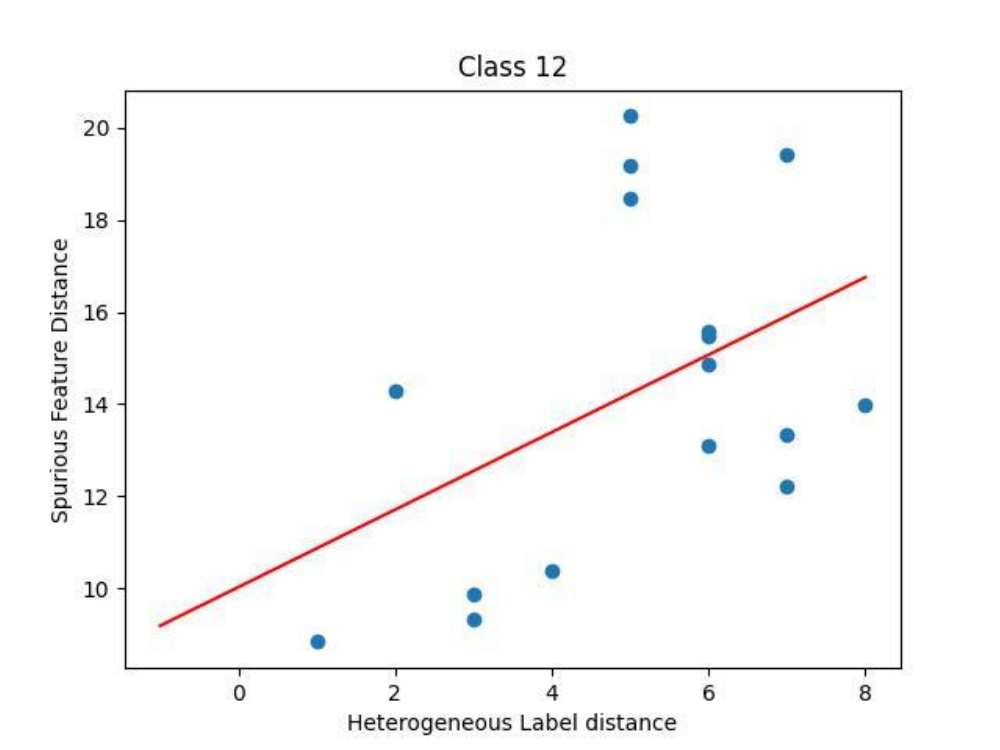}
    \end{subfigure}
    \begin{subfigure}{.19\textwidth}
        \centering
        \includegraphics[width=\linewidth]{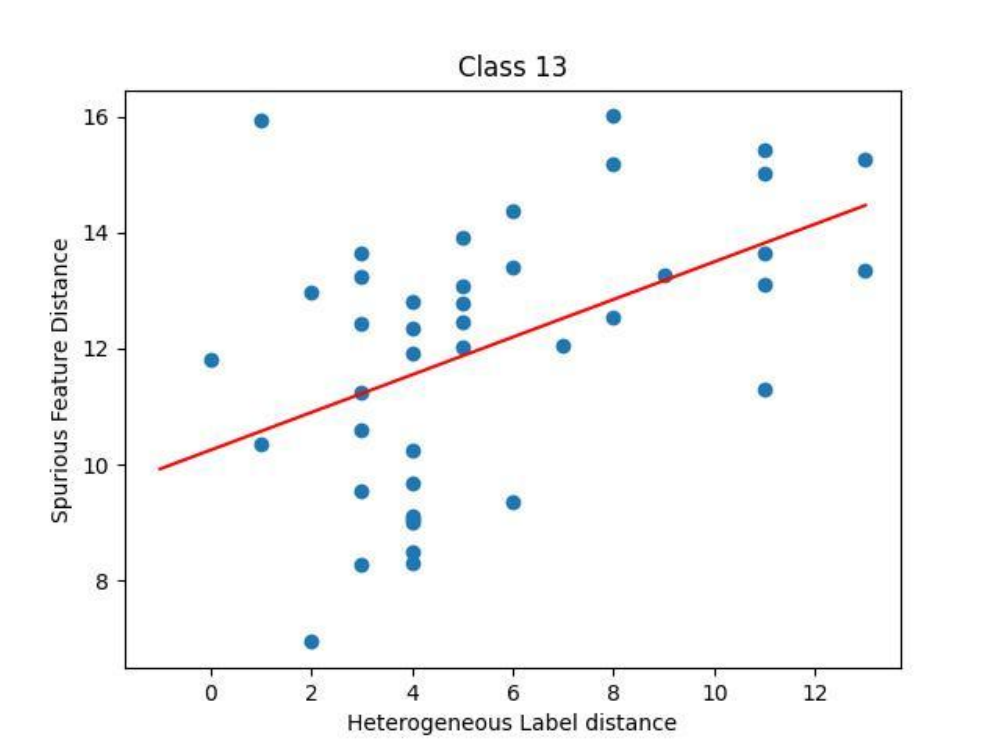}
    \end{subfigure}
    \begin{subfigure}{.19\textwidth}
        \centering
        \includegraphics[width=\linewidth]{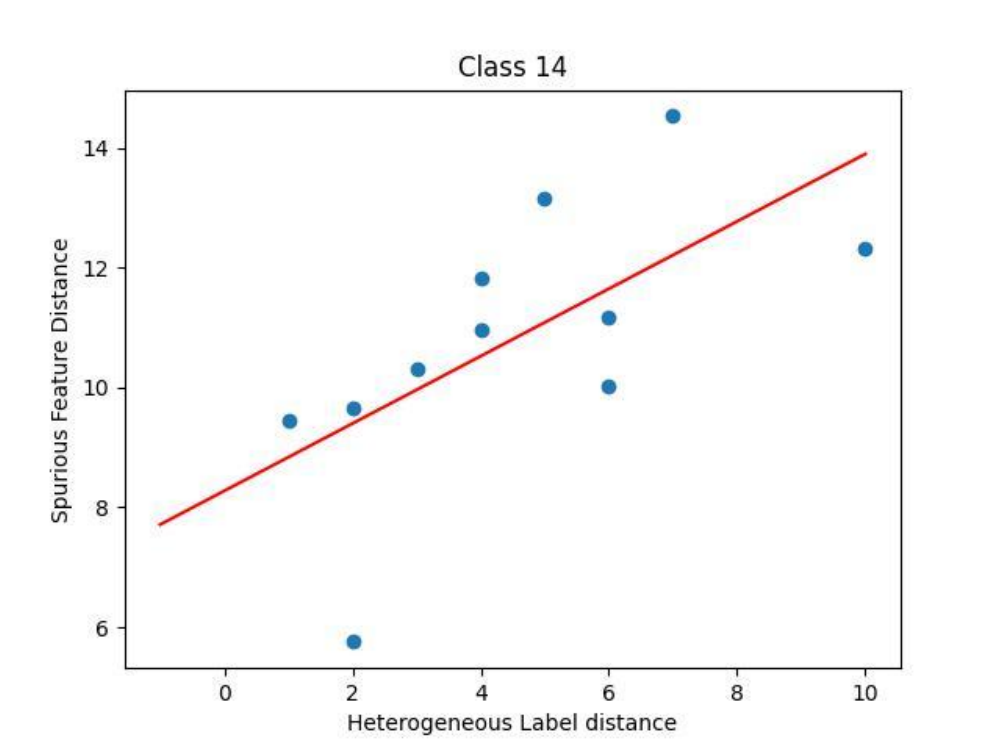}
    \end{subfigure}

    \begin{subfigure}{.19\textwidth}
        \centering
        \includegraphics[width=\linewidth]{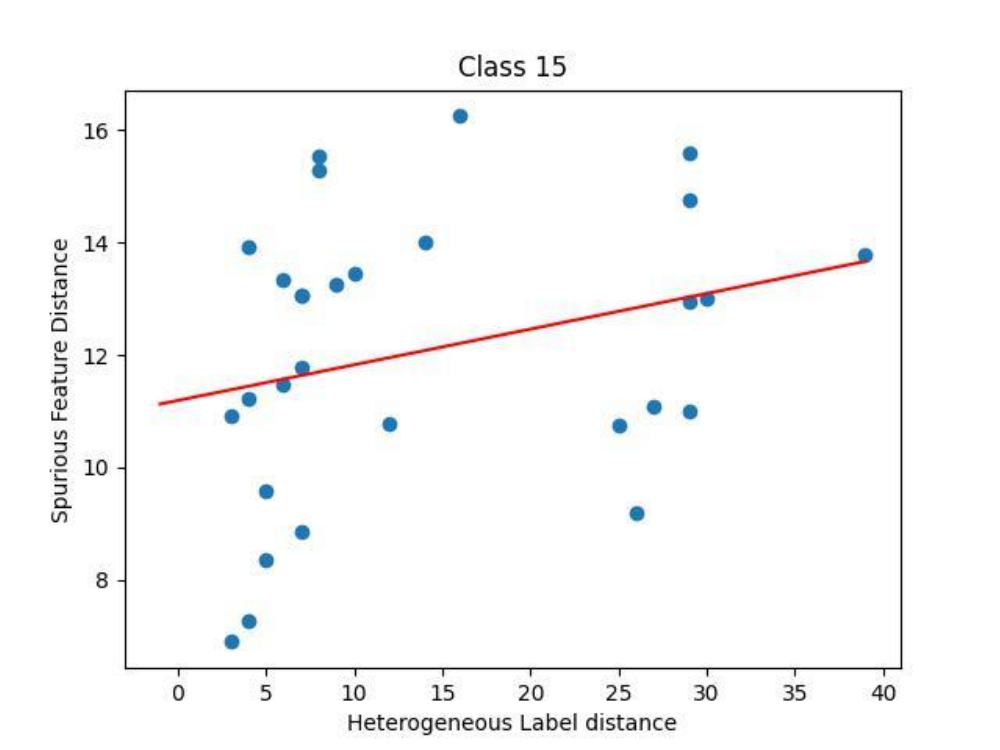}
    \end{subfigure}
    \begin{subfigure}{.19\textwidth}
        \centering
        \includegraphics[width=\linewidth]{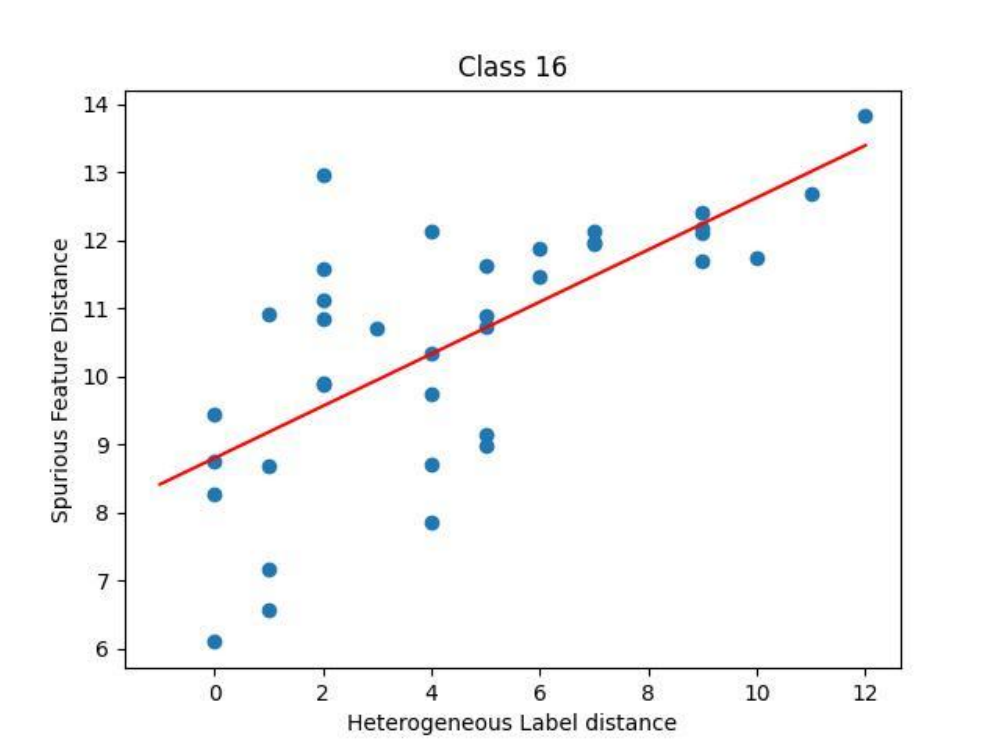}
    \end{subfigure}%
    \begin{subfigure}{.19\textwidth}
        \centering
        \includegraphics[width=\linewidth]{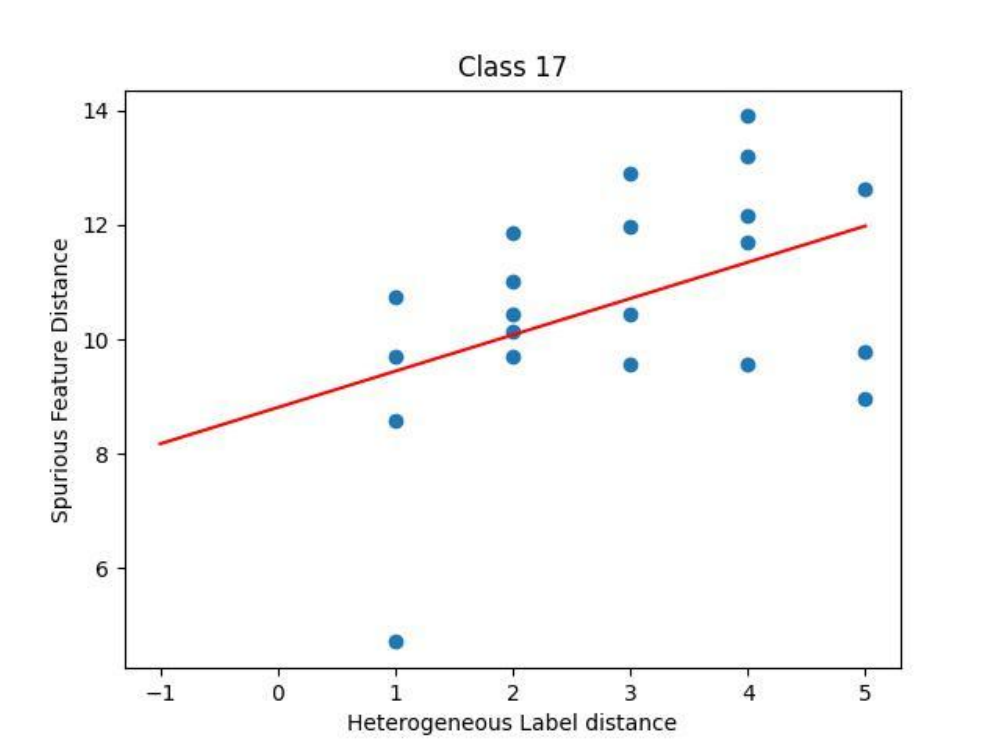}
    \end{subfigure}
    \begin{subfigure}{.19\textwidth}
        \centering
        \includegraphics[width=\linewidth]{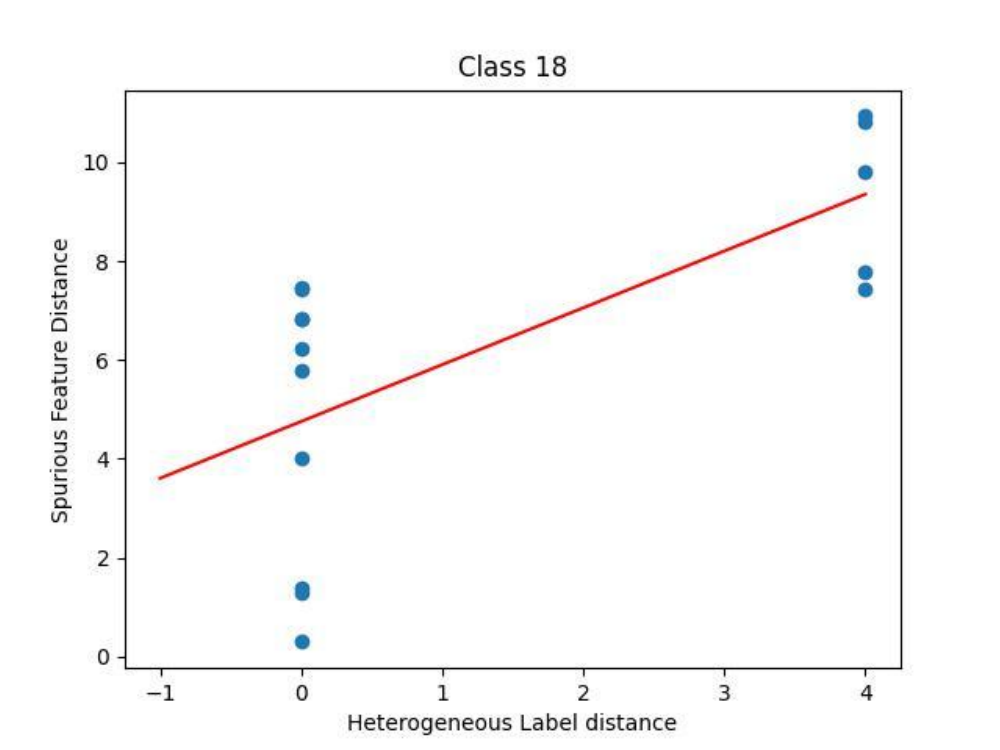}
    \end{subfigure}
    \begin{subfigure}{.19\textwidth}
        \centering
        \includegraphics[width=\linewidth]{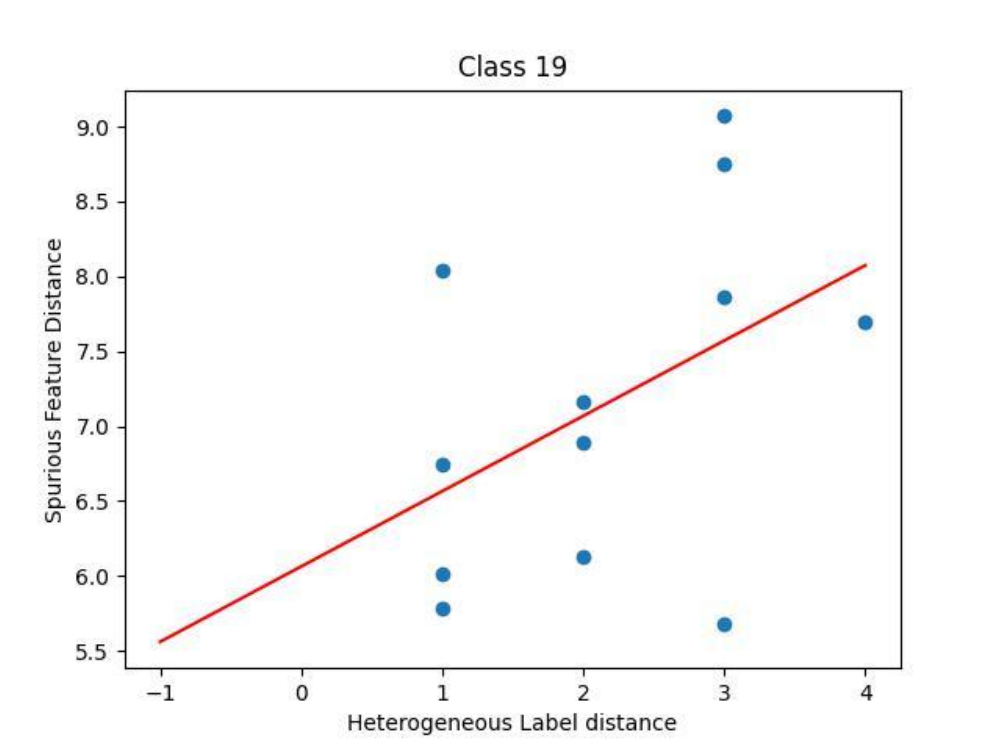}
    \end{subfigure}
    \caption{The relationship between the distance of the aggregated neighborhood representation and distance of HeteNLD on Cora \textit{word} domain, \textbf{covariate shift}. Each sub-figure is a class, and each dot in the figure represents a node pair in the graph. The red line is obtained by linear regression. The positive correlation is clear.}
    \label{reflect_sp_cora_word_cov}
\end{figure}

\begin{figure}
    \centering
    \begin{subfigure}{.19\textwidth}
        \centering
        \includegraphics[width=\linewidth]{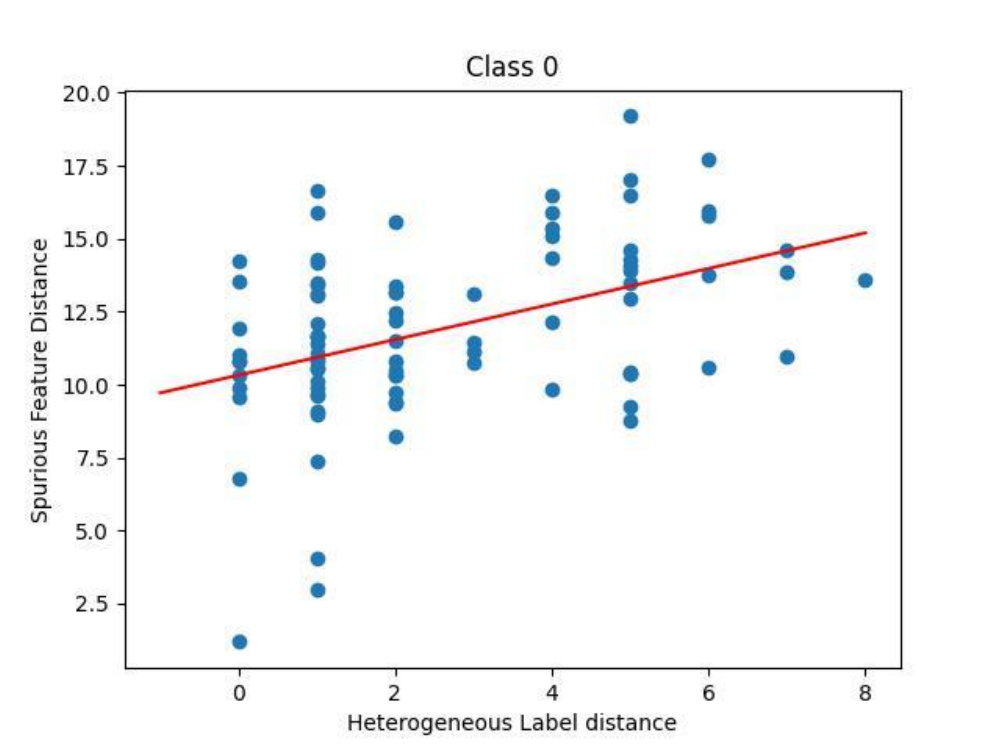}
    \end{subfigure}%
    \begin{subfigure}{.19\textwidth}
        \centering
        \includegraphics[width=\linewidth]{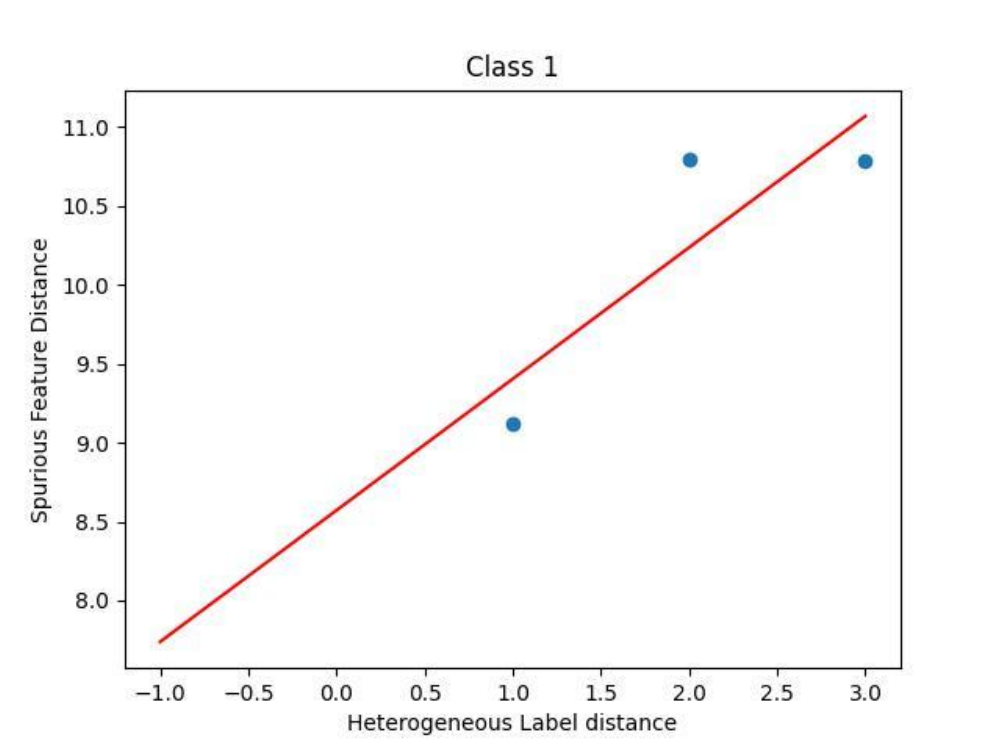}
    \end{subfigure}
    \begin{subfigure}{.19\textwidth}
        \centering
        \includegraphics[width=\linewidth]{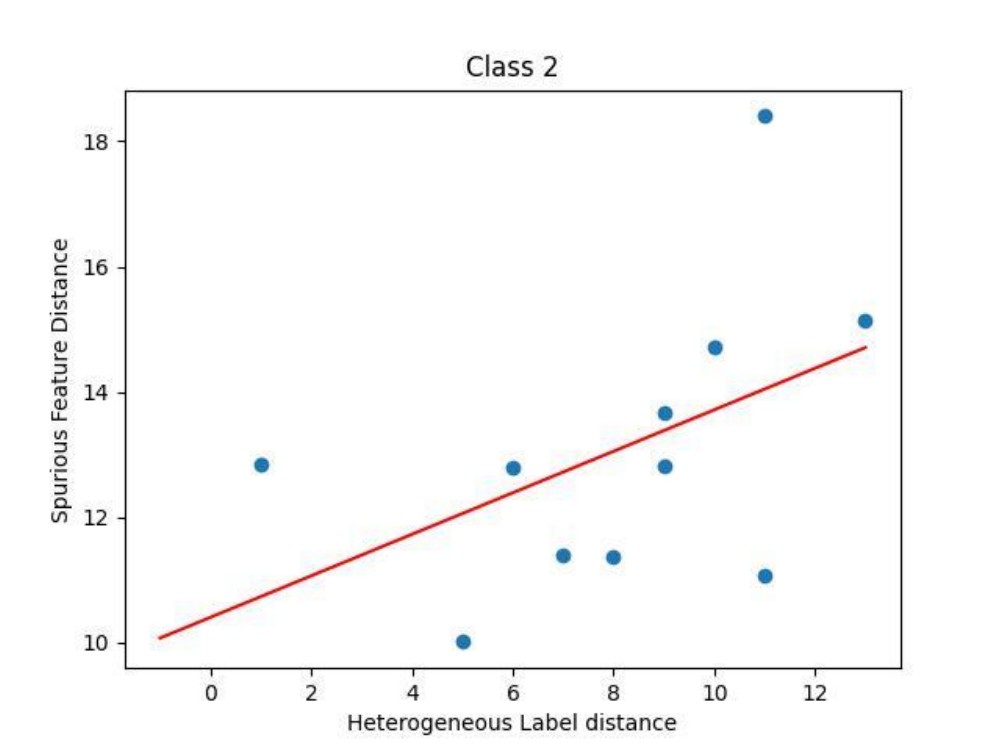}
    \end{subfigure}
    \begin{subfigure}{.19\textwidth}
        \centering
        \includegraphics[width=\linewidth]{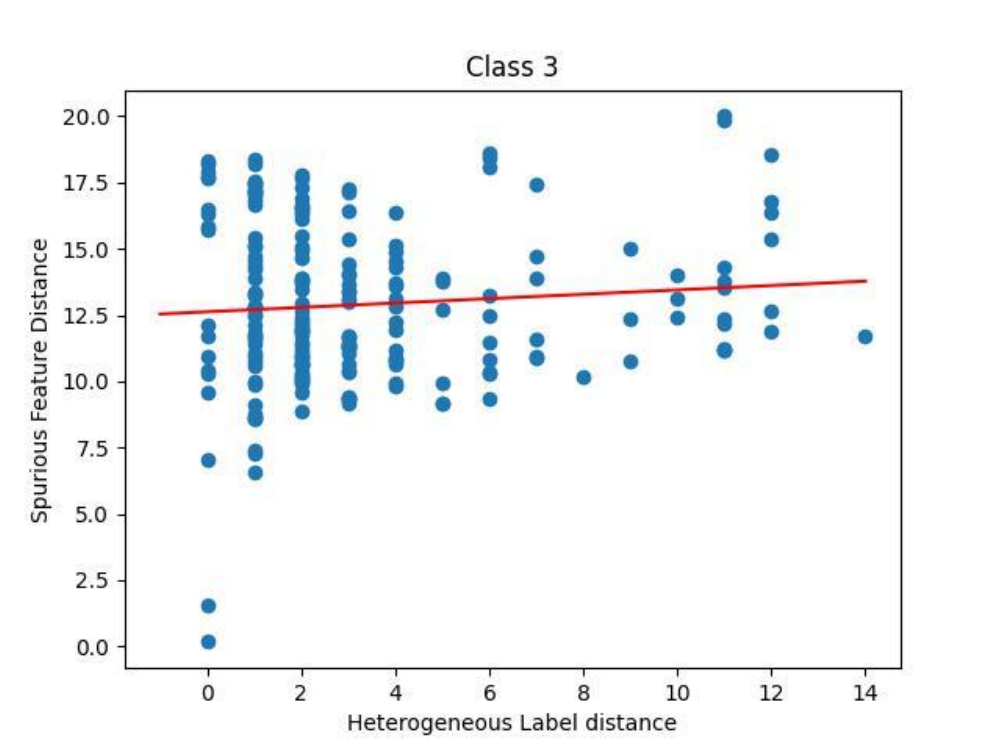}
    \end{subfigure}
    \begin{subfigure}{.19\textwidth}
        \centering
        \includegraphics[width=\linewidth]{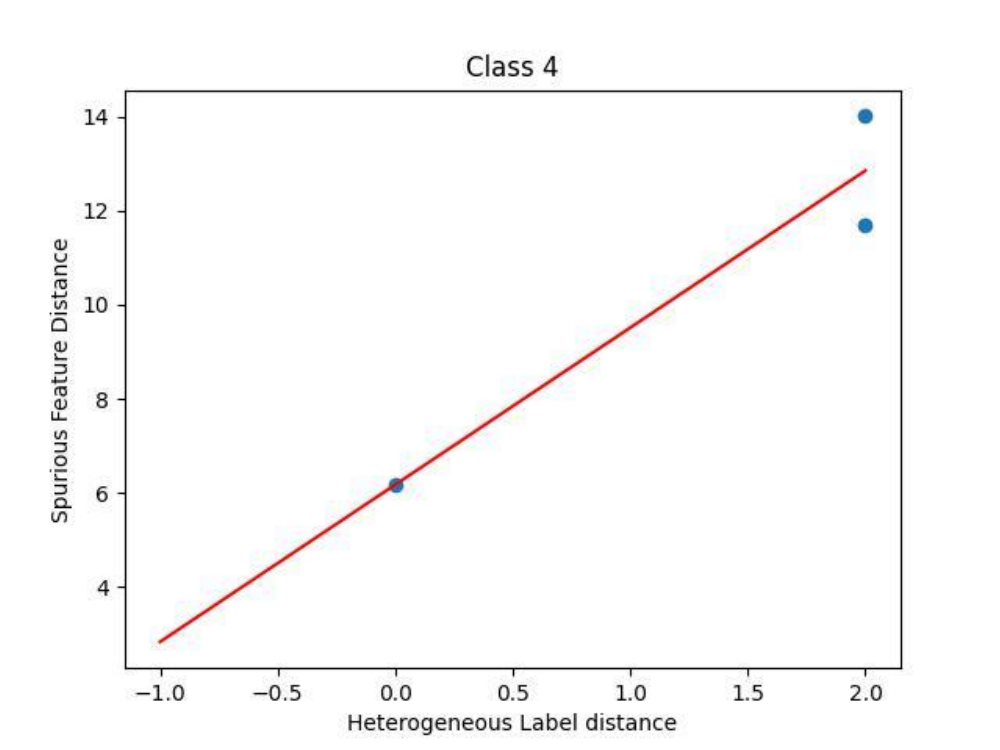}
    \end{subfigure}
    
    \begin{subfigure}{.19\textwidth}
        \centering
        \includegraphics[width=\linewidth]{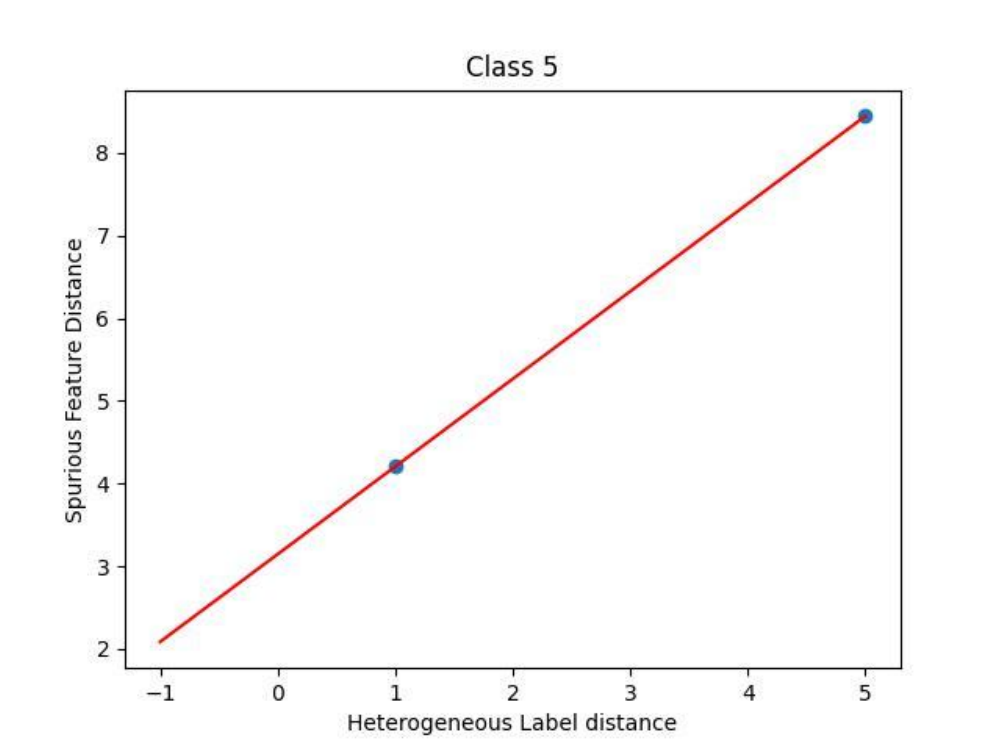}
    \end{subfigure}
    \begin{subfigure}{.19\textwidth}
        \centering
        \includegraphics[width=\linewidth]{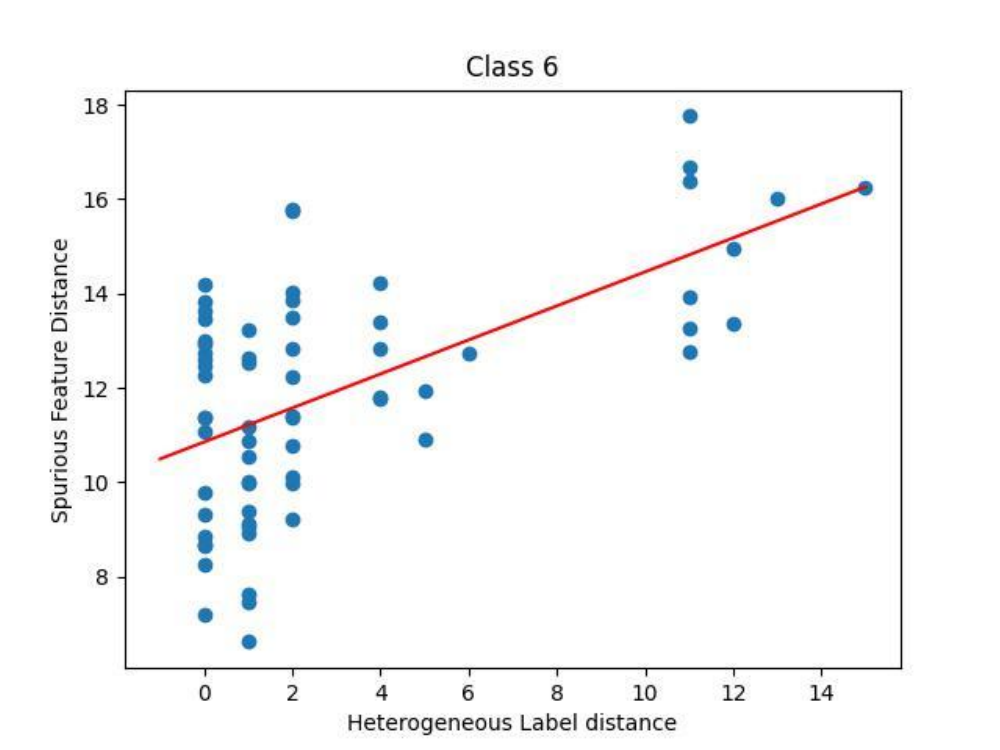}
    \end{subfigure}%
    \begin{subfigure}{.19\textwidth}
        \centering
        \includegraphics[width=\linewidth]{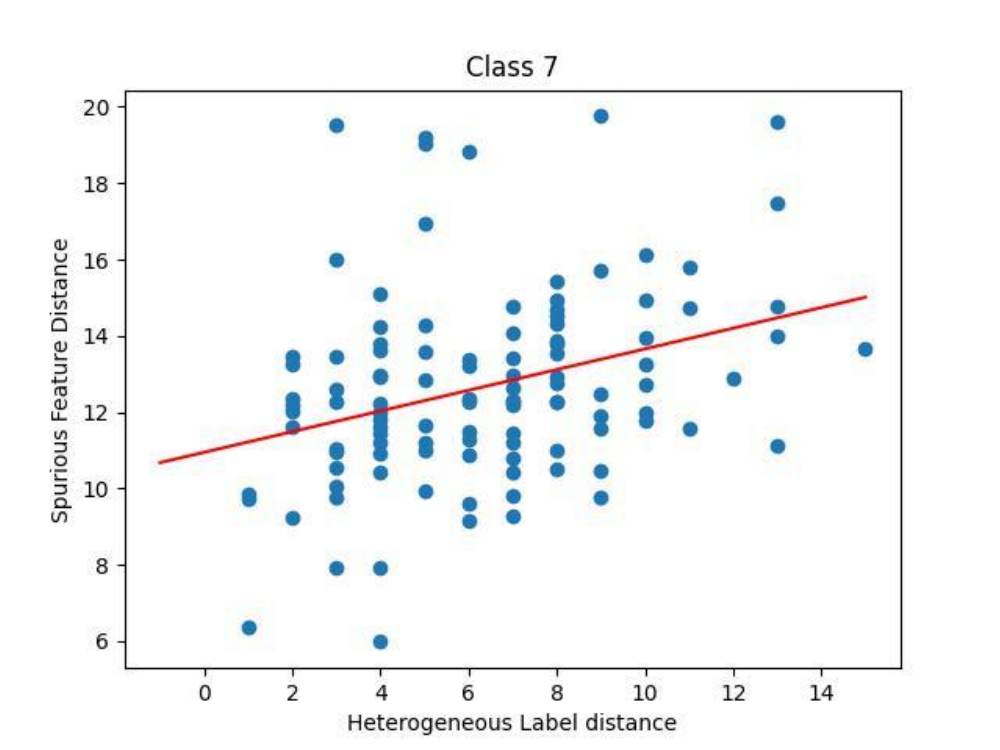}
    \end{subfigure}
    \begin{subfigure}{.19\textwidth}
        \centering
        \includegraphics[width=\linewidth]{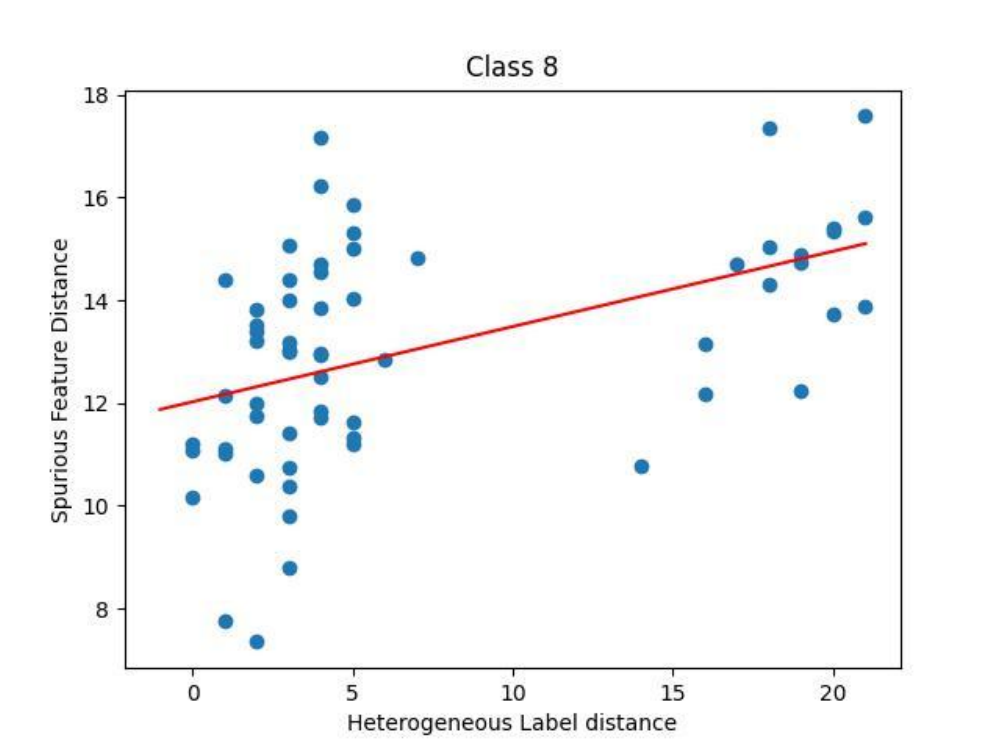}
    \end{subfigure}
    \begin{subfigure}{.19\textwidth}
        \centering
        \includegraphics[width=\linewidth]{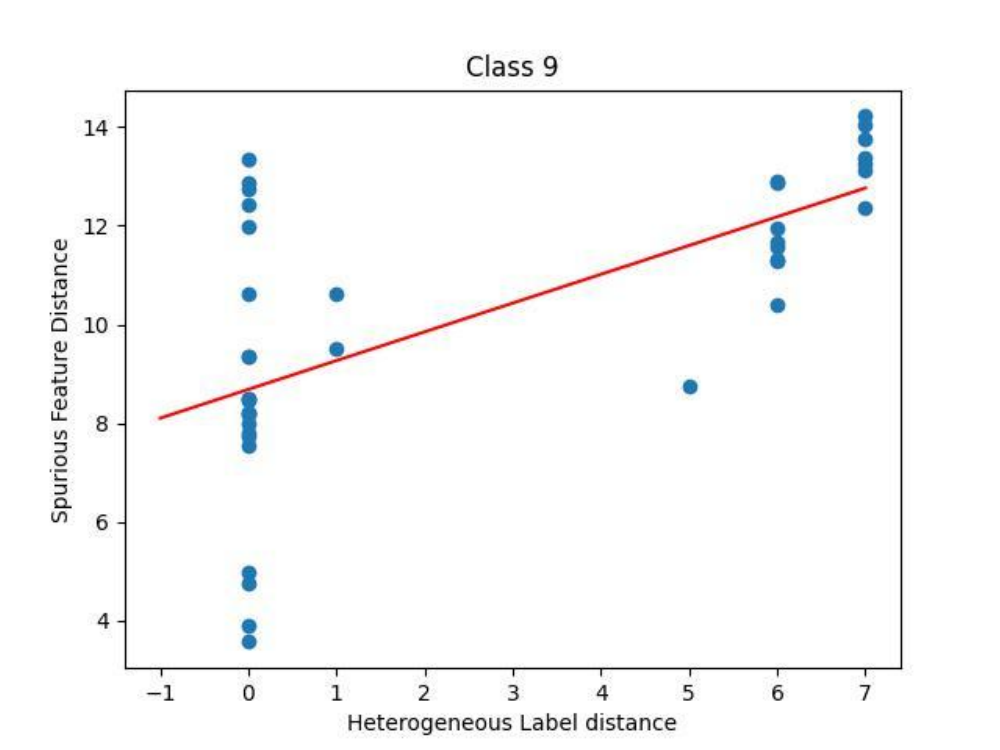}
    \end{subfigure}

    \begin{subfigure}{.19\textwidth}
        \centering
        \includegraphics[width=\linewidth]{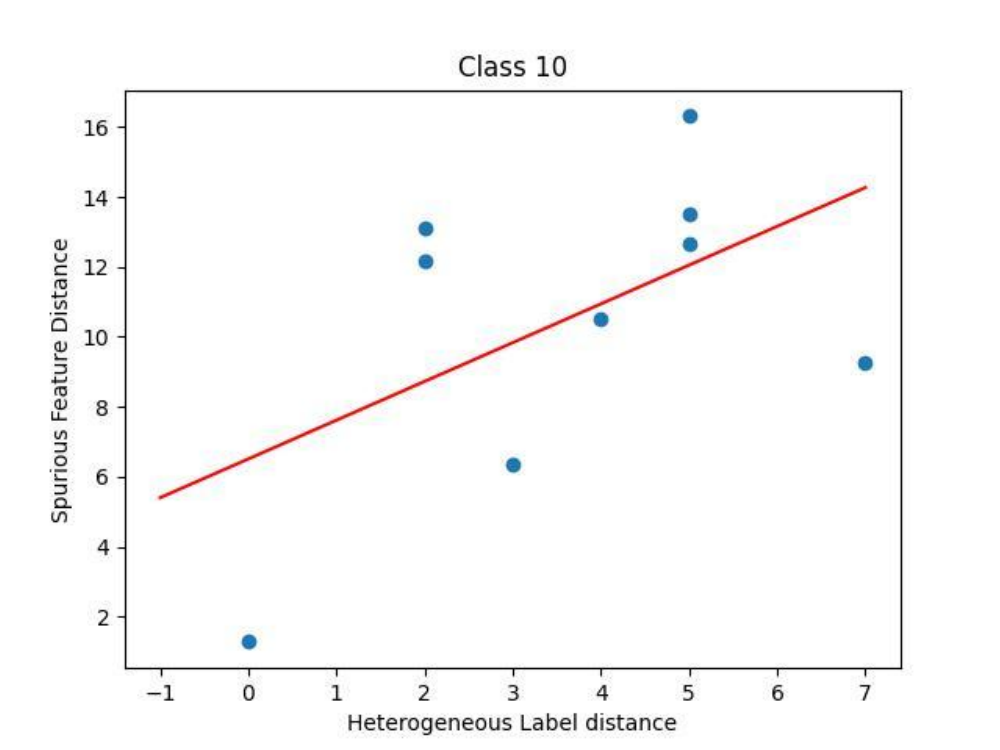}
    \end{subfigure}
    \begin{subfigure}{.19\textwidth}
        \centering
        \includegraphics[width=\linewidth]{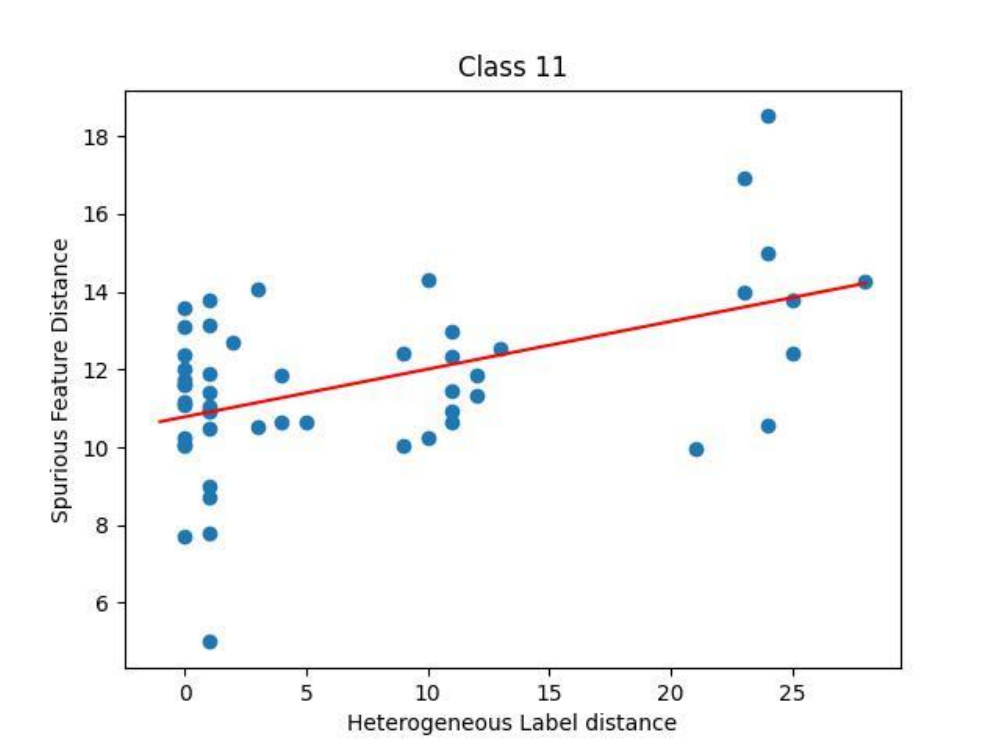}
    \end{subfigure}%
    \begin{subfigure}{.19\textwidth}
        \centering
        \includegraphics[width=\linewidth]{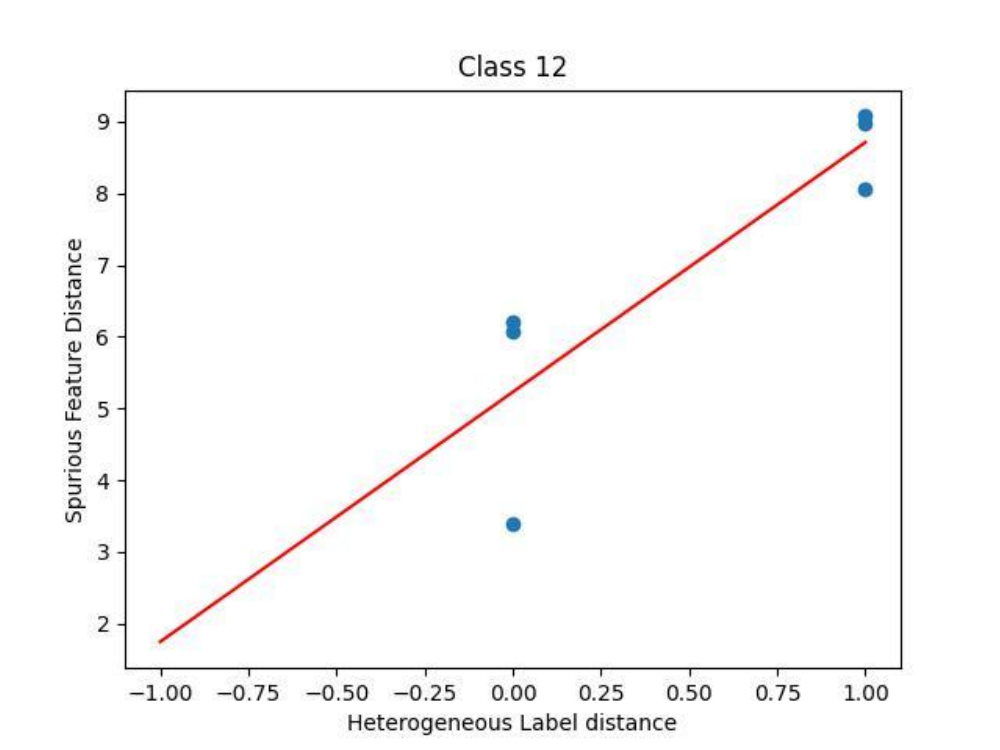}
    \end{subfigure}
    \begin{subfigure}{.19\textwidth}
        \centering
        \includegraphics[width=\linewidth]{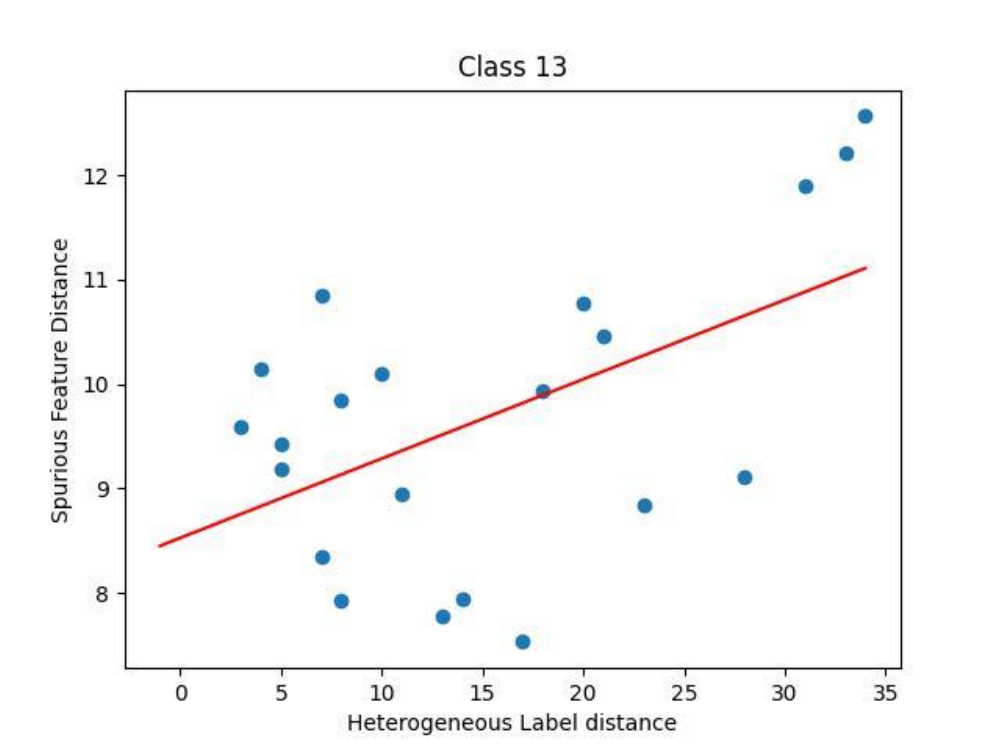}
    \end{subfigure}
    \begin{subfigure}{.19\textwidth}
        \centering
        \includegraphics[width=\linewidth]{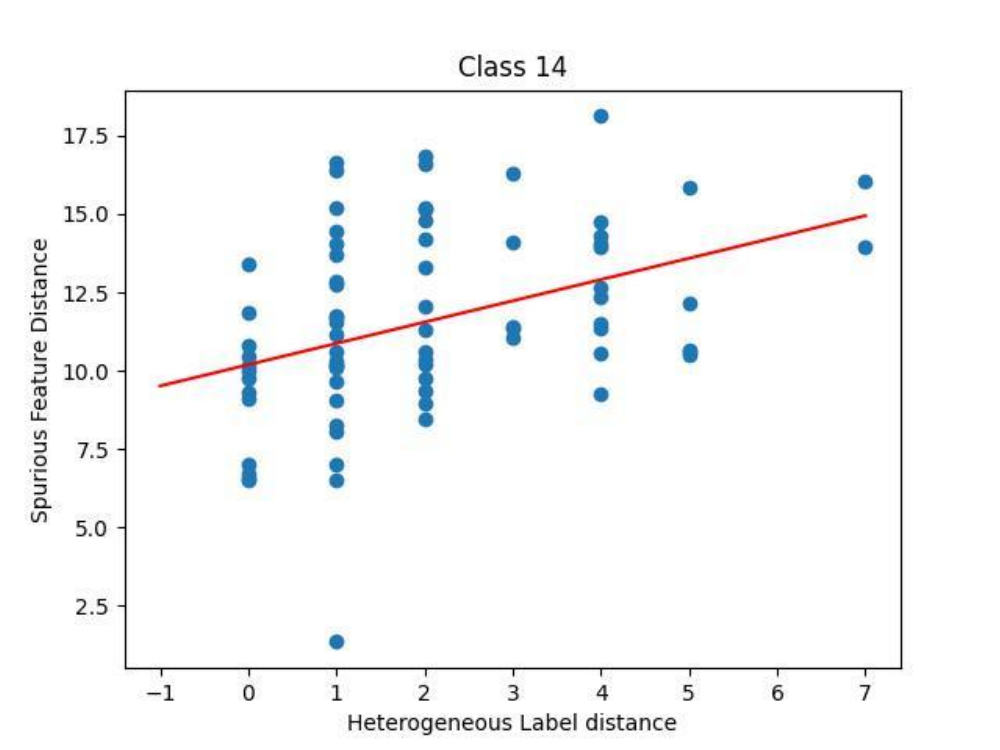}
    \end{subfigure}

    \begin{subfigure}{.19\textwidth}
        \centering
        \includegraphics[width=\linewidth]{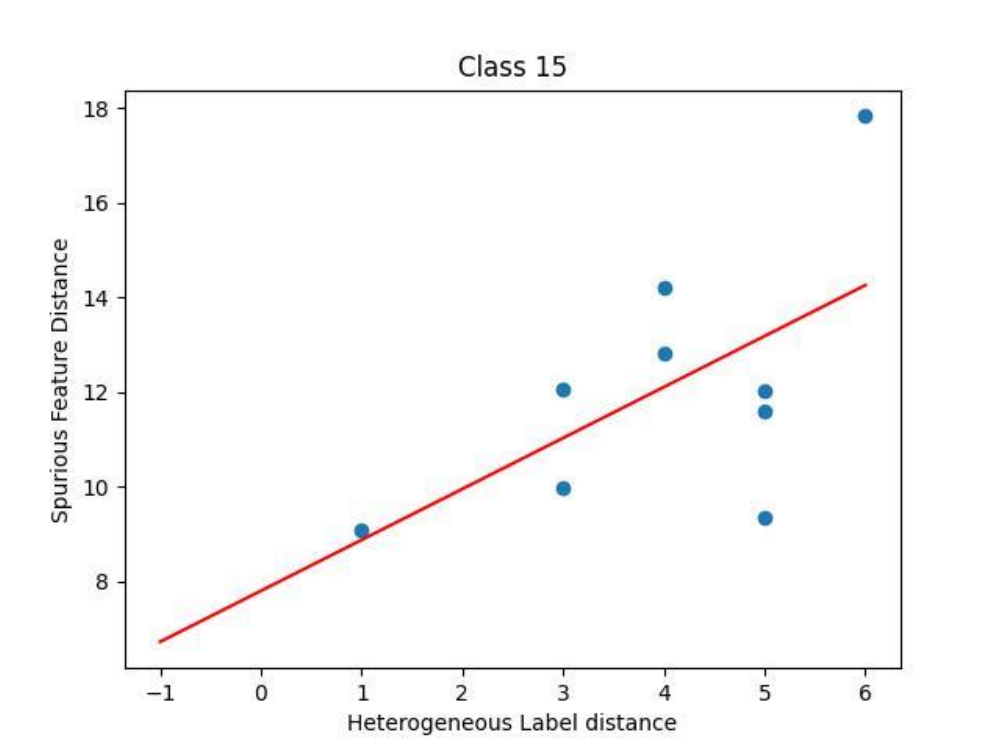}
    \end{subfigure}
    \begin{subfigure}{.19\textwidth}
        \centering
        \includegraphics[width=\linewidth]{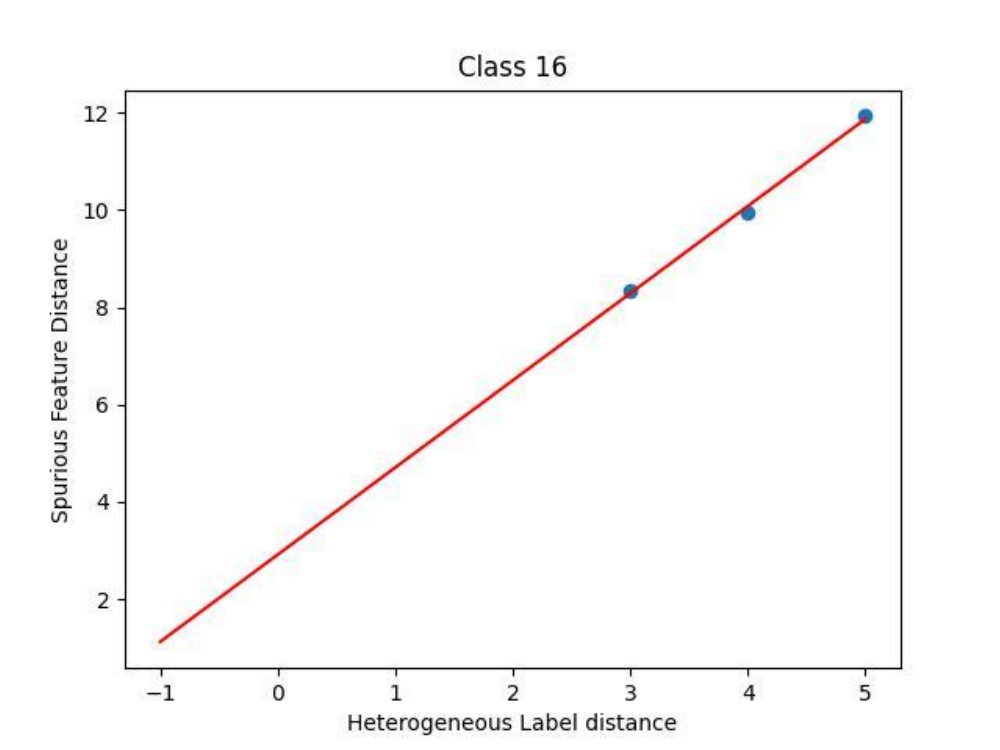}
    \end{subfigure}%
    \begin{subfigure}{.19\textwidth}
        \centering
        \includegraphics[width=\linewidth]{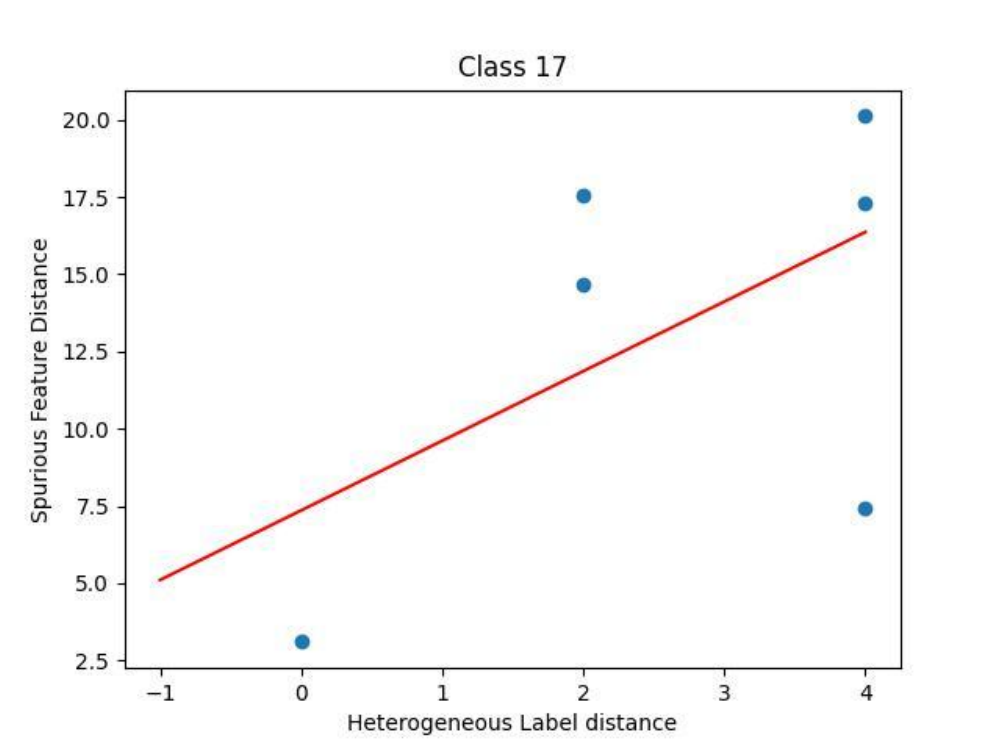}
    \end{subfigure}
    \begin{subfigure}{.19\textwidth}
        \centering
        \includegraphics[width=\linewidth]{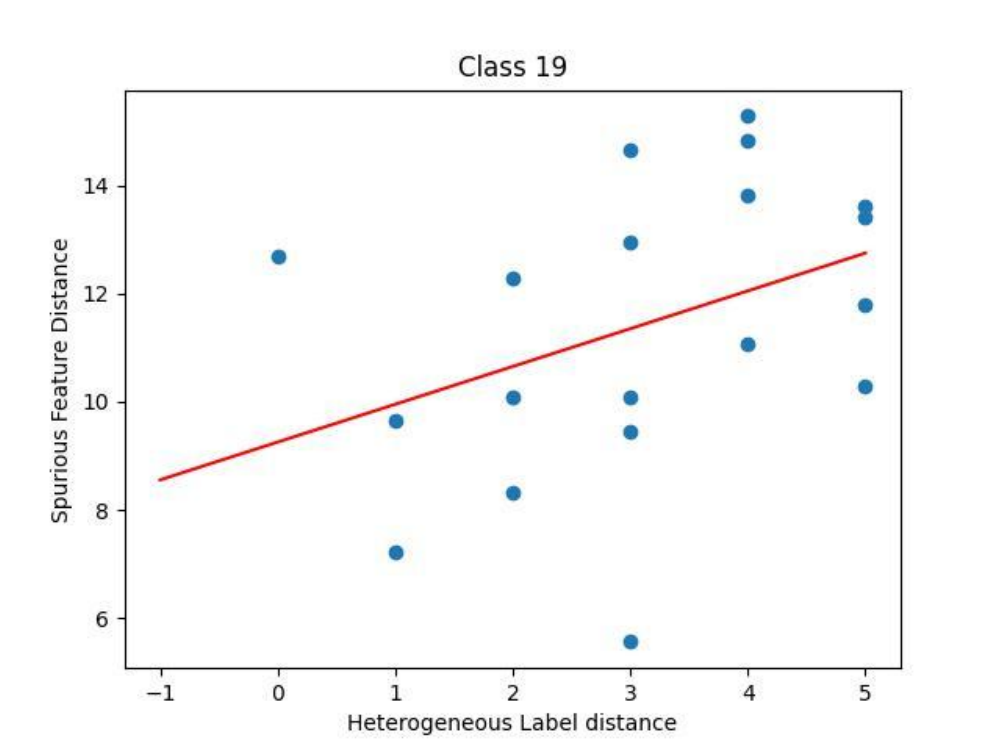}
    \end{subfigure}
    \begin{subfigure}{.19\textwidth}
        \centering
        \includegraphics[width=\linewidth]{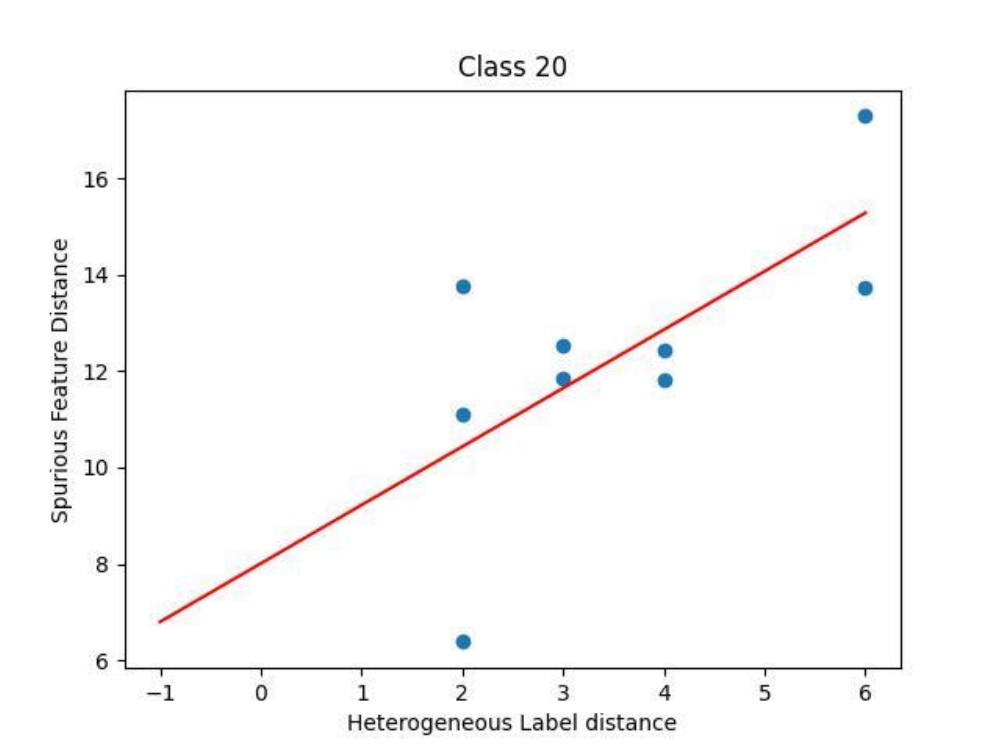}
    \end{subfigure}
    \caption{The relationship between the distance of the aggregated neighborhood representation and distance of HeteNLD on Cora \textit{degree} domain, \textbf{covariate shift}. The positive correlation is clear.}
    \label{reflect_sp_cora_deg_cov}
\end{figure}

\textbf{Concept shift.} As for concept shift, spurious features are correlated with labels, thus the label of a node contains information about spurious features correlated with this class. 
Hence, by observing HeteNLD, we can measure the distribution of the spurious feature. For concept shift, we train a GNN to predict environment labels to obtain spurious representations. \textbf{Table \ref{reflect_sp_cora_word_con} and \ref{reflect_sp_cora_deg_con} also show a clear positive correlation between spurious feature distance and HeteNLD discrepancy on concept shift, indicating that HeteNLD discrepancy can reflect the distance of the environmental spurious features. }

\begin{figure}
    \centering
    \begin{subfigure}{.19\textwidth}
        \centering
        \includegraphics[width=\linewidth]{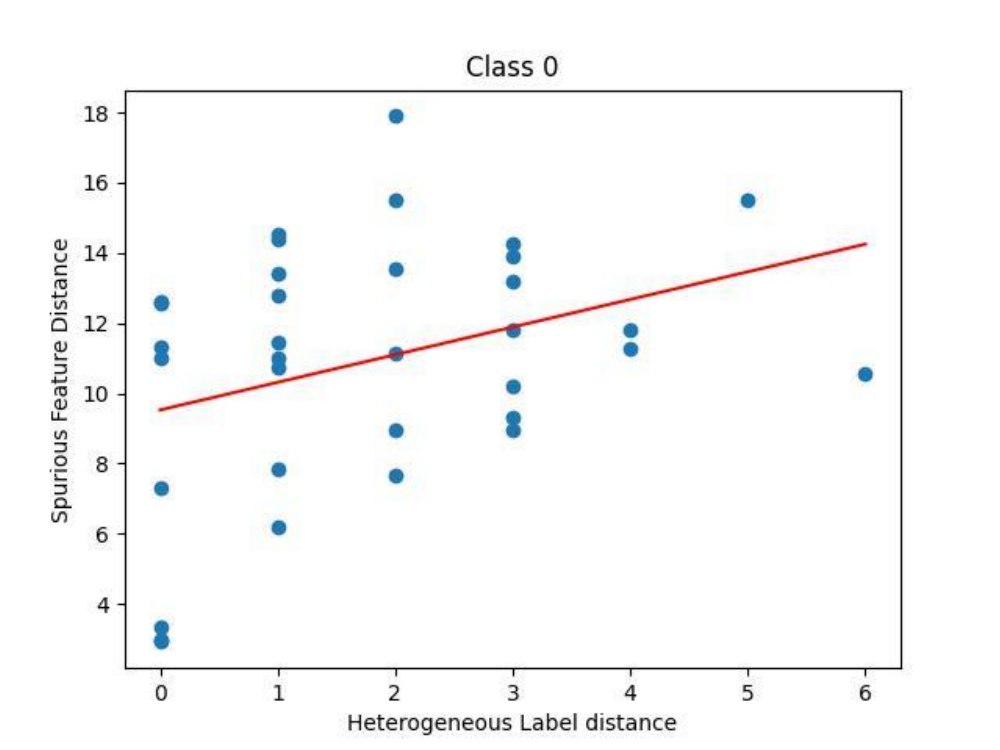}
    \end{subfigure}%
    \begin{subfigure}{.19\textwidth}
        \centering
        \includegraphics[width=\linewidth]{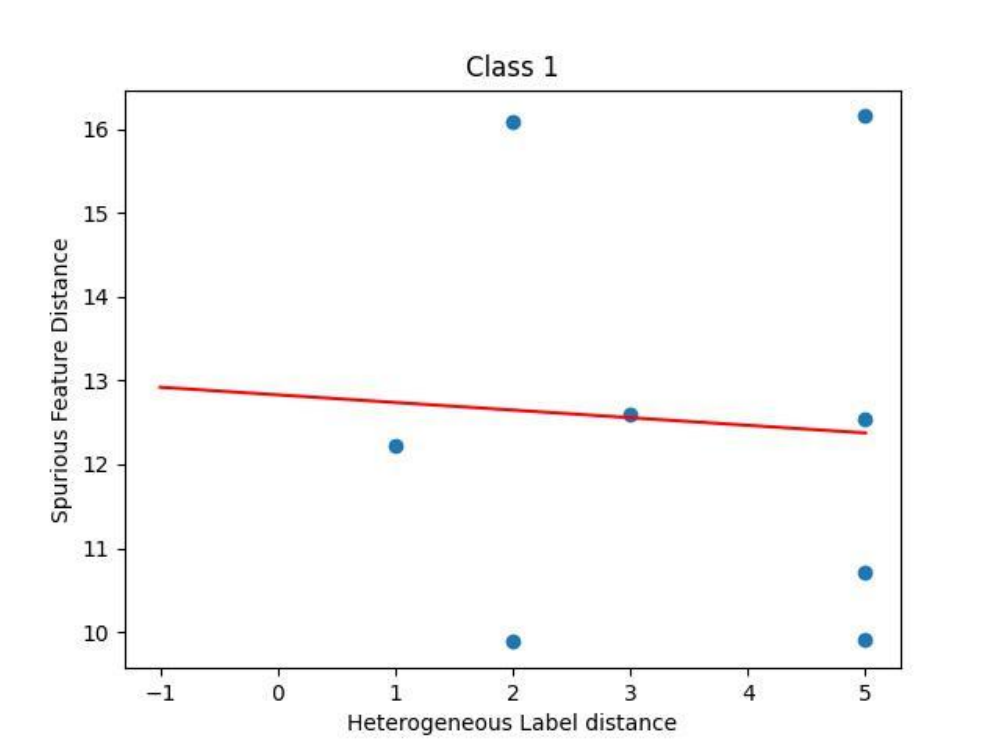}
    \end{subfigure}
    \begin{subfigure}{.19\textwidth}
        \centering
        \includegraphics[width=\linewidth]{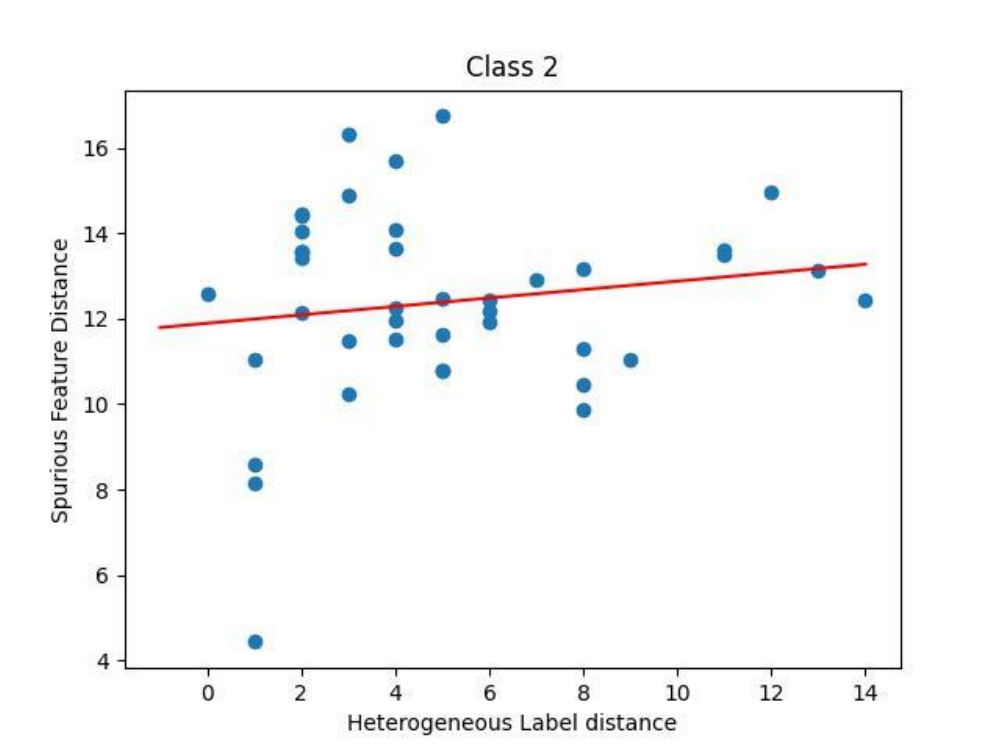}
    \end{subfigure}
    \begin{subfigure}{.19\textwidth}
        \centering
        \includegraphics[width=\linewidth]{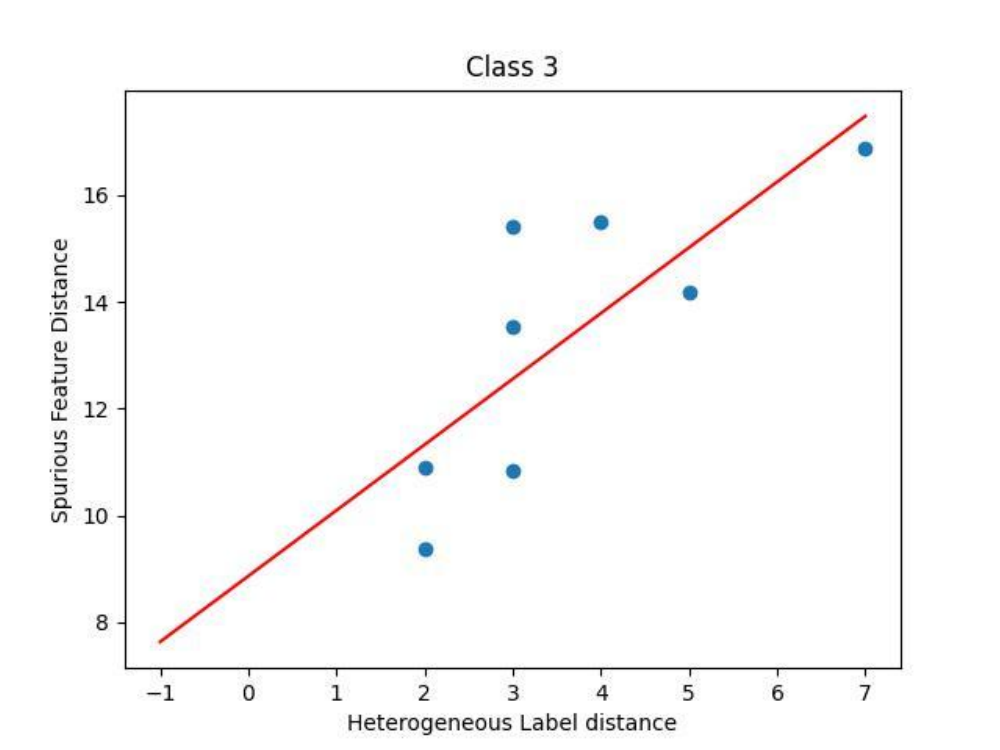}
    \end{subfigure}
    \begin{subfigure}{.19\textwidth}
        \centering
        \includegraphics[width=\linewidth]{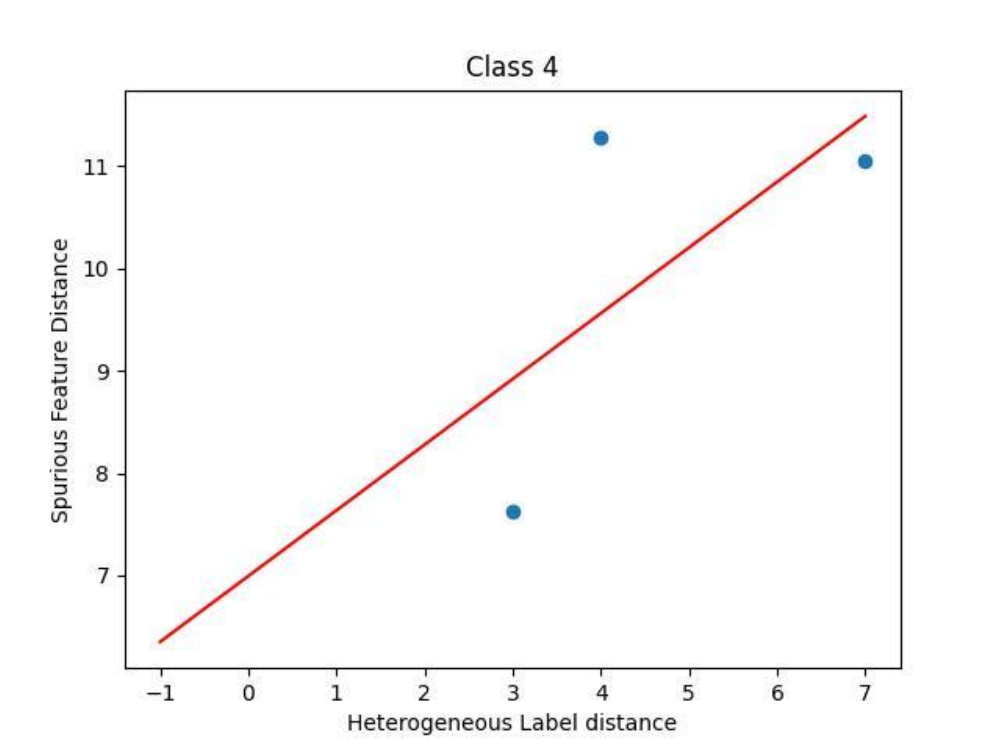}
    \end{subfigure}
    
    \begin{subfigure}{.19\textwidth}
        \centering
        \includegraphics[width=\linewidth]{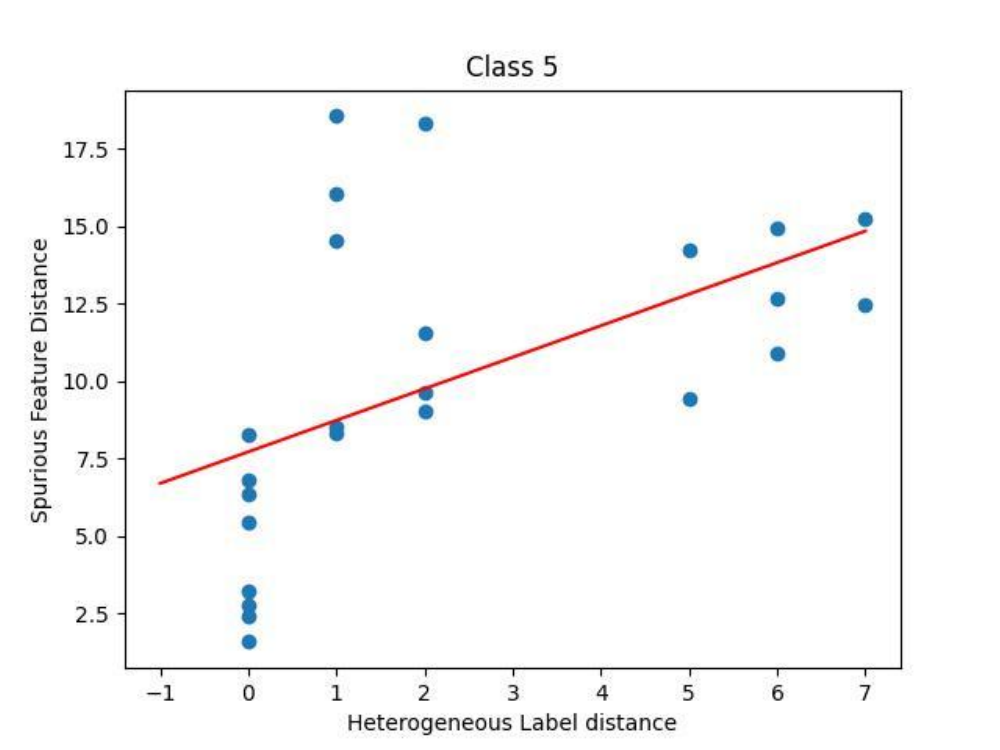}
    \end{subfigure}
    \begin{subfigure}{.19\textwidth}
        \centering
        \includegraphics[width=\linewidth]{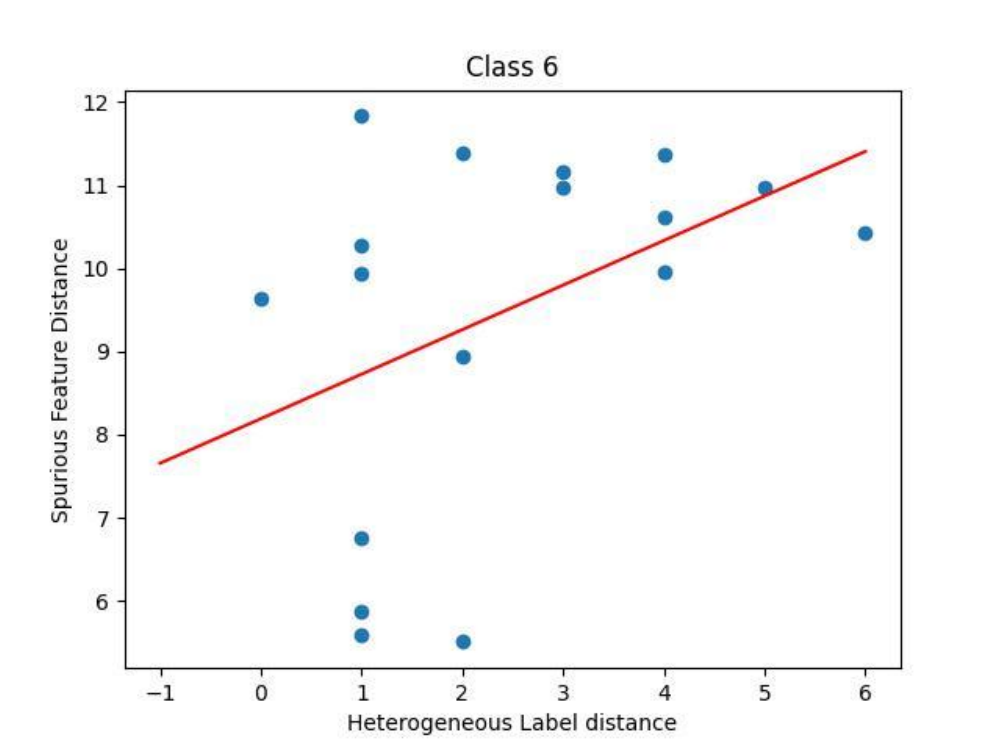}
    \end{subfigure}%
    \begin{subfigure}{.19\textwidth}
        \centering
        \includegraphics[width=\linewidth]{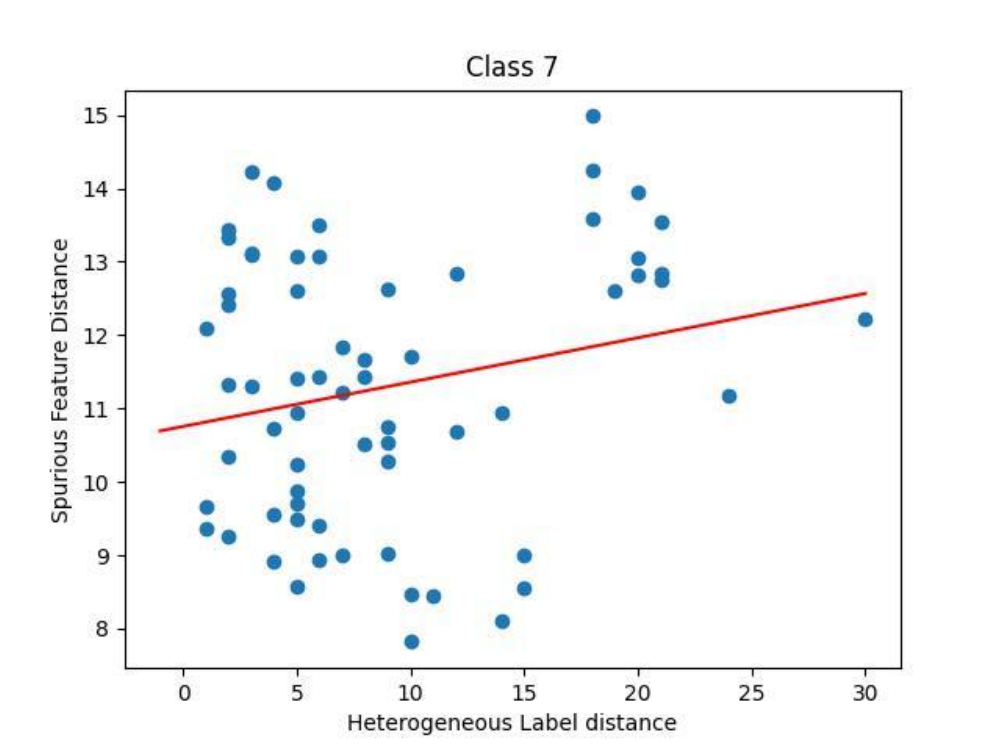}
    \end{subfigure}
    \begin{subfigure}{.19\textwidth}
        \centering
        \includegraphics[width=\linewidth]{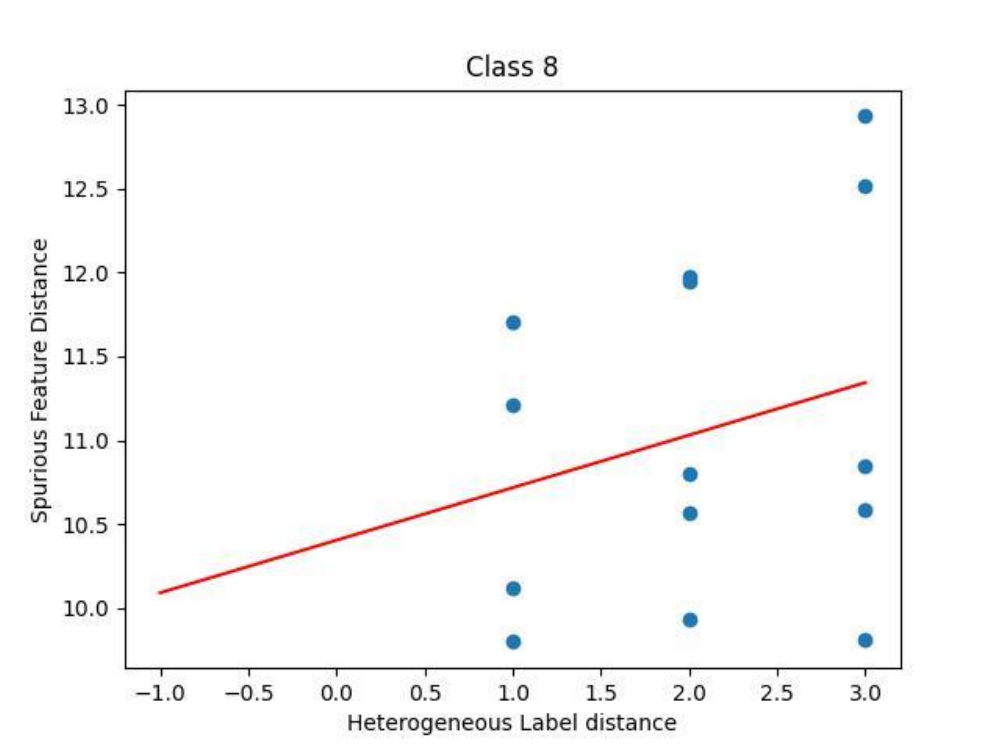}
    \end{subfigure}
    \begin{subfigure}{.19\textwidth}
        \centering
        \includegraphics[width=\linewidth]{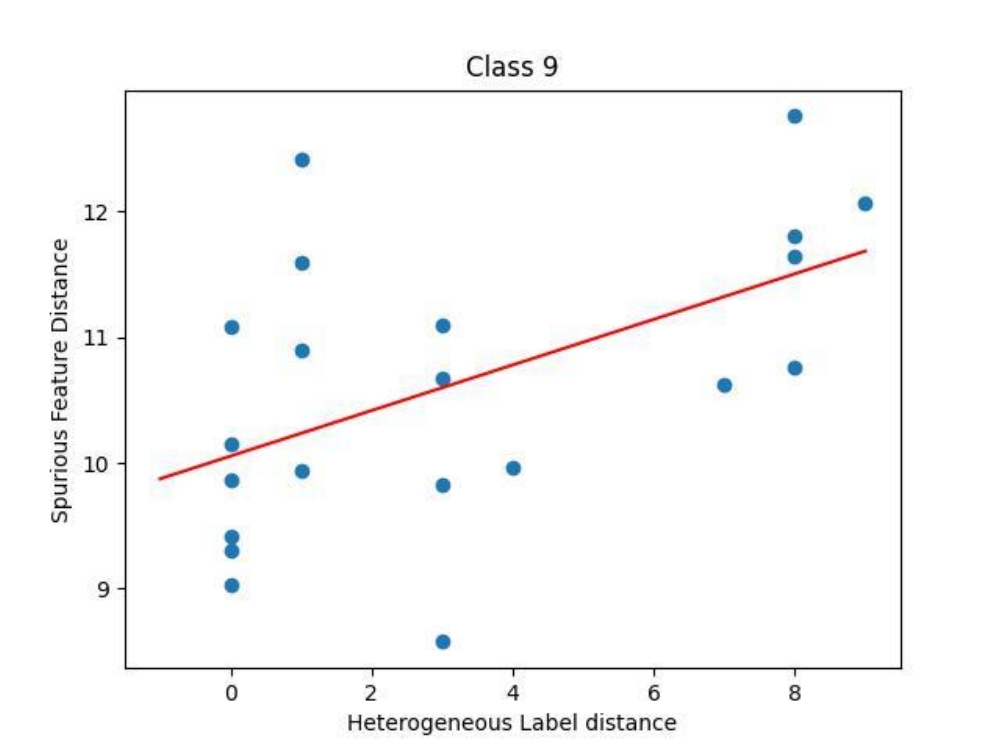}
    \end{subfigure}

    \begin{subfigure}{.19\textwidth}
        \centering
        \includegraphics[width=\linewidth]{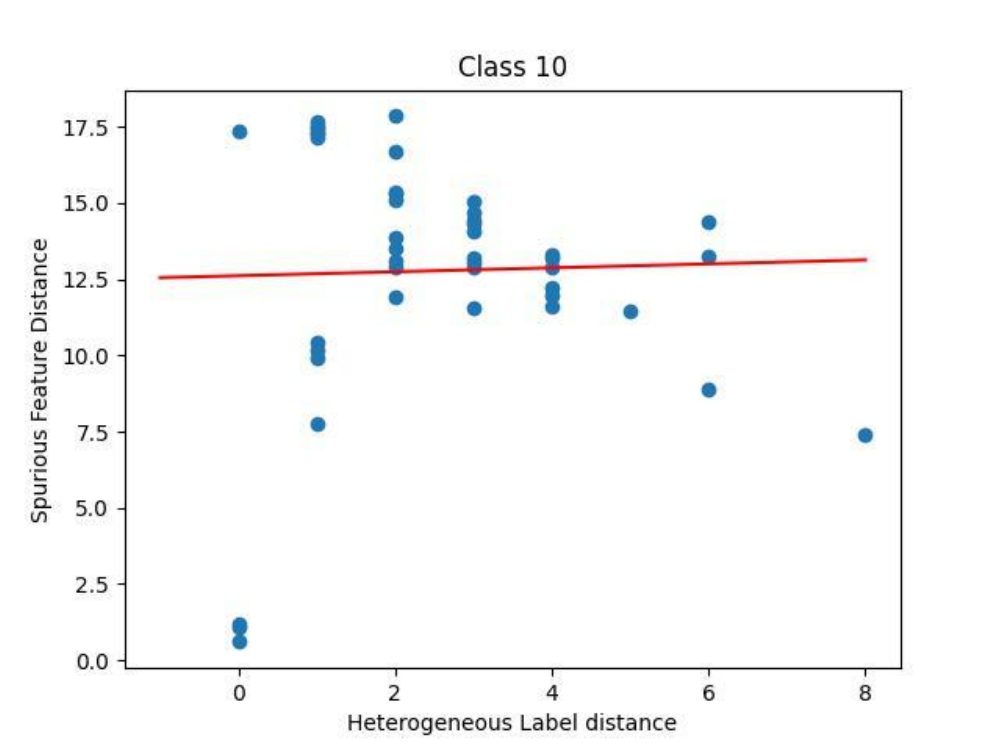}
    \end{subfigure}
    \begin{subfigure}{.19\textwidth}
        \centering
        \includegraphics[width=\linewidth]{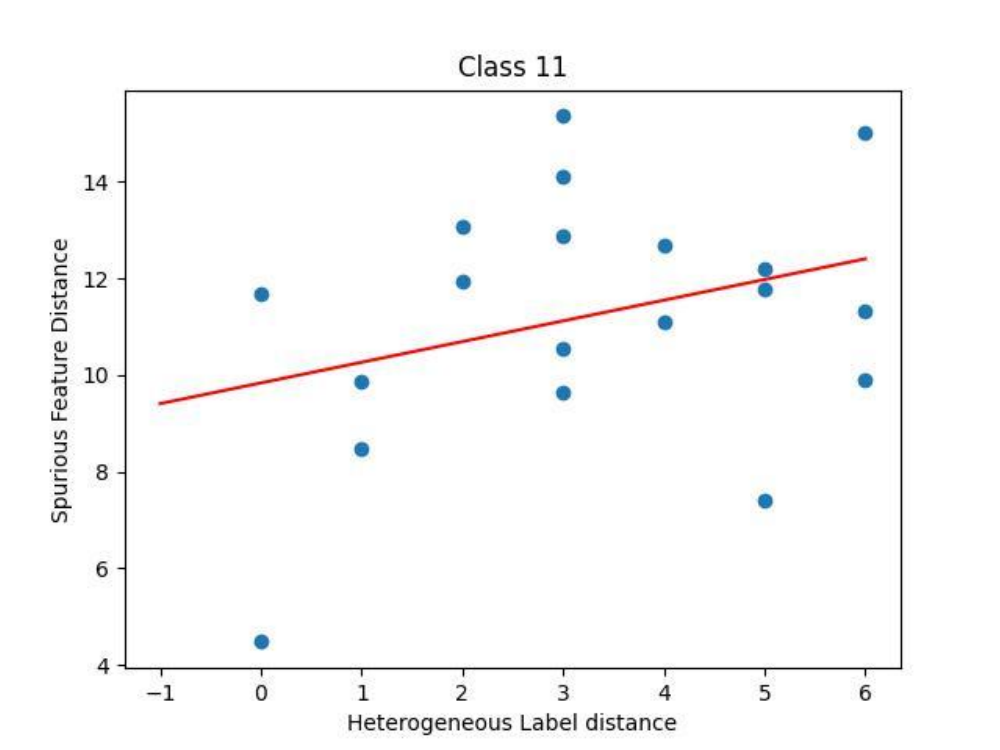}
    \end{subfigure}%
    \begin{subfigure}{.19\textwidth}
        \centering
        \includegraphics[width=\linewidth]{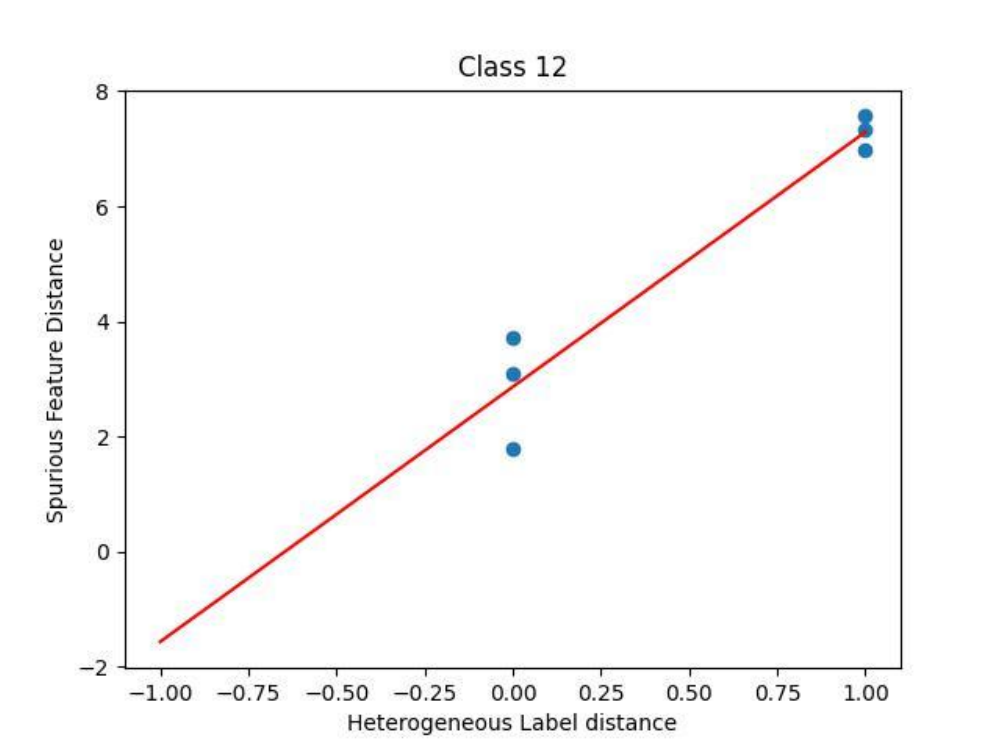}
    \end{subfigure}
    \begin{subfigure}{.19\textwidth}
        \centering
        \includegraphics[width=\linewidth]{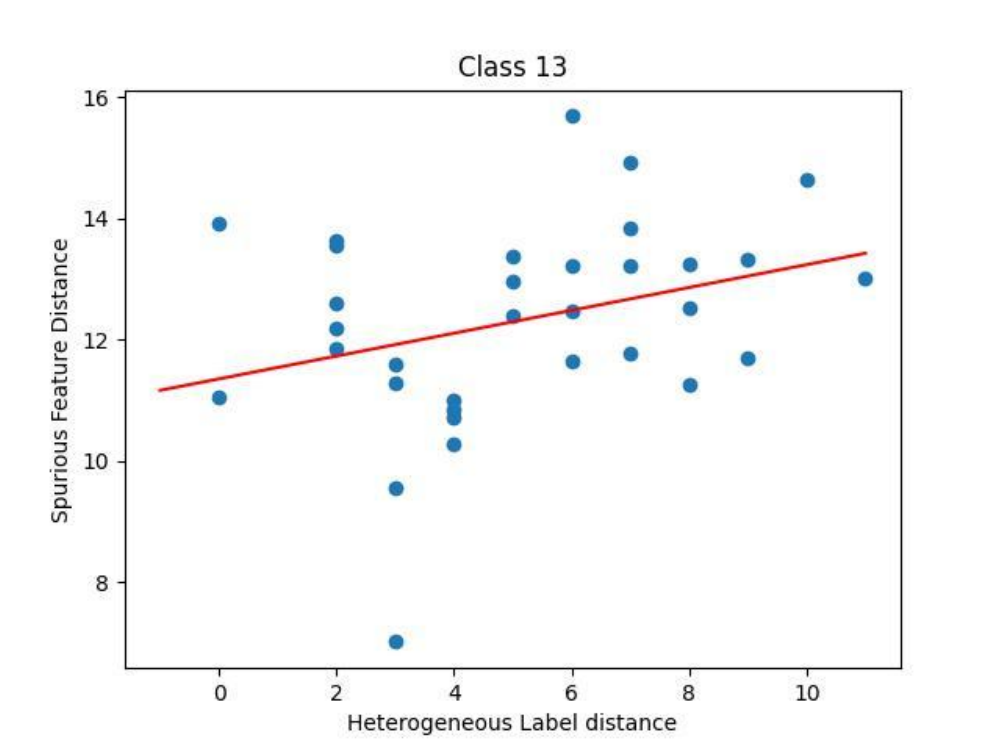}
    \end{subfigure}
    \begin{subfigure}{.19\textwidth}
        \centering
        \includegraphics[width=\linewidth]{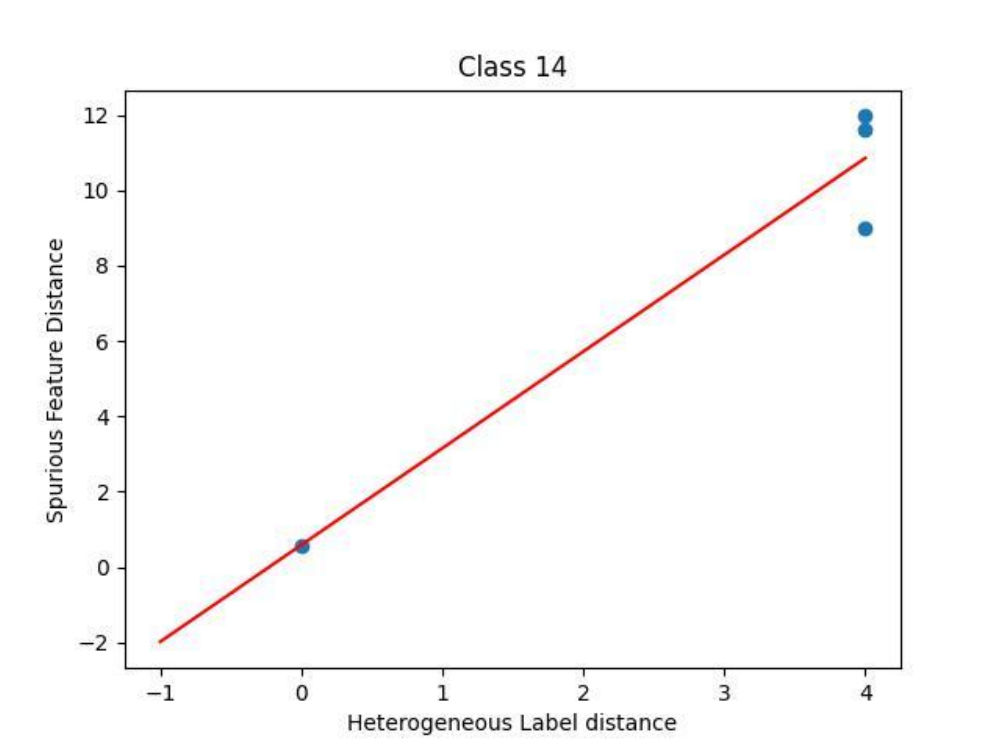}
    \end{subfigure}

    \begin{subfigure}{.19\textwidth}
        \centering
        \includegraphics[width=\linewidth]{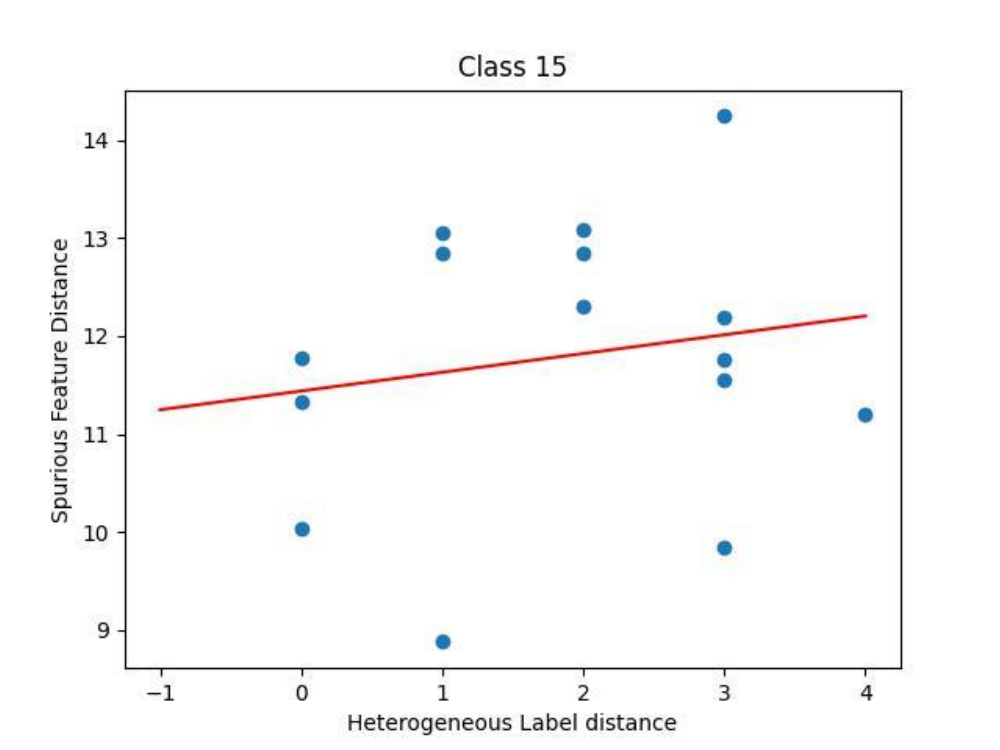}
    \end{subfigure}
    \begin{subfigure}{.19\textwidth}
        \centering
        \includegraphics[width=\linewidth]{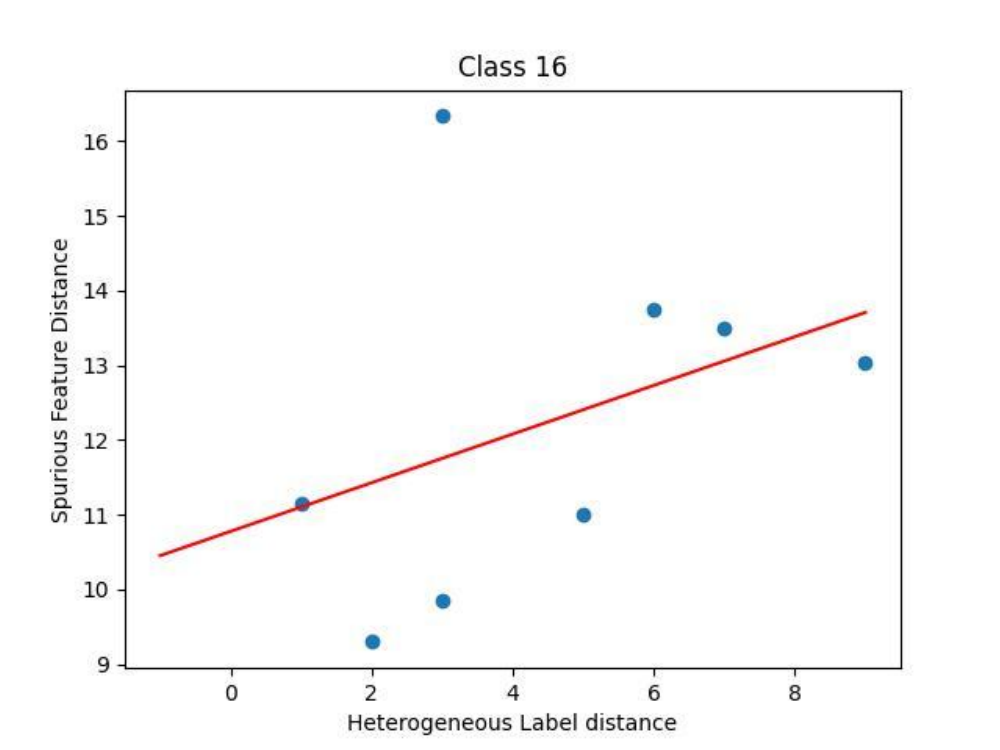}
    \end{subfigure}%
    \begin{subfigure}{.19\textwidth}
        \centering
        \includegraphics[width=\linewidth]{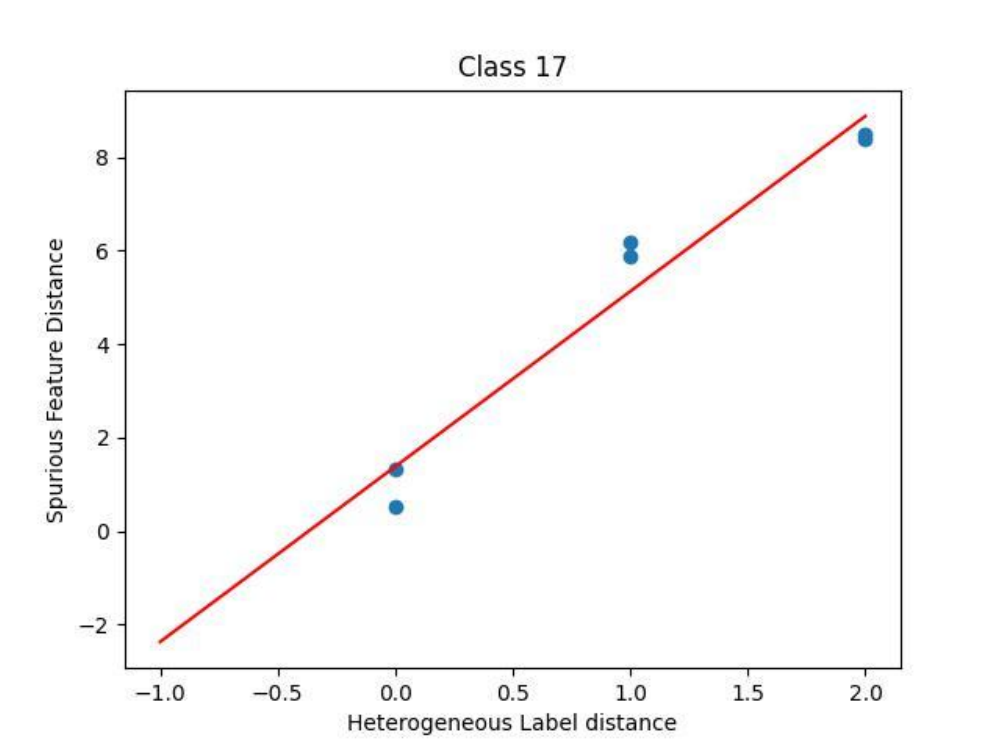}
    \end{subfigure}
    \begin{subfigure}{.19\textwidth}
        \centering
        \includegraphics[width=\linewidth]{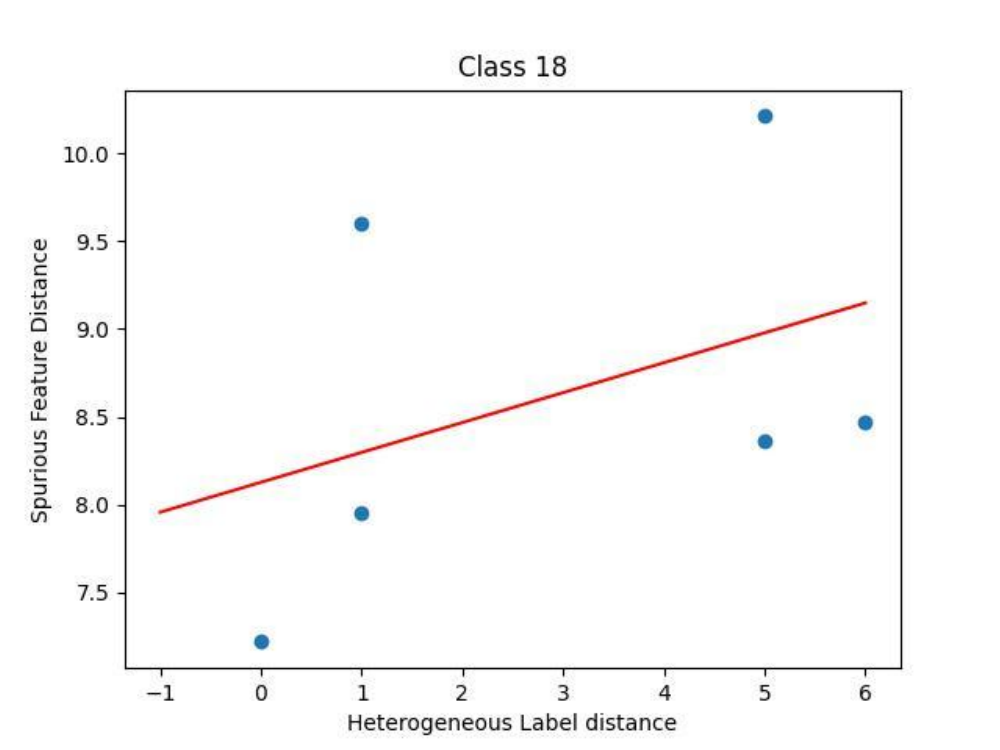}
    \end{subfigure}
    \begin{subfigure}{.19\textwidth}
        \centering
        \includegraphics[width=\linewidth]{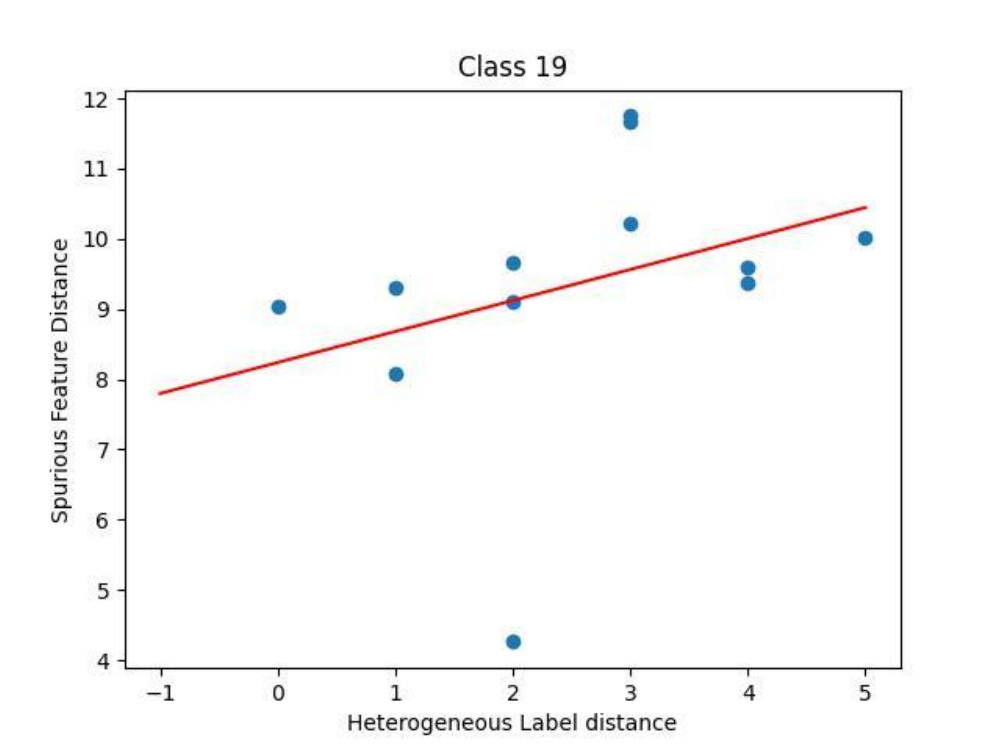}
    \end{subfigure}
    \caption{The relationship between the distance of environmental spurious features and distance of HeteNLD on Cora \textit{word}, \textbf{concept shift}. The positive correlation holds for most classes.}
    \label{reflect_sp_cora_word_con}
\end{figure}

\begin{figure}
    \centering
    \begin{subfigure}{.19\textwidth}
        \centering
        \includegraphics[width=\linewidth]{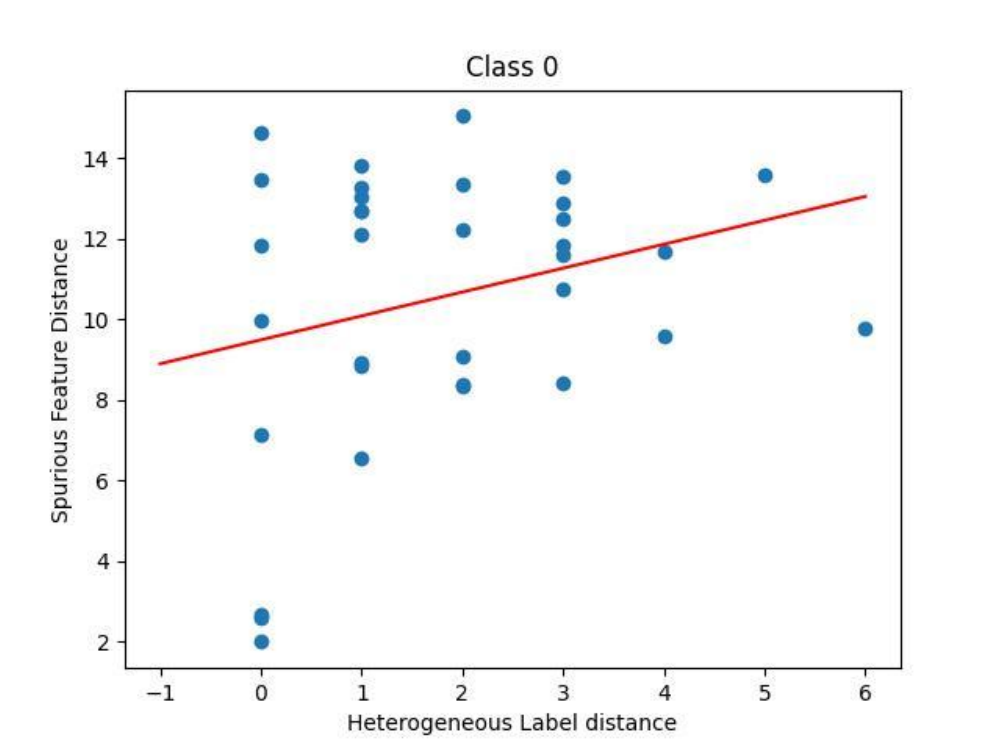}
    \end{subfigure}%
    \begin{subfigure}{.19\textwidth}
        \centering
        \includegraphics[width=\linewidth]{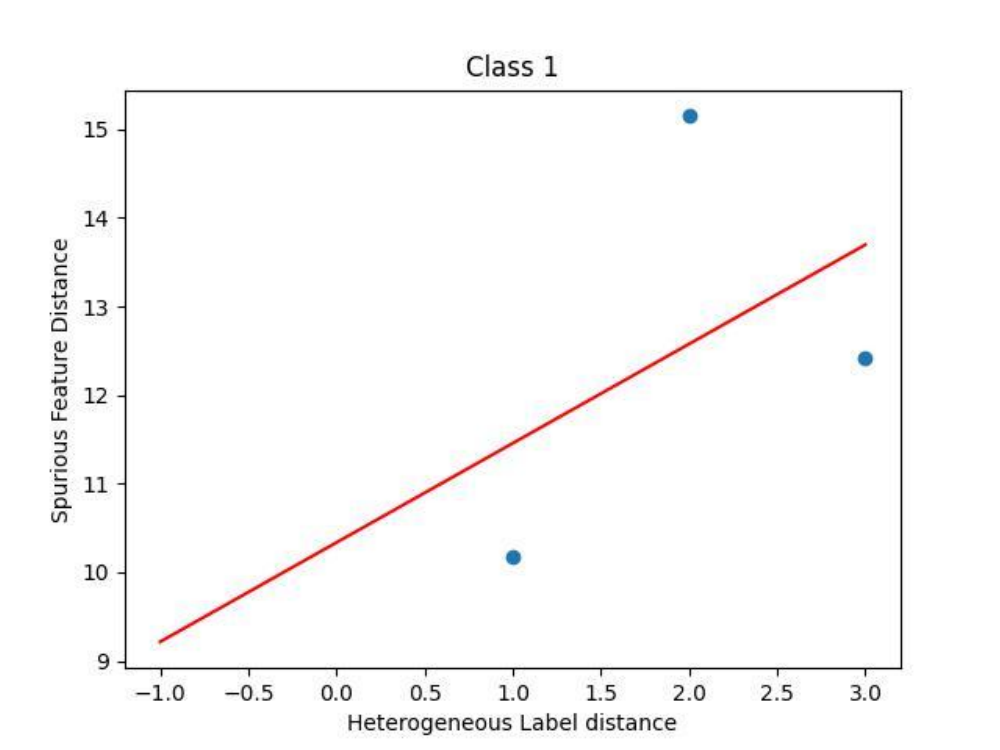}
    \end{subfigure}
    \begin{subfigure}{.19\textwidth}
        \centering
        \includegraphics[width=\linewidth]{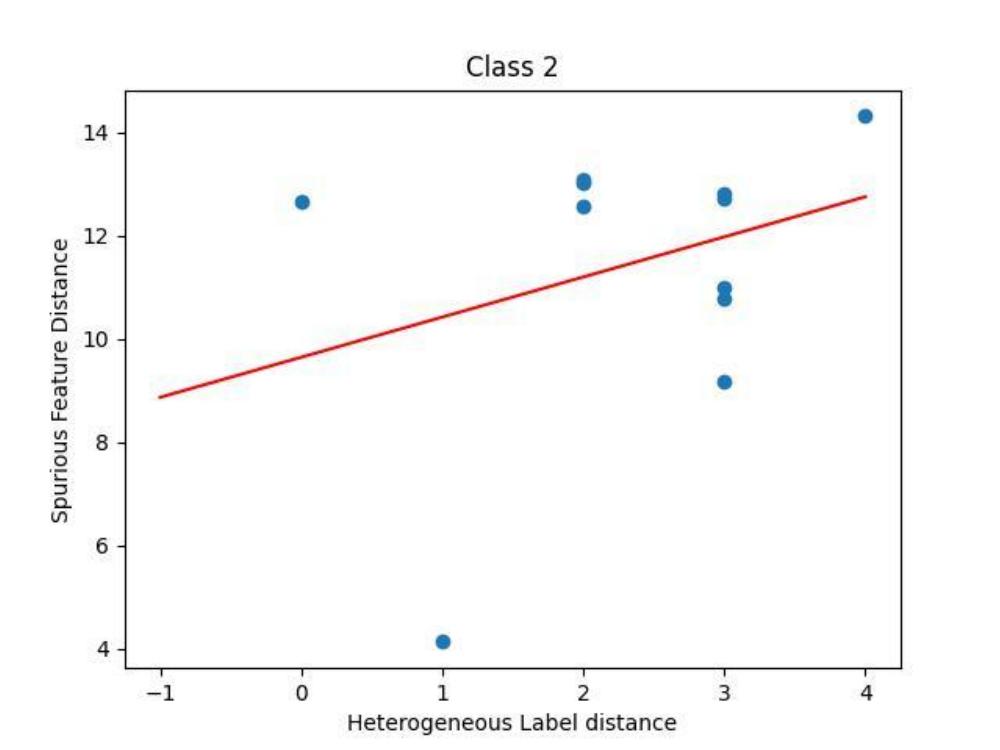}
    \end{subfigure}
    \begin{subfigure}{.19\textwidth}
        \centering
        \includegraphics[width=\linewidth]{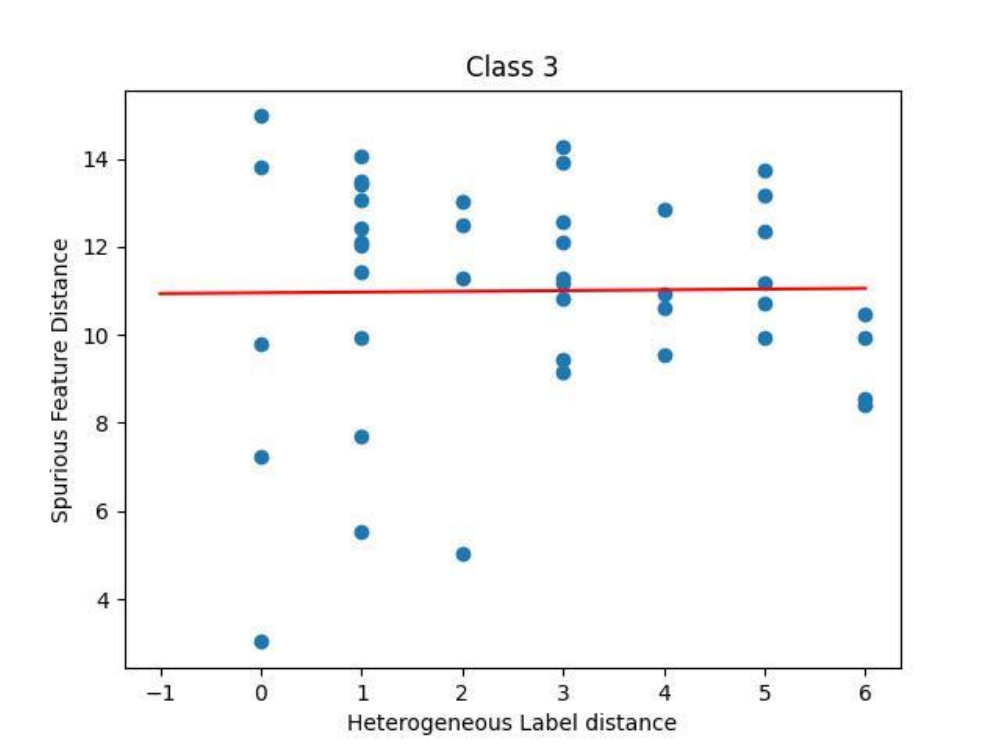}
    \end{subfigure}
    \begin{subfigure}{.19\textwidth}
        \centering
        \includegraphics[width=\linewidth]{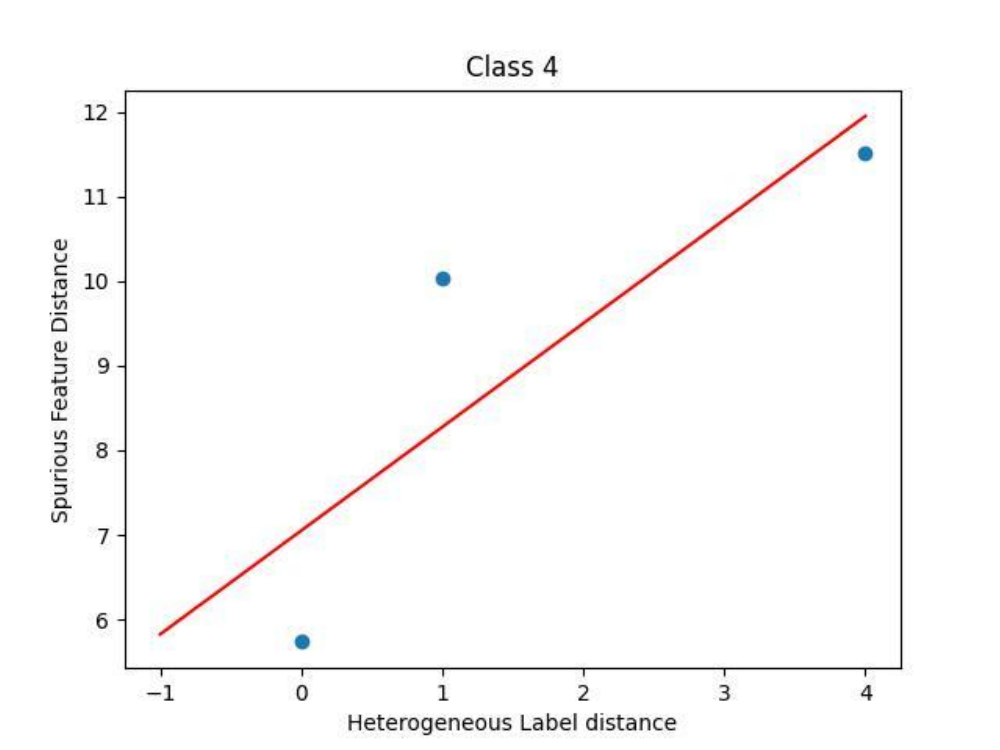}
    \end{subfigure}
    
    \begin{subfigure}{.19\textwidth}
        \centering
        \includegraphics[width=\linewidth]{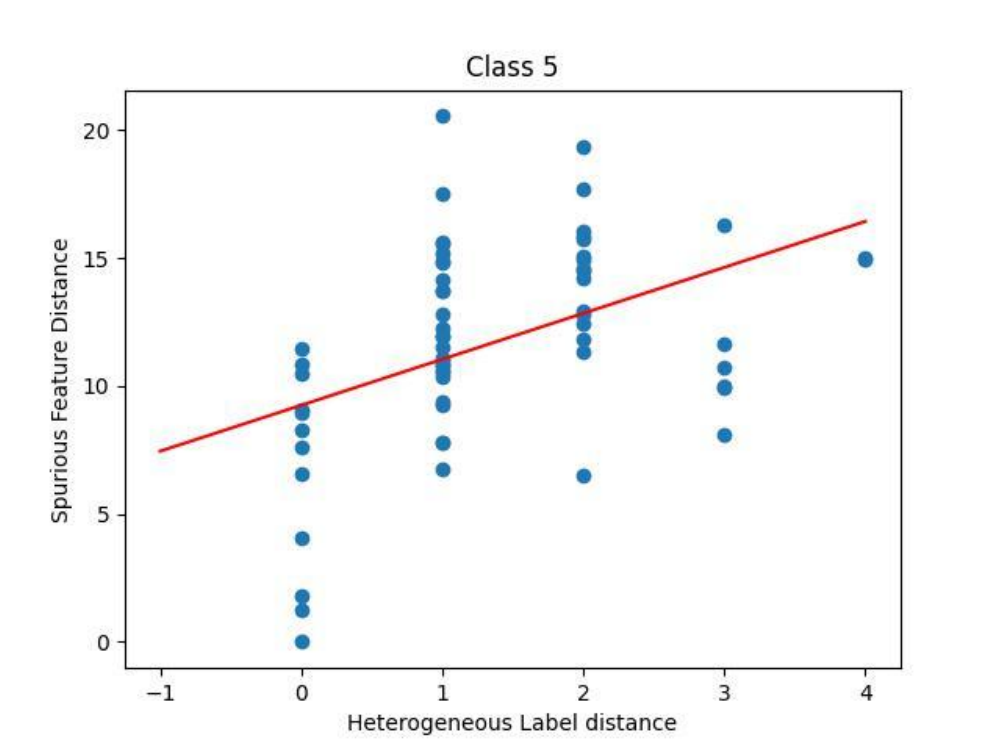}
    \end{subfigure}
    \begin{subfigure}{.19\textwidth}
        \centering
        \includegraphics[width=\linewidth]{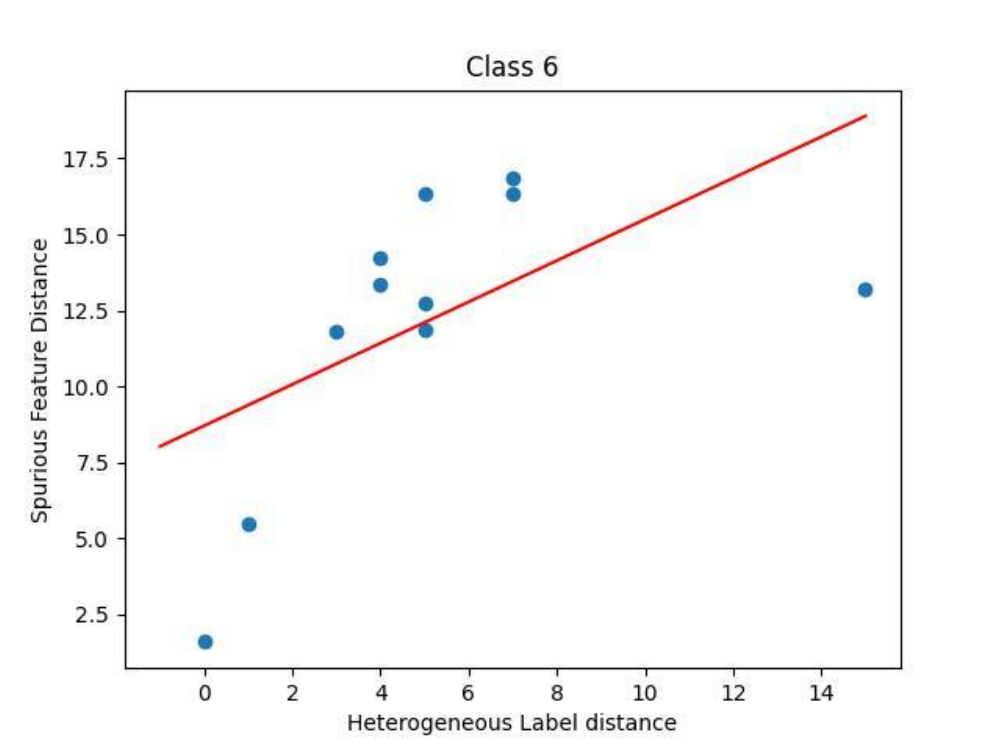}
    \end{subfigure}%
    \begin{subfigure}{.19\textwidth}
        \centering
        \includegraphics[width=\linewidth]{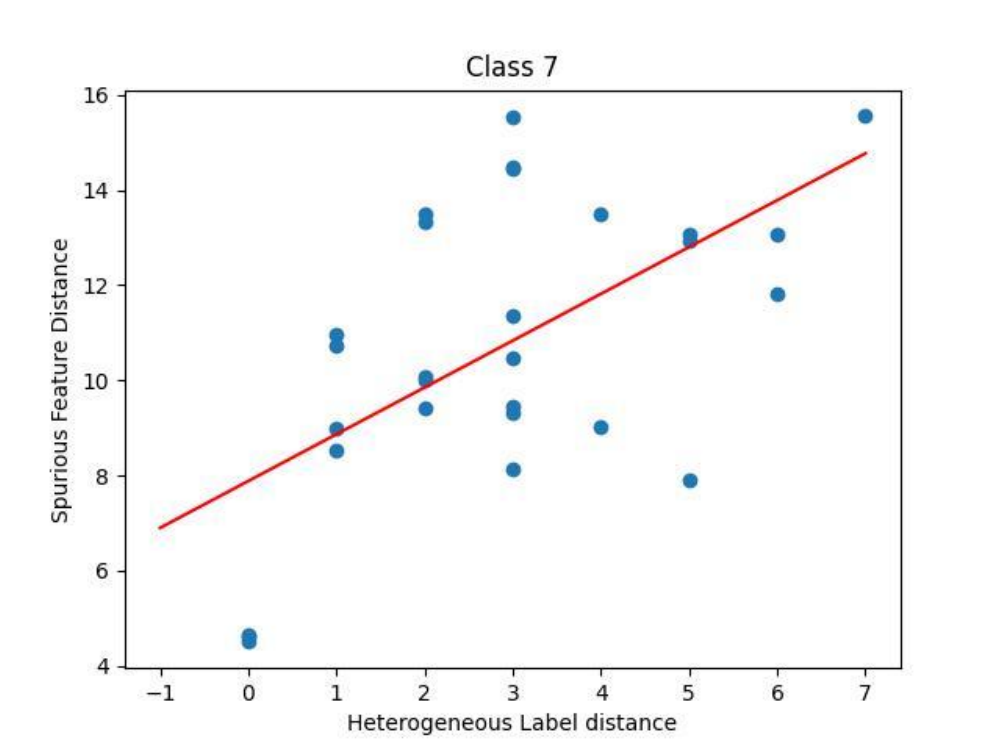}
    \end{subfigure}
    \begin{subfigure}{.19\textwidth}
        \centering
        \includegraphics[width=\linewidth]{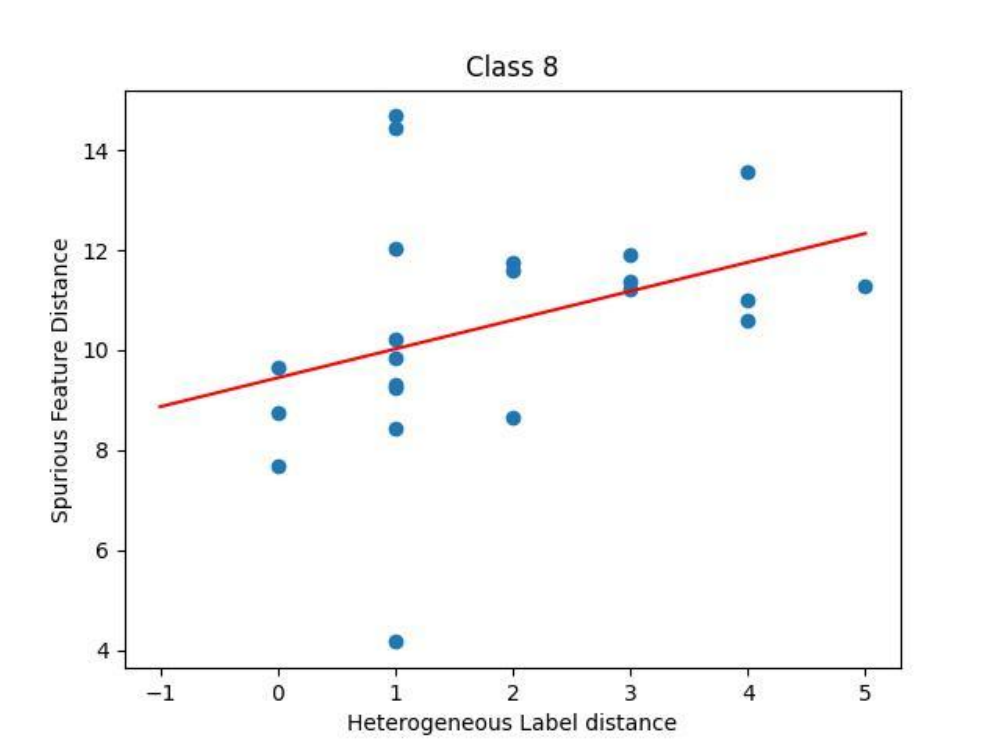}
    \end{subfigure}
    \begin{subfigure}{.19\textwidth}
        \centering
        \includegraphics[width=\linewidth]{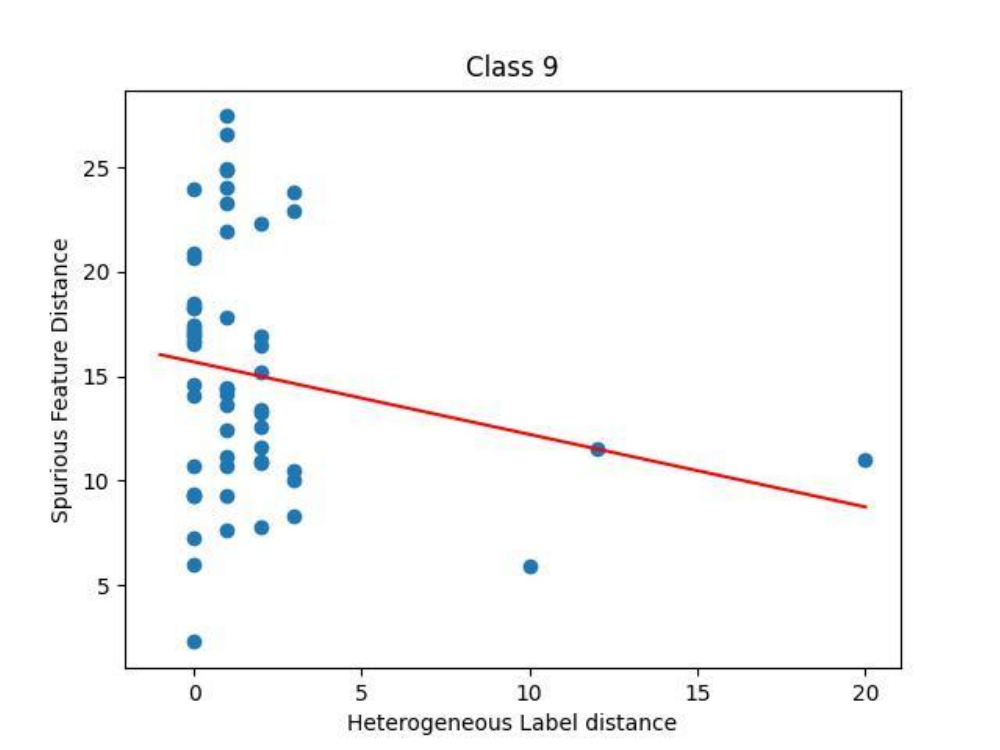}
    \end{subfigure}

    \begin{subfigure}{.19\textwidth}
        \centering
        \includegraphics[width=\linewidth]{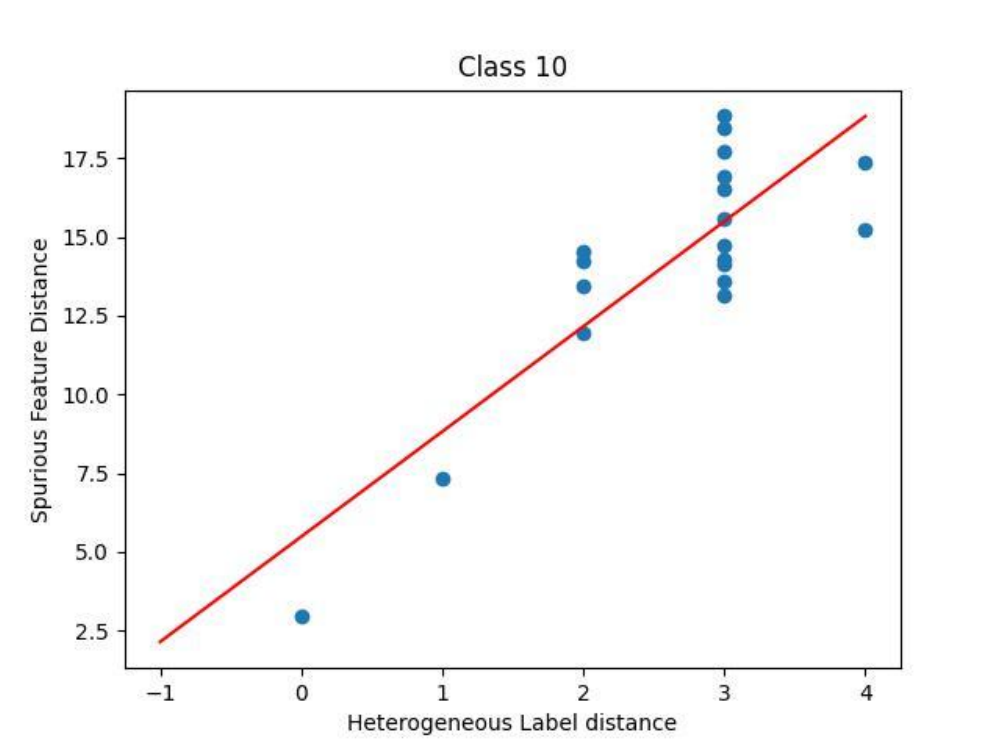}
    \end{subfigure}
    \begin{subfigure}{.19\textwidth}
        \centering
        \includegraphics[width=\linewidth]{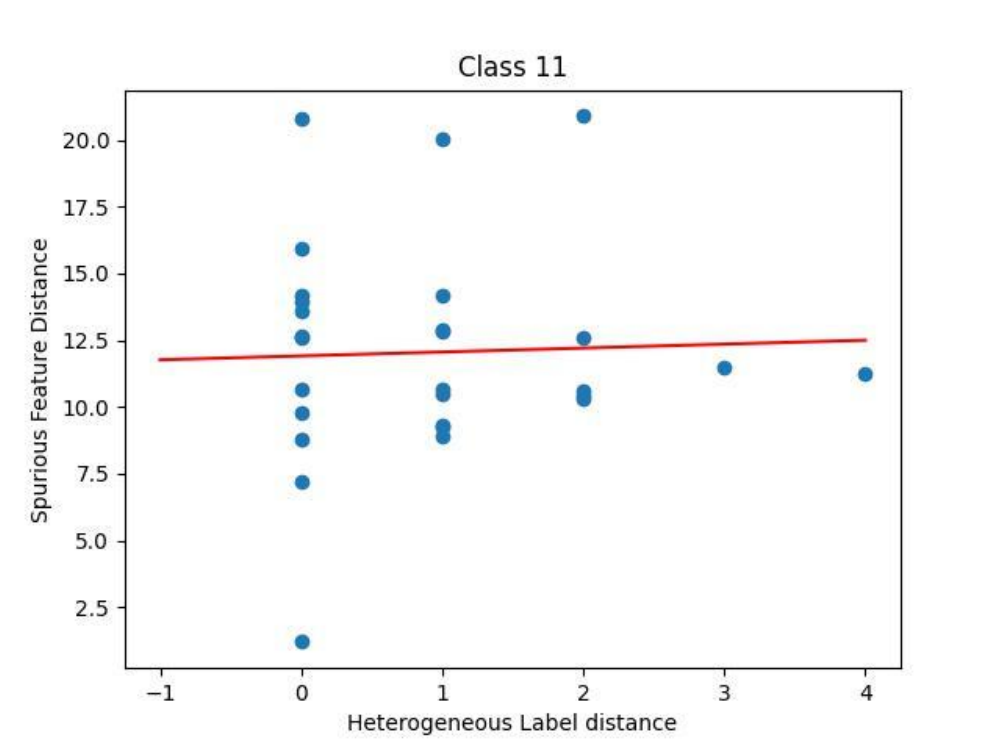}
    \end{subfigure}%
    \begin{subfigure}{.19\textwidth}
        \centering
        \includegraphics[width=\linewidth]{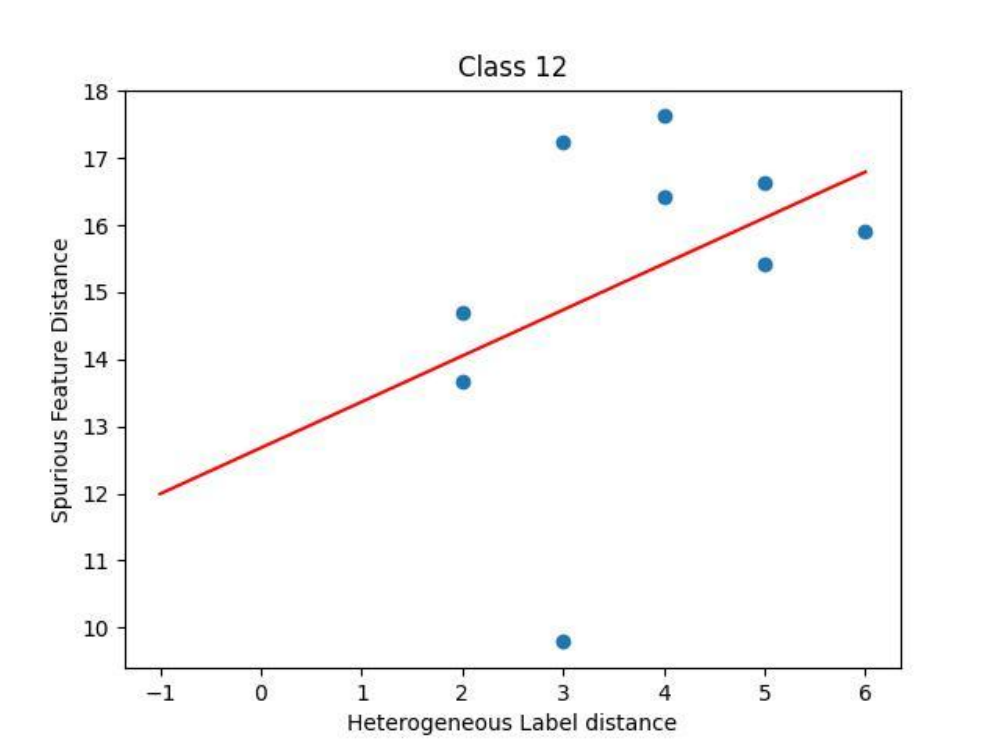}
    \end{subfigure}
    \begin{subfigure}{.19\textwidth}
        \centering
        \includegraphics[width=\linewidth]{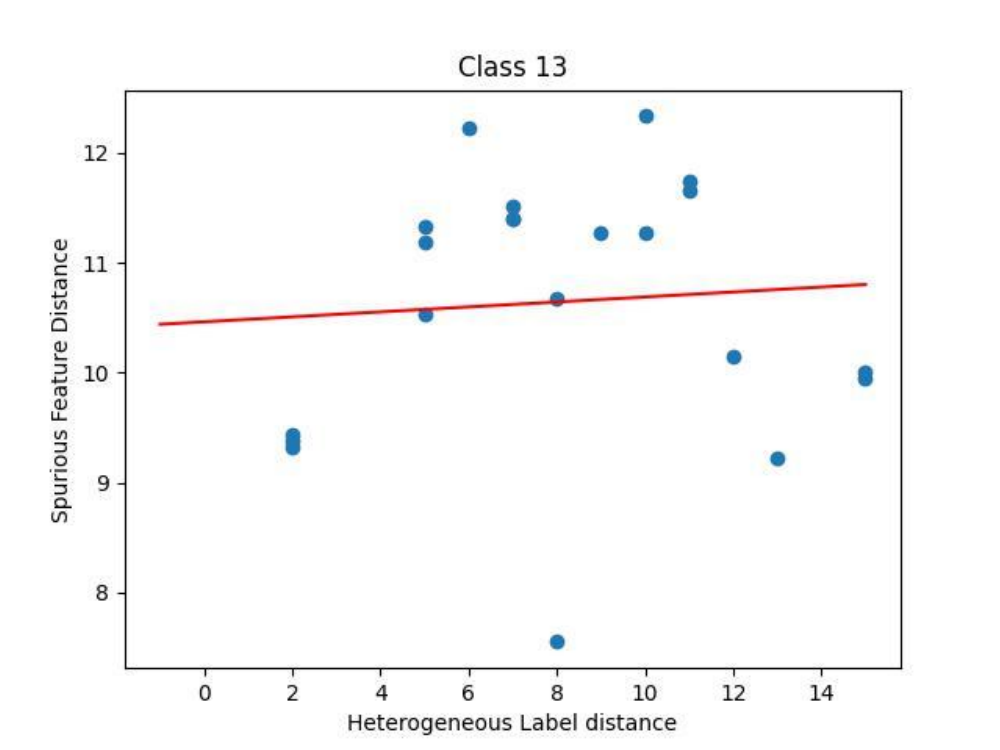}
    \end{subfigure}
    \begin{subfigure}{.19\textwidth}
        \centering
        \includegraphics[width=\linewidth]{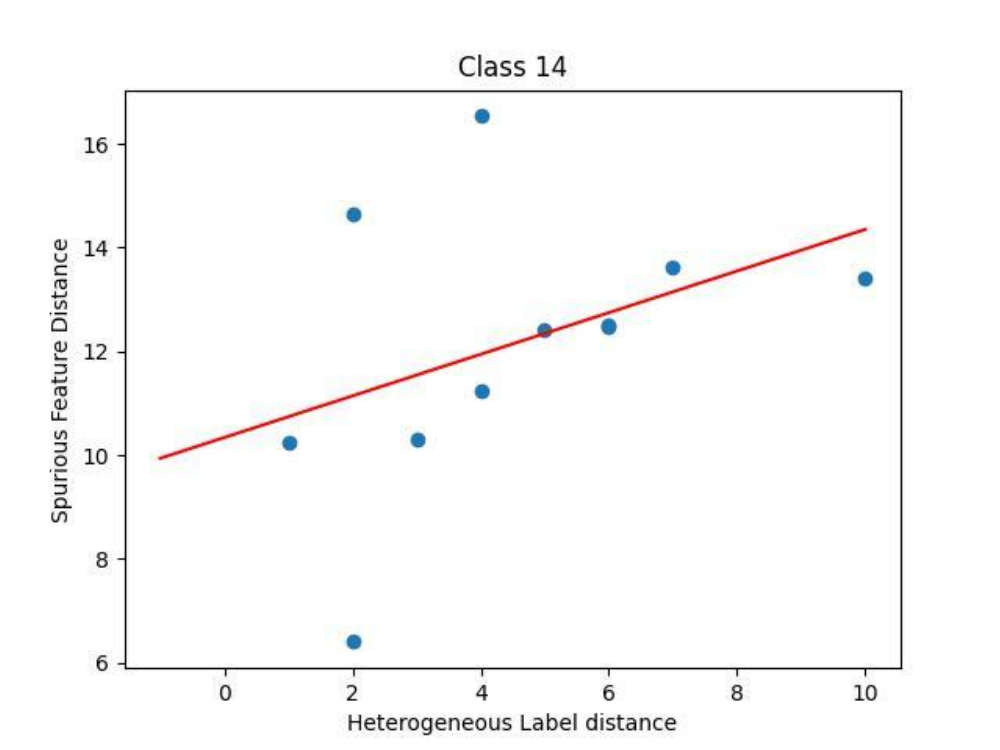}
    \end{subfigure}

    \begin{subfigure}{.19\textwidth}
        \centering
        \includegraphics[width=\linewidth]{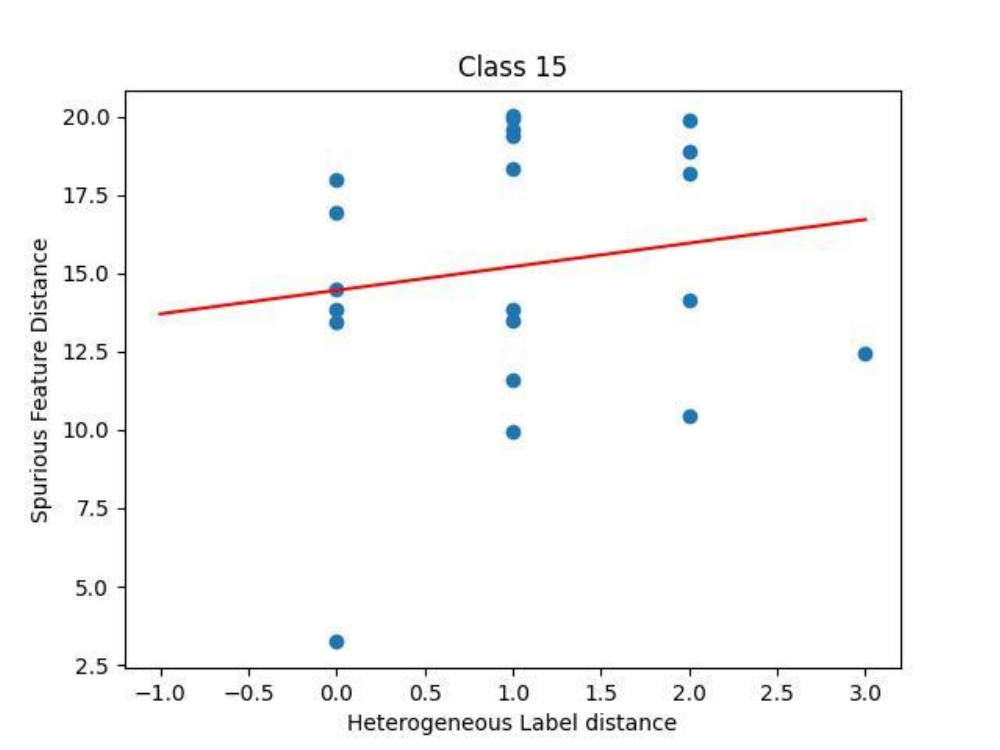}
    \end{subfigure}
    \begin{subfigure}{.19\textwidth}
        \centering
        \includegraphics[width=\linewidth]{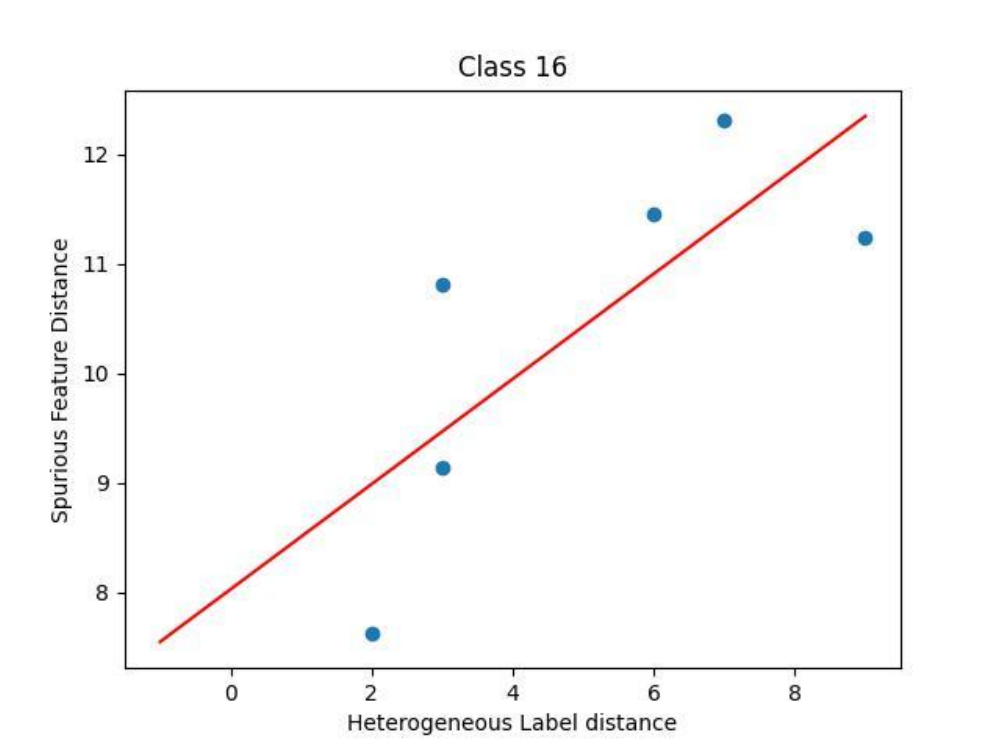}
    \end{subfigure}%
    \begin{subfigure}{.19\textwidth}
        \centering
        \includegraphics[width=\linewidth]{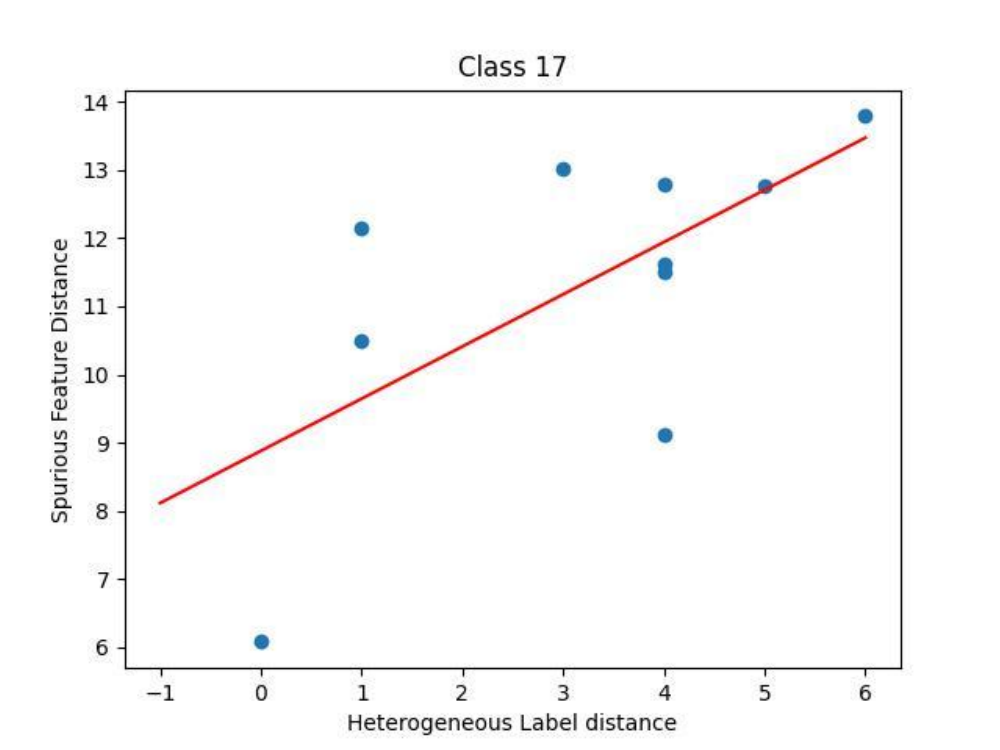}
    \end{subfigure}
    \begin{subfigure}{.19\textwidth}
        \centering
        \includegraphics[width=\linewidth]{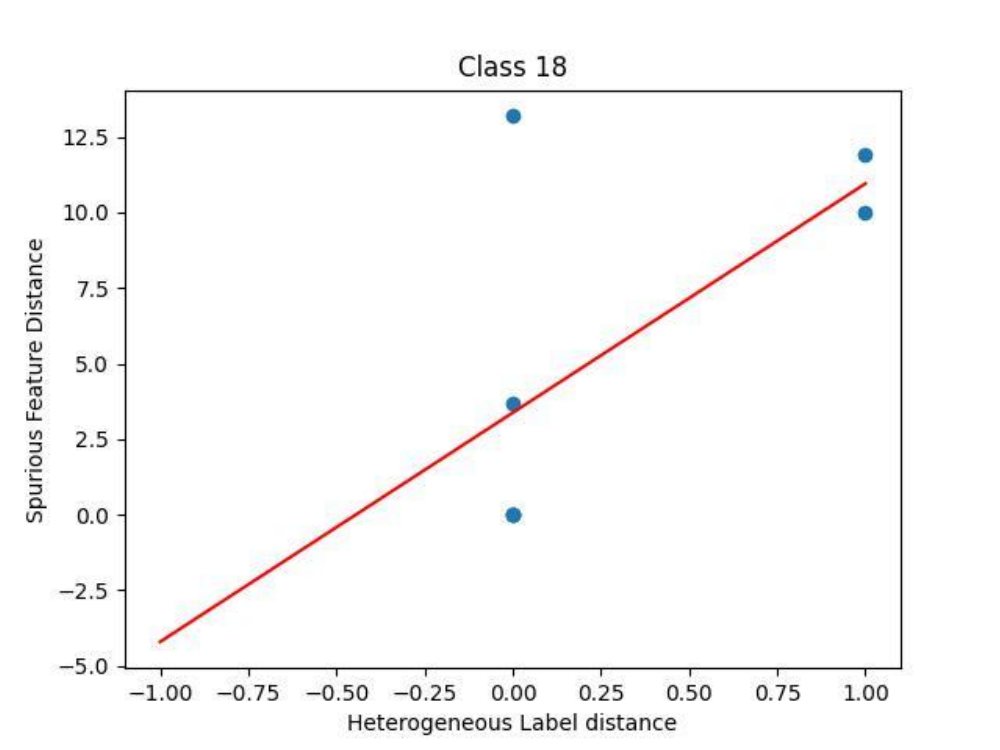}
    \end{subfigure}
    \begin{subfigure}{.19\textwidth}
        \centering
        \includegraphics[width=\linewidth]{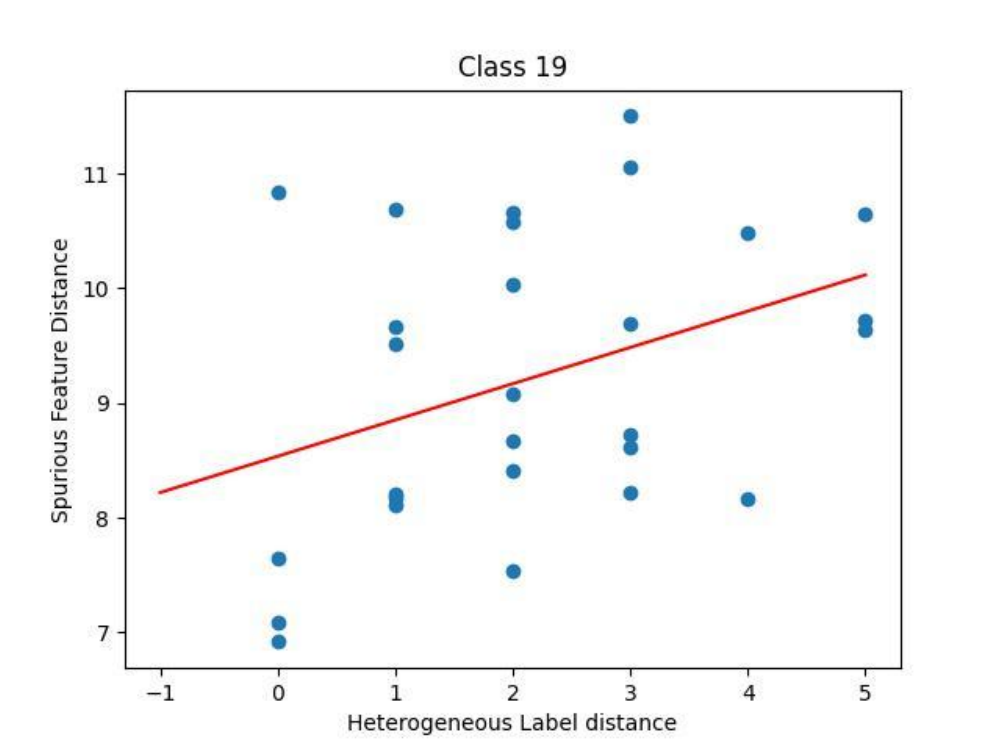}
    \end{subfigure}
    \caption{The relationship between the distance of environmental spurious features and distance of HeteNLD on Cora \textit{degree}, \textbf{concept shift}. The positive correlation holds for most classes.}
    \label{reflect_sp_cora_deg_con}
\end{figure}

\subsubsection{Homophilic Neighboring Labels Reflect Invariant Feature Distribution}
\label{reflect_inv}

Now will validate that the ratio of the same-class neighbors reflects the aggregated invariant representation. We use VREx to approximately extract invariant features and compute their distance w.r.t. the discrepancy of the ratio of the same-class neighbors. We evaluate on 4 splits of Cora: \textit{word}+\textbf{covariate}, \textit{word}+\textbf{concept}, \textit{degree}+\textbf{covariate} and \textit{degree}+\textbf{concept}. For each data split, we randomly choose 5 classes with sufficiently large differences in homophilic neighbor ratios for visualization.
\textbf{The results in Figure \ref{reflect_inv_cora} also show a positive correlation trend between the distance of the invariant representations and the difference in the ratio of same-class neighbors, indicating the latter can reflect the former.}

\begin{figure}
    \centering
    \begin{subfigure}{.19\textwidth}
        \centering
        \includegraphics[width=\linewidth]{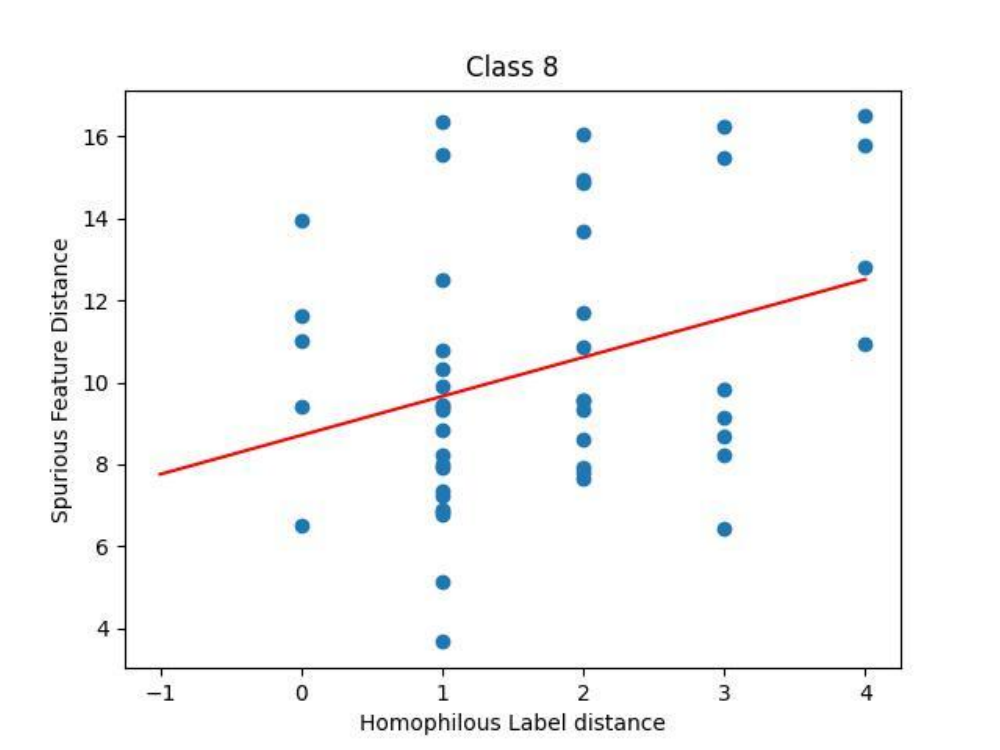}
    \end{subfigure}%
    \begin{subfigure}{.19\textwidth}
        \centering
        \includegraphics[width=\linewidth]{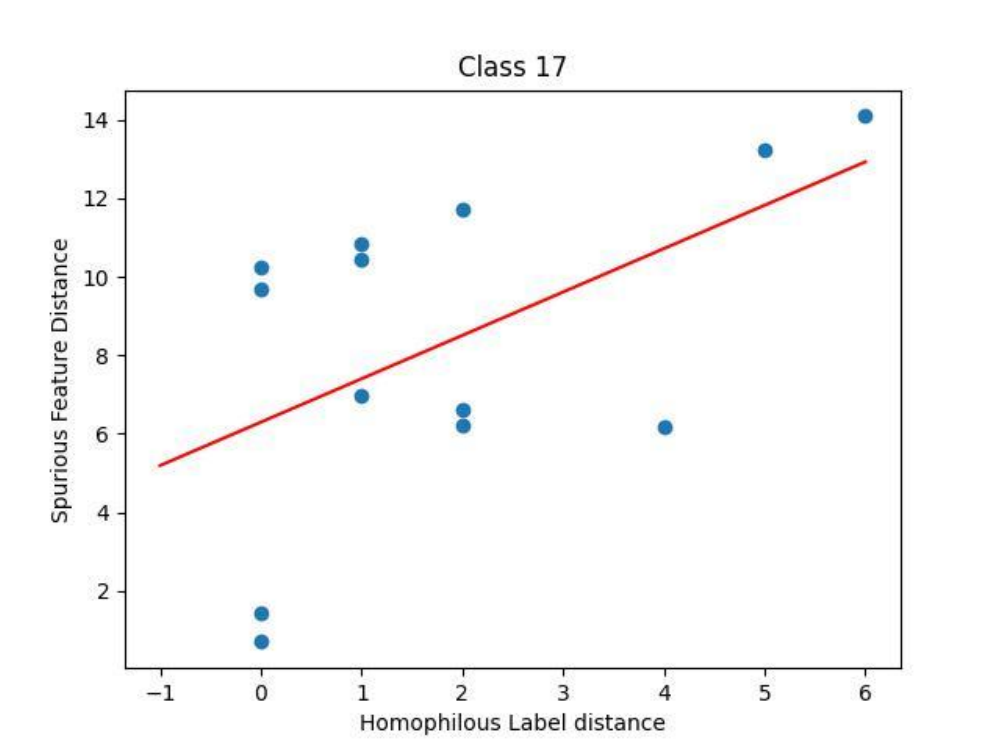}
    \end{subfigure}
    \begin{subfigure}{.19\textwidth}
        \centering
        \includegraphics[width=\linewidth]{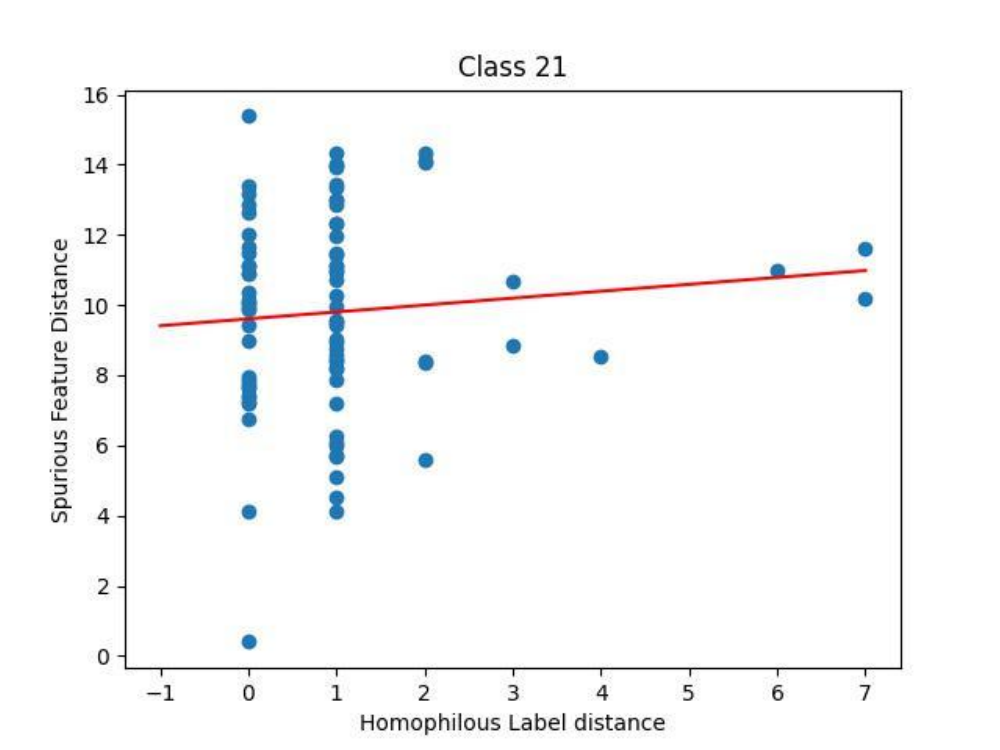}
    \end{subfigure}
    \begin{subfigure}{.19\textwidth}
        \centering
        \includegraphics[width=\linewidth]{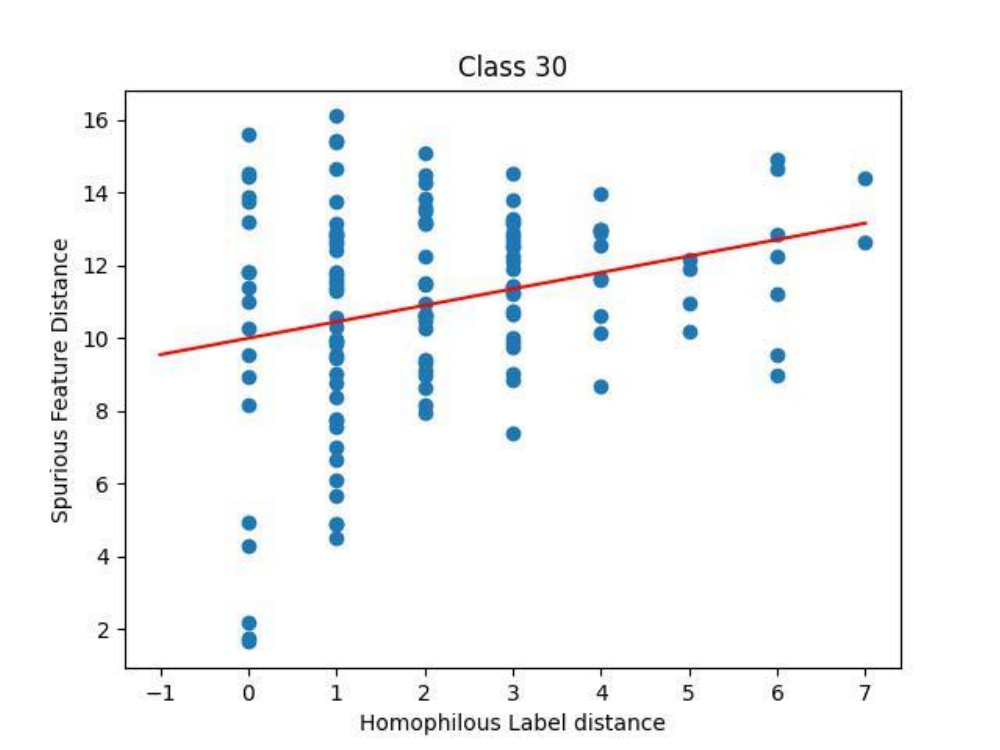}
    \end{subfigure}
    \begin{subfigure}{.19\textwidth}
        \centering
        \includegraphics[width=\linewidth]{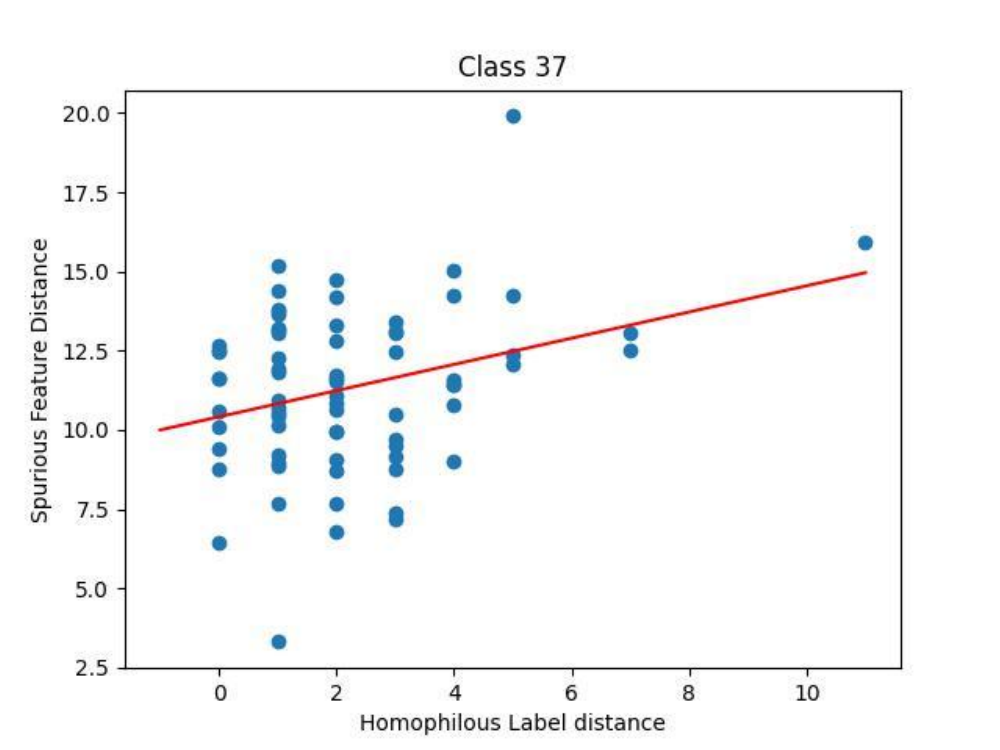}
    \end{subfigure}
    
    \begin{subfigure}{.19\textwidth}
        \centering
        \includegraphics[width=\linewidth]{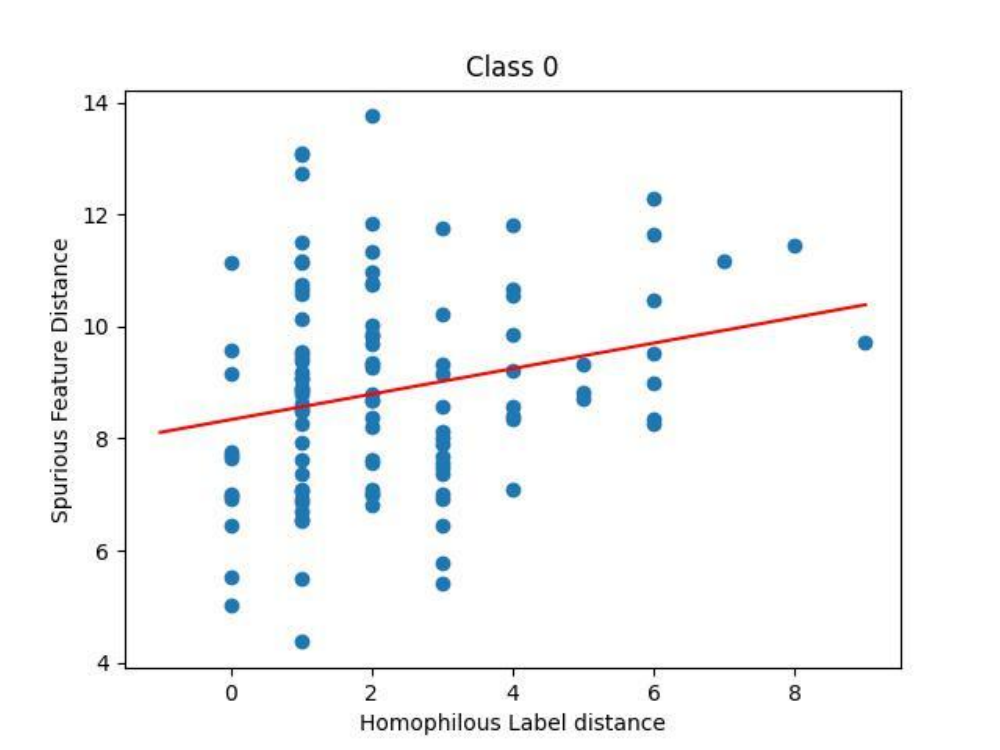}
    \end{subfigure}
    \begin{subfigure}{.19\textwidth}
        \centering
        \includegraphics[width=\linewidth]{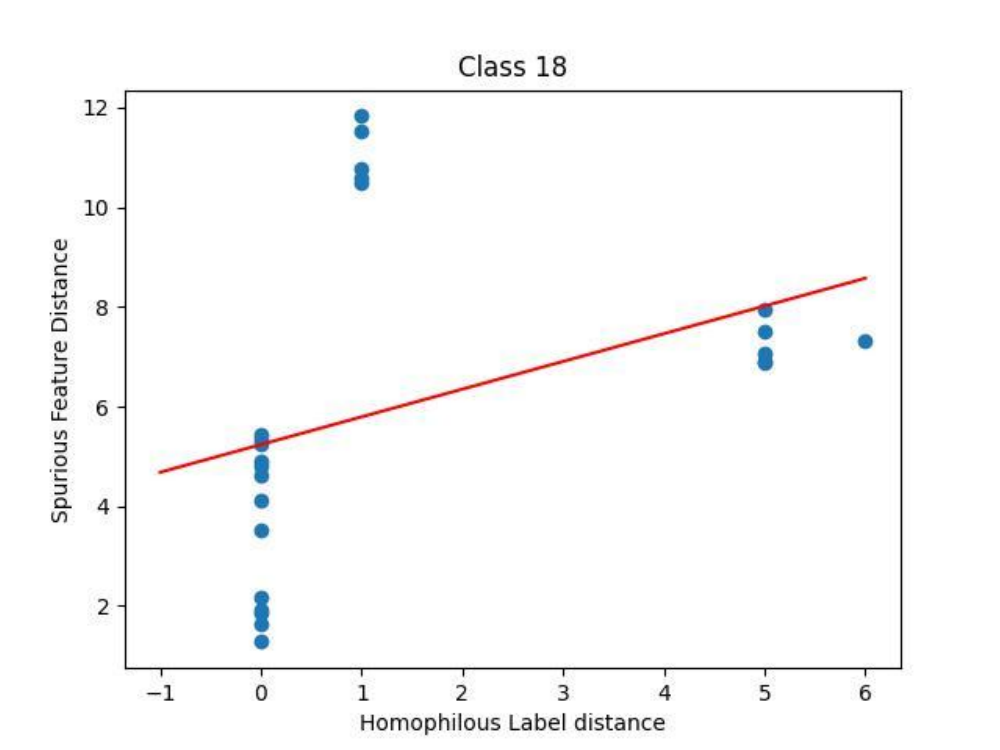}
    \end{subfigure}%
    \begin{subfigure}{.19\textwidth}
        \centering
        \includegraphics[width=\linewidth]{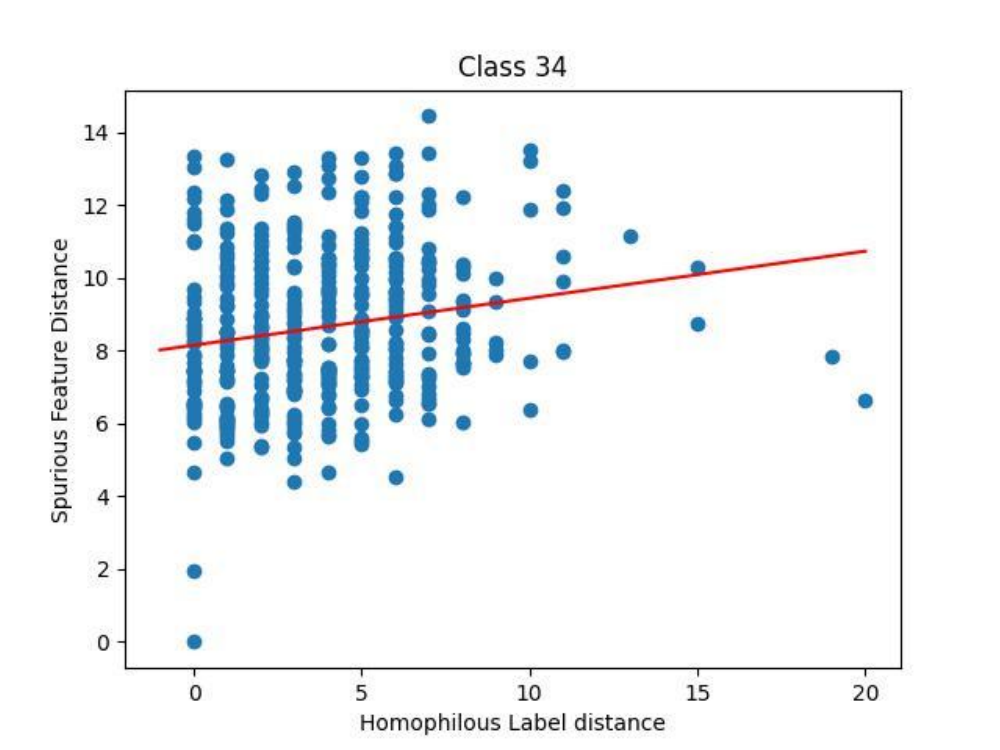}
    \end{subfigure}
    \begin{subfigure}{.19\textwidth}
        \centering
        \includegraphics[width=\linewidth]{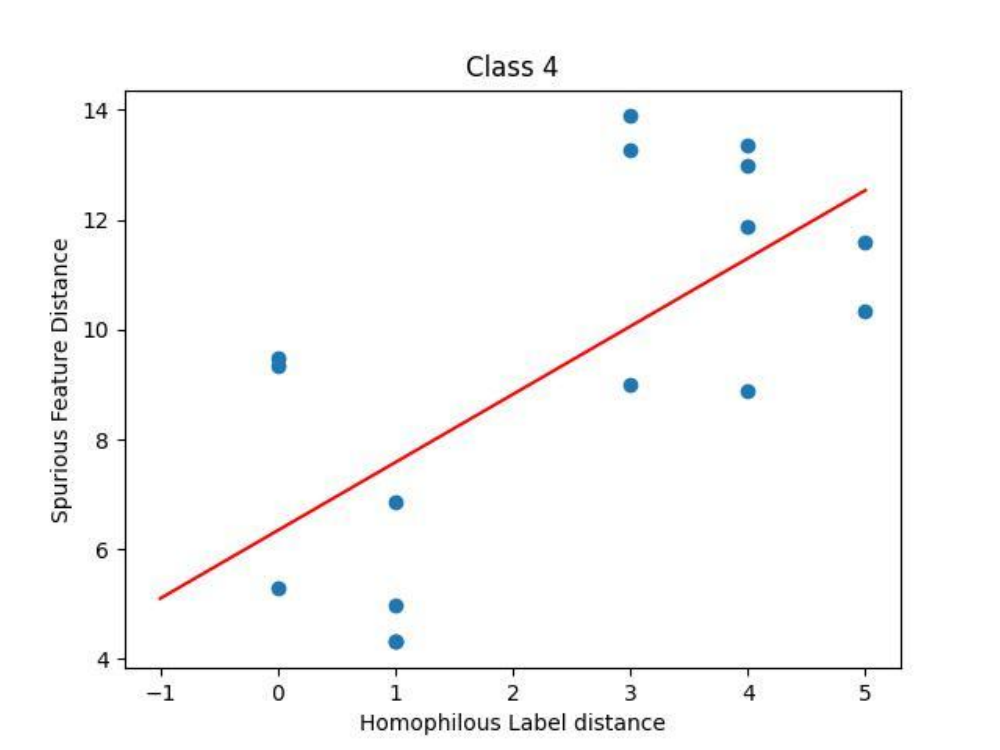}
    \end{subfigure}
    \begin{subfigure}{.19\textwidth}
        \centering
        \includegraphics[width=\linewidth]{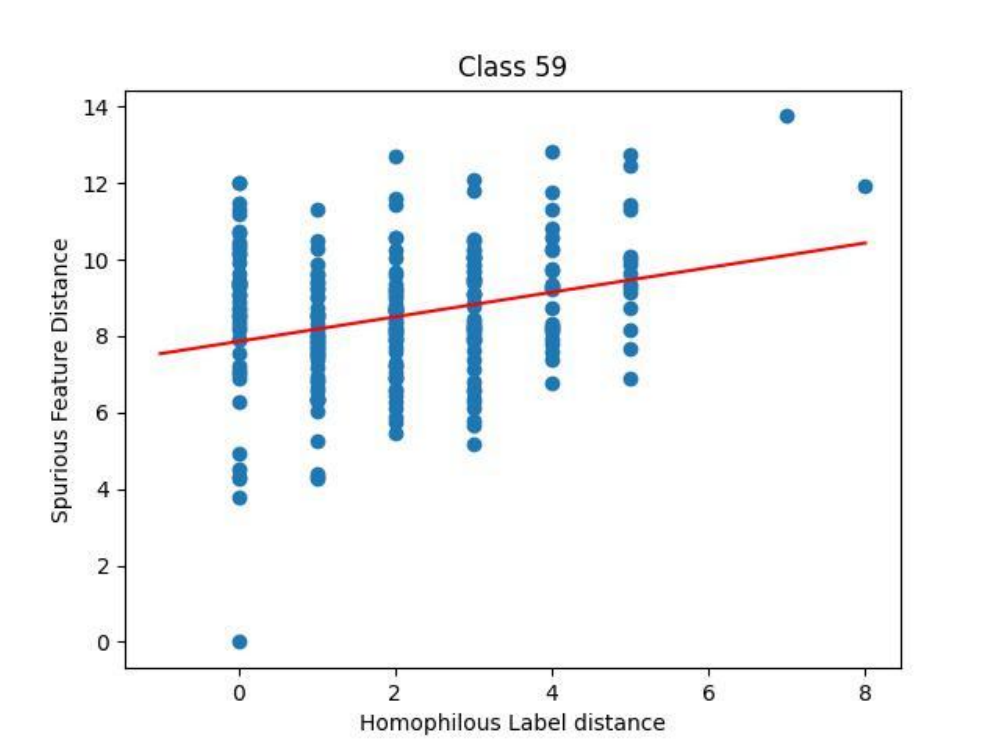}
    \end{subfigure}

    \begin{subfigure}{.19\textwidth}
        \centering
        \includegraphics[width=\linewidth]{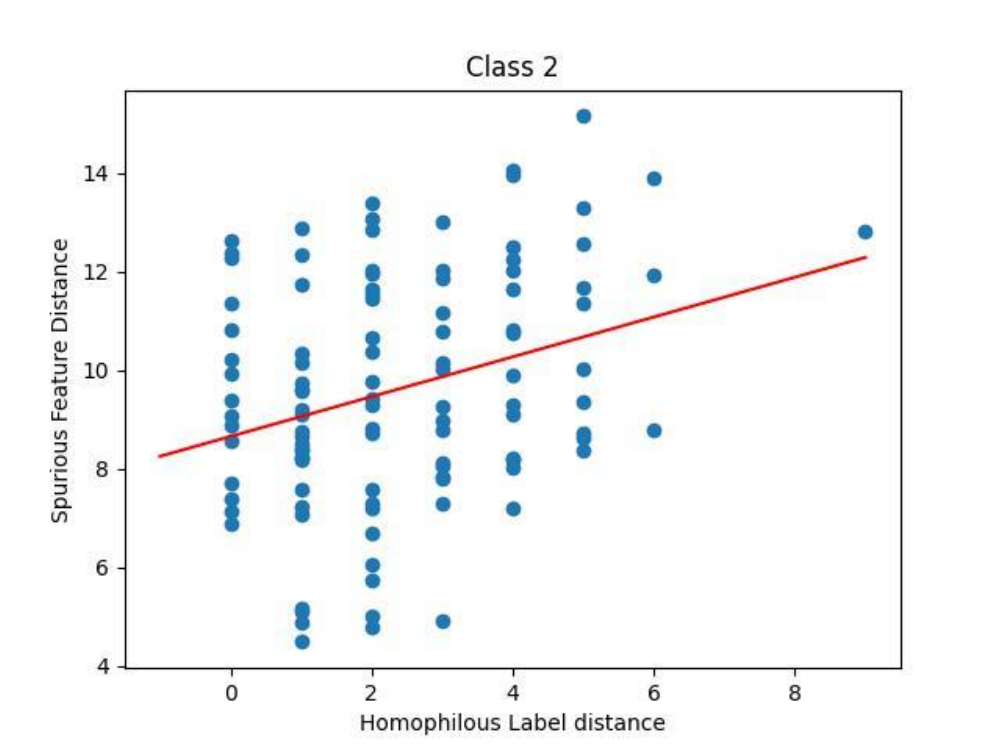}
    \end{subfigure}
    \begin{subfigure}{.19\textwidth}
        \centering
        \includegraphics[width=\linewidth]{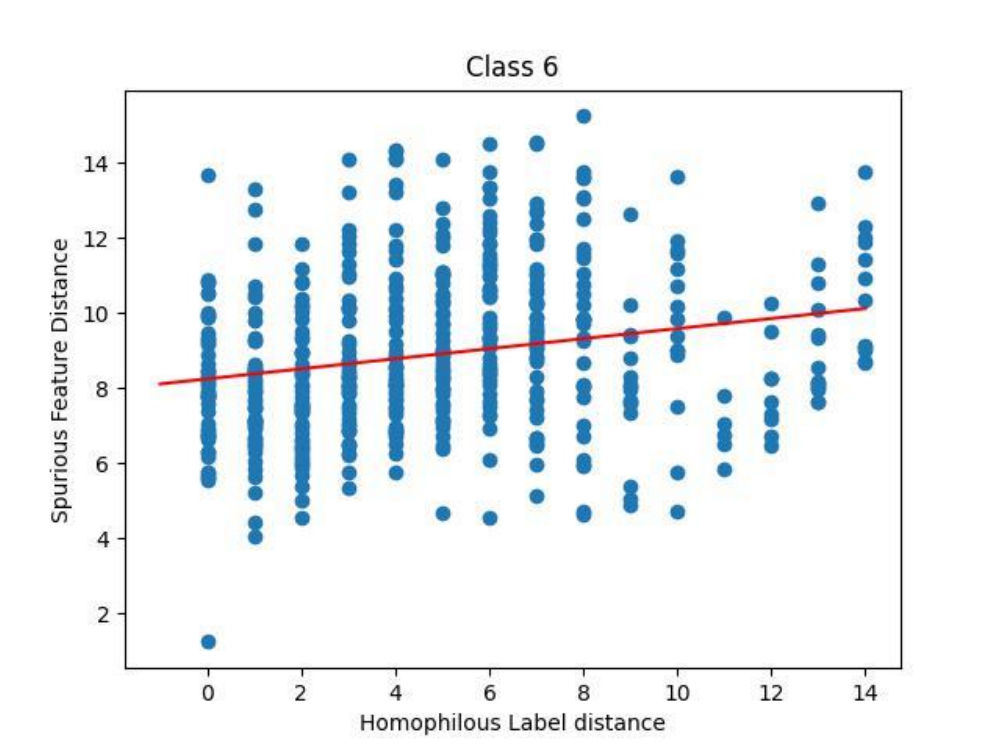}
    \end{subfigure}%
    \begin{subfigure}{.19\textwidth}
        \centering
        \includegraphics[width=\linewidth]{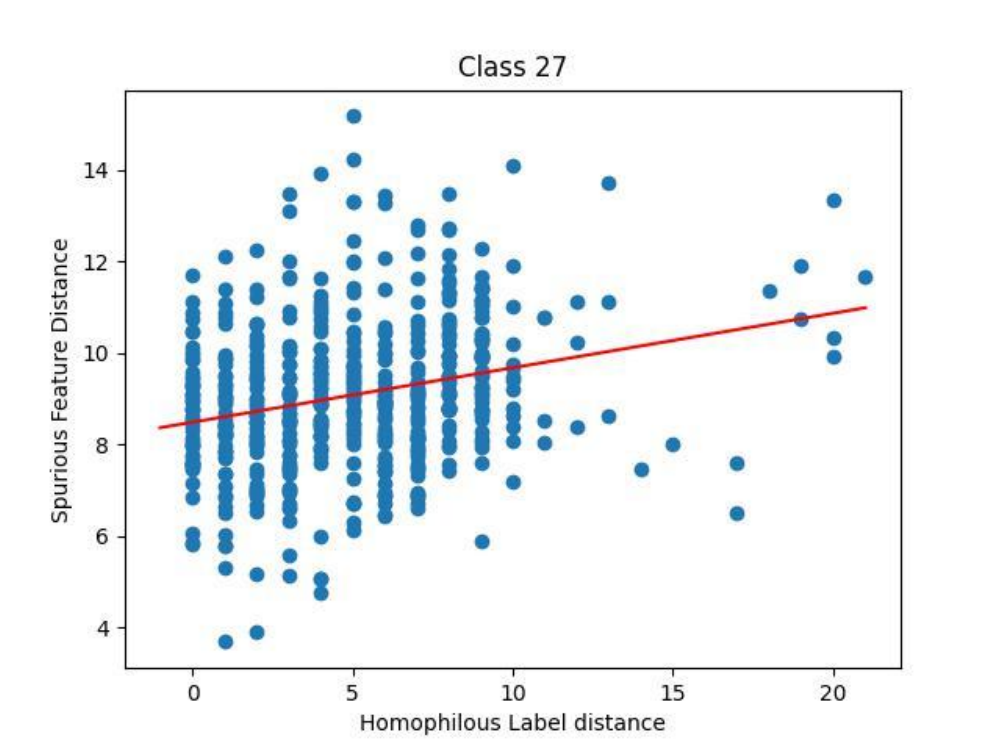}
    \end{subfigure}
    \begin{subfigure}{.19\textwidth}
        \centering
        \includegraphics[width=\linewidth]{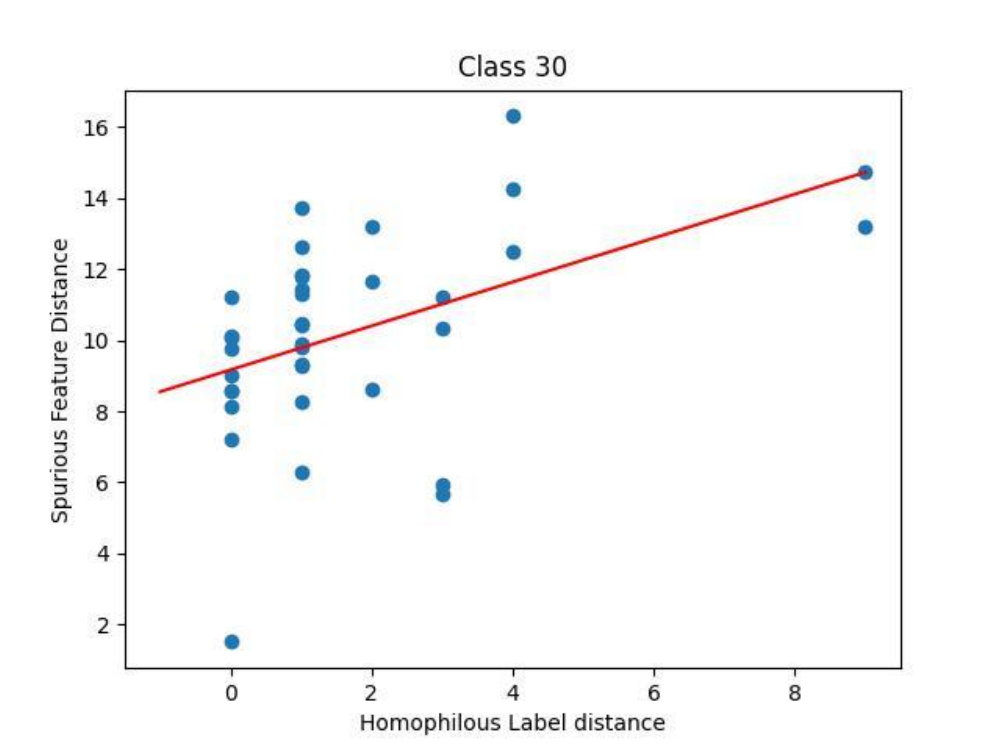}
    \end{subfigure}
    \begin{subfigure}{.19\textwidth}
        \centering
        \includegraphics[width=\linewidth]{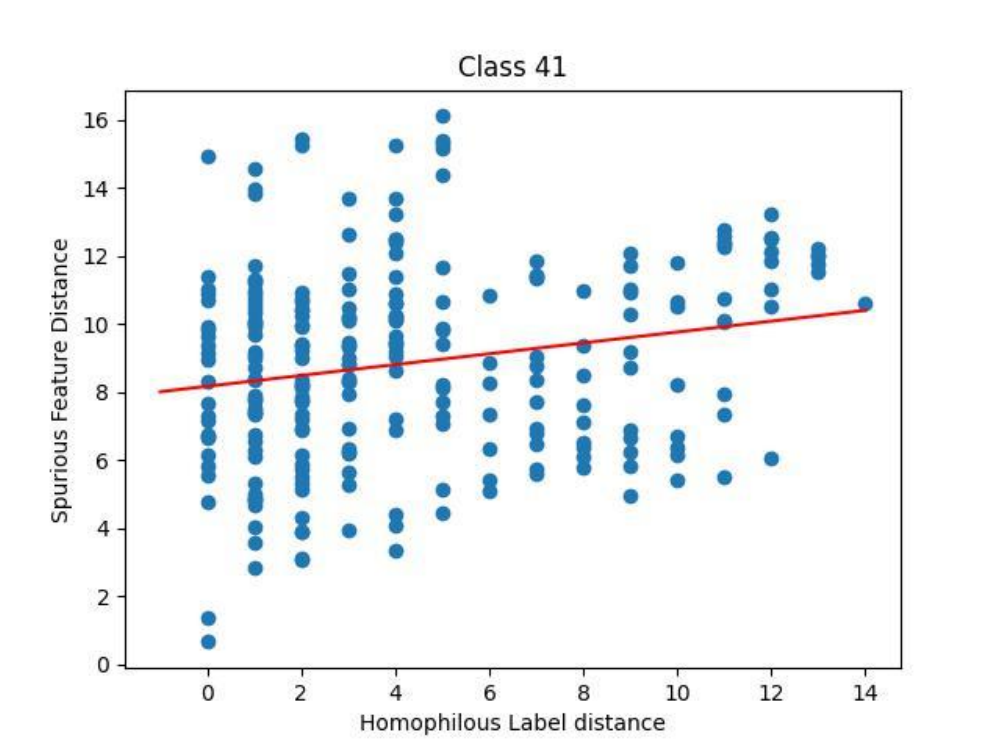}
    \end{subfigure}

    \begin{subfigure}{.19\textwidth}
        \centering
        \includegraphics[width=\linewidth]{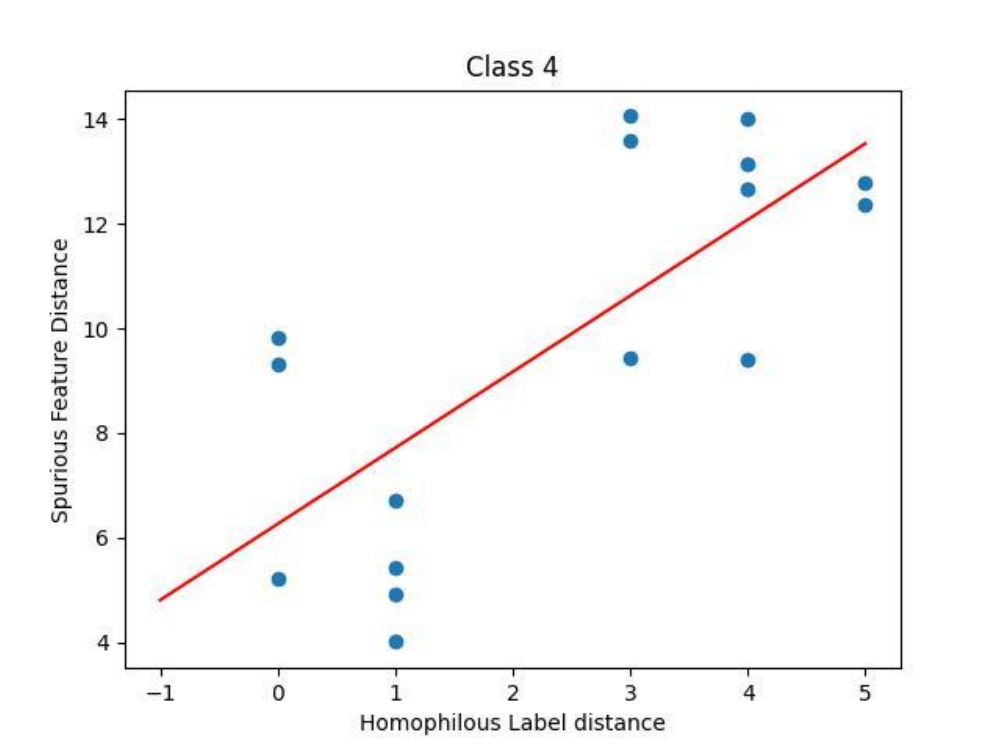}
    \end{subfigure}
    \begin{subfigure}{.19\textwidth}
        \centering
        \includegraphics[width=\linewidth]{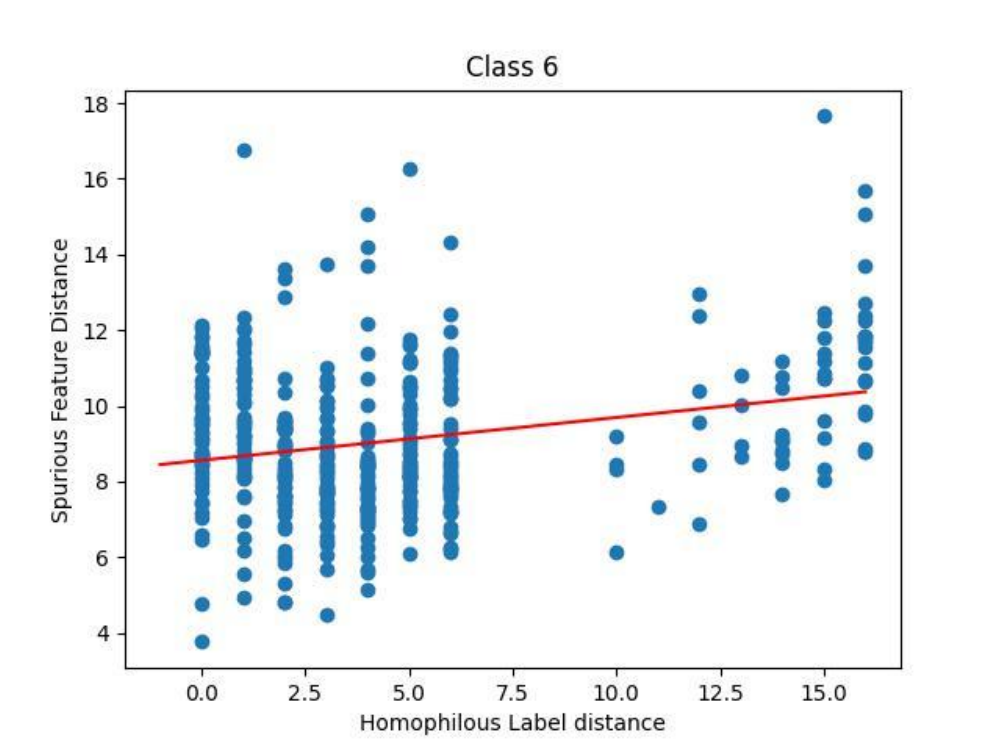}
    \end{subfigure}%
    \begin{subfigure}{.19\textwidth}
        \centering
        \includegraphics[width=\linewidth]{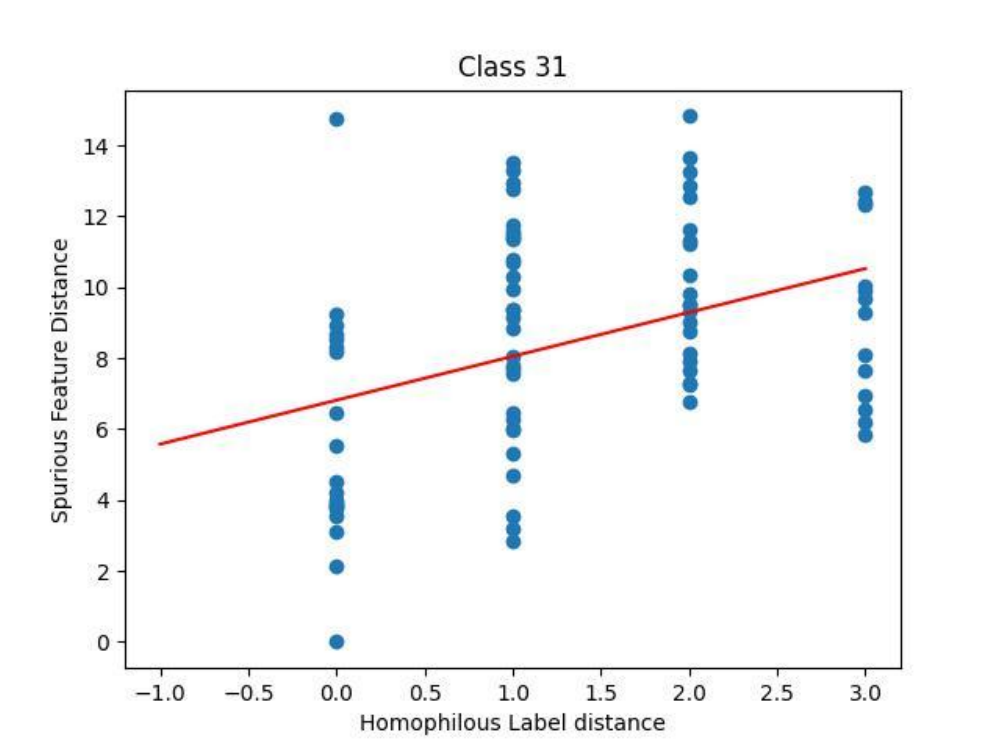}
    \end{subfigure}
    \begin{subfigure}{.19\textwidth}
        \centering
        \includegraphics[width=\linewidth]{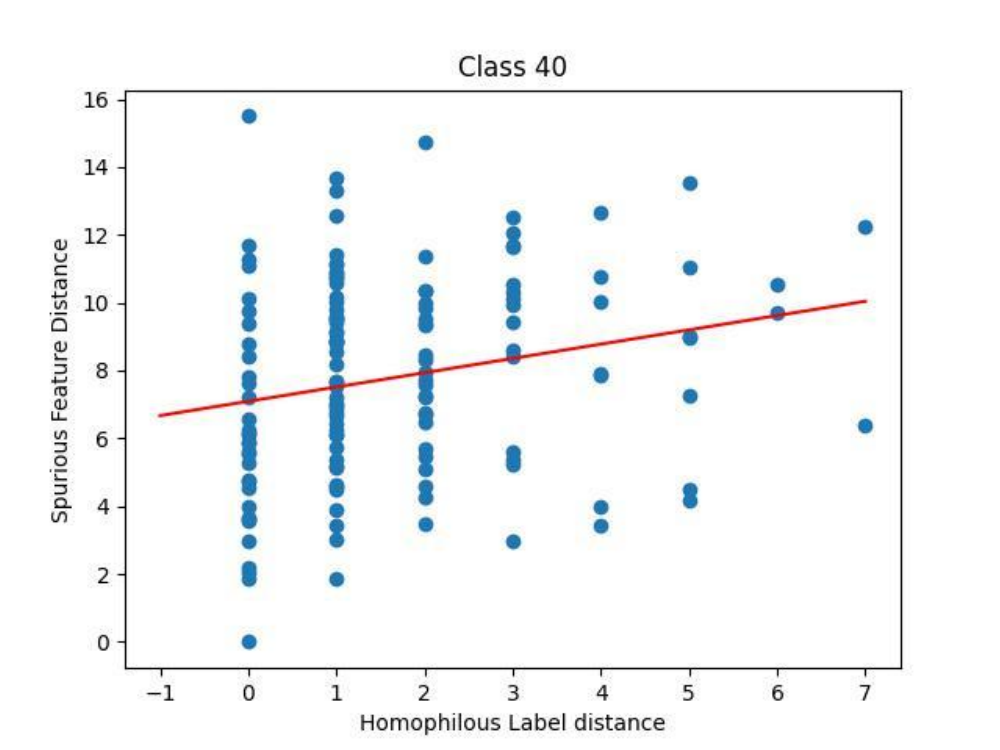}
    \end{subfigure}
    \begin{subfigure}{.19\textwidth}
        \centering
        \includegraphics[width=\linewidth]{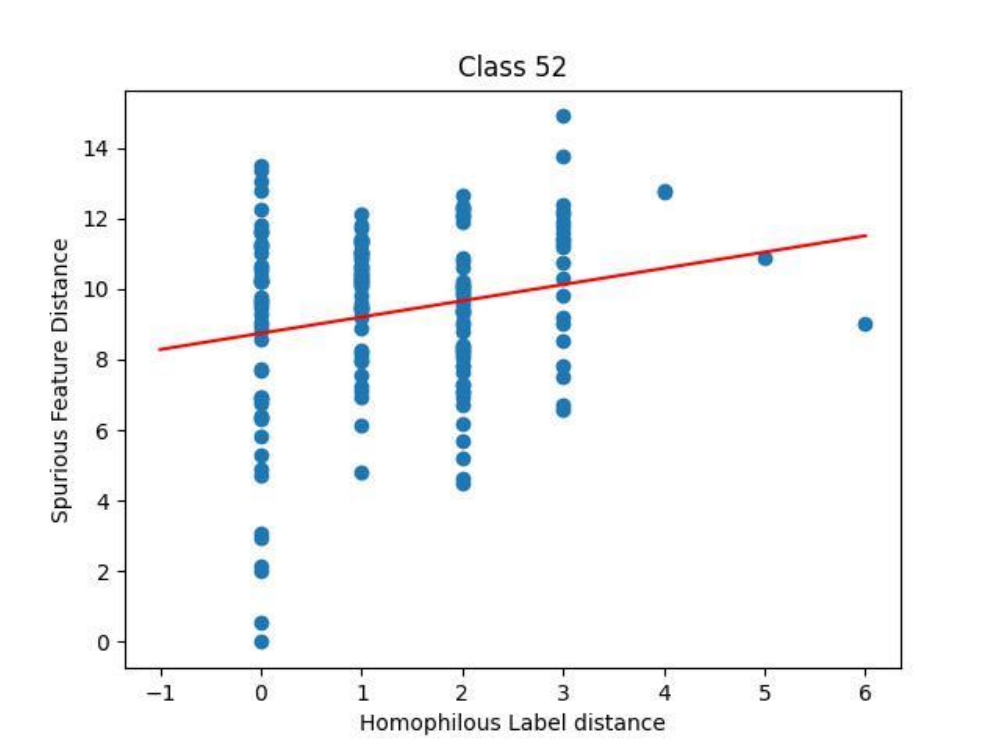}
    \end{subfigure}
    \caption{The relationship between the distance of invariant representations and discrepancy in the same-class neighbor ratio (the ratios in the figure are multiplied by node degree) on Cora \textit{degree}, \textbf{concept shift}. Line 1 to 4 are results of Cora \textit{word}+\textbf{covariate}, \textit{word}+\textbf{concept}, \textit{degree}+\textbf{covariate} and \textit{degree}+\textbf{concept}, respectively. Each subgraph marks a class, and each point in the subfigure represents a node pair. There is a positive correlation between the invariant feature distance and the difference in neighboring label ratio of the same class as the centered node.}
    \label{reflect_inv_cora}
\end{figure}

\subsection{Validation of the True Feature Generation Depth}
\label{depth}
For the theoretical model in Section \ref{theo_sec}, we assume that the number of layers of the GNN $L$ is greater than the depth of the causal pattern $k$. In this section, we empirically verify how large $k$ really is on real-world datasets. Specifically, we use GCN with different layers to predict the ground-truth label $Y$ on Cora and Arxiv datasets respectively (results are in Table \ref{cora_depth_inv} and \ref{arxiv_depth_inv}). As mentioned above, since a GCN with $L$ layers will aggregate features from $L$-hop neighbors for prediction, if the depth of the GCN is equal to the true generation depth, then the performance should be close to optimal. Therefore, we use the layer number that yields the optimal empirical performance (denoted as $L^*$) to approximate $k$. \textbf{ We find that the $L^*\leq 4$ in most cases.}  This indicates that our assumptions $L\leq k$ hold easily.

\subsection{Time Cost of CIA-LRA}
\label{app:time-cost}
 To show the running time of CIA-LRA, we show the time cost to reach the best test accuracy on our largest dataset Arxiv (with 50k~60k nodes). The results are in Table \ref{tab:time-cost} below. The time cost of CIA-LRA is comparable to baseline methods. 
 
\begin{table}
    \centering
    \caption{Time cost (seconds) to achieve optimal test performance on Arxiv using GAT on a single RTX 3090 GPU.}
    \resizebox{\linewidth}{!}{
    \begin{tabular}{ccccccccc}
    \toprule
         & ERM & Coral & Mixup & EERM & SRGNN & GTrans & CIT & CIA-LRA (6 hops)  \\
         \midrule
        Arxiv degree covariate & 74 & 551 & 758 & OOM & 34887 & OOM & OOM & 1248  \\
        \midrule
        Arxiv degree concept & 30 & 360 & 747 & OOM & 3960 & OOM & OOM & 1132  \\
        \midrule
        Arxiv time covariate & 46 & 246 & 1207 & OOM & 1993 & OOM & OOM & 292  \\
        \midrule
        Arxiv time concept & 440 & 1481 & 272 & OOM & 11628 & OOM & OOM & 989  \\
        \bottomrule
    \end{tabular}
    }
    \label{tab:time-cost}
\end{table}

\begin{table}[htbp]
\begin{subtable}{.5\linewidth}
    \centering
    \caption{OOD accuracy on causal prediction (\%) of GCN with different numbers of layers on Arxiv.}
    \resizebox{\linewidth}{!}{
        \begin{tabular}{cccccc}
            \toprule
            Dataset & Shift  & \(l=2\) & \(L=3\) & \(L=4\) &L=5 \\
            \midrule
           \multirow{2}{*}{Arxiv (degree)}& covariate &57.28(0.09) &58.92(0.14) &\textbf{60.18(0.41)} &60.17(0.12) \\
            \cmidrule{2-6}
           & concept &63.32(0.19) &62.92(0.21)&\textbf{65.41(0.13)}&63.93(0.58) \\
            \midrule
           \multirow{2}{*}{Arxiv (time)} & covariate &71.17(0.21)&70.98(0.20)&\textbf{71.71(0.21)}&70.84(0.11) \\
           \cmidrule{2-6}
           & concept &65.14(0.12)&67.36(0.07)&65.20(0.26)&\textbf{67.49(0.05)}\\
            \bottomrule
            \label{cora_depth_inv}
        \end{tabular}
    }
\end{subtable}
\begin{subtable}{.5\linewidth}
\centering
    \caption{OOD accuracy (\%) of GCN with different numbers of layers on Cora.}
    \resizebox{\linewidth}{!}{
        \begin{tabular}{cccccc}
            \toprule
            Dataset & Shift & \(L=1\) & \(l=2\) & \(L=3\) & \(L=4\) \\
            \midrule
           \multirow{2}{*}{Cora (degree)}& covariate &\textbf{59.04(0.15)} &58.44(0.44) &55.78(0.52) &55.15(0.24) \\
            \cmidrule{2-6}
           & concept &\textbf{62.88(0.34)}&61.53(0.48)&60.24(0.40)&60.51(0.17)  \\
            \midrule
           \multirow{2}{*}{Cora (word)} & covariate &64.05(0.18)&\textbf{65.81(0.12)}&65.07(0.52)&64.58(0.10) \\
           \cmidrule{2-6}
           & concept &64.32(0.15)&\textbf{64.85(0.10)}&64.61(0.11)&64.16(0.23)\\
            \bottomrule
             \label{arxiv_depth_inv}
        \end{tabular}
    }
\end{subtable}
\end{table}

\section{Detailed training procedure}
Table \ref{code} shows the detailed training procedure (pseudo code) of CIA-LRA. We use the same GNN encoder for the invariant subgraph extractor. Empirically, we add CIA or CIA-LRA after one epoch of ERM training.
\label{algo_sec}
\begin{algorithm}[htbp]
  \caption{Detailed Training Procedure of CIA-LRA} 
  \label{code} 
  \begin{algorithmic}[1] %
\REQUIRE~~\\ %
  A labeled training graph $\mathcal{G}=(A,X, Y)$, a GNN $f_{\Theta}$, and invariant subgraph generator GNN $f_{\theta_m}$. The number of hops $t$, CIA-LRA weight $\lambda$, the number of classes $C$, total iterations $T$, model learning rate $r_1$, invariant subgraph generator learning rate $r_2$, the number of $f_{\Theta}$'s layers $L$\\
  \ENSURE ~~\\%
  Updated model $f_{\Theta}$ with parameter $\Theta$.
  \FOR{iterations in $[1,2,...,T]$}
    \STATE Randomly sample a subgraph $A'\in \mathbb{R}^{N\times N}$ from $A\in \mathbb{R}^{N_0\times N_0}$.
    \STATE Compute and apply the edge mask according to Equation (\ref{edge_mask}) to obtain the masked adjacency matrix $A_{m}\leftarrow A'\odot M\in \mathbb{R}^{N\times N}$.
    \STATE Initialize $\mathcal{L}_{\text{CIA-LRA}}\leftarrow 0$.   
    \STATE Calculate the node representations $\phi(A_m,X)\in \mathbb{R}^{N\times D}$.
    \STATE \#\#\# Calculate $A^t$, where the $(i,j)$-th element of $A^t$ equals the length of the shortest path from node $i$ to $j$ if the length is less than $t$ else infinity:
    \STATE Initialize $A^t(i,j)\leftarrow \text{Inf}$ if $i\neq j$, $A^t(i,i)\leftarrow 1$, $A_{tem}\leftarrow A_m$
    \FOR{hop $h$ in $[1,2,...,t]$} 
        \STATE $A_{tem}\leftarrow A_mA_{tem}$
        \IF {$A_{tem}(i,j)>0$ \AND $A_{tem}(i,j)<A^t(i,j)$} 
            \STATE $A^t(i,j)\leftarrow h$
        \ENDIF
    \ENDFOR
    \STATE \#\#\# Compute the ratio of neighbored nodes of each class:
    \STATE Compute the normalized adjacency matrix $\bar{A}$, where $\bar{A}$'s $i$-th row $\bar{A}_i\leftarrow {A_{m}}_i/D_i$, ${A_{m}}_i$ is the $i$-th row of $A_m$ and $D_i\in \mathbb{R}$ is the degree of node $i$.
    \STATE Initialize the neighbored label ratio $R \leftarrow Y \in \mathbb{R}^{N\times C}$, where $R(i,c)$ is the ratio of node $i$'s neighbors of class $c$ within a $L$-hop range, $Y$ are the one hop labels.
    \FOR{$l$ in $[1,2,...,L]$}
        \STATE $R\leftarrow \bar{A}R$
    \ENDFOR
    \FOR{$c$ in $[1,2,...,C]$}
        \STATE Sample the nodes of class $c$ from $A^t$ and form $A^t_c$. Use $A^t_c$ to screen for pairs of nodes not exceeding a distance of $t$ hops $\Omega_c(t)$.
        \STATE Compute CIA-LRA loss of class $c$: $\mathcal{L}_{\text{CIA-LRA}}^c$ according to  Equation (\ref{LoReCIA}) using $\Omega_c(t)$, $R$, $A^t_c$ and  $\phi(A_m,X)$:
        \STATE $\mathcal{L}_{\text{CIA-LRA}}\leftarrow\mathcal{L}_{\text{CIA-LRA}}+\mathcal{L}_{\text{CIA-LRA}}^c$
    \ENDFOR
    \STATE Compute final loss $\mathcal{L}\leftarrow\mathcal{L}_{\text{ce}}(f_\Theta(A,X),Y)+\lambda \mathcal{L}_{\text{CIA-LRA}}$, $\mathcal{L}_{\text{ce}}$ is the cross-entropy loss
    \STATE Update model parameters $\Theta\leftarrow\Theta-r\nabla_\Theta\mathcal{L}$, $\theta_m\leftarrow\Theta_m-r\nabla_{\theta_m}\mathcal{L}$
  \ENDFOR
  \end{algorithmic}
\end{algorithm}

\section{Additional Discussion of Theoretical Settings and Results}
\subsection{Detailed Setup of the Theoretical Model in Section \ref{theo_sec}}
\label{theo_model_detail_sec}
\textbf{The proposed data generation process.} In the theoretical model of Equation (\ref{data_gene}), each dimension of  $n_1\in \mathbb{R}^{N^e\times 1}$ and $n_2\in\mathbb{R}^{N^e\times 1}$ are i.i.d, following a standard Gaussian distribution. $\epsilon^e \in \mathbb{R}^{N^e\times 1}$ is an environment spurious variable.  $\epsilon^e_i$ (each dimension of $\epsilon^e$) are independent random variables, $i=1,...,N^e$. We further assume the cross-environment expectation $\mathbb{E}_e[\epsilon^e]=\mathbf{0}$ and cross-environment variance $\mathbb{E}_e[\epsilon_i^e]=\sigma^2$, $i=1,...,N^e$ for brevity.

\textbf{The considered multi-layer GNN.}
In the analyzed GNN of Equation (\ref{GNN}), we simplify the classifier to an identity mapping. Such simplification has been adopted by various previous theoretical works on graphs \citep{wu2022non,tang2023towards}. We assume $L\geq k$ to ensure the model has enough capacity to learn invariant features.  We verify this assumption by using GCNs with different numbers of layers to predict the ground-truth labels (see Appendix \ref{depth}).

\subsection{Discussion of the Structural Feature Considered in the Theoretical Model and Justification for the Choice of the Analyzed GNN}
\label{sec_structural_shifts}
\textbf{Structural Features and Structural Shifts Considered in Section \ref{theo_model}.} To reflect reality as much as possible, it is necessary to consider both nodal and structural invariant and spurious features in the theoretical model. As mentioned in Section \ref{theo_model}, we model the invariant structural feature as the structure of the $k$-hop ego-subgraph. A natural question is raised here:
\begin{quote}
\emph{Can we find other ways to define the invariant/spurious structural features?}
\end{quote}
The answer is yes. For example, the invariant structure can be modeled as the subgraph of the ego-graph of a node, following \cite{li2023invariant}. However, it is fundamentally impossible for GNNs using mean aggregation (like GCN) to learn such causal structures. This is because such GNNs will assign fixed weights to each neighboring node feature, and they can't split the causal substructure from the neighbored ego-sgraph. Therefore, we make the causal structure feasible for GCN-like GNNs to learn by defining the causal structure as the whole $k$-hop neighboring ego-graph, rather than a subgraph, and show that OOD failure can still happen (Theorem \ref{prop1}). Then, under this setting, the remaining challenge becomes identifying the true $k$ by optimizing the shallow layer GNN parameters. However, in real practice, the invariant causal pattern may still be an ego-subgraph. This can be reflected in the performance gain of the invariant subgraph extractor used in CIA-LRA.

\textbf{Why Do We Choose Such a GNN in Section \ref{theo_model}?} From the above analysis, we show that such a choice is a compromise solution between the case of GCN (that can only extract a whole ego-graph) and GAT-like GNNs (that can extract a subgraph from an ego-graph). Although in this GNN each neighbored node is solely assigned the same weight, the shallow layer parameters can be optimized to realize the aggregation of different depths to capture the causal structures of different depths.

\subsection{Discussion of the Failure Solution for GNNs of VREx and IRMv1}
\label{spec_fail_sec}
In Theorem \ref{prop1_formal} and \ref{prop2_formal} (the formal version Theorem \ref{prop1}), we show that VREx and IRMv1 could induce a model that uses spurious features. Now we'll give an intuitive explanation of this failure mode. When the lower-layer parameters of the GNN ${\theta_1^1}^{(l)}$, ${\theta_2^1}^{(l)}$, ${\theta_1^2}^{(l)}$, ${\theta_2^2}^{(l)}$ take the specific solution $\Theta_0$ in Equation (\ref{spec_fail_solution}), we have 
\begin{equation}
\label{equation_degradation}
    H_1^{(L)}=\frac{\partial H_1^{(L)}}{\partial {\theta^i_1}^{(l)}}=\tilde{{A^e}}^{{s}}X_1, ~ H_2^{(L)}=\frac{\partial H_2^{(L)}}{\partial {\theta^i_2}^{(l)}}=\tilde{{A^e}}^{k+m}X_1,
\end{equation}
holds for $i=1,2, ~ l=1,...,L-1$ and every environment $e$. Thus, we get
\begin{equation}
    \begin{aligned}
        \frac{\partial \mathcal{L}}{\partial \theta_1}&=\frac{\partial \mathcal{L}}{\partial (H_1^{(L)}\theta_1)}\frac{\partial (H_1^{(L)} \theta_1)}{\partial \theta_1}=\frac{\partial \mathcal{L}}{\partial (H_1^{(L)}\theta_1)}H_1^{(L)}\\
        &\overset{(*)}{=}\frac{\partial \mathcal{L}}{\partial (H_1^{(L)}\theta_1)}\frac{\partial (H_1^{(L)})}{\partial {\theta^i_1}^{(l)}}=\frac{\partial \mathcal{L}}{\partial {\theta^i_1}^{(l)}}\frac{1}{\theta_1},\\
        &i=1,2,\quad l=1,...,L-1
    \end{aligned}
\end{equation}

$(*)$ is because of Equation (\ref{equation_degradation}). Therefore, $\frac{\partial \mathcal{L}}{\partial \theta_1}=0 \Rightarrow \frac{\partial \mathcal{L}}{\partial {\theta^i_1}^{(l)}}=0$. The same is true for $\frac{\partial \mathcal{L}}{\partial \theta_2}$ and $\frac{\partial \mathcal{L}}{\partial {\theta^i_2}^{(l)}}$. This means the solution of the top-level parameters $\theta_1$ and $\theta_2$ of the GNN will only be constrained by two equations, $\frac{\partial \mathcal{L}}{\partial \theta_1}=0$ and $\frac{\partial \mathcal{L}}{\partial \theta_2}=0$, rather than be constrained by all gradient functions $\frac{\partial \mathcal{L}}{\partial \theta^j_i}=0, ~i=1,2$. By analyzing the specific loss of VREx and IRMv1, we conclude that they will induce a non-zero $\theta_2$.

Note that the failure solution $\Theta_0$ here is not the unique one, we choose $\Theta_0$ just for the elegant expression and to better convey the intuition. In effect, the conclusion $\frac{\partial \mathcal{L}}{\partial \theta_1}=0 \Rightarrow \frac{\partial \mathcal{L}}{\partial {\theta^i_1}^{(l)}}=0$ and $\frac{\partial \mathcal{L}}{\partial \theta_2}=0 \Rightarrow \frac{\partial \mathcal{L}}{\partial {\theta^i_2}^{(l)}}=0$ holds as long as the lower-layer aggregation parameters satisfy
\begin{equation}
    \label{f2}
        \Theta'_0=\left\{
        \begin{array}{l}
            {\theta^1_1}^{(l)}=1, {\theta^2_1}^{(l)}=1, \quad l=L-1, ..., L-{s_1}+1  \\
            {\theta^1_1}^{(l)}=0, {\theta^2_1}^{(l)}=1, \quad l=L-{s_1}, L-{s_1}-1, ..., 1 \\
            {\theta^1_2}^{(l)}=1, {\theta^2_2}^{(l)}=1, \quad l=L-1, ..., L-{s_2}+1  \\
            {\theta^1_2}^{(l)}=0, {\theta^2_2}^{(l)}=1, \quad l=L-{s_2}, L-{s_2}-1, ..., 1 \\
        \end{array}
        \right., ~\text{for some $s_1,~s_2\in \mathbb{N}^+,~ 1<s_1\leq L,~ 1<s_2 \leq L$},
\end{equation}i.e., we don't require the spurious branch of the GNN to be identity mapping $I$ as equation (\ref{f1}) does.

 This failure mode can happen to GCN (all $\theta^i_j=1$, $i=1,2$, $j=1,2$) and also for GAT.

\subsection{The Superiority of the Proposed CSBM-OOD}
\label{discussion_CSBM}
Our CSBM-OOD in Section introduces several advancements over the conventional CSBMs \citep{ma2021subgroup, mao2023demystifying}: 1) It supports multi-class classification, extending beyond the binary classification framework of traditional CSBMs; 2) It accommodates unique neighboring label distributions for each node, in contrast to the traditional models that assume a uniform class-shared homophily/heterophily ratio across all nodes; 3) our model integrates OOD shifts, while traditional ones don't.

\subsection{Tightness of the Error Bound of Theorem \ref{bound}}
When there are no distributional shifts in spurious node features and heterophilic neighborhood distribution between training and test environments, the terms (a)-(d) in Eq. (\ref{eq:bound-formal}) becomes zero, and the upper bound becomes $\widehat{\mathcal{L}}_{e^{\text{tr}}}^\gamma(\tilde{h})+const=\widehat{\mathcal{L}}_{e^{\text{tr}}}^\gamma(\tilde{h})+\frac{1}{N_{e^{\text{tr}}}^{1-2 \alpha}}+\frac{1}{N_{e^{\text{tr}}}^{2 \alpha}} \ln \frac{L C\left(2 B_{e^{\text{te}}}\right)^{1 / L}}{\gamma^{1 / L} \delta}$, i.e., our bound only larger than the ideal error $\widehat{\mathcal{L}}_{e^{\text{tr}}}^\gamma(\tilde{h})$ by a constant $const$. When the number of training samples $N_{e^{\text{tr}}}$ is large, $const$ will be small enough and can be negligible. Hence, the tightness of our bound is guaranteed.

\section{Proofs of the Theoretical Results}
\label{proofs_sec}
\subsection{Proofs of the Concept Shift Case Presented in the Main Text}
In this section, we give proof of the propositions of the concept shift model presented in the main text.
\subsubsection{Proof of the non-Graph Success Case of VREx and IRMv1 under Concept Shift}
\label{proof_VREx_IRM_suc}
We restate Proposition \ref{prop_VREx_IRM_suc} as Proposition \ref{prop_VREx_IRM_suc_appdix} below.

\begin{proposition}
\label{prop_VREx_IRM_suc_appdix}
    For the non-graph version of the SCM in Equation (\ref{data_gene}), 
    \begin{equation}
        Y^e=X_1+n_1,~X_2^e=Y^e+n_2+{\epsilon^e},    
    \end{equation}
    VREx and IRMv1 will learn invariant features when using a 1-layer linear network: $f(X)=\theta_1 X_1 + \theta_2 X_2$.
\end{proposition}

\begin{proof}
        \textbf{VREx.}  Denote $X_1 \theta_1+X_2^e\theta_2-X_1-n_1$ as $l_e$.  The variance of loss across environments is:
        \begin{equation}
            \begin{aligned}
                \mathbb{V}_e[R(e)] & =\mathbb{E}_e[R^2(e)]-\mathbb{E}_e^2[R(e)]                                                                                                             \\
                                   & =\mathbb{E}_e\left[\left(\mathbb{E}_{n_1,n_2}\left\|X_1\theta_1+(X_1+n_1+n_2+\epsilon^e)\theta_2-X_1-n_1\right\|^2_2\right)^2\right]                 \\
                                   & -\mathbb{E}_e^2\left[\mathbb{E}_{n_1,n_2}\left\|X_1\theta_1+(X_1+n_1+n_2+\epsilon^e)\theta_2-(X_1-n_1\right\|^2_2\right].                             \\
                                   & =\mathbb{E}_e\left[ \mathbb{E}_{n_1,n_2}\left[(l_e^\top l_e)^2\right]\right]-\mathbb{E}^2_e\left[\mathbb{E}_{n_1,n_2}\left[l_e^\top l_e\right]\right].
            \end{aligned}
        \end{equation}

        Take the derivative of $\mathbb{V}_e[R(e)]$ with respect to $\theta_1$:
        \begin{equation}
            \label{eqq1}
            \begin{aligned}
                \frac{\partial \mathbb{V}_e[R(e)]}{\partial \theta_1} & =\mathbb{E}_e\left[2\mathbb{E}_{n_1,n_2}\left[l_e^\top l_e\right]\mathbb{E}_{n_1,n_2}\left[2l_e^\top X_1\right]\right]                              \\
                                                                    & -2\mathbb{E}_e\left[\mathbb{E}_{n_1,n_2}\left[l_e^\top l_e\right]\right]\mathbb{E}_e\left[\mathbb{E}_{n_1,n_2}\left[ 2 l_e^\top X_1 \right] \right]
            \end{aligned}
        \end{equation}

        Using the fact that $\mathbb{E}_{n_1,n_2}[n_1]=\mathbb{E}_{n_1,n_2}[n_2]=\mathbf{0}$, 
        $\mathbb{E}_{n_1,n_2}[n_1^\top n_2]=\mathbb{E}_{n_1,n_2}[n_2^\top n_1]=0$ and 
        $\mathbb{E}_{n_1,n_2}[n_1^\top n_1]=\mathbb{E}_{n_1,n_2}[n_2^\top n_2]=N^e$ if it is the noise from $e$,
         $ \mathbb{E}_{n_1, n_2}[n_1^\top \epsilon^e]=\mathbb{E}_{n_1, n_2}[n_2^\top \epsilon^e]=0 $ and  
        using the assumption that $ \mathbb{E}_{e}[(\epsilon^e_i)^2 ]=\sigma^2$, we have
        \begin{equation}
            \label{eqq2}
            \begin{aligned}
                &\mathbb{E}_e\left[2\mathbb{E}_{n_1,n_2}\left[l_e^\top l_e\right]\mathbb{E}_{n_1,n_2}\left[2l_e^\top X_1\right]\right]\\
                =&{\theta_1}^3 {X_1}^4 + 3 {\theta_1}^2{\theta_2} {X_1}^4 + \theta_1 {\theta_2}^2 [3 X_1^2 (X_1^2+\sigma^2) + 2X_1^2\mathbb{E}_e[N^e]]\\
                +&\theta_2^3 X_1 (X_1^3+X_1 \sigma^2 + \mathbb{E}_e[{\epsilon^e}^\top \epsilon^e \epsilon^e]+ 2X_1 \mathbb{E}_e[N^e])\\
                -&3 \theta_1^2 X_1^4 - \theta_1 \theta_2 (6X_1^4 + 2 X_1^2 \mathbb{E}_e[N^e]) - \theta_2^2(3 X_1^2(X_1^2 + \sigma^2)+ 4 X_1^2 \mathbb{E}_e[N^e])\\
                +&\theta_1 (3X_1^4 + X_1^2 \mathbb{E}_e[N^e]0) + 3 \theta_2 (X_1^4 + X_1^2 \mathbb{E}_e [N^e]) -X_1^2 (X_1^2 + \mathbb{E}_e[N^e] )
            \end{aligned}
        \end{equation}
        and
        \begin{equation}
            \label{eqq3}
            \begin{aligned}
                &\mathbb{E}_e\left[\mathbb{E}_{n_1,n_2}\left[l_e^\top l_e\right]\right]\mathbb{E}_e\left[\mathbb{E}_{n_1,n_2}\left[ 2 l_e^\top X_1 \right] \right]\\
                =&{\theta_1}^3 {X_1}^4 + 3 {\theta_1}^2{\theta_2} {X_1}^4 + \theta_1 {\theta_2}^2 [3 X_1^2 (X_1^2+\sigma^2) + 2X_1^2\mathbb{E}_e[N^e]]\\
                +&\theta_2^3 X_1 (X_1^3+X_1 \sigma^2 + 2X_1 \mathbb{E}_e[N^e])\\
                -&3 \theta_1^2 X_1^4 - \theta_1 \theta_2 (6X_1^4 + 2 X_1^2 \mathbb{E}_e[N^e]) - \theta_2^2(3 X_1^2(X_1^2 + \sigma^2)+ 4 X_1^2 \mathbb{E}_e[N^e])\\
                +&\theta_1 (3X_1^4 + X_1^2 \mathbb{E}_e[N^e]0) + 3 \theta_2 (X_1^4 + X_1^2 \mathbb{E}_e [N^e]) -X_1^2 (X_1^2 + \mathbb{E}_e[N^e] )
            \end{aligned}
        \end{equation}
        Plug Equation (\ref{eqq2}) and (\ref{eqq3}) back into Equation (\ref{eqq1}), we have 
        \begin{equation}
            \frac{\partial \mathbb{V}_e[R(e)]}{\partial \theta_1} = \theta_2^3 X_1^\top \mathbb{E}_e[{\epsilon^e}^\top \epsilon^e \epsilon^e] 
        \end{equation}

        Let $ \frac{\partial \mathbb{V}_e[R(e)]}{\partial \theta_1}=0$, we have $\theta_2=0$.

        Now we need to validate $\theta_2=0$ is also a solution to $ \frac{\partial \mathbb{V}_e[R(e)]}{\partial \theta_2}=0 $. 
        Let's calculate $\frac{\partial \mathbb{V}_e[R(e)]}{\partial \theta_2}$ when $ \theta_2=0 $:
        \begin{equation}
            \begin{aligned}
                \frac{\partial \mathbb{V}_e[R(e)]}{\partial \theta_2}& =\mathbb{E}_e\left[2\mathbb{E}_{n_1,n_2}\left[l_e^\top l_e\right]\mathbb{E}_{n_1,n_2}\left[2l_e^\top (X_1+n_1+n_2+\epsilon^e)\right]\right]                              \\
                & -2\mathbb{E}_e\left[\mathbb{E}_{n_1,n_2}\left[l_e^\top l_e\right]\right]\mathbb{E}_e\left[\mathbb{E}_{n_1,n_2}\left[ 2 l_e^\top (X_1+n_1+n_2+\epsilon^e) \right] \right]\\
                =&(\theta_1^2X_1^2-2 \theta_1 X_1^2 + X_1^2 + \mathbb{E}_e {N^e})(\theta_1 X_1^2 -X_1^2 - \mathbb{E}_e {N^e})\\
                -&(\theta_1^2X_1^2-2 \theta_1 X_1^2 + X_1^2 + \mathbb{E}_e {N^e})(\theta_1 X_1^2 -X_1^2 - \mathbb{E}_e {N^e})\\
                =&0
            .\end{aligned}
        \end{equation}
        So far, we have proved $\theta_2=0$ is the solution for VREx, hence it will learn invariant features. We finish the proof for VREx.

        \textbf{IRMv1.}  The objective of IRMv1 is $\mathbb{E}_e \|\nabla_w R(e) \|^2_2$. When IRMv1 loss is optimized to zero,
        we have $\nabla_w R(e)=0$ for all environments $e$.
        \begin{equation}
            \begin{aligned}
                    \nabla_w R(e)&=\mathbb{E}_{n_1,n_2}[2(\theta_1X_1+\theta_2X_2-(X_1 + n_1))^\top(\theta_1X_1 +\theta_2X_2)]\\
                    &=2((\theta_1)^2X_1^\top X_1 + (\theta_2)^2 (X_1^\top X_1 + {\epsilon^e}^\top{\epsilon^e}  +  2N^e) + 2 \theta_1 \theta_2 X_1^\top(X_1+\epsilon^e)\\
                    &-\theta_1X_1^\top X_1 - \theta_2(X_1^\top X_1 + X_1^\top \epsilon^e + N^e))
            \end{aligned}
        \end{equation}
        To realize $\nabla_w R(e)=0$ for all $e$, we must let $\theta_2=0$ (and consequently $\theta_1=1$), otherwise 
        the solution of $\theta_2$ will include terms related to $\epsilon^e$ and $N^e$ that vary with environments, and a single value $\theta_2$
        cannot fit all these values. Thus we finish the proof for IRMv1.
    \end{proof}

\subsubsection{Proof of the Failure Case on Graphs of VREx under Concept Shift}
We present the formal version of the VREx part in Theorem \ref{prop1} as Theorem \ref{prop1_formal} below.
\begin{theorem}
\label{prop1_formal}
    \textbf{(VREx will use spurious features on graphs under concept shift, formal)} Under the SCM of Equation (\ref{data_gene}), the objective $\min_\Theta \mathbb{V}_e[R(e)]$ has non-unique solutions for parameters of the GNN (\ref{GNN}) when part of the model parameters $\{{\theta^1_1}^{(l)},{\theta^2_1}^{(l)},{\theta^1_2}^{(l)},{\theta^2_2}^{(l)}\}$ take the values 
    \begin{equation}
    \label{f1}
        \Theta_0=\left\{
        \begin{array}{l}
            {\theta^1_1}^{(l)}=1, {\theta^2_1}^{(l)}=1, \quad l=L-1, ..., L-{s}+1  \\
            {\theta^1_1}^{(l)}=0, {\theta^2_1}^{(l)}=1, \quad l=L-{s}, L-{s}-1, ..., 1 \\
            {\theta^1_2}^{(l)}=0, {\theta^2_2}^{(l)}=1, \quad l=L-1, ..., 1       \\
        \end{array}
        \right.,
    \end{equation} for some $0<{s}<L$. Specifically, the VREx solutions of $\theta_1$ and $\theta_2$ are the sets of solutions of the cubic equation, some of which are spurious solutions that $\theta_2\neq 0$ (although $\theta_2=0$ is indeed one of the solutions, VREx is not guaranteed to reach this solution):
    \begin{equation}
        \left\{
        \begin{array}{l}
              (3c_1 \theta_1 \theta_2 + c_1 (\theta_2)^2 -2c_6 \theta_2) \sigma^2 
        - \mathbb{E}_e[N^e(2c_1(\theta_1+\theta_2)-c_6)]\sigma^2\theta_2 +c_7\theta_2  =0\\
        (\mathbb{E}_e[N^e(2c_1(\theta_1+\theta_2)-c_6)]\sigma^2\theta_2 -c_7)(c_3-c_4)\theta_2-[c_2(\theta_1+\theta_2)-c_5](\theta_2)^2=0
        \end{array}
        \right. .
    \end{equation}

    where $c_1=\mathbb{E}_e[(\tilde{{A^e}}^{{s}}X_1)^\top(\tilde{{A^e}}^{{s}}X_1)]$, 
$c_2=\mathbb{E}_e[N^e(\tilde{{A^e}}^{{s}}X_1)^\top\tilde{{A^e}}^{{s}}X_1]$, 
$c_3=\mathbb{E}_e[(\tilde{{A^e}}^{{s}}X_1)^\top\mathbf{1}]$,
$c_4=\mathbb{E}_e[((\tilde{{A^e}}^{k}X_1)^\top \mathbf{1}]$, 
$c_5=\mathbb{E}_e[N^e((\tilde{{A^e}}^{k}X_1)^\top \tilde{{A^e}}^{{s}}X_1+  \text{tr} ( (\tilde{{A^e}}^{k})^\top \tilde{{A^e}}^{k}) + N^e (1+\sigma^2))]$,
$c_6=\mathbb{E}_e[(\tilde{{A^e}}^{{s}}X_1)^\top(\tilde{{A^e}}^{k}X_1)]$,
$ c_7= \mathbb{E}_e[{\epsilon^e}^\top {\epsilon^e} {\epsilon^e}^\top (\tilde{{A^e}}^{{s}} X_1)]$.
\end{theorem}

\begin{proof}  \label{proof_VREx}
We will use some symbols to simplify the expression of the toy GNN. Denote $\tilde{A}^mn_1+n_2+{\epsilon}$ as ${\eta}$. Use the following notations to represent the components of the $L$-layer GNN model:
    \begin{equation}
    \label{def_C_Z}
        \begin{aligned}
            f_{\Theta}(A,X) & =H^{(L)}_1\theta_1+H^{(L)}_2\theta_2                                                                                                                                                                                                                                                                            \\
                            & =\underbrace{\left[{\theta^1_1}^{(L-1)}\bar{A}\left(...{\theta^1_1}^{(3)}\left({\theta^1_1}^{(2)}\bar{A}({\theta^1_1}^{(1)}\bar{A}+{\theta^2_1}^{(1)}\bar{I})X_1+{\theta^2_1}^{(2)}({\theta^1_1}^{(1)}\bar{A}+{\theta^2_1}^{(1)}\bar{I})X_1\right)+...\right)\right]}_{C_1}\theta_1                             \\
                            & +\underbrace{\left[{\theta^1_2}^{(L-1)}\bar{A}\left(...{\theta^1_2}^{(3)}\left({\theta^1_2}^{(2)}\bar{A}({\theta^1_2}^{(1)}\bar{A}+{\theta^2_2}^{(1)}\bar{I})\tilde{A}^{{k+m}}X_1+{\theta^2_2}^{(2)}({\theta^1_2}^{(1)}\bar{A}+{\theta^2_2}^{(1)}\bar{I})\tilde{A}^{{k+m}}X_1\right)+...\right)\right]}_{C_2}\theta_2 \\
                            & +\underbrace{\left[{\theta^1_2}^{(L-1)}\bar{A}\left(...{\theta^1_2}^{(3)}\left({\theta^1_2}^{(2)}\bar{A}({\theta^1_2}^{(1)}\bar{A}+{\theta^2_2}^{(1)}\bar{I}){\eta}+{\theta^2_2}^{(2)}({\theta^1_2}^{(1)}\bar{A}+{\theta^2_2}^{(1)}\bar{I}){\eta}\right)+...\right)\right]}_{Z}\theta_2                             \\
                            & =C_1\theta_1+(C_2+Z)\theta_2.
        \end{aligned}
    \end{equation}
    $C_1, C_2, Z\in \mathbb{R}^{N \times 1}$. 
    We use $C_1^e$, $C_2^e$, and $Z^e$ to denote the variables from the corresponding environment $e$.
    We further denote $C_2^e={C_2^e}' \tilde{{A^e}}^{{  s}}X_1$, $ Z^e= {C_2^e}' {\eta} $.

    Using these notations, the loss of environment $e$ is
    \begin{equation}
        \begin{aligned}
            R(e) & = \mathbb{E}_{n_1,n_2}\left[\left\|f_\Theta({A^e}, X^e)-Y^e\right\|_2^2\right]                                   \\
                 & =\mathbb{E}_{n_1,n_2}\left[\left\|C^e_1\theta_1+(C^e_2+Z^e)\theta_2-\tilde{{A^e}}^{k}X_1-n_1\right\|_2^2\right].
        \end{aligned}
    \end{equation} Denote the inner term $C^e_1\theta_1+(C^e_2+Z^e)\theta_2-\tilde{{A^e}}^{k}X_1-n_1$ as $l_e$.

    The variance of loss across environments is:
    \begin{equation}
        \begin{aligned}
            \mathbb{V}_e[R(e)] & =\mathbb{E}_e[R^2(e)]-\mathbb{E}_e^2[R(e)]                                                                                                             \\
                               & =\mathbb{E}_e\left[\left(\mathbb{E}_{n_1,n_2}\left\|C^e_1\theta_1+(C^e_2+Z^e)\theta_2-\tilde{{A^e}}^{k}X_1-n_1\right\|^2_2\right)^2\right]                 \\
                               & -\mathbb{E}_e^2\left[\mathbb{E}_{n_1,n_2}\left\|C^e_1\theta_1+(C^e_2+Z^e)\theta_2-\tilde{{A^e}}^{k}X_1-n_1\right\|^2_2\right].                             \\
                               & =\mathbb{E}_e\left[ \mathbb{E}_{n_1,n_2}\left[(l_e^\top l_e)^2\right]\right]-\mathbb{E}^2_e\left[\mathbb{E}_{n_1,n_2}\left[l_e^\top l_e\right]\right].
        \end{aligned}
    \end{equation}

    Take the derivative of $\mathbb{V}_e[R(e)]$ with respect to $\theta_1$:
    \begin{equation}
        \begin{aligned}
            \frac{\partial \mathbb{V}_e[R(e)]}{\partial \theta_1} & =\mathbb{E}_e\left[2\mathbb{E}_{n_1,n_2}\left[l_e^\top l_e\right]\mathbb{E}_{n_1,n_2}\left[2l_e^\top C_1^e\right]\right]                              \\
                                                                  & -2\mathbb{E}_e\left[\mathbb{E}_{n_1,n_2}\left[l_e^\top l_e\right]\right]\mathbb{E}_e\left[\mathbb{E}_{n_1,n_2}\left[ 2 l_e^\top C_1^e \right] \right]
        \end{aligned}
    \end{equation}
    
    Calculate the derivative by terms: 
    \begin{equation}
        \begin{aligned}
            \label{ll1}
            \mathbb{E}_{n_1,n_2}[l_e^\top l_e] & =\mathbb{E}_{n_1,n_2} [{C^e_1}^\top {C^e_1} (\theta_1)^2 +{C^e_1}^\top C^e_2 \theta_1 \theta_2 + {C^e_1}^\top Z^e \theta_1 \theta_2 -{C^e_1}^\top (\tilde{A^e})^kX_1 \theta_1 -{C^e_1}^\top n_1 \theta_1 \\
                                               & +{C_2^e}^\top C_1^e \theta_1 \theta_2 +{C_2^e}^\top {C_2^e} (\theta_2)^2 + {C_2^e}^\top Z^e \theta_1 \theta_2 - {C^e_2}^\top (\tilde{A^e})^k X_1 \theta_2 - {C^e_2}^\top n_1 \theta_2     \\
                                               & + {Z^e}^\top C^e_1\theta_1 \theta_2 + {Z^e}^\top C^e_2 (\theta_2)^2 + {Z^e}^\top {Z^e} (\theta_2)^2 - {Z^e}^\top (\tilde{A^e})^k X_1\theta_2 - {Z^e}^\top n_1 \theta_2\\
                                               & - ({(\tilde{A^e})^kX_1})^\top (C_1^e \theta_1 + C_2^e\theta_2 )-({(\tilde{A^e})^kX_1})^\top {Z^e} \theta_2
                                               +({(\tilde{A^e})^kX_1})^\top (\tilde{A^e})^kX_1 \\
                                               &+ ({(\tilde{A^e})^kX_1})^\top n_1 -n_1^\top (C_1^e \theta_1 + C_2^e \theta_2) -n_1^\top Z^e \theta_2
                                                + n_1^\top  (\tilde{A^e})^kX_1 +n_1^\top n_1 ]\\
        \end{aligned} 
    \end{equation}

    Since $n_1$ and $n_2$ are independent standard Gaussian noise, we have $\mathbb{E}_{n_1,n_2}[n_1]=\mathbb{E}_{n_1,n_2}[n_2]=\mathbf{0}$, 
    $\mathbb{E}_{n_1,n_2}[n_1^\top n_2]=\mathbb{E}_{n_1,n_2}[n_2^\top n_1]=0$ and 
    $\mathbb{E}_{n_1,n_2}[n_1^\top n_1]=\mathbb{E}_{n_1,n_2}[n_2^\top n_2]=N^e$ if it is the noise from $e$. 
    Also, since ${\epsilon^e}$ and 
    $n_1$, $n_2$ are independent, we have $ \mathbb{E}_{n_1, n_2}[n_1^\top {\epsilon^e}]=\mathbb{E}_{n_1, n_2}[n_2^\top {\epsilon^e}]=0 $.

    When \begin{equation}
        \label{sv}
        \left\{
        \begin{array}{l}
            {\theta^1_1}^{(l)}=1, {\theta^2_1}^{(l)}=1, \quad l=L-1, ..., L-{s}+1  \\
            {\theta^1_1}^{(l)}=0, {\theta^2_1}^{(l)}=1, \quad l=L-{s}, L-{s}-1, ..., 1 \\
            {\theta^1_2}^{(l)}=0, {\theta^2_2}^{(l)}=1, \quad l=L-1, ..., 1       \\
        \end{array}
        \right.,
    \end{equation} we have $ {C_2^e}'=I_{N^e}\in \mathbb{R}^{N^e \times N^e} $ and $ C_1^e=\tilde{{A^e}}^{{s}} X_1 $.  
    Consequently, we get $ \mathbb{E}_{n_1, n_2}[{Z^e}^\top n_1]= \text{tr}({C_2^e}' \tilde{{A^e}}^{k} )=\text{tr}(\tilde{{A^e}}^{k}) $, 
    $ \mathbb{E}_{n_1, n_2}[{Z^e}^\top Z^e] = \text{tr}\left((\tilde{{A^e}}^{k})^\top (\tilde{{A^e}}^{k})\right) + N^e +  {\epsilon^e}^\top {\epsilon^e}$.

    Use the above conclusions and rewrite Equation (\ref{ll1}) as (here we only plug in the value of ${C_2^e}'$):
    \begin{equation}
        \label{ll2}
        \begin{aligned}
            &\mathbb{E}_{n_1,n_2}[l_e^\top l_e] =\\
            &\left.\begin{array}{@{}l@{}}
            {C^e_1}^\top {C^e_1} (\theta_1)^2 +{C^e_1}^\top C^e_2 \theta_1 \theta_2 - {C^e_1}^\top (\tilde{A^e})^kX_1 \theta_1 +  {C_2^e}^\top C_1^e \theta_1 \theta_2 +{C_2^e}^\top {C_2^e} (\theta_2)^2 - {C^e_2}^\top (\tilde{A^e})^k X_1 \theta_2  \\
             + \text{tr}\left((\tilde{{A^e}}^{k})^\top (\tilde{{A^e}}^{k})\right) (\theta_2)^2 - ({(\tilde{A^e})^kX_1})^\top (C_1^e \theta_1 + C_2^e\theta_2 ) +({(\tilde{A^e})^kX_1})^\top (\tilde{A^e})^kX_1 + N^e \left(1+(\theta_2)^2 \right) \\
             - 2\text{tr}(\tilde{{A^e}}^{k})
            \end{array} \right\}  \text{($*$)}\\
            &\left.\begin{array}{@{}l@{}} + [{C^e_1}^\top {\epsilon^e}   + {C_2^e}^\top {\epsilon^e}  +  {\epsilon^e}^\top C_1^e] \theta_1 \theta_2 +  {\epsilon^e}^\top  {\epsilon^e}(\theta_2)^2 -2({(\tilde{A^e})^kX_1})^\top {\epsilon^e} \theta_2\\
            \end{array}\right\} \text{($**$)}
        ,\end{aligned}
    \end{equation}

    ($ * $) and ($ ** $) represent terms that are independent and associated with $ {\epsilon^e} $, respectively.
    Additionally, 

    \begin{equation}
        \begin{aligned}
            \label{lc1}
        \mathbb{E}_{n_1, n_2}[2 l_e^\top C_1^e]&=2 \left [ {C_1^e}^\top {C_1^e} \theta_1+
        {C_2^e}^\top {C_1^e} \theta_2 + ({C_2^e}' {\epsilon^e})^\top {C_1^e} \theta_2 - ((\tilde{A^e})^kX_1)^\top {C_1^e} \right ]\\
        &=2 \left [ {C_1^e}^\top {C_1^e} \theta_1+
        {C_2^e}^\top {C_1^e} \theta_2 +{\epsilon^e}^\top {C_1^e} \theta_2 - ((\tilde{A^e})^kX_1)^\top {C_1^e} \right ].
        \end{aligned}
    \end{equation}    

    Multiplying Equation (\ref{ll2}) and (\ref{lc1}) and take the expectation on $e$, 
    using the assumption that $ \mathbb{E}_{e}[({\epsilon^e}_i)^2 ]=\sigma^2$ ($ {\epsilon^e}_i $ is the $i$-th element of $ {\epsilon^e} $):
    \begin{equation}
        \label{eelllc}
        \begin{aligned}
            \mathbb{E}_{e}\left[ 2 \mathbb{ E}_{n_1, n_2}[l^\top_e l_e] \mathbb{ E}_{n_1, n_2}\left[2l_e^\top C_1\right] \right]
            &=4\mathbb{E}_e\left[ \text{($ * $)}\left( {C_1^e}^\top C_1^e \theta_1 + {C_2^e}^\top {C_2^e}' \theta_2 - ({(\tilde{A^e})^kX_1})^\top C_1^e \right) \right]\\
            &+ 4\mathbb{E}_e\left[  (\tilde{{A^e}}^{{s}} X_1)^\top \tilde{{A^e}}^{{s}} X_1 (3\theta_1 \theta_2+(\theta_2)^2)-2 (\tilde{{A^e}}^{{s}}X_1)^\top ((\tilde{A^e})^kX_1) \theta_2   \right]\theta_2 \sigma^2\\
            &+ 4\mathbb{E}_e [N^e {\epsilon^e}^\top {\epsilon^e} {\epsilon^e}^\top (\tilde{{A^e}}^{{s}} X_1)]\theta_2
        .\end{aligned}
    \end{equation}
    
    Next target is to compute $2\mathbb{E}_e [\mathbb{E}_{n_1,n_2}[l_e^\top l_e]]$ and $\mathbb{E}_e [\mathbb{E}_{n_1,n_2}[2l_e^\top C_1]]$.
    Since ${\epsilon^e}$ has zero mean, we have:

    \begin{equation}
        \label{eell}
        \begin{aligned}
            2 \mathbb{E}_e[\mathbb{E}_{n_1,n_2}[l_e^\top l_e]]=2\mathbb{E}_e[\text{($*$)}]+2\mathbb{E}_e[N^e](\theta_2)^2 \sigma^2
        \end{aligned}
    \end{equation}

    and

    \begin{equation}
        \label{eelc}
        \begin{aligned}
            \mathbb{E}_e[\mathbb{E}_{n_1,n_2}[2l_e^\top C_1^e]]=
            2 \mathbb{E}_e\left[ {C_1^e}^\top {C_1^e} \theta_1 +{C_2^e}^\top {C_1^e} \theta_2 - ({(\tilde{A^e})^kX_1})^\top C_1^e \right]
        .\end{aligned}
    \end{equation}

    Use Equation (\ref{eelllc}) (\ref{eell}) and (\ref{eelc}) and let $  \frac{\partial \mathbb{V}_e[R(e)]}{\partial \theta_1}=0 $, we have:

    \begin{equation}
    \label{pt1}
        \begin{aligned}
            &\mathbb{E}_e\left[ 3 (\tilde{{A^e}}^{{s}} X_1)^\top (\tilde{{A^e}}^{{s}} X_1)(\theta_1 \theta_2 + \frac{1}{3}(\theta_2)^2)- 2 (\tilde{{A^e}}^{{s}} X_1)^\top ((\tilde{A^e})^kX_1)\theta_2 \right] \sigma^2 
            + \mathbb{E}_e[{\epsilon^e}^\top {\epsilon^e} {\epsilon^e}^\top (\tilde{{A^e}}^{{s}} X_1)]\theta_2 \\
             -&\mathbb{E}_e[ N^e]\mathbb{E}_e\left[2 (\tilde{{A^e}}^{{s}} X_1)^\top (\tilde{{A^e}}^{{s}} X_1)(\theta_1+\theta_2)-((\tilde{A^e})^kX_1)^\top C_1^e\right] \theta_2 \sigma^2=0
        .\end{aligned}
    \end{equation}

Now we start calculating the expression of $ \frac{\partial \mathbb{V}_e[R(e)]}{\partial \theta_2} $: 
\begin{equation}
    \begin{aligned}
        \frac{\partial \mathbb{V}_e[R(e)]}{\partial \theta_2}& =\mathbb{E}_e\left[2\mathbb{E}_{n_1,n_2}\left[l_e^\top l_e\right]\mathbb{E}_{n_1,n_2}\left[2l_e^\top (C_2+Z^e)\right]\right]                              \\
        & -2\mathbb{E}_e\left[\mathbb{E}_{n_1,n_2}\left[l_e^\top l_e\right]\right]\mathbb{E}_e\left[\mathbb{E}_{n_1,n_2}\left[ 2 l_e^\top (C_2^e+Z^e) \right] \right]
    .\end{aligned}
\end{equation}

Let $\frac{\partial \mathbb{V}_e[R(e)]}{\partial \theta_2}=0$:
\begin{equation}
    \label{pt2}
    \begin{aligned}
        &\mathbb{E}_e\left[ ({C_1^e}^\top {C_2^e}' + {C_2^e}^\top {C_2^e}' + {{C_2^e}'}^\top {C_1^e}^\top )\theta_1 \theta_2 + ({C_2^e}')^\top {C_2^e}' (\theta_2)^2- 2 ({(\tilde{A^e})^kX_1})^\top {C_2^e}' \theta_2 \right]\\
        &\mathbb{E}_e\left[({{C_2^e}'}^\top C_2^e \theta_2 - ({(\tilde{A^e})^kX_1})^\top {C_2^e}')\right]\sigma^2\\
        &-\mathbb{E}_e \left[ N^e\sigma^2 \left( {C_1^e}^\top C_2^e \theta_1 + {C_2^e}^\top C_2^e \theta_2 
       - ({(\tilde{A^e})^kX_1})^\top C_2^e + \text{tr}((\tilde{{A^e}}^{k})^\top \tilde{{A^e}}^{k}) +N^e+{{C^e_2}'}^\top {C_2^e}' \sigma^2 \right)(\theta_2)^2  \right]\\
        & = 0
    .\end{aligned}
\end{equation}

Plug Equation (\ref{pt1}) in (\ref{pt2}), we reach:
\begin{equation}
    \begin{aligned}
            &\left[\mathbb{E}_e \left[N^e (\tilde{{A^e}}^{{s}}X_1)^\top (\tilde{{A^e}}^{{s}}X_1)(\theta_1+\theta_2) -((\tilde{A^e})^kX_1)^\top C_1^e
           \right]\sigma^2-\mathbb{E}_e[  {\epsilon^e}^\top {\epsilon^e} {\epsilon^e}^\top (\tilde{{A^e}}^{{s}}X_1)] \right]\theta_2\\
           &\mathbb{E}_e\left( (\tilde{{A^e}}^{{s}}X_1)^\top \mathbf{1_{N^e}} \theta_2 - (\tilde{{A^e}}^{k}X_1)^\top \mathbf{1_{N^e}}\right) \\
            & - \mathbb{E}_e \left[N^e\left((\tilde{{A^e}}^{{s}}X_1)^\top\tilde{{A^e}}^{{s}}X_1(\theta_1+\theta_2)-(\tilde{{A^e}}^{k}X_1)^\top \tilde{{A^e}}^{{s}}X_1+  \text{tr} ( (\tilde{{A^e}}^{k})^\top \tilde{{A^e}}^{k}) + N^e (1+\sigma^2)  \right)   \right](\theta_2)^2\\
            & =0
    .\end{aligned}
\end{equation}
Let $c_1=\mathbb{E}_e[(\tilde{{A^e}}^{{s}}X_1)^\top(\tilde{{A^e}}^{{s}}X_1)]$, 
$c_2=\mathbb{E}_e[N^e(\tilde{{A^e}}^{{s}}X_1)^\top\tilde{{A^e}}^{{s}}X_1$], 
$c_3=\mathbb{E}_e[(\tilde{{A^e}}^{{s}}X_1)^\top\mathbf{1}]$,
$c_4=\mathbb{E}_e[((\tilde{{A^e}}^{k}X_1)^\top \mathbf{1}]$, 
$c_5=\mathbb{E}_e[N^e((\tilde{{A^e}}^{k}X_1)^\top \tilde{{A^e}}^{{s}}X_1+  \text{tr} ( (\tilde{{A^e}}^{k})^\top \tilde{{A^e}}^{k}) + N^e (1+\sigma^2))]$,
$c_6=\mathbb{E}_e[(\tilde{{A^e}}^{{s}}X_1)^\top(\tilde{{A^e}}^{k}X_1)]$,
$ c_7= \mathbb{E}_e[{\epsilon^e}^\top {\epsilon^e} {\epsilon^e}^\top (\tilde{{A^e}}^{{s}} X_1)]$,

we conclude that

\begin{equation}
    \label{t1t2solu}
    \left\{
    \begin{array}{l}
          (3c_1 \theta_1 \theta_2 + c_1 (\theta_2)^2 -2c_6 \theta_2) \sigma^2 
        - \mathbb{E}_e[N^e(2c_1(\theta_1+\theta_2)-c_6)]\sigma^2\theta_2 +c_7 \theta_2 =0\\
        (\mathbb{E}_e[N^e(2c_1(\theta_1+\theta_2)-c_6)]\sigma^2\theta_2 -c_7)(c_3-c_4)\theta_2-[c_2(\theta_1+\theta_2)-c_5](\theta_2)^2=0
    \end{array}
    \right. .
\end{equation}

As for the derivative respect to $ {\theta_1^1}^{(l)} $, $ {\theta_1^2}^{(l)} $,
$ {\theta_2^1}^{(l)} $,$ {\theta_2^2}^{(l)} $, when they take the special value in Equation (\ref{sv}),
we have $\frac{\partial \mathbb{V}_e[R(e)]}{\partial \theta_1}=0 \Rightarrow \frac{\partial \mathbb{V}_e[R(e)]}{\partial {\theta_1^1}^{(l)}}=\frac{\partial \mathbb{V}_e[R(e)]}{\partial {\theta_1^2}^{(l)}}=0$
and $\frac{\partial \mathbb{V}_e[R(e)]}{\partial \theta_2}=0 \Rightarrow \frac{\partial \mathbb{V}_e[R(e)]}{\partial {\theta_2^1}^{(l)}}=\frac{\partial \mathbb{V}_e[R(e)]}{\partial {\theta_2^2}^{(l)}}=0$, $l=1,...,L$.
So we conclude the solution induced by Equation (\ref{t1t2solu}) is the solution of the objective, and $\theta_2=0$ is not a valid solution.

\end{proof}

\subsubsection{Proof of the Failure Case on Graphs of IRMv1 under Concept Shift}
\label{proof_irm_con_sec}
We present the formal version of the IRM part in  Theorem \ref{prop1} as Theorem \ref{prop2_formal} below.
\begin{theorem}
\label{prop2_formal}
    \textbf{(IRMv1 will use spurious features on graphs  under concept shift, formal)} Under the SCM of Equation (\ref{data_gene}), there exists $s\in \mathbb{N}^+$ that satisfies $0<s<L$ and $s\neq k$ such that optimizing the IRMv1 objective $\min_\Theta \mathbb{E}_e[\|\nabla_{w|w=1.0} R(e)\|^2]$ will not lead to the invariant solution $\theta_2=0$ for parameters of the GNN (\ref{GNN}) when $\{{\theta^1_1}^{(l)},{\theta^2_1}^{(l)},{\theta^1_2}^{(l)},{\theta^2_2}^{(l)}\}$ take the special solution:
    \begin{equation}
        \Theta_0=\left\{
        \begin{array}{l}
            {\theta^1_1}^{(l)}=1, {\theta^2_1}^{(l)}=1, \quad l=L-1, ..., L-{s}+1                                                                                                                                                                                                                                                                                                      \\
            {\theta^1_1}^{(l)}=0, {\theta^2_1}^{(l)}=1, \quad l=L-{s}, L-{s}-1, ..., 1                                                                                                                                                                                                                                                                                                     \\
            {\theta^1_2}^{(l)}=0, {\theta^2_2}^{(l)}=1, \quad l=L-1, ..., 1                                                                                                                                                                                                                                                                                                           \\
        \end{array}
        \right..
    \end{equation}
\end{theorem}
\begin{proof}
    \label{proof_IRM}
    From the proof of non-graph IRMv1 case Appendix \ref{proof_VREx_IRM_suc} we know that when IRMv1 objective is optimized, 
    we have $\nabla_w R(e)=0$ for all $e$. For the graph case, the expected risk of environment $e$ is 
    \begin{equation}
        R(e)=\mathbb{E}_{n_1,n_2}[\|\theta_1 C_1^e + \theta_2 (C_2^e+Z^e) - (\tilde{A^e})^kX_1 - n_1\|_2^2 ],
    \end{equation} where the definition of $C^e_1$, $C^e_2$ and $Z^e$ follows Equation (\ref{def_C_Z}). Now let's check if 
    the invariant solution $\theta_2=0$ is a valid solution. If $\theta_2=0$ holds, then the following equation must hold for every environment $e$:
\begin{equation}
    \begin{aligned}
            \nabla_w R(e)&=\mathbb{E}_{n_1,n_2}[(\theta_1C_1^e + \theta_2 (C_2^e + Z^e) -(\tilde{A^e})^k X_1 - n_1)^\top(\theta_1 C_1^e + \theta_2 (C_2^e + Z^e))]\\
            &=\mathbb{E}_{n_1,n_2}[(\theta_1)^2 {C_1^e}^\top C_1^e + (\theta_2)^2 (C_2^e+Z^e)^\top(C_2^e+Z^e) + 2 \theta_1 \theta_2 {C_1^e}^\top(C_2^e+Z^e)\\
            & -\theta_1 {C_1^e}^\top ((\tilde{A^e})^kX_1 + n_1) - \theta_2 (C_2^e + Z^e)((\tilde{A^e})^kX_1 + n_1)]\\
            &=(\theta_1)^2 ((\tilde{A^e})^{s}X_1)^\top ((\tilde{A^e})^sX_1) - \theta_1 ((\tilde{A^e})^{k}X_1)^\top ((\tilde{A^e})^sX_1) 
    \end{aligned}
\end{equation}
When $s\neq k$, we have $ \theta_1=\frac{ ((\tilde{A^e})^{k}X_1)^\top ((\tilde{A^e})^sX_1)}{ ((\tilde{A^e})^{s}X_1)^\top ((\tilde{A^e})^sX_1) } $. The value of this solution of $\theta_1$ varies with environment $e$, and thus is not a valid solution. 

However, now we will show that optimizing IRMv1 does not necessarily lead to lower-layer parameters such that $s=k$. To reveal this, by taking the derivative of $\mathcal{L}_{\text{IRMv1}}$ w.r.t. $\theta_1$ and $\theta_2$ and let them $=0$, we can get two cubic equations:

\begin{equation}
\label{irm_con_solu_t1}
    \begin{aligned}
            \frac{\partial \mathcal{L}_{\text{IRMv1}}}{\partial \theta_1}
            =&\mathbb{E}_e[({C_1^e}^\top C_1^e (\theta_1)^2 + ({C_2^e}^\top C_2^e+2 {C_2^e}^\top \epsilon^e+2N^e + \epsilon^\top \epsilon)(\theta_2)^2\\
            +&({C_1^e}^\top C_2^e+ {C_1^e}^\top \epsilon^e)\theta_1 \theta_2 - (\tilde{A^e}^{k}X_1)^\top C_1^e \theta_1 - [(\tilde{A^e}^{k}X_1)^\top + N^e + n_1^\top \epsilon^e]\theta_2)\\
            &(2 {C_1^e}^\top {C_1^e}^\top \theta_1 + ({C_1^e}^\top C_2^e + {C_1^e}^\top \epsilon)\theta_2 - (\tilde{A^e}^{k}X_1)^\top C_1^e )]=0
    \end{aligned}
\end{equation}
and
\begin{equation}
\label{irm_con_solu_t2}
    \begin{aligned}
            \frac{\partial \mathcal{L}_{\text{IRMv1}}}{\partial \theta_2}
            =&\mathbb{E}_e[({C_1^e}^\top C_1^e (\theta_1)^2 + ({C_2^e}^\top C_2^e+2 {C_2^e}^\top \epsilon^e+2N^e + \epsilon^\top \epsilon)(\theta_2)^2\\
            +&({C_1^e}^\top C_2^e+ {C_1^e}^\top \epsilon^e)\theta_1 \theta_2 - (\tilde{A^e}^{k}X_1)^\top C_1^e \theta_1 - [(\tilde{A^e}^{k}X_1)^\top + N^e + n_1^\top \epsilon^e]\theta_2)\\
            &(2 ({C_2^e}^\top C_2^e + 2 {C_2^e}^\top \epsilon^e + 2N^e + {\epsilon^e}^\top \epsilon^e) \theta_2 + ({C_1^e}^\top C_2^e + {C_1^e}^\top \epsilon)\theta_1\\
             -& (\tilde{A^e}^{k}X_1)^\top (C_2^e + \epsilon^e)+ N^e +n_1^\top \epsilon^e )]=0
    \end{aligned}
\end{equation}

From the analysis in Appendix \ref{spec_fail_sec}, we know that as long as the lower-layer parameters take any value that satisfies the form in Equation (\ref{f2}), even if $s\neq k$, we can get $\frac{\partial \mathcal{L}_{\text{IRMv1}}}{\partial \theta_1}=\frac{\partial \mathcal{L}_{\text{IRMv1}}}{\partial \theta_1^1}=\frac{\partial \mathcal{L}_{\text{IRMv1}}}{\partial \theta_1^2}$ and $\frac{\partial \mathcal{L}_{\text{IRMv1}}}{\partial \theta_2}=\frac{\partial \mathcal{L}_{\text{IRMv1}}}{\partial \theta_2^1}=\frac{\partial \mathcal{L}_{\text{IRMv1}}}{\partial \theta_2^2}$. Thus, IRM cannot necessarily learn a $s=k$. At this time (when $\frac{\partial \mathcal{L}_{\text{IRMv1}}}{\partial \theta_1}=\frac{\partial \mathcal{L}_{\text{IRMv1}}}{\partial \theta_1^1}=\frac{\partial \mathcal{L}_{\text{IRMv1}}}{\partial \theta_1^2}=0$ and $\frac{\partial \mathcal{L}_{\text{IRMv1}}}{\partial \theta_2}=\frac{\partial \mathcal{L}_{\text{IRMv1}}}{\partial \theta_2^1}=\frac{\partial \mathcal{L}_{\text{IRMv1}}}{\partial \theta_2^2}=0$ but $s\neq k$), from the form of Equations (\ref{irm_con_solu_t1}) and (\ref{irm_con_solu_t2}) we know that there exist solutions that $\theta_2\neq 0$, and the solution of $\theta_1$ and $\theta_2$ both depend on $\mathbb{E}_e(F(e))$, where $F(e)$ is some random variable associated with $e$. 

\end{proof}

\subsubsection{Proof of the Successful on Graphs Case of CIA under Concept Shift}
\begin{theorem}
    Optimizing the CIA objective will lead to the optimal solution $\Theta^*$:
    \begin{equation}
        \left\{
        \begin{array}{l}
            \theta_1= 1                                                                                      \\
            \theta_2=0 \quad \text{and} \quad  \exists l\in \{1,...,L-1\} ~\text{s.t.}~ {\theta_2^1}^{(l)}={\theta_2^2}^{(l)}=0 \\
            {\theta^1_1}^{(l)}=1, {\theta^2_1}^{(l)}=1, \quad l=L-1, ..., L-k+1                           \\
            {\theta^1_1}^{(l)}=0, {\theta^2_1}^{(l)}=1, \quad l=L-k, L-k-1, ..., 1                           \\
        \end{array}
        \right..
    \end{equation}
\end{theorem}

\begin{proof}
\label{proof_CIA}
    For brevity, denote a node representation of $ {C_1^e}_c $ as $ C_1^i $ and the one of $ {C_1^{e'}}_c $ as $ C_1^j $. 
    The same is true for $ C_2^i $ and $ C_2^j $. In this toy model, we need to consider the expectation of the noise, while in 
    real cases such noise is included in the node features so taking expectation on $e$ will
    handle this. Therefore, we add $ \mathbb{E}_{n_1,n_2} $ in this proof, and this expectation is excluded in
    the formal description of the objective in the main paper.
    \begin{equation}
        \begin{aligned}
            \mathcal{L}_{\text{CIA}}
            &=\mathbb{E}_{\substack{e, e'\\e\neq e'}} \mathbb{E}_{n_1,n_2}  \mathbb{E}_{c} \mathbb{E}_{\substack{i,j\\(i,j)\in \Omega^{e,e'}}}\left[\mathcal{D}(\phi_\Theta({A^e}, X^e)_{[c][v_i]}, \phi_\Theta(A^{e}, X^{e'})_{[c][v_j]})\right]\\
            &=\mathbb{E}_{\substack{e, e'\\e\neq e'}} \mathbb{E}_{n_1,n_2}  \mathbb{E}_{c} \mathbb{E}_{\substack{i,j\\(i,j)\in \Omega^{e,e'}}}\| {C_1^i}\theta_1 + (C_2^i+Z^e)\theta_2 - {C_1^j}\theta_1 - (C_2^j+Z^{e'})\theta_2  \|^2_2
        \end{aligned}
    \end{equation}
    \begin{equation}
        \frac{\partial \mathcal{L}_{\text{CIA}}}{ \partial \theta_1}=\mathbb{E}_{\substack{e, e'\\e\neq e'}} \mathbb{E}_{n_1,n_2}  \mathbb{E}_{c} \mathbb{E}_{\substack{i,j\\(i,j)\in \Omega^{e,e'}}}\left[{C_1^i}\theta_1 + ({C_2^i}+Z^e)\theta_2 - {C_1^{j}}\theta_1 - ({C_2^{j}}+Z^{e'})\theta_2\right]^\top ({C_1^i}-C_1^j)
    \end{equation}
    Let $ \frac{\partial \mathcal{L}_{\text{CIA}}}{ \partial \theta_1}=0 $, we have:
    \begin{equation}
        \label{cia_par1}
        \mathbb{E}_{\substack{e, e'\\e\neq e'}}  \mathbb{E}_{c} \mathbb{E}_{\substack{i,j\\(i,j)\in \Omega^{e,e'}}} \left[(C_1^i-C_1^j)^\top(C_1^i-C_1^j)\theta_1 + (C_2^i-C_2^k)^\top (C_1^i-C_1^j)\theta_2 \right]=0
    \end{equation}
    Also, we have:
    \begin{equation}
        \frac{\partial \mathcal{L}_{\text{CIA}}}{ \partial \theta_2}=\mathbb{E}_{\substack{e, e'\\e\neq e'}}  \mathbb{E}_{c} \mathbb{E}_{\substack{i,j\\(i,j)\in \Omega^{e,e'}}}\left[(C_1^i-C_1^j)^\top(C_2^i-C_2^j)\theta_1 + \left[(C_2^i-C_2^k)^\top (C_2^i-C_2^j) + (Z^e-Z^{e'})^\top(Z^e-Z^{e'})\right]\theta_2 \right]
    \end{equation}
    Further let $  \frac{\partial \mathcal{L}_{\text{CIA}}}{ \partial \theta_2}=0 $, combining Equation (\ref{cia_par1}) and using Assumption \ref{assump_stable_causal} we get
    \begin{equation}
        \left\{
            \begin{array}{l}
                \theta_1=0 \quad \text{or} \quad  \exists l\in \{1,...,L-1\} ~\text{s.t.}~ {\theta_1^1}^{(l)}={\theta_1^2}^{(l)}=0 \\
                \theta_2=0
            \end{array}
        \right..
    \end{equation}

    or, if $ \theta_1\neq 0 $ and $ \forall l~ \in \{1,...,L-1\} $, the parameters of that layer $l$ of the invariant branch of the 
    GNN are not all zero: $ {\theta_1^1}^{(l)}\neq 0 ~\text{or}~ {\theta_1^2}^{(l)}\neq 0 ~$, then we get

    \begin{equation}
        \theta_2 \underbrace{ \mathbb{E}_{\substack{e, e'\\e\neq e'}}  \mathbb{E}_{c} \mathbb{E}_{\substack{i,j\\(i,j)\in \Omega^{e,e'}}}\left[ -\frac{[(C_1^i-C_1^j)^\top(C_2^i-C_2^j)]^2}{(C_1^i-C_1^j)^\top(C_1^i-C_1^j)}+(C_2^i-C_2^j)^\top(C_2^i-C_2^j)+(Z^e-Z^{e'})^\top(Z^e-Z^{e'}) \right]}_{F}=0
    \end{equation}
    According to Cauchy–Schwarz inequality, $ F>0 $ unless $ \exists l\in \{1,...,L-1\} ~\text{s.t.}~ {\theta_2^1}^{(l)}={\theta_2^2}^{(l)}=0 $.
    To ensure $ \frac{\partial \mathcal{L}_{\text{CIA}}}{ \partial \theta_2} $, we conclude that $ \theta_2=0 $ or $ \exists l\in \{1,...,L-1\} ~\text{s.t.}~ {\theta_2^1}^{(l)}={\theta_2^2}^{(l)}=0 $.

    In conclusion, to satisfy the constraint of CIA, no matter whether the invariant branch has zero output, 
    the spurious branch must have zero parameters, i.e.,
    \begin{equation}
        \label{sp_out}
        \theta_2=0 \quad \text{or} \quad  \exists l\in \{1,...,L-1\} ~\text{s.t.}~ {\theta_2^1}^{(l)}={\theta_2^2}^{(l)}=0
    \end{equation}
    Thus, CIA will remove spurious features.
    
    Now we show that when CIA objective has been reached (the spurious branch has zero outputs),  the objective of $ \min_\Theta \mathbb{E}_e[\mathcal{L}(f_\Theta({A^e},X^e),Y^e)] $ 
    will help to learn predictive paramters of the invariant branch $ \theta_1 $, $ {\theta_1^1}^(l) $ and $ {\theta_1^2}^(l) $.
    When Equation (\ref{sp_out}) holds, 

    \begin{equation}
        \begin{aligned}
            \frac{\partial  \mathbb{E}_e[\mathcal{L}(f_\Theta({A^e},X^e),Y^e)]}{ \partial \theta_1}&=2\mathbb{E}_e \mathbb{E}_{n_1,n_2} \left[ \left({C_1^e}\theta_1-(\tilde{A^e})^kX_1-n_1 \right)^\top C_1^e \right]\\
            &=2 \mathbb{E}_e \left[ \left({C_1^e}\theta_1-(\tilde{A^e})^kX_1 \right)^\top C_1^e \right]
        \end{aligned}
    \end{equation}

    Let $ \frac{\partial  \mathbb{E}_e[\mathcal{L}(f_\Theta({A^e},X^e),Y^e)]}{ \partial \theta_1} =0$, we get the predictive parameters
    \begin{equation}
        \left\{
        \begin{array}{l}
            \theta_1= 1                                                                                      \\
            {\theta^1_1}^{(l)}=1, {\theta^2_1}^{(l)}=1, \quad l=L-1, ..., L-k+1                           \\
            {\theta^1_1}^{(l)}=0, {\theta^2_1}^{(l)}=1, \quad l=L-k, L-k-1, ..., 1                           \\
        \end{array}
        \right..
    \end{equation}
    Plug the final solution back in $ \frac{\partial \mathcal{L}_{\text{CIA}}}{ \partial {\theta^1_1}^{(l)}} $, 
    $ \frac{\partial \mathcal{L}_{\text{CIA}}}{ \partial {\theta^2_1}^{(l)}} $,
    $ \frac{\partial \mathcal{L}_{\text{CIA}}}{ \partial {\theta^1_2}^{(l)}} $,
    $ \frac{\partial \mathcal{L}_{\text{CIA}}}{ \partial {\theta^2_2}^{(l)}} $, 
    $ \frac{\partial  \mathbb{E}_e[\mathcal{L}(f_\Theta({A^e},X^e),Y^e)]}{ \partial {\theta^1_1}^{(l)}} $, 
    $ \frac{\partial  \mathbb{E}_e[\mathcal{L}(f_\Theta({A^e},X^e),Y^e)]}{ \partial {\theta^2_1}^{(l)}} $,
    $ \frac{\partial \mathbb{E}_e[\mathcal{L}(f_\Theta({A^e},X^e),Y^e)]}{ \partial {\theta^1_2}^{(l)}} $,
    $ \frac{\partial  \mathbb{E}_e[\mathcal{L}(f_\Theta({A^e},X^e),Y^e)]}{ \partial {\theta^2_2}^{(l)}} $, 
    we can verify that these terms are all when further letting $ \exists l\in \{1,...,L-1\} ~\text{s.t.}~ {\theta_2^1}^{(l)}={\theta_2^2}^{(l)}=0$.
\end{proof}

\subsection{Proof of the Covariate Shift Case}
\label{proof_cov}
\subsubsection{Proof of the non-Graph Success Case of VREx and IRMv1 under Covariate Shift}
\label{proof_IRM_VREx_suc_cov}
We restate Proposition \ref{prop_VREx_IRM_suc_cov} as Proposition \ref{prop_VREx_IRM_suc_cov2} below.
\begin{proposition}
\label{prop_VREx_IRM_suc_cov2}
    \textbf{(VREx and IRMv1 learn invariant features for non-graph tasks under covariate shift, proof is in )} For the non-graph version of the SCM in Equation (\ref{data_gene_cov}), 
    \begin{equation}
        Y^e=X_1+n_1,~X_2^e=n_2+{\epsilon^e},    
    \end{equation}
    Optimizing VREx $\min_\Theta \mathcal{L}_{\text{VREx}}=\mathbb{V}_e[R(e)]$ and IRMv1 $\min_\Theta \mathcal{L}_{\text{IRMv1}}=\mathbb{E}_e[\|\nabla_{w|w=1.0} R(e)\|^2]$ will learn invariant features when using a 1-layer linear network: $f(X)=\theta_1 X_1 + \theta_2 X_2$.
\end{proposition}
\begin{proof}
    \textbf{VREx.} For covariate shift, denote $X_1 \theta_1+X_2^e\theta_2-X_1-n_1$ as $l_e$, the expected risk in environment $e$ is 
    $ R(e)=\mathbb{E}_{n_1,n_2}\|\theta_1 X_1 + \theta_2 (n_2 + \epsilon^e) - (X_1+n_1) \|^2_2 $.
    Take the derivative of the VREx objective $\mathbb{V}_e[R(e)]$ with respect to $\theta_1$:
        \begin{equation}
            \label{eqqq1}
            \begin{aligned}
                \frac{\partial \mathbb{V}_e[R(e)]}{\partial \theta_1} & =\mathbb{E}_e\left[2\mathbb{E}_{n_1,n_2}\left[l_e^\top l_e\right]\mathbb{E}_{n_1,n_2}\left[2l_e^\top X_1\right]\right]                              \\
                                                                    & -2\mathbb{E}_e\left[\mathbb{E}_{n_1,n_2}\left[l_e^\top l_e\right]\right]\mathbb{E}_e\left[\mathbb{E}_{n_1,n_2}\left[ 2 l_e^\top X_1 \right] \right]
            \end{aligned}
        \end{equation}
    Calculate these terms:
    \begin{equation}
        \begin{aligned}
                &\mathbb{E}_e[\mathbb{E}_{n_1,n_2}[l_e^\top l_e]\mathbb{E}_{n_1,n_2}[l_e^\top X_1]]\\
                =&(\theta_1)^3 (X_1^\top X_1)^2 + \theta_1 (\theta_2)^2 (3 X_1^\top X_1 \sigma^2 + X_1^\top X_1\mathbb{E}_e [N^e]) + (\theta_2)^3 \mathbb{E}_e[(\epsilon^e)^\top \epsilon (\epsilon^e)^\top] X_1\\
                +& (\theta_1)^2 (X_1^\top X_1 - 2 (X_1^\top X_1)^2) - (\theta_2)^2 (X_1^\top X_1)^2 (3 \sigma^2 +\mathbb{E}_e [N^e])\\
                +& \theta_1 (X_1^\top X_1)^2 ( 3 (X_1^\top X_1)^2 +\mathbb{E}_e [N^e]) - (X_1^\top X_1)^2 ((X_1^\top X_1)^2 +\mathbb{E}_e [N^e])\\
                &\mathbb{E}_e[\mathbb{E}_{n_1,n_2}[l_e^\top l_e]]\mathbb{E}_e[\mathbb{E}_{n_1,n_2}[l_e^\top X_1]]\\
                =&(\theta_1)^3 (X_1^\top X_1)^2 + \theta_1 (\theta_2)^2 ( X_1^\top X_1 \sigma^2 + X_1^\top X_1\mathbb{E}_e [N^e]) \\
                +& (\theta_1)^2 (X_1^\top X_1 - 2 (X_1^\top X_1)^2) - (\theta_2)^2 (X_1^\top X_1)^2 ( \sigma^2 +\mathbb{E}_e [N^e])\\
                +& \theta_1 (X_1^\top X_1)^2 ( 3 (X_1^\top X_1)^2 +\mathbb{E}_e [N^e]) - (X_1^\top X_1)^2 ((X_1^\top X_1)^2 +\mathbb{E}_e [N^e])
        \end{aligned}
    \end{equation}
    Hence, 
    \begin{equation}
        \begin{aligned}
            \frac{\partial \mathbb{V}_e[R(e)]}{\partial \theta_1} = \theta_1 (\theta_2)^2 (2 X_1^\top X_1 \sigma^2) + (\theta_2)^3 X_1 - 2 (\theta_2)^2 X_1 \sigma^2
        \end{aligned}
    \end{equation}
    Let $ \frac{\partial \mathbb{V}_e[R(e)]}{\partial \theta_1} =0 $, we have $\theta_2=0$.

    When $\theta_2=0$ and when $\theta_1=1$, we get
    \begin{equation}
        \mathbb{E}_{n_1,n_2}[l_e^\top (n_2+\epsilon^e)]=\theta_1 X_1^\top \epsilon - X_1^\top \epsilon^e=0.
    \end{equation}
    As a result, 
    \begin{equation}
        \begin{aligned}
            \frac{\partial \mathbb{V}_e[R(e)]}{\partial \theta_2}&=\mathbb{E}_e\left[2\mathbb{E}_{n_1,n_2}\left[l_e^\top l_e\right]\mathbb{E}_{n_1,n_2}\left[2l_e^\top (n_2 + \epsilon^e)\right]\right]                              \\
            & -2\mathbb{E}_e\left[\mathbb{E}_{n_1,n_2}\left[l_e^\top l_e\right]\right]\mathbb{E}_e\left[\mathbb{E}_{n_1,n_2}\left[ 2 l_e^\top (n_2 + \epsilon^e) \right] \right]\\
            &=0-0\\
            &=0
        \end{aligned}
    \end{equation}
    In conclusion, $\theta_1=1$ and $\theta_2=0$ is the solution of the VREx objective in this non-graph covariate shift task.
    
    \textbf{IRMv1.}  The objective of IRMv1 is $\mathbb{E}_e \|\nabla_w R(e) \|^2_2$. When IRMv1 loss is optimized to zero,
    we have $\nabla_w R(e)=0$ for all environments $e$.
    \begin{equation}
        \begin{aligned}
                \nabla_w R(e)&=\mathbb{E}_{n_1,n_2}[2(\theta_1X_1+\theta_2X_2-(X_1 + n_1))(\theta_1X_1 +\theta_2X_2)]\\
                &=2((\theta_1)^2 X_1^\top X_1 + (\theta_2)^2((\epsilon^e)^\top \epsilon^e+ {N^e}) + 2 \theta_1 \theta_2 X_1^\top \epsilon^e \\
                &- \theta_1 X_1^\top X_1 - \theta_2 X_1^\top \epsilon^e)
        \end{aligned}
    \end{equation}
    To realize $\nabla_w R(e)=0$ for all $e$, we must let $\theta_2=0$ (and consequently $\theta_1=1$), otherwise 
    the solution of $\theta_2$ will include terms related to $\epsilon^e$ and $N^e$ that vary with environments, and a single value $\theta_2$
    cannot fit all these values. Thus we finish the proof for IRMv1.
\end{proof}

\subsubsection{Proof of the Failure Case on Graphs of VREx under Covariate Shift}
We restate Theorem \ref{VREx_fail_cov} as Theorem \ref{VREx_fail_cov2} below:
\begin{theorem}
\label{VREx_fail_cov2}
    \textbf{(VREx will use spurious features on graphs under covariate shift)}  Under the SCM of Equation (\ref{data_gene_cov}), the objective $\min_\Theta \mathbb{V}_e[R(e)]$ has non-unique solutions for parameters of the GNN (\ref{GNN}) when part of the model parameters $\{{\theta^1_1}^{(l)},{\theta^2_1}^{(l)},{\theta^1_2}^{(l)},{\theta^2_2}^{(l)}\}$ take the values 
    \begin{equation}
        \Theta_0=\left\{
        \begin{array}{l}
            {\theta^1_1}^{(l)}=1, {\theta^2_1}^{(l)}=1, \quad l=L-1, ..., L-s+1  \\
            {\theta^1_1}^{(l)}=0, {\theta^2_1}^{(l)}=1, \quad l=L-s, L-s-1, ..., 1 \\
            {\theta^1_2}^{(l)}=0, {\theta^2_2}^{(l)}=1, \quad l=L-1, ..., 1       \\
        \end{array}
        \right.,
    \end{equation} $0<s<L$ is some positive integer, $\theta_1$ and $\theta_2$ have four sets of solutions of the quadratic equation, some of which are spurious solutions that $\theta_2\neq 0$ (although $\theta_2=0$ is indeed one of the solutions, VREx is not guaranteed to reach this solution):
    \begin{equation}
        \left\{
        \begin{array}{l}
            c_1\sigma^2(2\theta_1 \theta_2+ (\theta_2)^2-2 c_2 \sigma^2 \theta_2)+c_3\theta_2-\mathbb{E}_e[N^e]c_1 \sigma^2 \theta_1 \theta_2 + \mathbb{E}_e[N^e]c_2 \sigma^2 \theta_2 =0 \\
            \left[ c_3\theta_2-\mathbb{E}_e[N^e]c_1 \sigma^2 \theta_1 \theta_2 + \mathbb{E}_e[N^e]c_2 \sigma^2 \theta_2 \right]c_4-c_5(\theta_2)^2=0
        \end{array}
        \right. .
    \end{equation}
    where  $ c_1= \mathbb{E}[ (\tilde{{A^e}}^{s}X_1)^\top (\tilde{{A^e}}^{s}X_1)]$, $c_2= \mathbb{E}[(\tilde{{A^e}}^{s}X_1)^\top (\tilde{{A^e}}^{k}X_1)] $, 
    $ c_3=\mathbb{E}_e[{\epsilon^e}^\top {\epsilon^e} {\epsilon^e}^\top (\tilde{{A^e}}^{s} X_1)] \sigma^2$, $c_4=\mathbb{E}_e\left[  ({(\tilde{A^e})^kX_1})^\top \mathbf{1_{N^e}} \right]\sigma^2 $,
    $ c_5=\mathbb{E}_e \left[N^e\left(\text{tr} ( (\tilde{{A^e}}^{k})^\top \tilde{{A^e}}^{k}) + N^e (1+\sigma^2)  \right)   \right] $.
\end{theorem}

\begin{proof}  \label{proof_VREx2}
We will use some symbols to simplify the expression of the toy GNN. 
Denote $n_2+{\epsilon^e}$ as ${\eta}$. Use the following notations to represent the components of the $L$-layer GNN model:
    \begin{equation}
    \label{def_C_Z_2}
        \begin{aligned}
            f_{\Theta}(A,X) & =H^{(L)}_1\theta_1+H^{(L)}_2\theta_2                                                                                                                                                                                                                                                           \\
                            & =\underbrace{\left[{\theta^1_1}^{(L-1)}\bar{A}\left(...{\theta^1_1}^{(3)}\left({\theta^1_1}^{(2)}\bar{A}({\theta^1_1}^{(1)}\bar{A}+{\theta^2_1}^{(1)}\bar{I})X_1+{\theta^2_1}^{(2)}({\theta^1_1}^{(1)}\bar{A}+{\theta^2_1}^{(1)}\bar{I})X_1\right)+...\right)\right]}_{C_1}\theta_1                             \\
                            & +\underbrace{\left[{\theta^1_2}^{(L-1)}\bar{A}\left(...{\theta^1_2}^{(3)}\left({\theta^1_2}^{(2)}\bar{A}({\theta^1_2}^{(1)}\bar{A}+{\theta^2_2}^{(1)}\bar{I}){\eta}+{\theta^2_2}^{(2)}({\theta^1_2}^{(1)}\bar{A}+{\theta^2_2}^{(1)}\bar{I}){\eta}\right)+...\right)\right]}_{Z}\theta_2                             \\
                            & =C_1\theta_1+Z\theta_2.
        \end{aligned}
    \end{equation}
    $C_1, Z\in \mathbb{R}^{N \times 1}$. 
    We use $C_1^e$ and $Z^e$ to denote the variables from the corresponding environment $e$.
    We further denote $ Z^e= {C_2^e}' {\eta} $.

    Using these notations, the loss of environment $e$ is
    \begin{equation}
        \begin{aligned}
            R(e) & = \mathbb{E}_{n_1,n_2}\left[\left\|f_\Theta({A^e}, X^e)-Y^e\right\|_2^2\right]                                   \\
                 & =\mathbb{E}_{n_1,n_2}\left[\left\|C^e_1\theta_1+Z^e\theta_2-\tilde{{A^e}}^{k}X_1-n_1\right\|_2^2\right].
        \end{aligned}
    \end{equation} Denote the inner term $C^e_1\theta_1+Z^e\theta_2-\tilde{{A^e}}^{k}X_1-n_1$ as $l_e$. The variance of loss across environments is:
    \begin{equation}
        \begin{aligned}
            \mathbb{V}_e[R(e)] & =\mathbb{E}_e[R^2(e)]-\mathbb{E}_e^2[R(e)]                                                                                                             \\
                               & =\mathbb{E}_e\left[\left(\mathbb{E}_{n_1,n_2}\left\|C^e_1\theta_1+Z^e\theta_2-\tilde{{A^e}}^{k}X_1-n_1\right\|^2_2\right)^2\right]                 \\
                               & -\mathbb{E}_e^2\left[\mathbb{E}_{n_1,n_2}\left\|C^e_1\theta_1+Z^e\theta_2-\tilde{{A^e}}^{k}X_1-n_1\right\|^2_2\right].                             \\
                               & =\mathbb{E}_e\left[ \mathbb{E}_{n_1,n_2}\left[(l_e^\top l_e)^2\right]\right]-\mathbb{E}^2_e\left[\mathbb{E}_{n_1,n_2}\left[l_e^\top l_e\right]\right].
        \end{aligned}
    \end{equation}

    Take the derivative of $\mathbb{V}_e[R(e)]$ with respect to $\theta_1$:
    \begin{equation}
        \begin{aligned}
            \frac{\partial \mathbb{V}_e[R(e)]}{\partial \theta_1} & =\mathbb{E}_e\left[2\mathbb{E}_{n_1,n_2}\left[l_e^\top l_e\right]\mathbb{E}_{n_1,n_2}\left[2l_e^\top C_1^e\right]\right]                              \\
                                                                  & -2\mathbb{E}_e\left[\mathbb{E}_{n_1,n_2}\left[l_e^\top l_e\right]\right]\mathbb{E}_e\left[\mathbb{E}_{n_1,n_2}\left[ 2 l_e^\top C_1^e \right] \right]
        \end{aligned}
    \end{equation}

    Calculate the derivative by terms: 
    \begin{equation}
        \begin{aligned}
            \label{ll1b}
            \mathbb{E}_{n_1,n_2}[l_e^\top l_e] & =\mathbb{E}_{n_1,n_2} [{C^e_1}^\top {C^e_1} (\theta_1)^2  + {C^e_1}^\top Z^e \theta_1 \theta_2 -{C^e_1}^\top (\tilde{A^e})^kX_1 \theta_1 -{C^e_1}^\top n_1 \theta_1 \\
                                               & + {Z^e}^\top C^e_1\theta_1 \theta_2 + {Z^e}^\top {Z^e} (\theta_2)^2 - {Z^e}^\top (\tilde{A^e})^kX_1 \theta_2 - {Z^e}^\top n_1 \theta_2\\
                                               & - ({(\tilde{A^e})^kX_1})^\top C_1^e \theta_1 -({(\tilde{A^e})^kX_1})^\top {Z^e} \theta_2
                                               +({(\tilde{A^e})^kX_1})^\top (\tilde{A^e})^kX_1 \\
                                               &+ ({(\tilde{A^e})^kX_1})^\top n_1 -n_1^\top C_1^e \theta_1  -n_1^\top Z^e \theta_2
                                                + n_1^\top  (\tilde{A^e})^kX_1 +n_1^\top n_1 ]\\
        \end{aligned} 
    \end{equation}

    Since $n_1$ and $n_2$ are independent standard Gaussian noise, we have $\mathbb{E}_{n_1,n_2}[n_1]=\mathbb{E}_{n_1,n_2}[n_2]=\mathbf{0}$, 
    $\mathbb{E}_{n_1,n_2}[n_1^\top n_2]=\mathbb{E}_{n_1,n_2}[n_2^\top n_1]=0$ and 
    $\mathbb{E}_{n_1,n_2}[n_1^\top n_1]=\mathbb{E}_{n_1,n_2}[n_2^\top n_2]=N^e$ if it is the noise from $e$. 
    Also, since ${\epsilon^e}$ and 
    $n_1$, $n_2$ are independent, we have $ \mathbb{E}_{n_1, n_2}[n_1^\top {\epsilon^e}]=\mathbb{E}_{n_1, n_2}[n_2^\top {\epsilon^e}]=0 $.

    When \begin{equation}
        \label{svb}
        \left\{
        \begin{array}{l}
            {\theta^1_1}^{(l)}=1, {\theta^2_1}^{(l)}=1, \quad l=L-1, ..., L-s+1  \\
            {\theta^1_1}^{(l)}=0, {\theta^2_1}^{(l)}=1, \quad l=L-s, L-s-1, ..., 1 \\
            {\theta^1_2}^{(l)}=0, {\theta^2_2}^{(l)}=1, \quad l=L-1, ..., 1       \\
        \end{array}
        \right.,
    \end{equation} we have $ {C_2^e}'=I_{N^e}\in \mathbb{R}^{N^e \times N^e} $ and $ C_1^e=\tilde{{A^e}}^{s} X_1 $.  
    Consequently, we get $ \mathbb{E}_{n_1, n_2}[{Z^e}^\top n_1]= 0$, 
    $ \mathbb{E}_{n_1, n_2}[{Z^e}^\top Z^e] =  N^e +  {\epsilon^e}^\top {\epsilon^e}$.

    Use the above conclusions and rewrite Equation (\ref{ll1b}) as (here we only plug in the value of ${C_2^e}'$):
    \begin{equation}
        \label{ll2b}
        \begin{aligned}
            &\mathbb{E}_{n_1,n_2}[l_e^\top l_e] =\\
            &\left.\begin{array}{@{}l@{}}
            {C^e_1}^\top {C^e_1} (\theta_1)^2  - {C^e_1}^\top (\tilde{A^e})^kX_1 \theta_1   \\
            - ({(\tilde{A^e})^kX_1})^\top C_1^e \theta_1 +({(\tilde{A^e})^kX_1})^\top (\tilde{A^e})^kX_1 + N^e \left(1+(\theta_2)^2 \right) \\
            \end{array} \right\}  \text{($*$)}\\
            &\left.\begin{array}{@{}l@{}} + [{C^e_1}^\top {\epsilon^e}   +  {\epsilon^e}^\top C_1^e] \theta_1 \theta_2 +  {\epsilon^e}^\top  {\epsilon^e}(\theta_2)^2 -2({(\tilde{A^e})^kX_1})^\top {\epsilon^e} \theta_2\\
            \end{array}\right\} \text{($**$)}
        ,\end{aligned}
    \end{equation}
    ($ * $) and ($ ** $) represent terms that are independent and associated with $ {\epsilon^e} $, respectively.

    Additionally, 
    \begin{equation}
        \begin{aligned}
            \label{lc1b}
        \mathbb{E}_{n_1, n_2}[2 l_e^\top C_1^e]&=2 \left [ {C_1^e}^\top {C_1^e} \theta_1
        +{\epsilon^e}^\top {C_1^e} \theta_2  - ((\tilde{A^e})^kX_1)^\top {C_1^e} \right ]\\.
        \end{aligned}
    \end{equation}    
    Multiplying Equation (\ref{ll2b}) and (\ref{lc1b}) and take the expectation on $e$, 
    using the assumption that $ \mathbb{E}_{e}[({\epsilon^e}_i)^2 ]=\sigma^2$ ($ {\epsilon^e}_i $ is the $i$-th element of $ {\epsilon^e} $):
    \begin{equation}
        \label{eelllcb}
        \begin{aligned}
            \mathbb{E}_{e}\left[ 2 \mathbb{ E}_{n_1, n_2}[l^\top_e l_e] \mathbb{ E}_{n_1, n_2}\left[2l_e^\top C_1\right] \right]
            &=4\mathbb{E}_e\left[ \text{($ * $)}\left( {C_1^e}^\top C_1^e \theta_1  - ({(\tilde{A^e})^kX_1})^\top C_1^e \right) \right]\\
            &+ 4\mathbb{E}_e\left[  (\tilde{{A^e}}^{s} X_1)^\top \tilde{{A^e}}^{s} X_1 (2\theta_1 \theta_2+(\theta_2)^2)-2 (\tilde{{A^e}}^{s}X_1)^\top ((\tilde{A^e})^kX_1) \theta_2   \right]\theta_2 \sigma^2\\
            &+ 4\mathbb{E}_e [N^e {\epsilon^e}^\top {\epsilon^e} {\epsilon^e}^\top (\tilde{{A^e}}^{s} X_1)]\theta_2
        .\end{aligned}
    \end{equation}

    Next target is to compute $2\mathbb{E}_e [\mathbb{E}_{n_1,n_2}[l_e^\top l_e]]$ and $\mathbb{E}_e [\mathbb{E}_{n_1,n_2}[2l_e^\top C_1]]$
    Since ${\epsilon^e}$ has zero mean, we have:
    \begin{equation}
        \label{eellb}
        \begin{aligned}
            2 \mathbb{E}_e[\mathbb{E}_{n_1,n_2}[l_e^\top l_e]]=\mathbb{E}[\text{($*$)}]+\mathbb{E}[2N^e] \sigma^2 (\theta_2)^2
        \end{aligned}
    \end{equation}
    and
    \begin{equation}
        \label{eelcb}
        \begin{aligned}
            \mathbb{E}_e[\mathbb{E}_{n_1,n_2}[2l_e^\top C_1^e]]=
            2 \mathbb{E}_e\left[ {C_1^e}^\top {C_1^e} \theta_1  - ({(\tilde{A^e})^kX_1})^\top C_1^e \right]
        .\end{aligned}
    \end{equation}

    Use Equation (\ref{eelllcb}) (\ref{eellb}) and (\ref{eelcb}) and let $  \frac{\partial \mathbb{V}_e[R(e)]}{\partial \theta_1}=0 $, we have:
    \begin{equation}
        \begin{aligned}
            \label{pt1b}
            &\mathbb{E}_e\left[ 2 (\tilde{{A^e}}^{s} X_1)^\top (\tilde{{A^e}}^{s} X_1)(\theta_1 \theta_2 + \frac{1}{2}(\theta_2)^2)- 2 (\tilde{{A^e}}^{s} X_1)^\top ((\tilde{A^e})^kX_1)\theta_2 \right] \sigma^2 
            + \mathbb{E}_e[{\epsilon^e}^\top {\epsilon^e} {\epsilon^e}^\top (\tilde{{A^e}}^{s} X_1)]\theta_2 \\
             -&\mathbb{E}_e[ N^e]\mathbb{E}_e\left[ (\tilde{{A^e}}^{s} X_1)^\top (\tilde{{A^e}}^{s} X_1)\theta_1-((\tilde{A^e})^kX_1)^\top C_1^e\right]  \theta_2 \sigma^2 =0
        .\end{aligned}
    \end{equation}

Now we start calculating the expression of $ \frac{\partial \mathbb{V}_e[R(e)]}{\partial \theta_2} $: 
\begin{equation}
    \begin{aligned}
        \frac{\partial \mathbb{V}_e[R(e)]}{\partial \theta_2}& =\mathbb{E}_e\left[2\mathbb{E}_{n_1,n_2}\left[l_e^\top l_e\right]\mathbb{E}_{n_1,n_2}\left[2l_e^\top Z^e\right]\right]                              \\
        & -2\mathbb{E}_e\left[\mathbb{E}_{n_1,n_2}\left[l_e^\top l_e\right]\right]\mathbb{E}_e\left[\mathbb{E}_{n_1,n_2}\left[ 2 l_e^\top Z^e \right] \right]
    .\end{aligned}
\end{equation}

Let $\frac{\partial \mathbb{V}_e[R(e)]}{\partial \theta_2}=0$:
\begin{equation}
    \label{pt2b}
    \begin{aligned}
        &\mathbb{E}_e\left[ ({C_1^e}^\top {C_2^e}' + {{C_2^e}'}^\top {C_1^e}^\top )\theta_1 \theta_2 + ({C_2^e}')^\top {C_2^e}' (\theta_2)^2- 2 ({(\tilde{A^e})^kX_1})^\top {C_2^e}' \theta_2 \right]\mathbb{E}_e\left[( - ({(\tilde{A^e})^kX_1})^\top {C_2^e}')\right]\sigma^2\\
        &-\mathbb{E}_e \left[ N^e\sigma^2 \left(  \text{tr}((\tilde{{A^e}}^{k})^\top \tilde{{A^e}}^{k}) +N^e+{{C^e_2}'}^\top {C_2^e}' \sigma^2 \right)(\theta_2)^2  \right]\\
        & = 0
    .\end{aligned}
\end{equation}

Plug Equation (\ref{pt1b}) in (\ref{pt2b}), we reach:
\begin{equation}
    \begin{aligned}
        &\left[\mathbb{E}_e \left[N^e (\tilde{{A^e}}^{s}X_1)^\top (\tilde{{A^e}}^{s}X_1)\theta_1 -N^e((\tilde{A^e})^kX_1)^\top C_1^e
        \right]\theta_2\sigma^2-\mathbb{E}_e[  {\epsilon^e}^\top {\epsilon^e} {\epsilon^e}^\top (\tilde{{A^e}}^{s}X_1)]\theta_2 \right]\mathbb{E}_e\left(  - (\tilde{{A^e}}^{k}X_1)^\top \mathbf{1_{N^e}}\right)\\
         & - \mathbb{E}_e \left[N^e\left(\text{tr} ( (\tilde{{A^e}}^{k})^\top \tilde{{A^e}}^{k}) + N^e (1+\sigma^2)  \right)   \right](\theta_2)^2\\
         & =0
    .\end{aligned}
\end{equation}
Let $ c_1= \mathbb{E}[ (\tilde{{A^e}}^{s}X_1)^\top (\tilde{{A^e}}^{s}X_1)]$, $c_2= \mathbb{E}[(\tilde{{A^e}}^{s}X_1)^\top (\tilde{{A^e}}^{k}X_1)] $, $ c_3=\mathbb{E}_e[{\epsilon^e}^\top {\epsilon^e} {\epsilon^e}^\top (\tilde{{A^e}}^{s} X_1)] \sigma^2$, $c_4=\mathbb{E}_e\left[  ({(\tilde{A^e})^kX_1})^\top \mathbf{1_{N^e}} \right]\sigma^2 $, $ c_5=\mathbb{E}_e \left[N^e\left(\text{tr} ( (\tilde{{A^e}}^{k})^\top \tilde{{A^e}}^{k}) + N^e (1+\sigma^2)  \right)   \right] $,

we conclude that
\begin{equation}
    \label{t1t2solub}
    \left\{
    \begin{array}{l}
          c_1\sigma^2(2\theta_1 \theta_2+ (\theta_2)^2-2 c_2 \sigma^2 \theta_2)+c_3-\mathbb{E}_e[N^e]c_1 \sigma^2 \theta_1 \theta_2 + \mathbb{E}_e[N^e]c_2 \sigma^2 \theta_2 =0 \\
          \left[ c_3\theta_2-\mathbb{E}_e[N^e]c_1 \sigma^2 \theta_1 \theta_2 + \mathbb{E}_e[N^e]c_2 \sigma^2 \theta_2 \right]c_4-c_5(\theta_2)^2=0
    \end{array}
    \right. .
\end{equation}

As for the derivative respect to $ {\theta_1^1}^{(l)} $, $ {\theta_1^2}^{(l)} $,
$ {\theta_2^1}^{(l)} $,$ {\theta_2^2}^{(l)} $, when they take the special value in Equation (\ref{svb}),
we have $\frac{\partial \mathbb{V}_e[R(e)]}{\partial \theta_1}=0 \Rightarrow \frac{\partial \mathbb{V}_e[R(e)]}{\partial {\theta_1^1}^{(l)}}=\frac{\partial \mathbb{V}_e[R(e)]}{\partial {\theta_1^2}^{(l)}}=0$
and $\frac{\partial \mathbb{V}_e[R(e)]}{\partial \theta_2}=0 \Rightarrow \frac{\partial \mathbb{V}_e[R(e)]}{\partial {\theta_2^1}^{(l)}}=\frac{\partial \mathbb{V}_e[R(e)]}{\partial {\theta_2^2}^{(l)}}=0$, $l=1,...,L$.
So we conclude the solution induced by Equation (\ref{t1t2solub}) is the solution of the objective, and $\theta_2=0$ is not a valid solution.

\end{proof}

\subsubsection{Proof of the Failure Case on Graphs of IRMv1 under Covariate Shift}
We restate Theorem \ref{IRM_fail_cov} as Theorem \ref{IRM_fail_cov2} below:
\begin{theorem}
\label{IRM_fail_cov2}
    \textbf{(IRMv1 will use spurious features on graphs under covariate shift)}   Under the SCM of Equation (\ref{data_gene_cov}),  there exists $s\in \mathbb{N}^+$ that satisfies $0<s<L$ and $s\neq k$ such that optimizing the IRMv1 objective $\min_\Theta \mathbb{E}_e[\|\nabla_{w|w=1.0} R(e)\|^2]$ will not lead to the invariant solution $\theta_2=0$ for parameters of the GNN (\ref{GNN}) when $\{{\theta^1_1}^{(l)},{\theta^2_1}^{(l)},{\theta^1_2}^{(l)},{\theta^2_2}^{(l)}\}$ take the special solution:
    \begin{equation}
        \Theta_0=\left\{
        \begin{array}{l}
            {\theta^1_1}^{(l)}=1, {\theta^2_1}^{(l)}=1, \quad l=L-1, ..., L-{s}+1                                                                                                                                                                                                                                                                                                      \\
            {\theta^1_1}^{(l)}=0, {\theta^2_1}^{(l)}=1, \quad l=L-{s}, L-{s}-1, ..., 1                                                                                                                                                                                                                                                                                                     \\
            {\theta^1_2}^{(l)}=0, {\theta^2_2}^{(l)}=1, \quad l=L-1, ..., 1                                                                                                                                                                                                                                                                                                           \\
        \end{array}
        \right..
    \end{equation}
\end{theorem}

\begin{proof}
    \label{proof_IRMb}
     From the proof of non-graph IRMv1 case Appendix \ref{proof_VREx_IRM_suc} we know that when the IRMv1 objective is optimized, 
    we have $\nabla_w R(e)=0$ for all $e$. For the graph case, the expected risk of environment $e$ is 
    \begin{equation}
        R(e)=\mathbb{E}_{n_1,n_2}[\|\theta_1 C_1^e + \theta_2 Z^e - \tilde{A^k_e}X_1 - n_1\|_2^2 ],
    \end{equation} where the definition of $C^e_1$ and $Z^e$ follows Equation (\ref{def_C_Z_2}). Now let's check if 
    the invariant solution $\theta_2=0$ is a valid solution. Let $\theta_2=0$, 
\begin{equation}
    \begin{aligned}
            \nabla_w R(e)&=\mathbb{E}_{n_1,n_2}[(\theta_1C_1^e + \theta_2 Z^e -(\tilde{A^e})^k X_1 - n_1)^\top(\theta_1 C_1^e + \theta_2  Z^e)]\\
            &=(\theta_1)^2 ((\tilde{A^e})^{s}X_1)^\top ((\tilde{A^e})^{s}X_1) + (\theta_2)^2 [N^e + (\epsilon^e)^\top \epsilon^e] + 2 \theta_1 \theta_2 ((\tilde{A^e})^{s}X_1)^\top \epsilon\\
            &- \theta_1 ((\tilde{A^e})^{s}X_1)^\top ((\tilde{A^e})^kX_1) -  \theta_2 (N^e+ (\epsilon^e)^\top ) ((\tilde{A^e})^kX_1)
    \end{aligned}
\end{equation}
If $\theta_2=0$ is a feasible solution, then the following equation must hold for every environment $e$:
\begin{equation}
    (\theta_1)^2 ((\tilde{A^e})^{s}X_1)^\top ((\tilde{A^e})^{s}X_1)- \theta_1 ((\tilde{A^e})^{s}X_1)^\top ((\tilde{A^e})^kX_1) =0.
\end{equation}
When $s\neq k$, we get $ \theta_1=\frac{ ((\tilde{A^e})^{s}X_1)^\top ((\tilde{A^e})^sX_1)}{ ((\tilde{A^e})^{s}X_1)^\top ((\tilde{A^e})^kX_1) } $ 
(we discard the trivial solution of $ \theta_1=0 $). The value of this solution of $\theta_1$ varies with environment $e$, and thus is not a valid solution.
However, now we will show that optimizing IRMv1 does not necessarily lead to lower-layer parameters such that $s=k$. To reveal this, by taking the derivative of $\mathcal{L}_{\text{IRMv1}}$ w.r.t. $\theta_1$ and $\theta_2$ and let them $=0$, we can get two cubic equations:
\begin{equation}
\label{e99}
    \begin{aligned}
            \frac{\partial \mathcal{L}_{\text{IRMv1}}}{\partial \theta_1}
            =&\mathbb{E}_e[({C_1^e}^\top C_1^e (\theta_1)^2 + (N^e + \epsilon^\top \epsilon)(\theta_2)^2+ {C_1^e}^\top \epsilon^e\theta_1 \theta_2 - (\tilde{A^e}^{k}X_1)^\top C_1^e \theta_1 )\\
            &(2 {C_1^e}^\top {C_1^e}^\top \theta_1 +  {C_1^e}^\top \epsilon \theta_2 - (\tilde{A^e}^{k}X_1)^\top C_1^e )]=0
    \end{aligned}
\end{equation}

\begin{equation}
\label{e100}
    \begin{aligned}
            \frac{\partial \mathcal{L}_{\text{IRMv1}}}{\partial \theta_2}
            =&\mathbb{E}_e[({C_1^e}^\top C_1^e (\theta_1)^2 + (N^e + \epsilon^\top \epsilon)(\theta_2)^2+{C_1^e}^\top \epsilon^e\theta_1 \theta_2 - (\tilde{A^e}^{k}X_1)^\top C_1^e \theta_1)\\
            &(2 (N^e + {\epsilon^e}^\top \epsilon^e) \theta_2 + ({C_1^e}^\top C_2^e + {C_1^e}^\top \epsilon)\theta_1)]=0
    \end{aligned}
\end{equation}

From the analysis in Appendix \ref{spec_fail_sec}, we know that as long as the lower-layer parameters take any value that satisfies the form in Equation (\ref{f2}), even if $s\neq k$, we can get $\frac{\partial \mathcal{L}_{\text{IRMv1}}}{\partial \theta_1}=\frac{\partial \mathcal{L}_{\text{IRMv1}}}{\partial \theta_1^1}=\frac{\partial \mathcal{L}_{\text{IRMv1}}}{\partial \theta_1^2}$ and $\frac{\partial \mathcal{L}_{\text{IRMv1}}}{\partial \theta_2}=\frac{\partial \mathcal{L}_{\text{IRMv1}}}{\partial \theta_2^1}=\frac{\partial \mathcal{L}_{\text{IRMv1}}}{\partial \theta_2^2}$. At this time (when $\frac{\partial \mathcal{L}_{\text{IRMv1}}}{\partial \theta_1}=\frac{\partial \mathcal{L}_{\text{IRMv1}}}{\partial \theta_1^1}=\frac{\partial \mathcal{L}_{\text{IRMv1}}}{\partial \theta_1^2}=0$ and $\frac{\partial \mathcal{L}_{\text{IRMv1}}}{\partial \theta_2}=\frac{\partial \mathcal{L}_{\text{IRMv1}}}{\partial \theta_2^1}=\frac{\partial \mathcal{L}_{\text{IRMv1}}}{\partial \theta_2^2}=0$ but $s\neq k$), from the form of Equations (\ref{e99}) and (\ref{e100}) we know that there exist solutions that $\theta_2\neq 0$, and the solution of $\theta_1$ and $\theta_2$ both depend on $\mathbb{E}_e(F(e))$, where $F(e)$ is some random variable associated with $e$.

\end{proof}

\subsubsection{Proof of the Successful Case on Graphs of CIA under Covariate Shift}
\begin{theorem}
\label{CIA_succ_cov2}
    Optimizing the CIA objective will lead to the optimal solution $\Theta^*$:
    \begin{equation}
        \left\{
        \begin{array}{l}
            \theta_1= 1                                                                                      \\
            \theta_2=0 \\
            {\theta^1_1}^{(l)}=1, {\theta^2_1}^{(l)}=1, \quad l=L-1, ..., L-k+1                           \\
            {\theta^1_1}^{(l)}=0, {\theta^2_1}^{(l)}=1, \quad l=L-k, L-k-1, ..., 1                           \\
        \end{array}
        \right..
    \end{equation}
\end{theorem}

\begin{proof}
\label{proof_CIAb}

    For brevity, denote a node representation of $ {C_1^e}_c $ as $ C_1^i $ and the one of $ {C_1^{e'}}_c $ as $ C_1^j $. 
    The same is true for $ C_2^i $ and $ C_2^j $. In this toy model, we need to consider the expectation of the noise, while in real cases such noise is included in the node features so taking expectation on $e$ will
    handle this. Therefore, we add $ \mathbb{E}_{n_1,n_2} $ in this proof, and this expectation is excluded in
    the formal description of the objective in the main paper.
    \begin{equation}
        \begin{aligned}
            \mathcal{L}_{\text{CIA}}c
            &=\mathbb{E}_{\substack{e, e'\\e\neq e'}} \mathbb{E}_{n_1,n_2}  \mathbb{E}_{c} \mathbb{E}_{\substack{i,j\\(i,j)\in \Omega^{e,e'}}}\left[\mathcal{D}(\phi_\Theta({A^e}, X^e)_{[c][v_i]}, \phi_\Theta(A^{e}, X^{e'})_{[c][v_j]})\right]\\
            &=\mathbb{E}_{\substack{e, e'\\e\neq e'}} \mathbb{E}_{n_1,n_2}  \mathbb{E}_{c} \mathbb{E}_{\substack{i,j\\(i,j)\in \Omega^{e,e'}}}\| {C_1^i} + Z^e - {C_1^j} - Z^{e'}  \|^2_2
        \end{aligned}
    \end{equation}
    \begin{equation}
        \frac{\partial \mathcal{L}_{\text{CIA}}}{ \partial \theta_1}=\mathbb{E}_{\substack{e, e'\\e\neq e'}} \mathbb{E}_{n_1,n_2}  \mathbb{E}_{c} 
        \mathbb{E}_{\substack{i,j\\(i,j)\in \Omega^{e,e'}}}\left[{C_1^i}\theta_1 + Z^e \theta_2- {C_1^{j}}\theta_1 - Z^{e'}\theta_2\right]^\top ({C_1^i}-C_1^j)
    \end{equation}
    Let $ \frac{\partial \mathcal{L}_{\text{CIA}}}{ \partial \theta_1}=0 $, we have:
    \begin{equation}
        \label{cia_par1b}
        \mathbb{E}_{\substack{e, e'\\e\neq e'}}  \mathbb{E}_{c} \mathbb{E}_{\substack{i,j\\(i,j)\in \Omega^{e,e'}}} \left[(C_1^i-C_1^j)^\top(C_1^i-C_1^j)\theta_1  \right]=0
    \end{equation}
    Thus, we get two possible solutions of the invariant branch. The first valid solution is the optimal one:
    \begin{equation}
        \left\{
        \begin{array}{l}
            \theta_1= 1                                                                                      \\
            {\theta^1_1}^{(l)}=1, {\theta^2_1}^{(l)}=1, \quad l=L-1, ..., L-k+1                           \\
            {\theta^1_1}^{(l)}=0, {\theta^2_1}^{(l)}=1, \quad l=L-k, L-k-1, ..., 1                           \\
        \end{array}
        \right..
    \end{equation}
    The second valid solution is a trivial one:
    \begin{equation}
            \theta_1= 0   \quad \text{or} \quad  \exists l\in \{1,...,L-1\} ~\text{s.t.}~ {\theta_1^1}^{(l)}={\theta_1^2}^{(l)}=0            
    \end{equation}
    Take the derivative of the objective w.r.t. $ \theta_2 $:
    \begin{equation}
        \frac{\partial \mathcal{L}_{\text{CIA}}}{ \partial \theta_2}=\mathbb{E}_{\substack{e, e'\\e\neq e'}}  \mathbb{E}_{c} \mathbb{E}_{\substack{i,j\\(i,j)\in \Omega^{e,e'}}}
        \left[\left[ (Z^e-Z^{e'})^\top(Z^e-Z^{e'})\right]\theta_2 \right]=2\sigma^2\theta_2
    \end{equation}
    Let $  \frac{\partial \mathcal{L}_{\text{CIA}}}{ \partial \theta_2}=0 $,  we get $\theta_2=0$.
    Thus, CIA will remove spurious features.

    Now we show that when CIA objective has been reached (the spurious branch has zero outputs),  the objective of $ \min_\Theta \mathbb{E}_e[\mathcal{L}(f_\Theta({A^e},X^e),Y^e)] $ 
    will help to learn predictive parameters of the invariant branch $ \theta_1 $, $ {\theta_1^1}^(l) $ and $ {\theta_1^2}^(l) $.
    When $ \theta_2=0 $:
    \begin{equation}
        \begin{aligned}
            \frac{\partial  \mathbb{E}_e[\mathcal{L}(f_\Theta({A^e},X^e),Y^e)]}{ \partial \theta_1}&=2\mathbb{E}_e \mathbb{E}_{n_1,n_2} \left[ \left({C_1^e}\theta_1-(\tilde{A^e})^kX_1-n_1 \right)^\top C_1^e \right]\\
            &=2 \mathbb{E}_e \left[ \left({C_1^e}\theta_1-(\tilde{A^e})^kX_1 \right)^\top C_1^e \right]
        \end{aligned}
    \end{equation}

    Let $ \frac{\partial  \mathbb{E}_e[\mathcal{L}(f_\Theta({A^e},X^e),Y^e)]}{ \partial \theta_1} =0$, we get the predictive parameters
    \begin{equation}
        \left\{
        \begin{array}{l}
            \theta_1= 1                                                                                      \\
            {\theta^1_1}^{(l)}=1, {\theta^2_1}^{(l)}=1, \quad l=L-1, ..., L-k+1                           \\
            {\theta^1_1}^{(l)}=0, {\theta^2_1}^{(l)}=1, \quad l=L-k, L-k-1, ..., 1                           \\
        \end{array}
        \right..
    \end{equation}
    Plug the final solution back in $ \frac{\partial \mathcal{L}_{\text{CIA}}}{ \partial {\theta^1_1}^{(l)}} $, 
    $ \frac{\partial \mathcal{L}_{\text{CIA}}}{ \partial {\theta^2_1}^{(l)}} $,
    $ \frac{\partial \mathcal{L}_{\text{CIA}}}{ \partial {\theta^1_2}^{(l)}} $,
    $ \frac{\partial \mathcal{L}_{\text{CIA}}}{ \partial {\theta^2_2}^{(l)}} $, 
    $ \frac{\partial  \mathbb{E}_e[\mathcal{L}(f_\Theta({A^e},X^e),Y^e)]}{ \partial {\theta^1_1}^{(l)}} $, 
    $ \frac{\partial  \mathbb{E}_e[\mathcal{L}(f_\Theta({A^e},X^e),Y^e)]}{ \partial {\theta^2_1}^{(l)}} $,
    $ \frac{\partial \mathbb{E}_e[\mathcal{L}(f_\Theta({A^e},X^e),Y^e)]}{ \partial {\theta^1_2}^{(l)}} $,
    $ \frac{\partial  \mathbb{E}_e[\mathcal{L}(f_\Theta({A^e},X^e),Y^e)]}{ \partial {\theta^2_2}^{(l)}} $, 
    we can verify that these terms are all 0.
    
\end{proof}

\subsection{Proof of Theorem \ref{bound}}
\label{bound_proof}
The following proof of Theorem \ref{bound} is adapted from \cite{ma2021subgroup} and \cite{mao2023demystifying}. We restate Theorem \ref{bound} as Theorem \ref{theo:bound_formal} below.

\begin{theorem}
    \label{theo:bound_formal}
     (Subgroup Generalization Bound for GNNs). Let $\tilde{h}$ be any classifier in a function family $\mathcal{H}$ with parameters 
$\left\{\widetilde{W}_l\right\}_{l=1}^L$. 
Under Assumption \ref{ass:ma-ass-2} and \ref{ass:ma-ass-3}, for any $e^{\text{te}}\in \mathcal{E}_{\text{te}}, \gamma \geq 0$, 
and large enough $N_{e^{\text{tr}}}$, with probability at least $1-\delta$, we have\begin{equation}
    \begin{aligned}
    \label{eq:bound-formal}
        \mathcal{L}_{e^{\text{te}}}^0(\tilde{h}) &\leq \widehat{\mathcal{L}}_{e^{\text{tr}}}^\gamma(\tilde{h})\\
        &+O( \underbrace{\frac{1}{\sigma^2}(\sum^{C}_{c=1} \sum_{c'\neq c}(\sqrt{|[(\mu_c-\mu_{c'})^\top;(\mu_c^{e^\text{te}}-\mu_{c'}^{e^\text{te}})^\top]|}+2\sqrt{2})\epsilon_{e^{\text{te}},e^{\text{tr}}}}_{\textbf{(a)}}\\
        &+  \underbrace{2\sum_{c=1}^C(C-1)B_{e^{\text{te}}}|\mu_c^{e^\text{te}}-\mu_c^{e^\text{tr}}|)}_{\textbf{(b)}} \\
        &+ \underbrace{ \frac{1}{2 \sigma^2} \frac{1}{N_{e^{\text{tr}}}}\sum_{i\in V_{e^{\text{tr}}}}\frac{1}{N_{e^{\text{te}}}}\sum_{j \in V_{e^{\text{te}}}^{(i)}}\sum^{C}_{c=1} \sum_{c'\neq c}
            |p^{\text{ht}}_j(c'|c)-p^{\text{ht}}_i(c'|c)|}_{\textbf{(c)}}\\
        &+\underbrace{\frac{b \sum_{l=1}^L\left\|\widetilde{W}_l\right\|_F^2}{(\gamma / 8)^{2 / L} N_{e^{\text{tr}}}^\alpha}\left(\epsilon_{e^{\text{te}},e^{\text{tr}}}\right)^{2 / L}}_{\textbf{(d)}}
        +const )
    \end{aligned}
\end{equation}
where $p^{\text{ht}}_i(c'|c)$ is the ratio of heterophilic neighbors of class $c'$ when $y_i=c$,  $const=\frac{1}{N_{}^{1-2 \alpha}}+\frac{1}{N_{e^{\text{tr}}}^{2 \alpha}} \ln \frac{L C\left(2 B_{e^{\text{te}}}\right)^{1 / L}}{\gamma^{1 / L} \delta}$ 
is a term independent of aggregated feature distance and the difference in neighboring heterophilic label distribution, where 
$B_{e^{\text{te}}}=\max_{i\in V_{e^{\text{tr}}}\cup V_{e^{\text{te}}}}\|g_i\|_2$ is the maximum feature norm.
\end{theorem}

To prove Theorem \ref{theo:bound_formal}, we need the following lemma that bounds the expected loss discrepancy between the train and the test node subgroups. 
\begin{lemma}
\label{lm1}
    Given a distribution $P$ over $\mathcal{H}$, for $\lambda>0$, define the expected loss discrepancy between $V_e$ and $V_{e'}$ with respect to $(P, \gamma, \lambda)$ as $D_{e, e'}^\gamma(P ; \lambda):=\ln \mathbb{E}_{h \sim P} e^{\lambda\left(\mathcal{L}_e^{\gamma / 2}(h)-\mathcal{L}_{e'}^\gamma(h)\right)}$. Under Assumption \ref{ass:ma-ass-2}, we have 
    \begin{equation}
    \begin{aligned}
        D^{\gamma/2}_{e^{\text{te}},e^{\text{tr}}} \leq & \frac{1}{\sigma^2}(\sum^{C}_{c=1} \sum_{c'\neq c}(\sqrt{|[(\mu_c-\mu_{c'})^\top;(\mu_c^{e^\text{te}}-\mu_{c'}^{e^\text{te}})^\top]|}+2\sqrt{2})\epsilon_{e^{\text{te}},e^{\text{tr}}}
            +2\sum_{c=1}^C(C-1)B_{e^{\text{te}}}|\mu_c^{e^\text{te}}-\mu_c^{e^\text{tr}}|)\\
            + & \frac{1}{2 \sigma^2} \frac{1}{N_{e^{\text{tr}}}}\sum_{i\in V_{e^{\text{tr}}}}\frac{1}{N_{e^{\text{te}}}}\sum_{j \in V_{e^{\text{te}}}^{(i)}}\sum^{C}_{c=1} \sum_{c'\neq c}
            |p^{\text{ht}}_j(c'|c)-p^{\text{ht}}_i(c'|c)|
    \end{aligned}
    \end{equation}
\end{lemma}
\begin{proof}
    According to Eq. (14) of Lemma 5 in \cite{ma2021subgroup}, under \ref{ass:ma-ass-2}, we already have
    \begin{equation}
    \begin{aligned}
            &\mathcal{L}_{e^{\text{te}}}^0(\tilde{h}) - \widehat{\mathcal{L}}_{e^{\text{tr}}}^\gamma(\tilde{h})
            \\
            \leq& \frac{1}{N_{e^{\text{tr}}}}\sum_{i\in V_{e^{\text{tr}}}}\frac{1}{N_{e^{\text{te}}}}\sum_{j \in V_{e^{\text{te}}}^{(i)}}\sum^{C}_{c=1}
            \left( 1\cdot \left( \mathcal{L}^{\frac{\gamma}{2}}(h_j, c)-\mathcal{L}^{\gamma}(h_i,k) \right) + (\text{Pr}(y_j=c|g_j)-\text{Pr}(y_i=c|g_i))\cdot 1 \right)\\
            \leq& \frac{1}{N_{e^{\text{tr}}}}\sum_{i\in V_{e^{\text{tr}}}}\frac{1}{N_{e^{\text{te}}}}\sum_{j \in V_{e^{\text{te}}}^{(i)}}\sum^{C}_{c=1}
            (\text{Pr}(y_j=c|g_j)-\text{Pr}(y_i=c|g_i))
        \end{aligned}
    \end{equation}
    In the following proof, the main goal is to bound the term $(\text{Pr}(y_j=c|g_j)-\text{Pr}(y_i=c|g_i))$ under the multi-classification OOD generalization setting. 
    Using the Bayes theorem, we have 
    \begin{equation}
        \begin{aligned}
            &\frac{1}{N_{e^{\text{tr}}}}\sum_{i\in V_{e^{\text{tr}}}}\frac{1}{N_{e^{\text{te}}}}\sum_{j \in V_{e^{\text{te}}}^{(i)}}\sum^{C}_{c=1}
            (\text{Pr}(y_j=c|g_j)-\text{Pr}(y_i=c|g_i))\\
            =&\frac{1}{N_{e^{\text{tr}}}}\sum_{i\in V_{e^{\text{tr}}}}\frac{1}{N_{e^{\text{te}}}}\sum_{j \in V_{e^{\text{te}}}^{(i)}}\sum^{C}_{c=1}
            \frac{\text{Pr}(g_j|y_j=c)\sum_{c'\neq c}\text{Pr}(g_i|y_j=c')-\text{Pr}(g_i|y_i=c)\sum_{c'\neq c}\text{Pr}(g_j|y_j=c')}
            {(\sum_{c'=1}^{C}\text{Pr}(g_j|y_j=c'))(\sum_{c'=1}^{C}\text{Pr}(g_i|y_i=c'))}
        \end{aligned}
    \end{equation}
    Let consider the term $\frac{\text{Pr}(g_j|y_j=c)\text{Pr}(g_i|y_j=c')-\text{Pr}(g_i|y_i=c)\text{Pr}(g_j|y_j=c')}
    {(\sum_{c'=1}^{C}\text{Pr}(g_j|y_j=c'))(\sum_{c'=1}^{C}\text{Pr}(g_i|y_i=c'))}$ for each $ c'\neq c $, and bound all these terms respectively. 
    When $ y_i=c $,  denote $ \mu_{i}(c) \in \mathcal{R}^{D \times (C-1)} $ as the matrix composed of all the class means of node $v_i$'s heterophilic neighbors, 
    i.e. the columns of $ \mu_{i}(c) $ are $ \mu_c',~ c'\in \{1,2,...,C\}/\{c\}$. The elements of $p_i^{\text{ht}}\in \mathcal{R}^{C-1}$ are corresponding 
    $ p^{\text{ht}}_i(c'), ~ c'\in \{1,2,...,C\}/\{c\}$. 
    According to Definition \ref{csbm}, we have
    \begin{equation}
        \text{Pr}(g_i|y_i=c)=\frac{1}{\sqrt{(2 \pi  \sigma^2)^D}} \exp(
            -\frac{1}{2 \sigma^2}\|g_i-[p^{\text{hm}}_i \mu_c^\top+(\mu_{i}(c)p^{\text{ht}}_i)^\top;
        p^{\text{hm}}_i{\mu_c^{e^\text{te}}}^\top+({\mu_i(c)}^{e^{\text{te}}}p^{\text{ht}}_i)^\top]^\top\|^2_2) 
    \end{equation}
    Denote $(\sum_{c'=1}^{C}\text{Pr}(g_j|y_j=c'))(\sum_{c'=1}^{C}\text{Pr}(g_i|y_i=c'))$ as $\exp(M)\in [0,C]$, where $ M>0 $ is a constant, we have
    \begin{equation}
        \begin{aligned}
            &|\frac{\text{Pr}(g_j|y_j=c)\text{Pr}(g_i|y_j=c')-\text{Pr}(g_i|y_i=c)\text{Pr}(g_j|y_j=c')}
            {(\sum_{c'=1}^{C}\text{Pr}(g_j|y_j=c'))(\sum_{c'=1}^{C}\text{Pr}(g_i|y_i=c'))} |\\
            =&|\exp(
                -\frac{1}{2 \sigma^2}\|g_j-[p^{\text{hm}}_j \mu_c^\top+(\mu_{i}(c)p^{\text{ht}}_j)^\top;
            p^{\text{hm}}_j{\mu_c^{e^\text{te}}}^\top+({\mu_j(c)}^{e^{\text{te}}}p^{\text{ht}}_j)^\top]^\top\|^2_2)\\
            &\exp(
            -\frac{1}{2 \sigma^2}\|g_i-[p^{\text{hm}}_i \mu_{c'}^\top+(\mu_{i}({c'})p^{\text{ht}}_i)^\top;
        p^{\text{hm}}_i{\mu_{c'}^{e^\text{tr}}}^\top+({\mu_i({c'})}^{e^{\text{tr}}}p^{\text{ht}}_i)^\top]^\top\|^2_2)\\
        -&\exp(
                -\frac{1}{2 \sigma^2}\|g_j-[p^{\text{hm}}_j \mu_{c'}^\top+(\mu_{i}({c'})p^{\text{ht}}_j)^\top;
            p^{\text{hm}}_j{\mu_{c'}^{e^\text{te}}}^\top+({\mu_j({c'})}^{e^{\text{te}}}p^{\text{ht}}_j)^\top]^\top\|^2_2)\\
            &\exp(
            -\frac{1}{2 \sigma^2}\|g_i-[p^{\text{hm}}_i \mu_c^\top+(\mu_{i}(c)p^{\text{ht}}_i)^\top;
        p^{\text{hm}}_i{\mu_c^{e^\text{tr}}}^\top+({\mu_i(c)}^{e^{\text{tr}}}p^{\text{ht}}_i)^\top]^\top\|^2_2)|/|\exp(-M)|\\
         \end{aligned}
    \end{equation}
    
    Further arranging the formula, we obtain:
    \begin{equation}
        \begin{aligned}
        &|\frac{\text{Pr}(g_j|y_j=c)\text{Pr}(g_i|y_j=c')-\text{Pr}(g_i|y_i=c)\text{Pr}(g_j|y_j=c')}
            {(\sum_{c'=1}^{C}\text{Pr}(g_j|y_j=c'))(\sum_{c'=1}^{C}\text{Pr}(g_i|y_i=c'))} |\\
        =&|\exp(-\frac{\|g_j-[p^{\text{hm}}_j \mu_c^\top+(\mu_{i}(c)p^{\text{ht}}_j)^\top;
        p^{\text{hm}}_j{\mu_c^{e^\text{te}}}^\top+({\mu_j(c)}^{e^{\text{te}}}p^{\text{ht}}_j)^\top]^\top\|^2_2}{2 \sigma^2}\\
        &-\frac{\|g_i-[p^{\text{hm}}_i \mu_{c'}^\top+(\mu_{i}({c'})p^{\text{ht}}_i)^\top;
        p^{\text{hm}}_i{\mu_{c'}^{e^\text{tr}}}^\top+({\mu_i({c'})}^{e^{\text{tr}}}p^{\text{ht}}_i)^\top]^\top\|^2_2}{2 \sigma^2}-M)\\
        -&\exp(-\frac{\|g_j-[p^{\text{hm}}_j \mu_{c'}^\top+(\mu_{i}({c'})p^{\text{ht}}_j)^\top;
        p^{\text{hm}}_j{\mu_{c'}^{e^\text{te}}}^\top+({\mu_j({c'})}^{e^{\text{te}}}p^{\text{ht}}_j)^\top]^\top\|^2_2}{2 \sigma^2}\\
        &-\frac{\|g_i-[p^{\text{hm}}_i \mu_c^\top+(\mu_{i}(c)p^{\text{ht}}_i)^\top;
        p^{\text{hm}}_i{\mu_c^{e^\text{tr}}}^\top+({\mu_i(c)}^{e^{\text{tr}}}p^{\text{ht}}_i)^\top]^\top\|^2_2}{2 \sigma^2}-M)|
        \end{aligned}
    \end{equation}
    Using Lagrange mean value theorem, we have
    \begin{equation}
        \label{lagrange}
        |\exp (-x)-\exp (-y)|=\exp (-\xi)|y-x| \leq|x-y|
    \end{equation}
    for $ x,~y>0 $ and $ \xi \in [x, y]$ (or $ \xi \in [y, x]$). Using Equation (\ref{lagrange}), we get
    \begin{equation}
        \begin{aligned}
            &|\frac{\text{Pr}(g_j|y_j=c)\text{Pr}(g_i|y_j=c')-\text{Pr}(g_i|y_i=c)\text{Pr}(g_j|y_j=c')}
            {(\sum_{c'=1}^{C}\text{Pr}(g_j|y_j=c'))(\sum_{c'=1}^{C}\text{Pr}(g_i|y_i=c'))} |\\
            \leq & \frac{1}{2 \sigma^2}(\|g_j-[p^{\text{hm}}_j \mu_c^\top+(\mu_{i}(c)p^{\text{ht}}_j)^\top;
            p^{\text{hm}}_j{\mu_c^{e^\text{te}}}^\top+({\mu_j(c)}^{e^{\text{te}}}p^{\text{ht}}_j)^\top]^\top\|^2_2\\
             + &\|g_i-[p^{\text{hm}}_i \mu_{c'}^\top+(\mu_{i}({c'})p^{\text{ht}}_i)^\top;
            p^{\text{hm}}_i{\mu_{c'}^{e^\text{tr}}}^\top+({\mu_i({c'})}^{e^{\text{tr}}}p^{\text{ht}}_i)^\top]^\top\|^2_2\\
             - &\|g_j-[p^{\text{hm}}_j \mu_{c'}^\top+(\mu_{i}({c'})p^{\text{ht}}_j)^\top;
            p^{\text{hm}}_j{\mu_{c'}^{e^\text{te}}}^\top+({\mu_j({c'})}^{e^{\text{te}}}p^{\text{ht}}_j)^\top]^\top\|^2_2\\
             - &\|g_i-[p^{\text{hm}}_i \mu_c^\top+(\mu_{i}(c)p^{\text{ht}}_i)^\top;
            p^{\text{hm}}_i{\mu_c^{e^\text{tr}}}^\top+({\mu_i(c)}^{e^{\text{tr}}}p^{\text{ht}}_i)^\top]^\top\|^2_2)\\
            = & \left.\begin{array}{@{}l@{}} -\frac{1}{\sigma^2} ( [(\mu_c-\mu_{c'})^\top;(\mu_c^{e^\text{te}}-\mu_{c'}^{e^\text{te}})^\top] (g_j p^{\text{hm}}_j - g_i p^{\text{hm}}_i) 
            + [0;(\mu_c^{e^\text{te}}-{\mu_{c'}^{e^\text{te}}})^\top-(\mu_c^{e^\text{tr}}-{\mu_{c'}^{e^\text{tr}}})^\top] g_i p^{\text{hm}}_i \\
            +[\mu_c^\top;{\mu_c^{e^\text{tr}}}^\top](p^{\text{ht}}_i(c)g_i-p^{\text{ht}}_j(c)g_j)+[0;{\mu_c^{e^\text{tr}}}^\top]p^{\text{ht}}_j(c)g_j\\
            -[\mu_{c'}^\top;{\mu_{c'}^{e^\text{tr}}}^\top](p^{\text{ht}}_i(c')g_i-p^{\text{ht}}_j(c')g_j)-[0;{\mu_{c'}^{e^\text{tr}}}^\top]p^{\text{ht}}_j(c')g_j\\
            -[0;{\mu_c^{e^\text{te}}}^\top]p^{\text{ht}}_j(c) g_j + [0;{\mu_{c'}^{e^\text{te}}}^\top] p^{\text{hm}}_j(c') g_j) \\
            \end{array} \right\} \text{(A)}\\
            +& \left.\begin{array}{@{}l@{}} \frac{1}{2 \sigma^2}((p^{\text{hm}}_j-p^{\text{hm}}_i)(\|\mu_c\|_2^2-\|\mu_{c'}\|^2_2)+p^{\text{hm}}_j \|\mu_c^{e^\text{te}}\|^2_2
            -p^{\text{hm}}_i \|\mu_c^{e^\text{tr}}\|^2_2 + p^{\text{hm}}_i \|\mu_{c'}^{e^\text{tr}}\|^2_2 -p^{\text{hm}}_j \|{\mu_{c'}^{e^\text{te}}}\|^2_2 ) \end{array} \right\} \text{(B)}\\
            +&  \left.\begin{array}{@{}l@{}} \frac{1}{2 \sigma^2} ( ({p^{\text{ht}}_j(c)}^2-{p^{\text{ht}}_i(c)}^2)\|\mu_c\|_2^2 + {p^{\text{ht}}_j(c)}^2\|\mu_c^{e^\text{te}}\|_2^2 - {p^{\text{ht}}_i(c)}^2 \|\mu_c^{e^\text{tr}}\|_2^2\\
            - ({p^{\text{ht}}_j(c')}^2-{p^{\text{ht}}_i(c')}^2)\|\mu_{c'}\|_2^2 - {p^{\text{ht}}_j(c')}^2 \|\mu_{c'}^{e^\text{te}}\|_2^2 + {p^{\text{ht}}_i(c')}^2 \|\mu_{c'}^{e^\text{te}}\|_2^2 ) \end{array} \right\} \text{(C)}\\
            \end{aligned}
    \end{equation}
Let's bound (A), (B) and (C) respectively.
    \begin{equation}
        \begin{aligned}    
            \text{(A)}\overset{(a)}{\leq} & \frac{1}{\sigma^2} (|[(\mu_c-\mu_{c'})^\top;(\mu_c^{e^\text{te}}-\mu_{c'}^{e^\text{te}})^\top]|\|g_i-g_j\|_2
            +(|\mu_c^{e^\text{te}}-\mu_c^{e^\text{tr}}| + |\mu_{c'}^{e^\text{te}}-\mu_{c'}^{e^\text{tr}}|)|g_i|\\
            +&|[\mu_c^\top;{\mu_c^{e^\text{tr}}}^\top]|\|g_i-g_j\|_2 +|[\mu_{c'}^\top;{\mu_{c'}^{e^\text{tr}}}^\top]|\|g_i-g_j\|_2\\
            +&(|\mu_c^{e^\text{te}}-\mu_c^{e^\text{tr}}| + |\mu_{c'}^{e^\text{te}}-\mu_{c'}^{e^\text{tr}}|)|g_j|)\\
            \overset{(b)}{=}&\frac{1}{\sigma^2}\left( (\sqrt{|[(\mu_c-\mu_{c'})^\top;(\mu_c^{e^\text{te}}-\mu_{c'}^{e^\text{te}})^\top]|}+2\sqrt{2})\|g_i-g_j\|_2+
            (|\mu_c^{e^\text{te}}-\mu_c^{e^\text{tr}}| + |\mu_{c'}^{e^\text{te}}-\mu_{c'}^{e^\text{tr}}|)(|g_i|+|g_j|) \right)\\
            \leq& \frac{1}{\sigma^2} \left( (\sqrt{|[(\mu_c-\mu_{c'})^\top;(\mu_c^{e^\text{te}}-\mu_{c'}^{e^\text{te}})^\top]|}+2\sqrt{2})\|g_i-g_j\|_2+
            2B_{e^{\text{te}}}(|\mu_c^{e^\text{te}}-\mu_c^{e^\text{tr}}| + |\mu_{c'}^{e^\text{te}}-\mu_{c'}^{e^\text{tr}}|)\right)
        \end{aligned}
    \end{equation}
    $(a)$ uses the fact that $ |g_j p^{\text{hm}}_j - g_i p^{\text{hm}}_i| \leq \|g_i-g_j\|_2 $ and $ p^{\text{hm}}_i,~p^{\text{ht}}_i(c),~p^{\text{ht}}_i(c')
    ,~p^{\text{hm}}_j,~p^{\text{ht}}_j(c),~p^{\text{ht}}_j(c') $ are all $<1$, and $ (b) $ is because of the assumption that 
    $ |[\mu_c^\top;{\mu_c^{e^\text{tr}}}^\top]| = |[\mu_{c'}^\top;{\mu_{c'}^{e^\text{tr}}}^\top]| = \sqrt{2}$.

    Also, since the class means are normalized, we have $ \text{(B)}=0 $.

    Similarly, we have
    \begin{equation}
        \begin{aligned}
            \text{(C)}&=\frac{1}{2 \sigma^2} ( ({p^{\text{ht}}_j(c)}^2-{p^{\text{ht}}_i(c)}^2)
            -  ({p^{\text{ht}}_j(c')}^2-{p^{\text{ht}}_i(c')}^2))\\
            &\leq \frac{1}{2 \sigma^2} \left( | {p^{\text{ht}}_j(c)}-{p^{\text{ht}}_i(c)}|+
              |{p^{\text{ht}}_j(c')}-{p^{\text{ht}}_i(c')}| \right)
        \end{aligned}
    \end{equation}

    Using these results, we have
    \begin{equation}
        \begin{aligned}
            &\mathcal{L}_{e^{\text{te}}}^0(\tilde{h}) - \widehat{\mathcal{L}}_{e^{\text{tr}}}^\gamma(\tilde{h})\\
             \leq &\frac{1}{N_{e^{\text{tr}}}}\sum_{i\in V_{e^{\text{tr}}}}\frac{1}{N_{e^{\text{te}}}}\sum_{j \in V_{e^{\text{te}}}^{(i)}}\sum^{C}_{c=1}
             \frac{\text{Pr}(g_j|y_j=c)\sum_{c'\neq c}\text{Pr}(g_i|y_j=c')-\text{Pr}(g_i|y_i=c)\sum_{c'\neq c}\text{Pr}(g_j|y_j=c')}
             {(\sum_{c'=1}^{C}\text{Pr}(g_j|y_j=c'))(\sum_{c'=1}^{C}\text{Pr}(g_i|y_i=c'))} \\
             = & \frac{1}{N_{e^{\text{tr}}}}\sum_{i\in V_{e^{\text{tr}}}}\frac{1}{N_{e^{\text{te}}}}\sum_{j \in V_{e^{\text{te}}}^{(i)}}\sum^{C}_{c=1} \sum_{c'\neq c}
             \frac{1}{\sigma^2}( (\sqrt{|[(\mu_c-\mu_{c'})^\top;(\mu_c^{e^\text{te}}-\mu_{c'}^{e^\text{te}})^\top]|}+2\sqrt{2})\|g_i-g_j\|_2)\\
             +& \frac{2}{\sigma^2} \sum_{c=1}^C (C-1)B_{e^{\text{te}}}|\mu_c^{e^\text{te}}-\mu_c^{e^\text{tr}}| \\
             + & \frac{1}{2 \sigma^2} \frac{1}{N_{e^{\text{tr}}}}\sum_{i\in V_{e^{\text{tr}}}}\frac{1}{N_{e^{\text{te}}}}\sum_{j \in V_{e^{\text{te}}}^{(i)}}\sum^{C}_{c=1} \sum_{c'\neq c}
            |p^{\text{ht}}_j(c'|c)-p^{\text{ht}}_i(c'|c)|\\
            \leq & \frac{1}{\sigma^2}(\sum^{C}_{c=1} \sum_{c'\neq c}(\sqrt{|[(\mu_c-\mu_{c'})^\top;(\mu_c^{e^\text{te}}-\mu_{c'}^{e^\text{te}})^\top]|}+2\sqrt{2})\epsilon_{e^{\text{te}},e^{\text{tr}}}
            +2\sum_{c=1}^C(C-1)B_{e^{\text{te}}}|\mu_c^{e^\text{te}}-\mu_c^{e^\text{tr}}|)\\
            + & \frac{1}{2 \sigma^2} \frac{1}{N_{e^{\text{tr}}}}\sum_{i\in V_{e^{\text{tr}}}}\frac{1}{N_{e^{\text{te}}}}\sum_{j \in V_{e^{\text{te}}}^{(i)}}\sum^{C}_{c=1} \sum_{c'\neq c}
            |p^{\text{ht}}_j(c'|c)-p^{\text{ht}}_i(c'|c)|\\
        \end{aligned}
    \end{equation}
\end{proof}

Now we can proceed with the proof of Theorem \ref{theo:bound_formal}.
\begin{proof}
    Directly replace the result of Lemma 5 in \cite{ma2021subgroup} with that of Lemma \ref{lm1} and following the proof of Lemma 6 and Theorem 3 in \cite{ma2021subgroup}, under Assumption \ref{ass:ma-ass-3}, we finish the proof of Theorem \ref{theo:bound_formal}.
\end{proof}

\section{Limitations}
\label{limit}
The theoretical analysis is limited to linear GNNs. However, there some justification for using a linear GNN. 1) Some recent works \citep{zhu2021simple, wang2021dissecting} observed that linear GNNs achieve comparable performance to nonlinear GNNs. \cite{tang2023towards} also theoretically proved that SGC can outperform GCN under some mild assumptions. 2) many recent works on the theoretical analysis of graphs/OOD generalization adopt linear networks \citep{lin2023spurious, wu2022non, mao2023demystifying}. 3) Our theory matches the experimental results on the nonlinear GCN and GAT that CIA outperforms IRM and VREx.

We didn't implement some of the node-level OOD methods in our main experiments including \cite{li2023invariant, liu2023flood} because they didn't release the code. Another limitation is that we didn't provide an in-depth theoretical comparison between our method and existing node-level OOD methods, but only revealed the difficulty of invariant learning on graphs through the examples of VREx and IRM. However, this theoretical finding is sufficient to inspire the CIA solution and its enhanced version, CIA-LRA. We reserve a more comprehensive analysis of the failure of OOD methods on graphs and broader guidance for designing graph-OOD work for future research.

\section{Broader Impacts}
\label{impact}
Our work contributes to improving the node-level OOD generalization of GNN models. We believe our work could positively impact various fields by improving predictive accuracy in areas like healthcare, social networking, etc., potentially leading to better-personalized services and enhanced safety.
\section{Compute Resources}
\label{compute}
We use one NVIDIA GeForce RTX 3090 or 4090 GPU for each single experiments. All algorithms except EERM and GTrans can be executed on a single 24GB GPU when processing the largest dataset, Arxiv, without encountering out-of-memory.

\newpage
\newpage

\section*{NeurIPS Paper Checklist}

\begin{enumerate}

\item {\bf Claims}
    \item[] Question: Do the main claims made in the abstract and introduction accurately reflect the paper's contributions and scope?
    \item[] Answer: \answerYes{} %
    \item[] Justification: The abstract and introduction concisely state how we dissect the challenges of invariant learning on graphs and out strategies to tackle it. In the introduction, we list all our contributions and indicate the sections in which they are made.

\item {\bf Limitations}
    \item[] Question: Does the paper discuss the limitations of the work performed by the authors?
    \item[] Answer: \answerYes{} %
    \item[] Justification: We point out the limitations in Appendix \ref{limit}: we didn't compare the empirical results with some of the node-level OOD works because their codes are unavailable; we didn't theoretically analyze all invariant learning methods.   

\item {\bf Theory Assumptions and Proofs}
    \item[] Question: For each theoretical result, does the paper provide the full set of assumptions and a complete (and correct) proof?
    \item[] Answer: \answerYes{} %
    \item[] Justification: All supplementary assumptions, formal versions of the theorems, and full proofs can be found in Appendix \ref{proofs_sec}.

    \item {\bf Experimental Result Reproducibility}
    \item[] Question: Does the paper fully disclose all the information needed to reproduce the main experimental results of the paper to the extent that it affects the main claims and/or conclusions of the paper (regardless of whether the code and data are provided or not)?
    \item[] Answer: \answerYes{} %
    \item[] Justification: We have provided all basic experimental settings and the hyperparameter search space in Appendix \ref{exp_detail} to reproduce all our results. The detailed training procedure of CIA-LRA is provided in Appendix \ref{code}. We will release our implementation of CIA and CIA-LRA following the acceptance of our work.

\item {\bf Open access to data and code}
    \item[] Question: Does the paper provide open access to the data and code, with sufficient instructions to faithfully reproduce the main experimental results, as described in supplemental material?
    \item[] Answer: \answerNo{} %
    \item[] Justification: The datasets utilized in our paper are publicly accessible and remain unmodified. We will release our code following the acceptance of our work, complete with detailed instructions to ensure the faithful reproduction of the main experimental results.

\item {\bf Experimental Setting/Details}
    \item[] Question: Does the paper specify all the training and test details (e.g., data splits, hyperparameters, how they were chosen, type of optimizer, etc.) necessary to understand the results?
    \item[] Answer: \answerYes{} %
    \item[] Justification: We adopt the default training details open-source GOOD benchmark, and we have stated some modifications of the setting in the main text and the appendix. 

\item {\bf Experiment Statistical Significance}
    \item[] Question: Does the paper report error bars suitably and correctly defined or other appropriate information about the statistical significance of the experiments?
    \item[] Answer: \answerYes{} %
    \item[] Justification: The improvements offered by CIA-LRA exceed the error bars of the best existing methods. Also, CIA significantly outperforms VREx and IRM across most dataset splits.

\item {\bf Experiments Compute Resources}
    \item[] Question: For each experiment, does the paper provide sufficient information on the computer resources (type of compute workers, memory, time of execution) needed to reproduce the experiments?
    \item[] Answer: \answerYes{} %
    \item[] Justification: In Appendix \ref{compute}, we report the type of GPU we use and the memory costs. 
    
\item {\bf Code Of Ethics}
    \item[] Question: Does the research conducted in the paper conform, in every respect, with the NeurIPS Code of Ethics \url{https://neurips.cc/public/EthicsGuidelines}?
    \item[] Answer: \answerYes{} %
    \item[] Justification: Our research conforms, in every respect, to the NeurIPS Code of Ethics.

\item {\bf Broader Impacts}
    \item[] Question: Does the paper discuss both potential positive societal impacts and negative societal impacts of the work performed?
    \item[] Answer: \answerYes{} %
    \item[] Justification: We discuss the potential impacts of out work in Appendix \ref{impact}.
    
\item {\bf Safeguards}
    \item[] Question: Does the paper describe safeguards that have been put in place for responsible release of data or models that have a high risk for misuse (e.g., pretrained language models, image generators, or scraped datasets)?
    \item[] Answer: \answerNA{} %
    \item[] Justification: \justificationTODO{}

\item {\bf Licenses for existing assets}
    \item[] Question: Are the creators or original owners of assets (e.g., code, data, models), used in the paper, properly credited and are the license and terms of use explicitly mentioned and properly respected?
    \item[] Answer: \answerYes{} %
    \item[] Justification: Yes. All assets used in our research, including code, data, and models, are properly credited.

\item {\bf New Assets}
    \item[] Question: Are new assets introduced in the paper well documented and is the documentation provided alongside the assets?
    \item[] Answer: \answerNA{} %
    \item[] Justification: \justificationTODO{}

\item {\bf Crowdsourcing and Research with Human Subjects}
    \item[] Question: For crowdsourcing experiments and research with human subjects, does the paper include the full text of instructions given to participants and screenshots, if applicable, as well as details about compensation (if any)? 
    \item[] Answer: \answerNA{} %
    \item[] Justification: \justificationTODO{}

\item {\bf Institutional Review Board (IRB) Approvals or Equivalent for Research with Human Subjects}
    \item[] Question: Does the paper describe potential risks incurred by study participants, whether such risks were disclosed to the subjects, and whether Institutional Review Board (IRB) approvals (or an equivalent approval/review based on the requirements of your country or institution) were obtained?
    \item[] Answer: \answerNA{} %
    \item[] Justification: \justificationTODO{}

\end{enumerate}

\end{document}